\documentclass[11pt]{article}

\usepackage{noisetilt_arxiv}
\usepackage{graphicx}
\usepackage{booktabs}
\usepackage{microtype}
\usepackage[accsupp]{axessibility}

\usepackage[numbers,sort&compress]{natbib}
\usepackage[hidelinks]{hyperref}
\usepackage[nameinlink,capitalize,noabbrev]{cleveref}

\usepackage{amsmath}
\usepackage{amssymb}
\usepackage{mathtools}

\usepackage{bm}

\usepackage{amsfonts}
\usepackage{array}
\usepackage{tabularx}
\usepackage{makecell}
\usepackage{multirow}
\usepackage[table]{xcolor}
\usepackage{pifont}
\usepackage{capt-of}
\usepackage{placeins}
\usepackage{wrapfig}

\renewcommand{\eqref}[1]{equation~\textup{\ref{#1}}}
\crefname{section}{Section}{Sections}
\crefname{subsection}{Section}{Sections}

\DeclareMathOperator*{\argmax}{arg\,max}

\definecolor{lightgreen}{HTML}{A4DBB7}
\definecolor{darkgreen}{HTML}{65A682}

\newcommand{\Ours}[0]{\texttt{NTRK}}
\newcommand{\Oursbf}[0]{\fontfamily{lmtt}\fontseries{b}\selectfont \Ours{}}

\newcommand{\cmark}{\ding{51}}
\newcommand{\xmark}{\ding{55}}

\begin{document}

\ntaffiliations{
  $^{1}$KAIST
  \qquad
  $^{2}$The University of Tokyo
}

\ntmaketitle
  {\textsc{NoiseTilt}: Noise-Tilted Reverse Kernels for Diffusion Reward Alignment}
  {Jisung Hwang$^{1,\dagger}$, Yunhong Min$^{1}$, Jaihoon Kim$^{1}$, I-Chao Shen$^{2,\ddagger}$, Minhyuk Sung$^{1,\ddagger}$}
  {}
\begingroup
\renewcommand{\thefootnote}{\ensuremath{\dagger}}
\footnotetext{This work was done in part while the author was a visiting researcher at The University of Tokyo.}
\renewcommand{\thefootnote}{\ensuremath{\ddagger}}
\footnotetext{I-Chao Shen and Minhyuk Sung are co-corresponding authors.}
\endgroup

\begin{figure}[H]
    \centering
    \includegraphics[width=\linewidth,trim=24pt 9pt 16pt 13pt,clip]{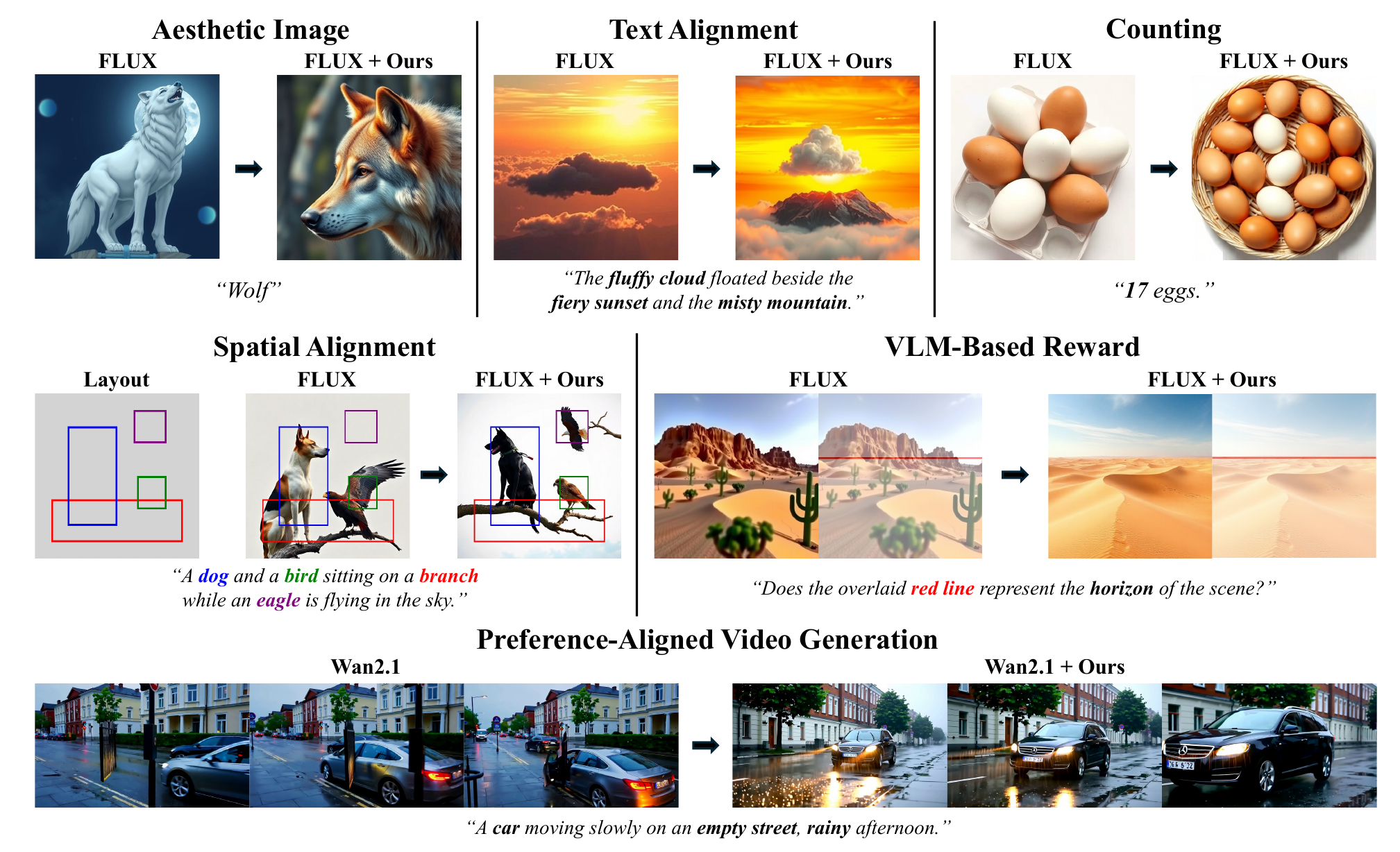}
    \caption{Inference-time reward alignment results using our Noise-Tilted Reverse Kernels (\Ours{}) across diverse reward-guided diffusion applications.}
    \label{fig:teaser}
\end{figure}

\begin{abstract}
    We introduce the Noise-Tilted Reverse Kernel (\Ours{}), a reward-guided diffusion sampler that injects reward gradients through the noise term, leaving the pretrained reverse kernel unchanged and requiring only a single sample per step. Reward-guided sampling at inference time has greatly expanded the versatility of pretrained diffusion models. Yet existing methods face a trade-off. Gradient-based guidance shifts the reverse mean, steering generation but pushing intermediate states outside the region that the model was trained on and degrading quality. Search-based methods preserve quality but gain no gradient signal. No prior method achieves both. \Ours{} resolves this by keeping the reverse mean fixed and biasing the noise term toward high reward. This is enabled by a whitening operator, the central mechanism behind \Ours{}, which converts reward gradients into noise-compatible perturbations without losing their guiding signal. Across various reward alignment tasks, \Ours{} outperforms recent state-of-the-art baselines without losing sample quality. Remarkably, on aesthetic generation, \Ours{} surpasses the reward of the best baseline at 500 NFEs using only 25 NFEs, a $20\times$ reduction in compute.
\keywords{Noise-Tilted Reverse Kernel \and Whitening Operator \and Reward Alignment \and Inference-Time Scaling}

\end{abstract}

\clearpage
\twocolumn

\section{Introduction}
\label{sec:intro}

Inference-time scaling has become one of the most powerful levers for expanding the capabilities of pretrained generative models, enabling alignment to user- and task-specific preferences without costly retraining. 
In particular, diffusion models, which underpin state-of-the-art image and video generation~\cite{Rombach:2022LDM, flux2024, cai2025:zimage, wan2025}, are well suited for inference-time scaling thanks to their iterative denoising process. 
When preferences can be expressed as a reward function, this iterative procedure offers a model-agnostic mechanism to reshape the sampling distribution toward high-reward regions.
As a result, reward-guided sampling at inference time has been applied broadly, including deblurring~\cite{Chung:2023DPS, rozet2024:moment_matching, he2024mpgd, song2023:pseudoinverse, Ye:2024TFG, Dou2024:FPS, Cardoso2024:MCGDiff}, super-resolution~\cite{Chung:2023DPS, chung2022:mcg, song2023:pseudoinverse, Ye:2024TFG}, aesthetic image generation~\cite{Ma2025:SoP, Kim:2025DAS, yoon2025:psi}, spatial alignment~\cite{Bansal:2023UGD, yoon2025:psi}, text–image alignment~\cite{Kim:2025DAS, kim2025:rbf}, and increasingly, controllable video generation~\cite{jang2026frameguidance, zhang2024:controlvideo, lu2024:freelong, qiu2024:freetraj, liu2025:videoalign}.

Behind all these applications lies the same question: how to steer each denoising step toward high-reward outcomes.
These methods broadly fall into two paradigms, distinguished by how they modify or utilize the reverse Gaussian kernel.
The first, \textbf{mean-shifted} (gradient guidance) methods, pioneered by DPS~\cite{Chung:2023DPS}, augment the mean of the reverse Gaussian kernel with reward gradients, steering the trajectory toward high-reward regions.
MCMC-based extensions further build on this principle~\cite{Yu:2023FreeDOM, Song:2023LGD, Bansal:2023UGD}.
The second, \textbf{search-based} methods~\cite{Ma2025:SoP, Li2024:SVDD, kim2025:rbf}, draw $K$ candidates from the stochastic reverse kernel and select among them by a reward criterion such as argmax or importance sampling, without gradients or kernel modification.
Several works combine both via particle filtering~\cite{Kim:2025DAS, wu2023:tds, yoon2025:psi}.

The two paradigms represent opposite ends of a fundamental trade-off between gradient guidance and noise-compatibility.
Mean-shifted methods gain the former but introduce a critical mismatch: adding reward gradients to the reverse mean displaces sampled states outside the noise-compatible regime, the narrow region where the pretrained model was trained to operate, with the displacement growing as the gradient magnitude increases (see \cref{fig:method_illust}).
This displacement drives intermediate states out-of-distribution and degrades generation quality.
It also makes these methods prone to reward hacking~\cite{Gao:2023Scaling}.
Search-based methods preserve the latter but forgo gradient guidance entirely, relying on random sampling to find reward-improving directions.
To the best of our knowledge, \textbf{no prior method achieves both}, and our work addresses this gap.

\begin{figure}[t]
\centering
\setlength{\tabcolsep}{1.0pt}
\renewcommand{\arraystretch}{1.0}
\small

\newcommand{\ImgBox}[1]{\fbox{\includegraphics[width=\linewidth]{#1}}}
\setlength{\fboxsep}{0pt}
\setlength{\fboxrule}{0.25pt}

\newcommand{\TNoise}[1]{\ImgBox{Figures/NoiseDistinct/noises/#1}}
\newcommand{\ANoise}[1]{\ImgBox{Figures/NoiseDistinct/noises/#1}}
\newcommand{\cellimg}[1]{\fbox{\includegraphics[width=0.465\columnwidth]{Figures/init_injc_sample/#1}}}

\begin{tabularx}{\linewidth}{@{} p{4mm} >{\centering\arraybackslash}X >{\centering\arraybackslash}X c @{}}
\toprule
& \multicolumn{2}{c}{\small Latent Visualization}
& \small Sampled Image \\
\midrule
\multirow{2}{*}[3em]{\rotatebox[origin=c]{90}{\normalsize\textbf{Typical Noise}}}
& \TNoise{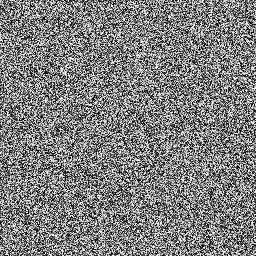} & \TNoise{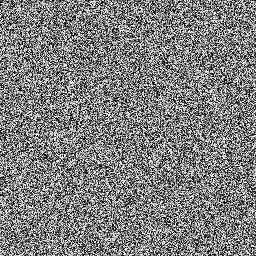}
& \raisebox{4.4em}{\multirow{2}{*}{\cellimg{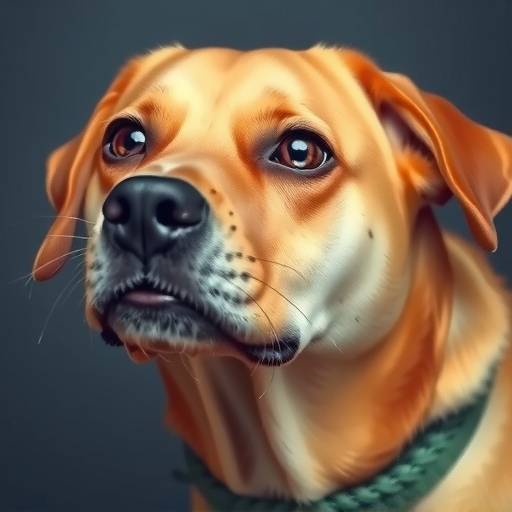}}} \\[-2pt]
& \TNoise{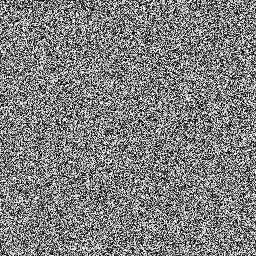} & \TNoise{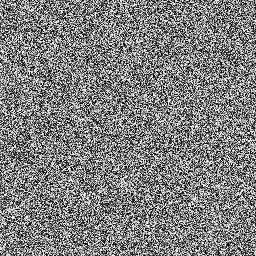} & \\[-2pt]
\midrule
\multirow{2}{*}[3.2em]{\rotatebox[origin=c]{90}{\normalsize\textbf{Atypical Noise}}}
& \ANoise{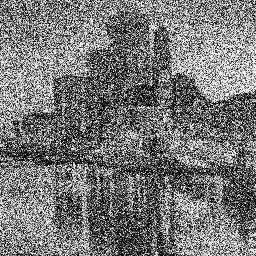} & \ANoise{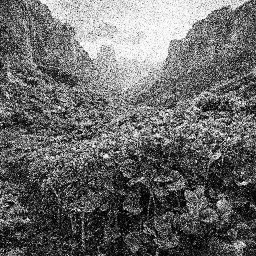}
& \raisebox{4.4em}{\multirow{2}{*}{\cellimg{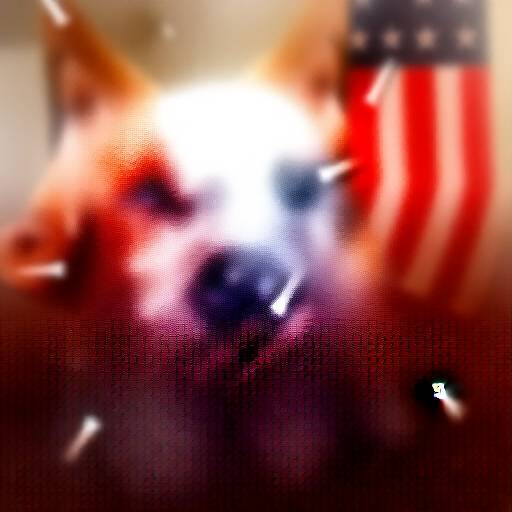}}} \\[-2pt]
& \ANoise{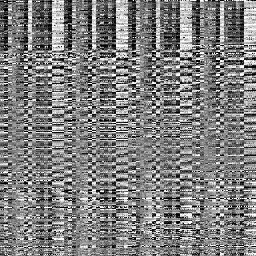} & \ANoise{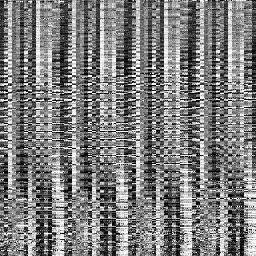} & \\[-2pt]
\bottomrule
\end{tabularx}

\caption{
\footnotesize
\textbf{Typical vs. atypical noise.}
Typical noise produces high-quality outputs, whereas atypical noise induces artifacts.
}

\label{fig:noise_inject_examples}
\vspace{-1.0\baselineskip}
\end{figure}

To resolve this trade-off, we propose the \emph{Noise-Tilted Reverse Kernel} (\Ours{}), a drop-in alternative that leaves the reverse-kernel mean entirely unchanged and routes reward information through the noise term at each step.

The central challenge is that any signal injected through the noise term must still be \textbf{\emph{typical}} Gaussian noise~\cite{hwang2025:mpgr, hwang2026gradient}, a multi-faceted property encompassing norm concentration, spatial uncorrelatedness, and collective Gaussian statistics. A raw reward gradient is usually structured and deterministic; injecting it directly through the noise channel produces atypical noise and again drives states out-of-distribution, as illustrated in \cref{fig:noise_inject_examples}.
Inspired by the concept of whitening in statistical learning~\cite{hyvarinen2001ica,kessy2018optimal}, we introduce a \emph{whitening operator} that processes the reward gradient before injection, suppressing structured components incompatible with the pretrained model while preserving the directional information of the gradient.
Whitening is therefore not an auxiliary step but the core mechanism of \Ours{}, and reward alignment improves directly with the quality of the whitening operator.

In our experiments, we demonstrate the effectiveness of \Ours{}, which can also be combined with multi-particle strategies such as Best-of-N (BoN)~\cite{Stiennon:2020BoN}. 
Across a range of applications, including aesthetic image generation, text-aligned image generation, VLM-based reward alignment, preference-aligned video generation as showcased in \cref{fig:teaser}, we show that \Ours{}, paired with simple BoN, can outperform recent state-of-the-art approaches that employ more complex techniques such as trajectory rollback~\cite{Yu:2023FreeDOM}, particle filtering~\cite{Kim:2025DAS, Li2024:SVDD, kim2025:rbf}, and initial-point search~\cite{yoon2025:psi}. 
In particular, for aesthetic image generation, our method surpasses the reward achieved by the state-of-the-art baseline with 500 NFEs using only \emph{25 NFEs} (i.e., 5\%), while maintaining image quality.

\section{Related Work}
\label{sec:related_work}
We review prior work on reward alignment for diffusion models at inference time.
Fine-tuning-based approaches include direct backpropagation~\cite{Clark2024:DRaFT,prabhudesai2024:VADER} and reinforcement learning~\cite{Black2024:DDPO,Fan:2023DPOK,Rafailov2023:DPO,Wallace:2024DiffusionDPO}; we focus on inference-time methods that operate directly on pretrained models without additional training, and can further be applied on top of fine-tuned models~\cite{li2026:mixgrpo,liu2025:flowgrpo}.
We organize these into two threads: methods that modify the per-step reverse kernel, grouped into three categories summarized in \cref{tab:kernel_summary}, and methods for making reward gradients noise-compatible, which underpin our whitening operator.

\subsection{Reverse Kernel Methods}

Rather than manipulating attention maps~\cite{hertz2022:p2p,ma2025:omnipainter,hertz2024:stylealigned,chefer2023:attend,kumari2023:custom}, which tend to be task- and model-specific, we focus on methods that modify the reverse kernel using a differentiable reward.
A separate line that treats the injected noise as an optimization variable~\cite{tang2025dno,eyring2024reno} is covered in the second thread below.

\paragraph{Mean-shifted reverse kernels.}
A foundational method in this line is DPS~\cite{Chung:2023DPS}, which incorporates reward information directly into the diffusion sampling process by shifting the mean of each reverse kernel using gradients derived from the reward model via Tweedie's formula~\cite{Robbins1992}.
By leveraging the pretrained diffusion prior, this approach can be applied to diverse tasks and has become one of the most widely adopted frameworks for inference-time alignment.
Numerous extensions have since been proposed to improve sampling efficiency, accuracy, and applicability: FreeDoM~\cite{Yu:2023FreeDOM} incorporates additional Monte Carlo sampling, other works develop more advanced solvers~\cite{rozet2024:moment_matching,Song:2023LGD,song2023:pseudoinverse,wu2023:tds,he2024mpgd}, and the framework has been extended to latent diffusion models~\cite{rout2023:psld,rout2024:STSL} enabling scalable applications in large generative systems~\cite{Rombach:2022LDM}, and applied across image and video generation~\cite{Bansal:2023UGD,lee2023:syncdiffusion,Ye:2024TFG,kwon2024:video,daras2024:warped}.
As shown in \cref{tab:kernel_summary}, however, shifting the kernel mean pushes samples outside the noise-compatible Gaussian regime that pretrained models are trained on.

\paragraph{Search-based methods.}
Rather than modifying the kernel mean, search-based methods exploit the stochasticity of the reverse diffusion process by drawing $K$ candidates from the unmodified base kernel and selecting those with higher rewards~\cite{Ma2025:SoP,ramesh2025:search,singhal2025:fk}.
These methods trade additional computation for improved sample quality while strictly preserving the native reverse kernel.
SVDD~\cite{Li2024:SVDD} is a representative example that selects the highest-reward sample at each denoising step; RBF~\cite{kim2025:rbf} improves efficiency by dynamically allocating the sampling budget across timesteps.
These methods remain noise-compatible but forgo gradient guidance and require $K$ draws per step.

\begin{table}[t]
\centering
\caption{
\footnotesize
\textbf{Comparison of reverse-kernel guidance mechanisms.}
\Ours{} is the only approach satisfying all three properties simultaneously.
}
\label{tab:kernel_summary}
\vspace{0.05cm}
\renewcommand{\arraystretch}{1.5}
\setlength{\tabcolsep}{3pt}
\scriptsize
\begin{tabular}{lccccc}
\toprule
& \textbf{Base} & \makecell{\textbf{Mean-}\\\textbf{Shifted}} & \makecell{\textbf{Search-}\\\textbf{Based}} & \textbf{Hybrid} & \makecell{\textbf{Noise-}\\\textbf{Tilted}} \\
\midrule
\makecell[l]{Examples}
& ---
& \tiny DPS~\cite{Chung:2023DPS}
& \tiny SVDD~\cite{Li2024:SVDD}
& \tiny DAS~\cite{Kim:2025DAS}
& \Ours{} \\
\makecell[l]{Gradient\\[-0.2em]guidance}
& \xmark
& \cmark
& \xmark
& \cmark
& \cmark \\
\makecell[l]{Pretrained\\[-0.2em]-compatible}
& \cmark
& \xmark
& \cmark
& \xmark
& \cmark \\
Draws / step
& $1$
& $1$
& $K$
& $K$
& $1$ \\
\bottomrule
\end{tabular}
\vspace{-0.2cm}
\end{table}

\paragraph{Hybrid methods.}
A third group combines gradient guidance with multi-particle sampling, gaining the directionality of gradient guidance and the diversity of multiple draws.
DAS~\cite{Kim:2025DAS} incorporates reward gradients within a Sequential Monte Carlo (SMC) framework~\cite{Doucet2001:SMC}, and $\Psi$-Sampler~\cite{yoon2025:psi} additionally reshapes the initial particle distribution toward the reward-aligned posterior.
However, because these methods still shift the kernel mean to inject gradient guidance, they inherit the noise-compatibility issue of mean-shifted kernels.

Across all three groups, no existing approach simultaneously achieves gradient guidance and noise-compatibility.

\subsection{Noise-Compatible Perturbations}

Our work addresses this gap with \Ours{}, which routes reward information through the stochastic noise term rather than the mean and is the only approach satisfying all three properties in \cref{tab:kernel_summary}. The central challenge is that the injected perturbation must remain noise-compatible, a property more demanding than simply matching a Gaussian norm~\cite{hwang2025:mpgr, hwang2026gradient}.
Inspired by statistical whitening~\cite{hyvarinen2001ica, kessy2018optimal}, our whitening operator processes the reward gradient to exhibit Gaussian statistics and spatial uncorrelatedness before injection.

\paragraph{Regularization-based methods.}
Several prior works have studied how to encourage Gaussian typicality in vectors involved in diffusion inference, each from a different context: earlier approaches apply norm-based regularization to keep optimized noise or latent vectors near the Gaussian concentration shell~\cite{samuel2023:norm, eyring2024reno, Benhamu:2024DFlow}, while DNO~\cite{tang2025dno}, MPGR~\cite{hwang2025:mpgr}, and StressDream~\cite{seo2026:stressdream} additionally regularize higher-order spatial statistics alongside norm, in the respective contexts of noise-space optimization, gradient guidance stabilization, and video world model steering.
All of these impose soft constraints for specific properties of Gaussian noise, however, and cannot guarantee those properties.

\paragraph{Projection-based methods.}
WGNC~\cite{hwang2026gradient} is the first to reframe noise-compatibility as a \emph{projection} problem, directly mapping the reward gradient onto a white Gaussian noise feasible set defined by hard blockwise norm constraints in the Fourier domain. This is the first approach in the spirit of true statistical whitening, achieving typicality enforcement in a single pass.
However, its hard equality constraints also distort noise that is already noise-compatible.
Our whitening operator upgrades this projection approach: by replacing hard equality constraints with confidence-interval projections, it strongly suppresses atypical structure while leaving genuine Gaussian noise nearly unchanged (see~\cref{app:whitening_comparison} for a detailed comparison).

\section{Overview}
We formalize the reward alignment problem and analyze four design choices for the per-step reverse kernel.
The central question is how reward information can be incorporated at each denoising step without violating the noise-compatible regime that the pretrained model relies on.

\subsection{Problem Definition}

Given a pretrained diffusion model that maps a source noise distribution, $p_T = \mathcal{N}(\bm{0}, \bm{I})$, to a data distribution $p_0$, our objective is to generate high-reward samples $\bm{x}_0$, a task generally known as reward alignment. 
Formally, this objective is formulated as finding a target distribution $p^{*}_0$~\cite{Korbak:2022RLDM,Uehara:2024Finetuning,Uehara:2024Bridging} such that:
\begin{equation}
    \label{eq:reward_max_obj}
    \begin{aligned}
    p^{*}_0 = \argmax_{q}\;\mathbb{E}_{\bm{x}_0 \sim q} \left[ r(\bm{x}_0) \right]-\beta\,\mathcal{D}_{\text{KL}} \left[ q \| p_0 \right],
    \end{aligned}
\end{equation}
which maximizes the expected reward $r(\bm{x}_0)$ while the KL divergence acts as a regularizer, preventing deviation from the pretrained data distribution.
The temperature $\beta>0$ controls this trade-off: smaller $\beta$ produces stronger reward tilting, while larger $\beta$ stays closer to the pretrained distribution.

The optimal reverse kernel $p^*_\theta(\bm{x}_{t-1} \mid \bm{x}_{t})$ required to sample from the target distribution in \cref{eq:reward_max_obj} can be approximated as follows:
\begin{equation}
    \label{eq:optimal_policy}
    \begin{aligned}
    p^*_\theta(\bm{x}_{t-1} | \bm{x}_{t}) 
    \propto
    p_\theta(\bm{x}_{t-1} | \bm{x}_{t})\exp\left( \tfrac{V_{t-1}(\bm{x}_{t-1})}{\beta} \right), 
    \end{aligned}
\end{equation}
where the exact value function $V_t(\bm{x}_t)$ is intractable and is approximated using Tweedie's posterior mean~\cite{Robbins1992}. 
That is, $V_t(\bm{x}_t) \approx r(\hat{\bm{x}}_{0|t})$, where $\hat{\bm{x}}_{0|t} := \mathbb{E}[\bm{x}_0 \mid \bm{x}_t]$.
The design of the per-step reverse kernel determines how this approximation is realized in practice, and whether the resulting updates remain noise-compatible.

In the following sections, we first recall the base diffusion reverse kernel and then compare two common approximation routes to reward guidance: mean-shifted reverse kernels and search-based selection.
Lastly, in \cref{subsec:whiten_kernel}, we introduce our proposed noise-tilted reverse kernel that preserves the base mean while injecting reward information through the stochastic term.
These four kernels are illustrated in \cref{fig:method_illust}.

\clearpage
\newpage
\begin{figure*}[t]
\centering
\setlength{\tabcolsep}{2pt}
\renewcommand{\arraystretch}{0.95}
\footnotesize

\newcommand{\cellimg}[1]{\includegraphics[width=0.95\linewidth]{Figures/Illst/#1}}
\newcommand{\paneltitle}[3]{%
\makecell[c]{%
{\large\textbf{#1}}\\[-1pt]
{\normalsize #3 \ | \ \scriptsize #2}%
}%
}

\begin{tabularx}{\linewidth}{@{} *{2}{X} @{}}

\paneltitle
    {Base Kernel}
    {\cref{subsec:rev_kernel}}
    {1 draw / step}
&
\paneltitle
    {Mean-Shifted}
    {DPS~\cite{Chung:2023DPS}, \cref{subsec:grad_kernel}}
    {1 draw / step}
\\

\cellimg{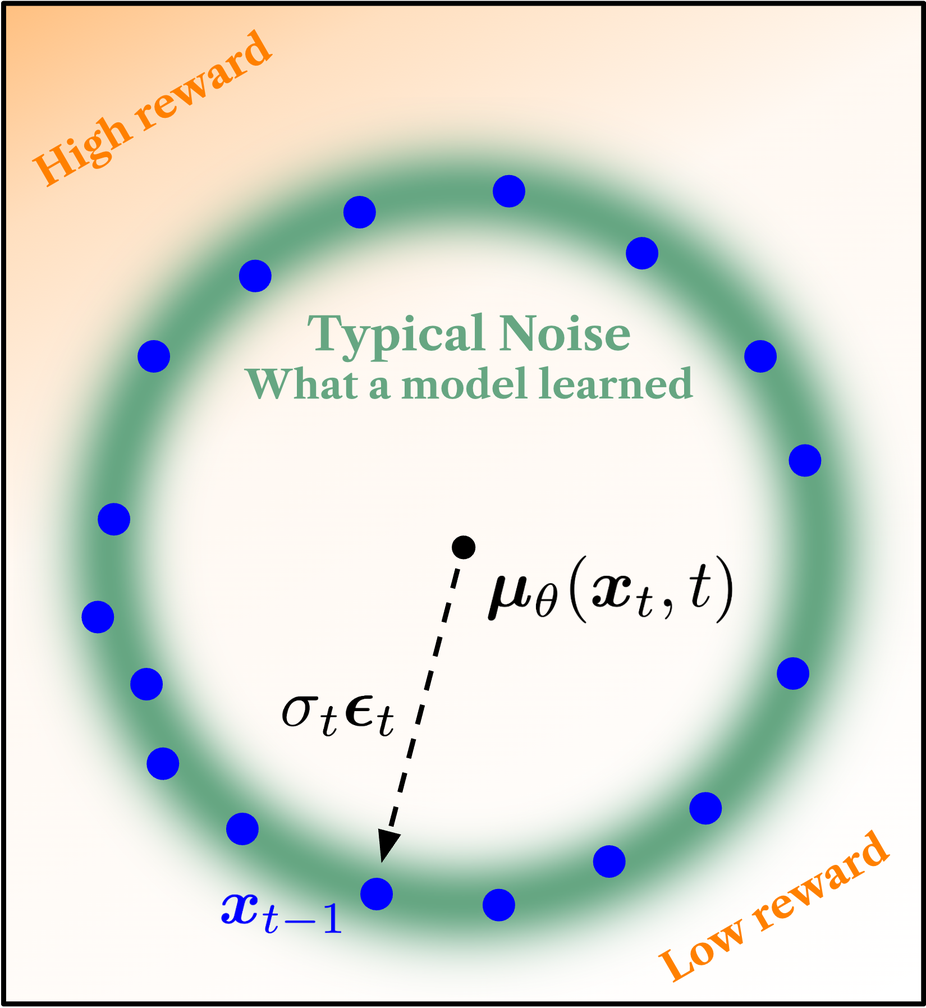} &
\cellimg{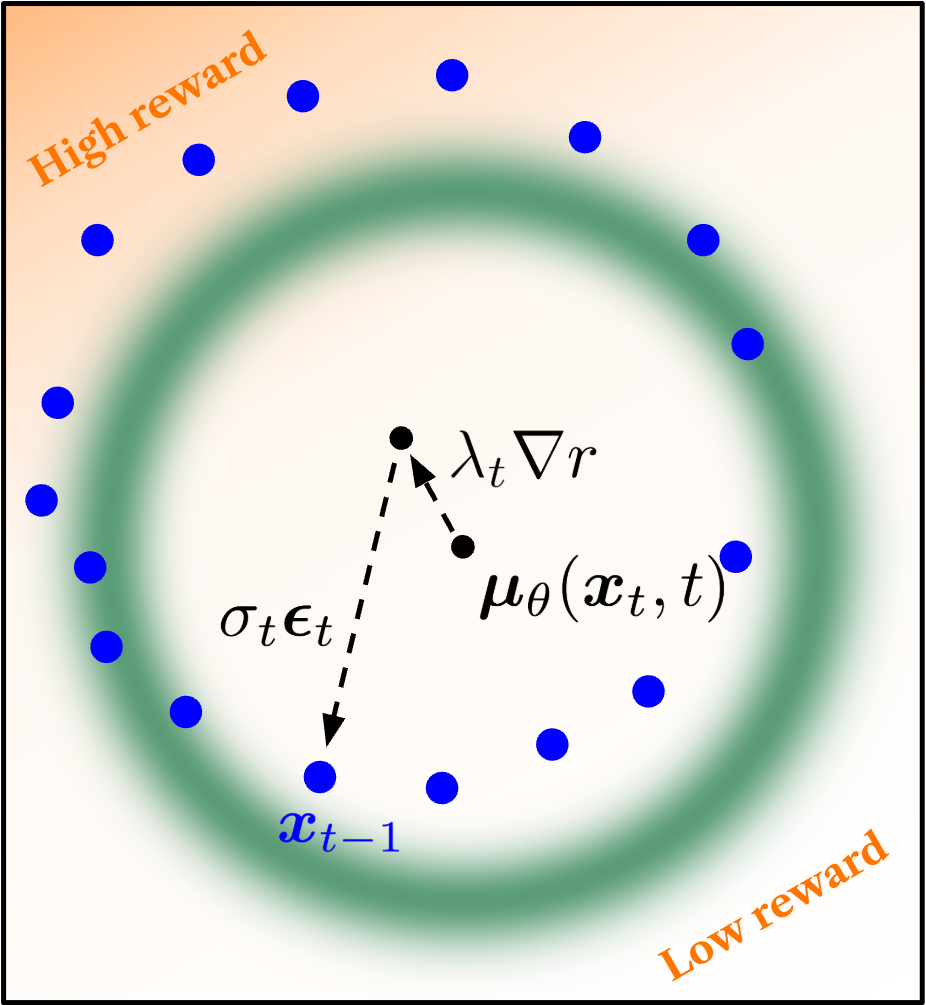} \\

\makecell{{\large $\bm{x}_{t-1}=\bm{\mu}_\theta(\bm{x}_t,t)+\sigma_t\bm{\epsilon}_t,$}} &
\makecell{{\large $\bm{x}_{t-1}=\tilde{\bm{\mu}}_\theta(\bm{x}_t,t)+\sigma_t\bm{\epsilon}_t,$}}
\\
\normalsize \makecell{$\bm{\epsilon}_t\sim\mathcal{N}(\bm{0},\bm{I}).$} &
\normalsize \makecell{$\tilde{\bm{\mu}}_\theta(\bm{x}_t,t):=\bm{\mu}_\theta(\bm{x}_t,t)+\lambda_t\nabla_{\bm{x}_t}r(\hat{\bm{x}}_{0|t}).$}
\\[0.5cm]

\paneltitle
    {Search-Based}
    {SVDD~\cite{Li2024:SVDD}, \cref{subsec:search_kernel}}
    {$K$ draws / step}
&
\paneltitle
    {Noise-Tilted}
    {\textbf{Ours}, \cref{subsec:whiten_kernel}}
    {1 draw / step}
\\

\cellimg{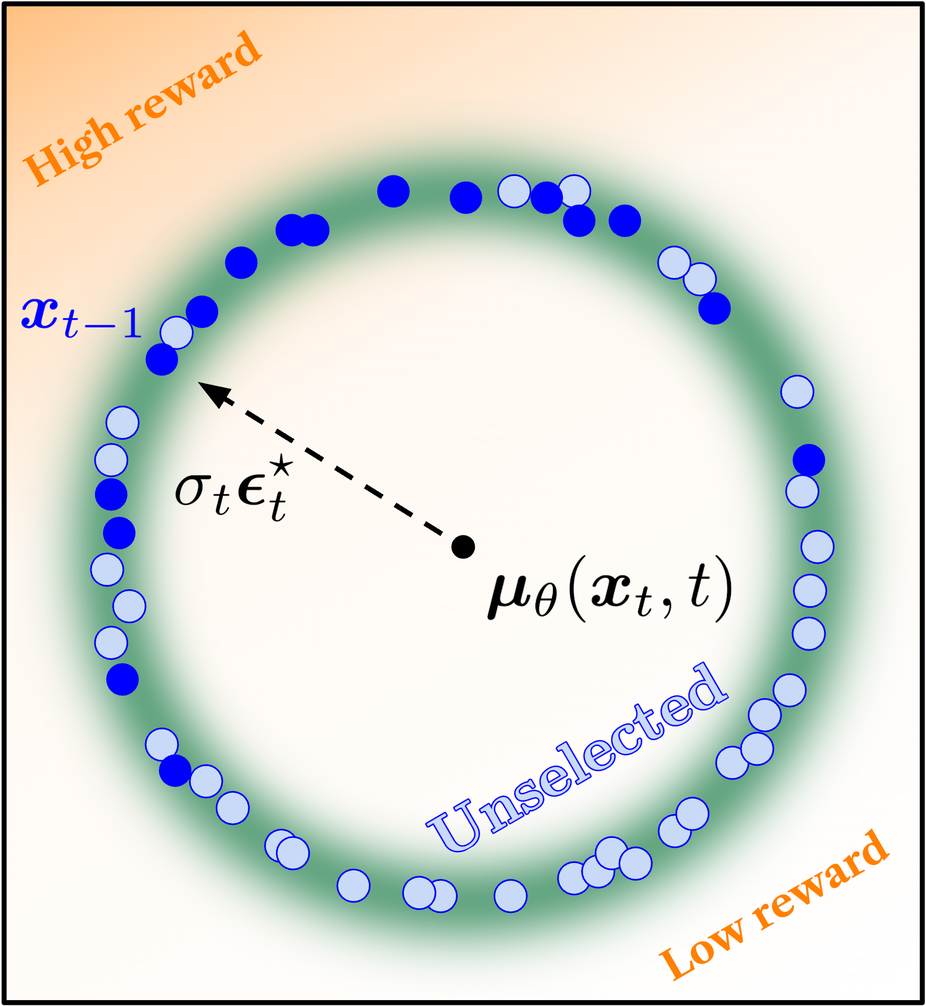} &
\cellimg{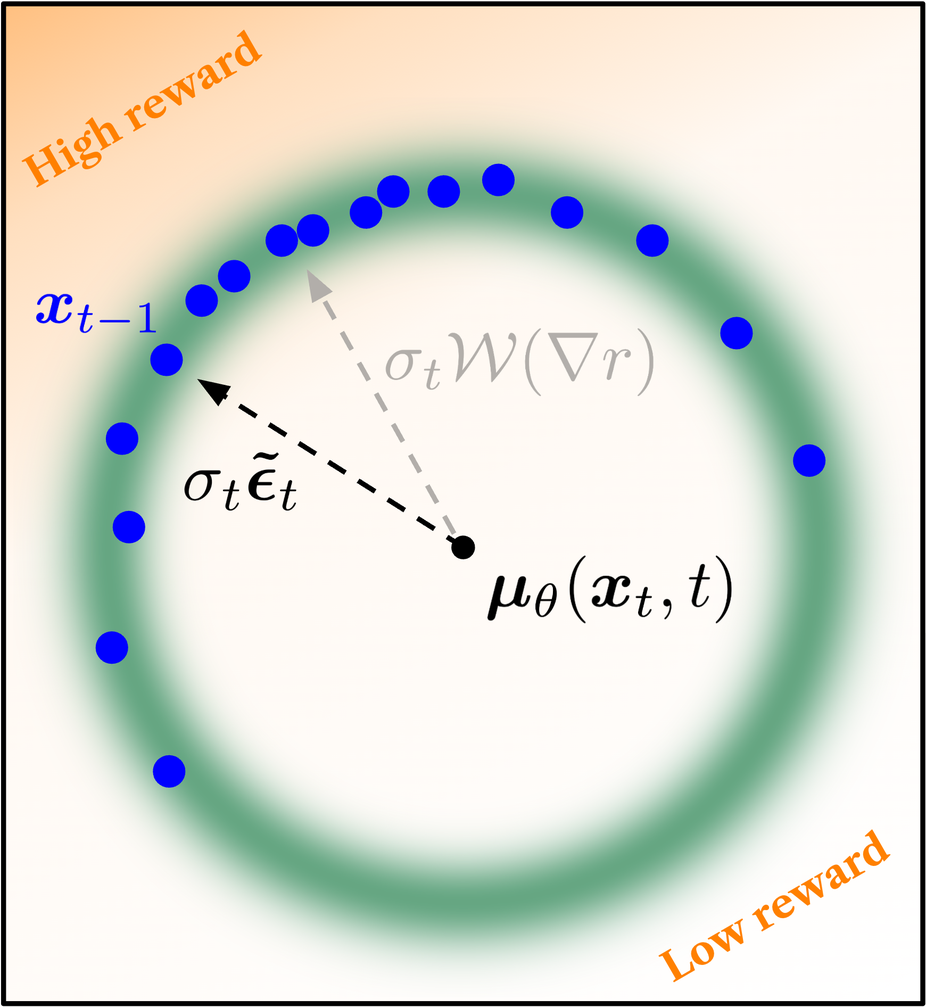} \\

\makecell{{\large $\bm{x}_{t-1} = \bm{\mu}_\theta(\bm{x}_t,t) + \sigma_t \bm{\epsilon}_t^\star,$}} &
\makecell{{\large $\bm{x}_{t-1}=\bm{\mu}_\theta(\bm{x}_t,t)+\sigma_t\bm{\tilde{\epsilon}}_t,$}}
\\
\normalsize \makecell{$\bm{\epsilon}_t^\star = \bm{\epsilon}_t^{(i^\star)},\quad i^\star = \underset{i\in\{1,\dots,K\}}{\argmax}\; r(\hat{\bm{x}}_{0|t-1}^{(i)}),$\\
$\bm{x}_{t-1}^{(i)} = \bm{\mu}_\theta(\bm{x}_t,t) + \sigma_t\bm{\epsilon}_t^{(i)}.$} &
\normalsize \makecell{$\bm{\tilde{\epsilon}}_t=\sqrt{\rho_t}\,\mathcal{W}(\nabla_{\bm{x}_t}r(\hat{\bm{x}}_{0|t}))+\sqrt{1-\rho_t}\,\bm{\epsilon}_t.$}
\\

\end{tabularx}

\vspace{0.1\baselineskip}
\caption{
\footnotesize
\textbf{Reverse-kernel guidance mechanisms.}
The green annulus indicates the noise-compatible regime; blue dots show the induced sample distribution.
Mean-shifted guidance injects reward information by shifting the reverse mean, pushing samples outside the noise-compatible regime.
Search-based guidance preserves the base mean by selecting the best among $K$ candidate noise draws.
\Ours{} also preserves the base mean but constructs a single reward-tilted noise draw, achieving reward alignment without leaving the noise-compatible regime.
}
\vspace{-0.65\baselineskip}
\label{fig:method_illust}
\end{figure*}

\clearpage

\subsection{Diffusion Reverse Kernel}
\label{subsec:rev_kernel}
A diffusion model~\cite{Ho:2020DDPM, Song:2021SDE} progressively denoises an initial Gaussian noise by sequentially applying a Markovian reverse kernel, modeled as a Gaussian transition from $\bm{x}_t$ to $\bm{x}_{t-1}$ over discrete timesteps $t = T, \dots, 0$:
\begin{equation}
\label{eq:reverse_kernel}
p_\theta(\bm{x}_{t-1} \mid \bm{x}_t)
=
\mathcal{N}\big(\bm{x}_{t-1};
\bm{\mu}_\theta(\bm{x}_t,t),\sigma_t^2\bm{I}\big),
\end{equation}
where the mean $\bm{\mu}_\theta(\bm{x}_t, t)$ is parameterized using a neural network, and the variance $\sigma_t^2$ controls the amount of stochasticity injected each step.
The visualization of \cref{eq:reverse_kernel} is presented in \cref{fig:method_illust}.

To generate a sample in practice, $\bm{x}_{t-1}$ is drawn from the distribution in \cref{eq:reverse_kernel} using the standard reparameterization trick:
\begin{equation}
    \label{eq:base_sampling}
    \begin{aligned}
    \bm{x}_{t-1}
    &=
    \bm{\mu}_\theta(\bm{x}_t,t)+\sigma_t\bm{\epsilon}_t,\\
    \bm{\epsilon}_t&\sim\mathcal{N}(\bm{0},\bm{I}).
    \end{aligned}
\end{equation}
Since the reverse kernel is parameterized as a Gaussian distribution, any valid transition between consecutive latent states must satisfy the condition that the \textit{standardized perturbation} $\bm{\eta}_t$ is a Gaussian noise:
\begin{equation}
    \label{eq:eta_base}
    \bm{\eta}_t := \frac{\bm{x}_{t-1} - \bm{\mu}_\theta(\bm{x}_t, t)}{\sigma_t}.
\end{equation}
Under the base reverse process, $\bm{\eta}_t = \bm{\epsilon}_t$ holds trivially, so the standardized perturbation is inherently a standard Gaussian noise.

\subsection{Mean-Shifted Reverse Kernel~\cite{Chung:2023DPS}}
\label{subsec:grad_kernel}
For efficient reward alignment, first-order reward information can be injected directly into the reverse transition in \cref{eq:reverse_kernel}.
Concretely, the mean-shifted reverse kernel~\cite{Chung:2023DPS} modifies the base reverse kernel by adding this gradient to its mean:
\begin{equation}
    \label{eq:grad_kernel}
    \bm{x}_{t-1}
    =
    \bm{\mu}_\theta(\bm{x}_t,t)
    +
    \lambda_t\nabla_{\bm{x}_t}r(\hat{\bm{x}}_{0|t})
    +
    \sigma_t \bm{\epsilon}_t,
\end{equation}
where $\bm{\epsilon}_t\sim\mathcal{N}(\bm{0},\bm{I})$, $\lambda_t>0$ is a guidance hyperparameter, and detailed derivations are provided in \cref{subsec:mean-shift_app}.
As illustrated in the top-right panel of \cref{fig:method_illust}, \cref{eq:grad_kernel} injects reward guidance by shifting the reverse-kernel mean while keeping the injected noise term unchanged.

Since its introduction, this mean-shifted reverse kernel has been widely adopted and extended, including advanced sampling schemes~\cite{Yu:2023FreeDOM, he2024mpgd}, particle-based sampling~\cite{Kim:2025DAS, yoon2025:psi}, and applications ranging from images~\cite{Bansal:2023UGD, lee2023:syncdiffusion, Ye:2024TFG} to videos~\cite{kwon2024:video, daras2024:warped}.

Despite its practical success, a closer look at the standardized perturbation reveals a subtle yet critical mismatch with the base reverse dynamics.
Identifying and analyzing this discrepancy is one of the main contributions of this work.
To make the mismatch explicit, consider the standardized perturbation induced by \cref{eq:grad_kernel} relative to the base mean $\bm{\mu}_\theta(\bm{x}_t,t)$:
\begin{equation}
    \label{eq:eta_mean}
    \bm{\eta}_t^\text{mean}
    =
    \bm{\epsilon}_t
    +
    \frac{\lambda_t}{\sigma_t}\nabla_{\bm{x}_t}r(\hat{\bm{x}}_{0|t}).
\end{equation}
The deterministic reward gradient term in \cref{eq:eta_mean} means that the standardized perturbation is no longer a Gaussian noise assumed by the pretrained reverse kernel.
Consequently, the reverse updates can drift away from the learned intermediate distributions, and repeated iterations may push samples toward regions that are insufficiently supported by the pretrained diffusion in practice.
An alternative approach avoids this issue by preserving the base mean and instead biasing the stochastic perturbation, as we describe in the following subsection.

\subsection{Search-Based Reverse Kernel}
\label{subsec:search_kernel}

To approximate the optimal reverse transition in \cref{eq:optimal_policy}, search-based methods draw $K$ candidates from the base reverse kernel, all centered at the pretrained mean:
\begin{equation}
\label{eq:search_kernel}
\bm{x}_{t-1}^{(i)} = \bm{\mu}_\theta(\bm{x}_t,t)+\sigma_t\bm{\epsilon}_t^{(i)},
\quad
\bm{\epsilon}_t^{(i)}\sim\mathcal{N}(\bm{0},\bm{I}),
\end{equation}
where $i=1,\dots,K$. As a representative search-based method, SVDD~\cite{Li2024:SVDD} selects the candidate whose predicted reward score is highest:
\begin{equation}
\label{eq:search_select}
i^\star = \argmax_{i\in\{1,\dots,K\}}\,r\!\left(\hat{\bm{x}}_{0|t-1}^{(i)}\right),
\quad
\bm{x}_{t-1} = \bm{x}_{t-1}^{(i^\star)}.
\end{equation}
Reward information therefore appears through the selected stochastic perturbation $\bm{\epsilon}_t^{(i^\star)}$ rather than through a deterministic mean shift.
As illustrated in \cref{fig:method_illust}, search-based guidance can therefore be interpreted as an implicit form of noise tilting.
However, when high-reward samples are in low-density regions of the pretrained distribution, this becomes inefficient, requiring a prohibitively large number of samples.
We provide a more detailed derivation of this view in \cref{subsec:search_app}.
In the following subsection, we propose a method that retains this advantage of preserving the base mean while reducing the sampling cost to a single draw. 

\subsection{Noise-Tilted Reverse Kernel (\Ours{})}
\label{subsec:whiten_kernel}

Our key idea is to preserve the pretrained reverse mean $\bm{\mu}_\theta(\bm{x}_t,t)$ and inject reward information only through the noise, thereby avoiding the noise-compatibility mismatch of mean-shifted kernels. Unlike search-based selection, which requires drawing $K$ candidates per step, our approach achieves this with a single noise draw.

The challenge is that a raw reward gradient is usually structured and deterministic rather than Gaussian, so it cannot directly serve as a noise-compatible perturbation.
We therefore introduce a whitening operator $\mathcal{W}$ that maps the reward gradient to a noise-compatible direction, and define the whitened reward direction $\bm{w}_t:=\mathcal{W}(\nabla_{\bm{x}_t}r(\hat{\bm{x}}_{0|t}))$.
We formalize $\mathcal{W}$ in \cref{sec:method}.

To incorporate $\bm{w}_t$ into the noise term while preserving its Gaussian form, we rely on a basic identity: if $\bm{\epsilon}_1,\bm{\epsilon}_2\sim\mathcal{N}(\bm{0},\bm{I})$ are independent, then for any $\rho\in[0,1]$,
\begin{equation}
\label{eq:noise_add}
\sqrt{\rho}\,\bm{\epsilon}_1+\sqrt{1-\rho}\,\bm{\epsilon}_2
\sim
\mathcal{N}(\bm{0},\bm{I}).
\end{equation}
We define the guided injected noise by mixing the whitened direction with unbiased stochasticity:
\begin{equation}
\label{eq:tilde_eps}
\bm{\tilde{\epsilon}}_t
=
\sqrt{\rho_t}\,\bm{w}_t
+
\sqrt{1-\rho_t}\,\bm{\epsilon}_t,
\end{equation}
where $\bm{\epsilon}_t \sim \mathcal{N}(\bm{0},\bm{I})$ and $\rho_t\in[0,1]$ controls the guidance strength.
Intuitively, the first term injects a whitened guidance component that carries reward-gradient information, while the second term preserves unbiased stochasticity.
A heuristic interpretation of $\rho_t$ relative to search-based guidance under a local linearity assumption is provided in \cref{subsec:rho_app}.

Using \cref{eq:tilde_eps}, we replace the base sampling rule in \cref{eq:base_sampling}, resulting in our noise-tilted reverse kernel, \Ours{}:
\begin{equation}
\label{eq:noise_kernel}
\bm{x}_{t-1}
=
\bm{\mu}_\theta(\bm{x}_t,t)
+
\sigma_t\bm{\tilde{\epsilon}}_t.
\end{equation}
Most importantly, the standardized noise relative to the base mean becomes
\begin{equation}
\label{eq:eta_ours}
\bm{\eta}_t^{\text{noise}}
:=
\frac{\bm{x}_{t-1}-\bm{\mu}_\theta(\bm{x}_t,t)}{\sigma_t}
=
\bm{\tilde{\epsilon}}_t.
\end{equation}
The kernel is thus \emph{noise-tilted} in that reward guidance is expressed entirely through the noise term rather than the mean, as formalized in \cref{eq:eta_ours}.
Crucially, \cref{eq:noise_kernel} preserves the pretrained reverse-kernel structure, using the same base mean $\bm{\mu}_\theta(\bm{x}_t,t)$ and the same nominal noise scale $\sigma_t$ as \cref{eq:reverse_kernel}.
Consequently, each step remains noise-compatible with the pretrained reverse dynamics.
As shown in the right panel of \cref{fig:method_illust}, samples of $\bm{x}_{t-1}$ are drawn toward higher-reward regions while remaining within the regime where the pretrained model operates reliably.

A core contribution of this work is the construction of the whitening operator $\mathcal{W}$ so that it is statistically well-founded in high dimensions and expressive enough to preserve reward signal.
In \cref{sec:method}, we formalize the required statistical conditions and introduce our confidence-interval projection method to realize $\mathcal{W}$ in practice.

\section{Whitening Operator}
\label{sec:method}

\begin{figure*}[t]
\centering
\setlength{\tabcolsep}{1pt}
\renewcommand{\arraystretch}{0.5}
\small

\newcommand{\cellimg}[1]{\includegraphics[width=\linewidth]{Figures/whitening_process/#1}}

\newcolumntype{C}{>{\centering\arraybackslash}X}

\begin{tabularx}{\linewidth}{@{} c *{4}{C} || *{2}{C} @{}}
\toprule
& Input & \makecell{2OS\\Projection} & +\makecell{Tile-wise\\Statistics} & +\makecell{Multiple\\Domains} & \makecell{Sampled\\Noise $\bm{z}$} & $\mathcal{W}(\bm{z})$ \\

\midrule
\raisebox{1.0ex}{\rotatebox{90}{\scriptsize \makecell[c]{Latent\\Visualization}}} &
\cellimg{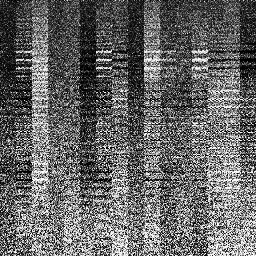} &
\cellimg{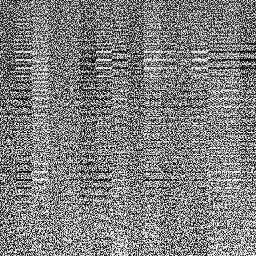} &
\cellimg{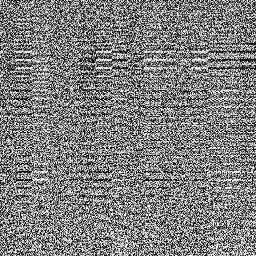} &
\cellimg{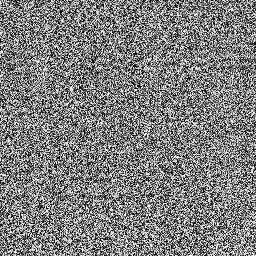} &
\cellimg{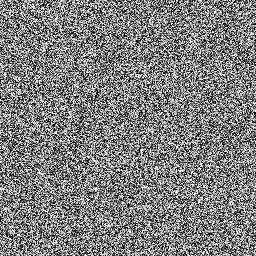} &
\cellimg{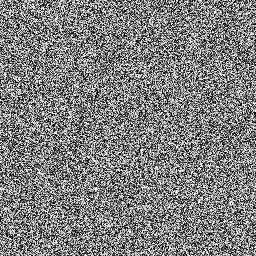} \\
\raisebox{2.8ex}{\rotatebox{90}{\scriptsize \makecell[c]{Sampled\\Image}}} &
\cellimg{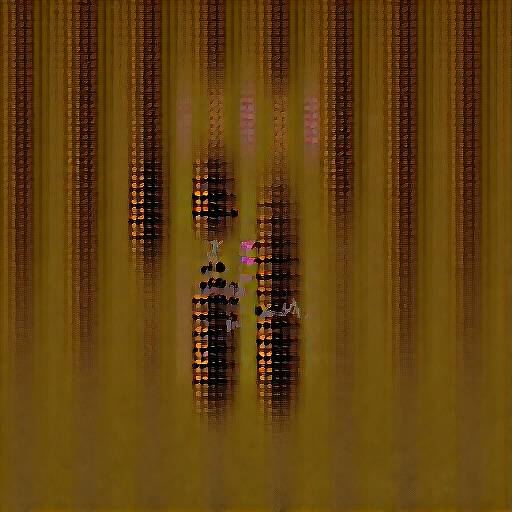} &
\cellimg{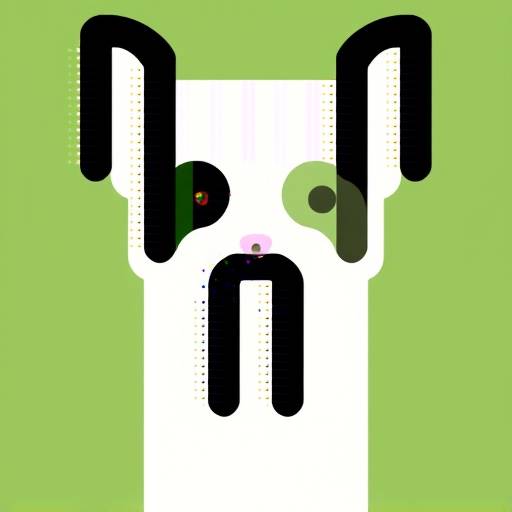} &
\cellimg{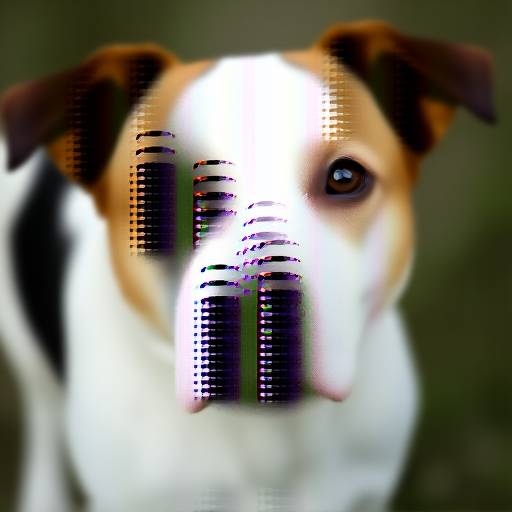} &
\cellimg{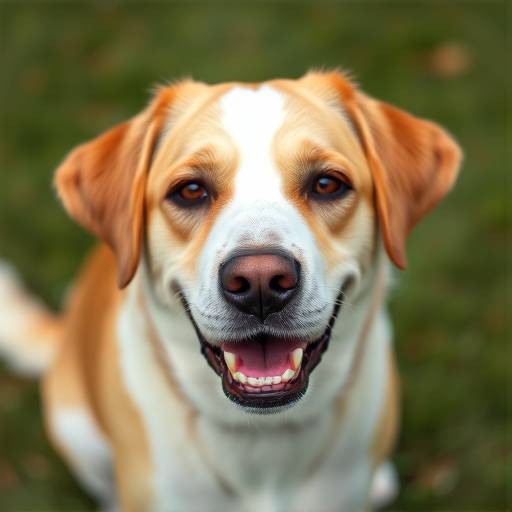} &
\cellimg{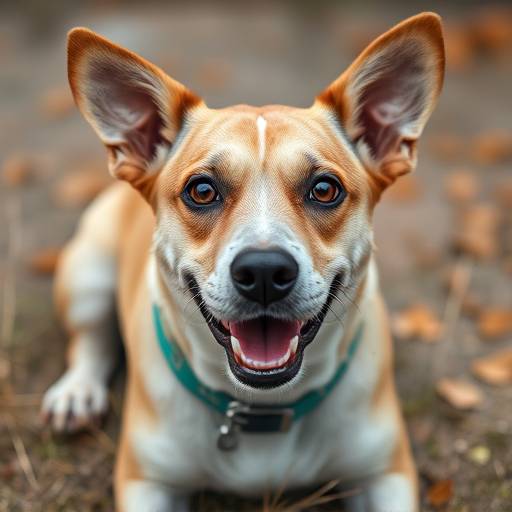} &
\cellimg{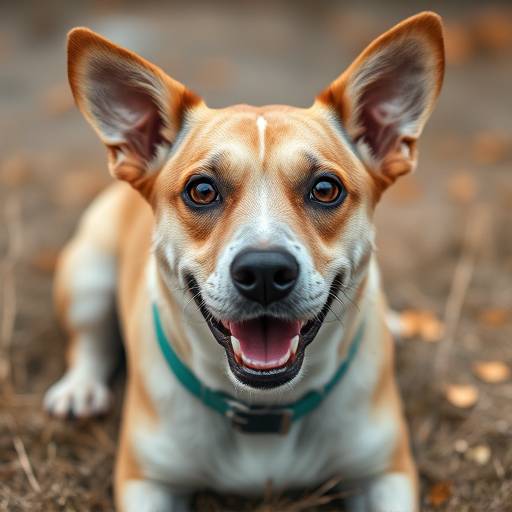} \\[-2px]
\bottomrule
\end{tabularx}

\caption{
\footnotesize
\textbf{Effect of our whitening operator $\mathcal{W}$.}
As we apply the components of $\mathcal{W}$ to a structured latent (left to right), it becomes more similar to typical Gaussian noise with more realistic samples. 
For typical Gaussian noise, $\mathcal{W}(\bm{z})$ changes negligibly (two rightmost columns). 
}
\label{fig:whitening_process}
\vspace{-0.25\baselineskip}
\end{figure*}

Our goal is to transform an arbitrary input vector into a \emph{typical} standard Gaussian noise vector.
Although $\mathcal{N}(\bm{0},\bm{I})$ has nonzero density everywhere in $\mathbb{R}^N$, in high dimensions almost all probability mass concentrates on a narrow \emph{typical} region.
Pretrained generative models are trained on trajectories whose injected perturbations lie in this typical region; consequently, perturbations that drift away from it can act as out-of-distribution inputs to the learned reverse dynamics (see \cref{fig:noise_inject_examples}).
We therefore design a whitening operator $\mathcal{W}:\mathbb{R}^N\rightarrow\mathbb{R}^N$ that \emph{moves} the reward gradient direction toward typical standard Gaussian noise before it is mixed into the stochastic term of the reverse kernel.

The typical set is difficult to characterize exactly in closed form.
Instead, we approximate it using a collection of \emph{high-confidence} constraints induced by known statistics of the standard normal distribution.
Concretely, we precompute $99.99\%$ confidence bounds and define $\mathcal{W}$ as a sequence of projections onto the corresponding confidence sets.
A key building block is a two-level order-statistic projection (\textbf{2OS}), which we describe next.

\paragraph{Two-level order statistics.}
Let $\bm{x}\in\mathbb{R}^N$ be the vector to whiten, reshaped into a tile matrix $\bm{Y}\in\mathbb{R}^{M\times D}$ (so $N=MD$).
The 2OS statistic is obtained by sorting \emph{within} each tile and then sorting \emph{across} tiles at each within-tile rank:
\begin{equation}
\bm{Z}=\mathrm{sort}_0\bigl(\mathrm{sort}_1(\bm{Y})\bigr),
\label{eq:whiten_2os_def}
\end{equation}
where $\mathrm{sort}_1$ sorts each row of $\bm{Y}$ and $\mathrm{sort}_0$ sorts each column.
Intuitively, $\bm{Z}$ captures a ``rank-of-rank'' summary: $Z_{r,j}$ is the $r$-th smallest value among the $j$-th order statistics collected from all tiles.
For standard Gaussian noise, each entry $Z_{r,j}$ concentrates sharply, allowing tight confidence bounds for each $(r,j)$.

\paragraph{Confidence bounds for the 2OS statistic.}
Fix $\alpha=10^{-4}$ and let $\Phi$ denote the CDF of the standard normal distribution.
For each rank pair $(r,j)$, we compute quantile bounds $(q^{\mathrm{lo}}_{r,j},q^{\mathrm{hi}}_{r,j})$ from the nested order-statistic distribution, and map them to value-domain bounds
\begin{equation}
(L_{r,j},U_{r,j})
=
\bigl(\Phi^{-1}(q^{\mathrm{lo}}_{r,j}),
\Phi^{-1}(q^{\mathrm{hi}}_{r,j})\bigr).
\label{eq:whiten_2os_quantile_lo}
\end{equation}
The exact Beta-quantile construction is given in \cref{app:whitening_2os}.

These bounds specify a $1-\alpha$ confidence set for the doubly-sorted matrix $\bm{Z}$:
\begin{equation}
\mathcal{C}_{\mathrm{2os}}
=
\{\bm{Y}:L_{r,j}\le Z_{r,j}\le U_{r,j},\ \forall r,j\}.
\label{eq:whiten_confset_2os}
\end{equation}
\paragraph{2OS projection via sort--clip--unsort.}
Given $\bm{Y}$, we compute $\bm{Z}$ as in \cref{eq:whiten_2os_def}, clip each element to its confidence interval
\begin{equation}
    Z_{r,j}\leftarrow\operatorname{clip}(Z_{r,j};L_{r,j},U_{r,j}),
\label{eq:whiten_clip}    
\end{equation}
and then invert the two sorting permutations to map the clipped $\bm{Z}$ back to the original tile layout.
We prove in \cref{app:whitening_2os} that this \emph{sort--clip--unsort} operation equals the Euclidean projection onto $\mathcal{C}_{\mathrm{2os}}$, and we therefore call it the \emph{2OS projection}.

The 2OS projection prevents extreme values from concentrating in a few tiles.
Because clipping is applied to each rank pair $(r,j)$ of the doubly-sorted statistic, each tile contains a balanced spread of small-to-large values consistent with typical standard Gaussian noise.
Applying 2OS additionally to tile-wise mean and energy statistics constrains block-level moments, while repeating across orthogonally transformed domains captures structured correlations that the value-domain projection alone cannot reach.

The full whitening operator $\mathcal{W}$ is composed of 2OS projections applied to various tile-wise statistics across multiple transformed domains, as described in \cref{app:whitening_operator}. We visualize this progression in \cref{fig:whitening_process}.
Since all constraints are derived from $99.99\%$ confidence intervals under the reference distribution, typical standard Gaussian noise passes through with negligible modification in practice (cosine similarity $>0.99999$; see the two rightmost columns in \cref{fig:whitening_process}), whereas structured inputs are substantially whitened, yielding a perturbation that is noise-compatible with the pretrained reverse kernel.

\begin{table*}[t!]
\centering
\caption{
\footnotesize
\textbf{Quantitative comparison on image reward alignment.}
Left: aesthetic image generation (target: Aesthetic Score~\cite{Schuhmann:aesthetics}).
Right: text-aligned image generation (target: PickScore~\cite{Kirstain2023:pickapic}).
For single-particle methods we augment sampling with Best-of-N to match the total NFE, denoted with \textsuperscript{\dag}.
Dark green cells indicate the best result for each metric, light green the second best.
}
\label{tab:aesthetic_full}
\label{tab:pick_full}
\renewcommand{\arraystretch}{1.10}
\setlength{\tabcolsep}{3.5pt}
\footnotesize
\resizebox{\textwidth}{!}{%
\begin{tabular}{@{} l c | ccccc | ccccc @{}}
\toprule

\multirow{3}{*}[-3.0ex]{\textbf{Method}}
& \multirow{3}{*}[-3.0ex]{\textbf{NFE}}
& \multicolumn{5}{c|}{\normalsize \textbf{Aesthetic Image Generation}}
& \multicolumn{5}{c}{\normalsize \textbf{Text-Aligned Image Generation}} \\

\cmidrule(lr){3-7} \cmidrule(lr){8-12}

&
& \multicolumn{1}{c}{\makecell{\textbf{Target}\\\textbf{Reward}}}
& \multicolumn{4}{c|}{\textbf{Held-Out Reward}}
& \multicolumn{1}{c}{\makecell{\textbf{Target}\\\textbf{Reward}}}
& \multicolumn{4}{c}{\textbf{Held-Out Reward}} \\

\cmidrule(lr){3-3} \cmidrule(lr){4-7} \cmidrule(lr){8-8} \cmidrule(lr){9-12}

&
& \makecell{\textbf{Aesthetic}\\\textbf{Score}} $\uparrow$
& \makecell{\textbf{Pick}\\\textbf{-Score}} $\uparrow$
& \textbf{HPSv2} $\uparrow$
& \makecell{\textbf{Image}\\\textbf{Reward}} $\uparrow$
& \makecell{\textbf{VQA}\\\textbf{Score}} $\uparrow$
& \makecell{\textbf{Pick}\\\textbf{-Score}} $\uparrow$
& \makecell{\textbf{Aesthetic}\\\textbf{Score}} $\uparrow$
& \textbf{HPSv2} $\uparrow$
& \makecell{\textbf{Image}\\\textbf{Reward}} $\uparrow$
& \makecell{\textbf{VQA}\\\textbf{Score}} $\uparrow$ \\

\midrule

Base~\cite{flux2024}
& 25
& 6.0282 & 0.2144 & 0.2759 & 1.0538 & 0.9644
& 0.2054 & 5.4664 & 0.2316 & 0.1710 & 0.8011 \\

BoN~\cite{Stiennon:2020BoN}
& 500
& 6.7310 & \cellcolor{lightgreen}0.2197 & 0.2890 & 1.1419 & 0.9597
& 0.2146 & 5.8582 & 0.2619 & 0.6883 & 0.8021 \\

DPS\textsuperscript{\dag}~\cite{Chung:2023DPS}
& 500
& 6.7647 & 0.2191 & 0.2861 & 1.0639 & 0.9624
& 0.2147 & 5.8073 & 0.2622 & 0.6310 & 0.8028 \\

FreeDoM\textsuperscript{\dag}~\cite{Yu:2023FreeDOM}
& 533
& 6.8406 & 0.2185 & 0.2853 & 0.9941 & 0.9635
& 0.2133 & 5.8492 & 0.2572 & 0.5354 & 0.7990 \\

SVDD~\cite{Li2024:SVDD}
& 500
& 7.1363 & 0.2177 & 0.2814 & 1.0256 & 0.9510
& 0.2204 & \cellcolor{lightgreen}5.8743 & \cellcolor{lightgreen}0.2699 & \cellcolor{darkgreen}0.7592 & \cellcolor{lightgreen}0.8201 \\

RBF~\cite{kim2025:rbf}
& 500
& 6.9900 & 0.2183 & 0.2826 & 1.0761 & 0.9689
& 0.2202 & 5.8618 & 0.2682 & \cellcolor{lightgreen}0.7583 & 0.8149 \\

DAS~\cite{Kim:2025DAS}
& 500
& 6.9384 & 0.2183 & 0.2860 & 1.0568 & 0.9706
& 0.2139 & 5.8385 & 0.2568 & 0.5226 & 0.7990 \\

$\Psi$-Sampler~\cite{yoon2025:psi}
& 500
& 7.0116 & 0.2188 & 0.2847 & 1.1235 & \cellcolor{darkgreen}0.9737
& 0.2120 & 5.7329 & 0.2551 & 0.4590 & 0.8145 \\

\midrule

\textbf{{\Oursbf{}} (Ours)}
& 25
& \cellcolor{lightgreen}7.4510 & \cellcolor{darkgreen}0.2200 & \cellcolor{lightgreen}0.2928 & \cellcolor{darkgreen}1.2565 & \cellcolor{lightgreen}0.9728
& \cellcolor{lightgreen}0.2224 & 5.7720 & 0.2601 & 0.5257 & \cellcolor{darkgreen}0.8210 \\

{\Oursbf{}}\textsuperscript{\dag} \textbf{(Ours)}
& 500
& \cellcolor{darkgreen}7.9656 & \cellcolor{lightgreen}0.2197 & \cellcolor{darkgreen}0.2932 & \cellcolor{lightgreen}1.1669 & 0.9609
& \cellcolor{darkgreen}0.2327 & \cellcolor{darkgreen}5.9020 & \cellcolor{darkgreen}0.2817 & 0.7370 & 0.8017 \\

\bottomrule
\end{tabular}%
}
\end{table*}

\section{Experiments}
\label{sec:exp}
In this section, we present experimental results demonstrating the effectiveness of \Ours{} compared to prior baselines, with the experimental setup described in \cref{subsec:exp_setup}. We first evaluate aesthetic image generation and text-aligned image generation in \cref{subsec:aesthetic} and \cref{subsec:human-preference-image}, respectively. We then extend the evaluation to preference-aligned video generation in \cref{subsec:human-preference-video}.
Finally, we show that our method can also be applied on top of fine-tuned models in \cref{subsec:fine-tune-plug}.
We provide additional applications, including counting tasks and VLM-based reward alignment, in \cref{app:counting,app:vlm_reward_alignment}, and alignment results using a different diffusion model~\cite{cai2025:zimage} in \cref{app:zimage_results}.
Additional quantitative and qualitative results for the main experiments are provided in \cref{app:additional_main}.

\subsection{Experiment Setup}
\label{subsec:exp_setup}
\paragraph{Tasks.} We evaluate \Ours{} across three reward-guided generation settings in the main text: aesthetic image generation, text-aligned image generation, and preference-aligned video generation.
For aesthetic image generation, we use $45$ animal prompts from previous work, DDPO~\cite{Black2024:DDPO}. For text-aligned image generation, we use $100$ prompts in complex category of T2I-CompBench++~\cite{Huang:2025T2ICompBench++}. For preference-aligned video generation, we use 200 prompts from VBench~\cite{huang2023:vbench} animal and scenery categories.

For image and video generation applications, we use FLUX~\cite{flux2024} and Wan2.1~\cite{wan2025} as the base flow models, respectively.
In all experiments, we fix the sampling steps to $25$.
As a reference, we also include the results of the base models without any guidance method applied.

\begin{figure*}[t!]
\centering
\setlength{\tabcolsep}{0pt}
\renewcommand{\arraystretch}{1.05}
\scriptsize

\begin{tabular}{@{}
>{\centering\arraybackslash}p{0.03\textwidth}
*{6}{>{\centering\arraybackslash}m{0.161\textwidth}}
@{}}
\toprule

&
Base~\cite{flux2024} &
BoN~\cite{Stiennon:2020BoN} &
DPS\textsuperscript{\dag}~\cite{Chung:2023DPS} &
SVDD~\cite{Li2024:SVDD} &
$\Psi$-Sampler~\cite{yoon2025:psi} &
{\Oursbf{}}\textsuperscript{\dag} \textbf{(Ours)} \\

\midrule

\multirow{6}{0.03\textwidth}[-20pt]{\centering\rotatebox{90}{\normalsize \textbf{Aesthetic Image Generation}}}
& \multicolumn{6}{@{}c@{}}{\textit{``Sheep''}} \\[2pt]
& \includegraphics[width=0.162\textwidth]{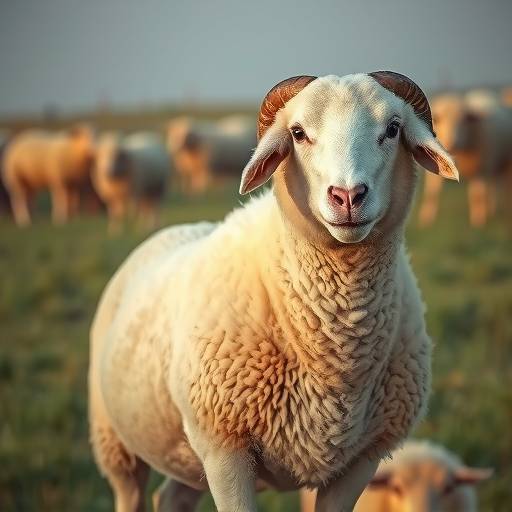} &
  \includegraphics[width=0.162\textwidth]{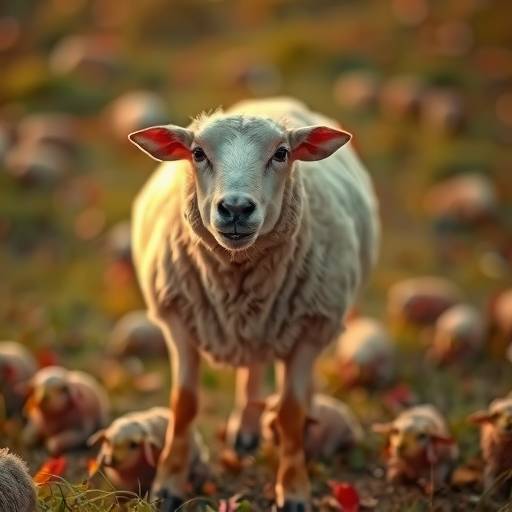} &
  \includegraphics[width=0.162\textwidth]{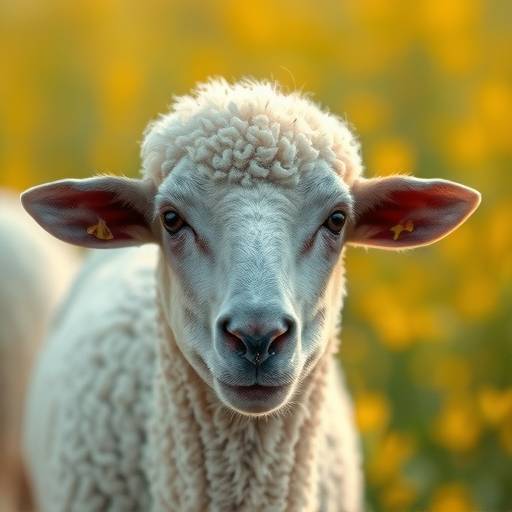} &
  \includegraphics[width=0.162\textwidth]{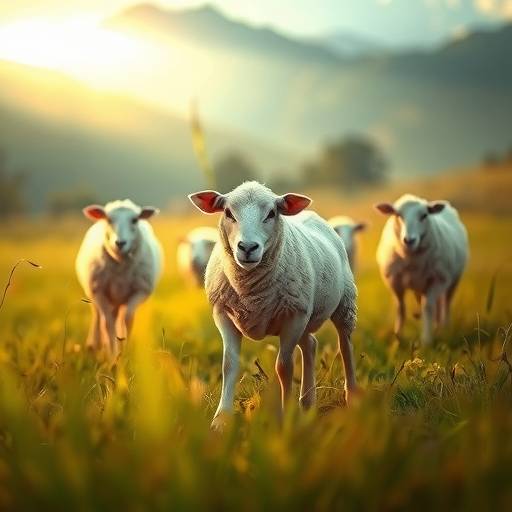} &
  \includegraphics[width=0.162\textwidth]{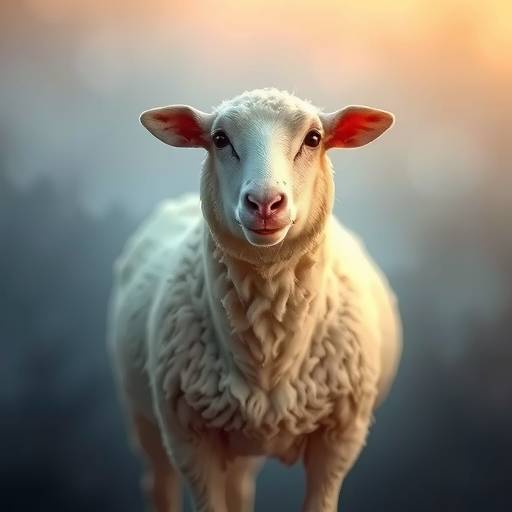} &
  \includegraphics[width=0.162\textwidth]{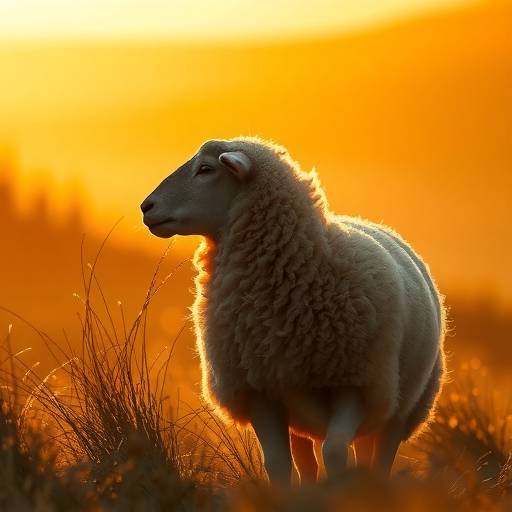} \\
& \textit{6.5477} & \textit{6.6494} & \textit{6.9178} & \textit{6.8442} & \textit{7.0671} & \textit{8.5842} \\[3pt]
& \multicolumn{6}{@{}c@{}}{\textit{``Frog''}} \\[2pt]
& \includegraphics[width=0.162\textwidth]{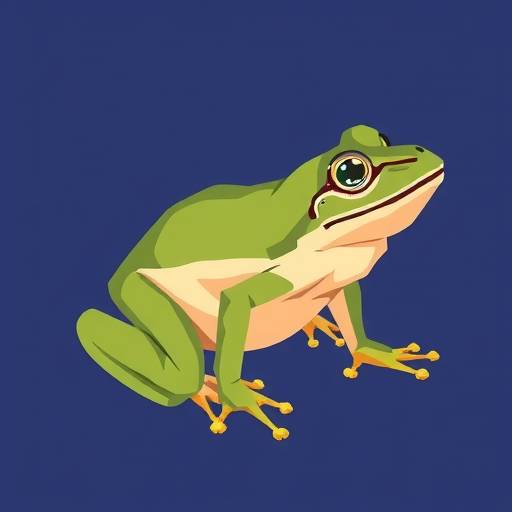} &
  \includegraphics[width=0.162\textwidth]{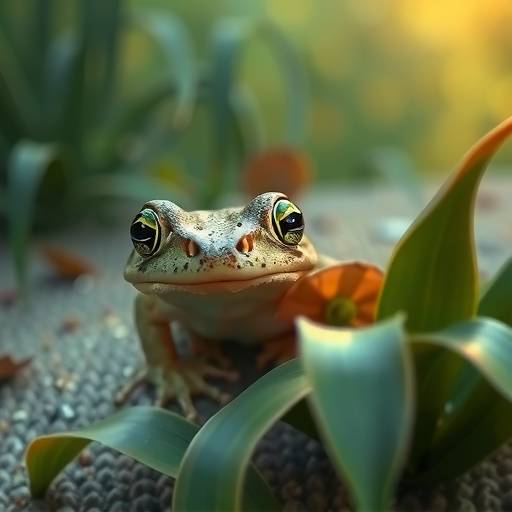} &
  \includegraphics[width=0.162\textwidth]{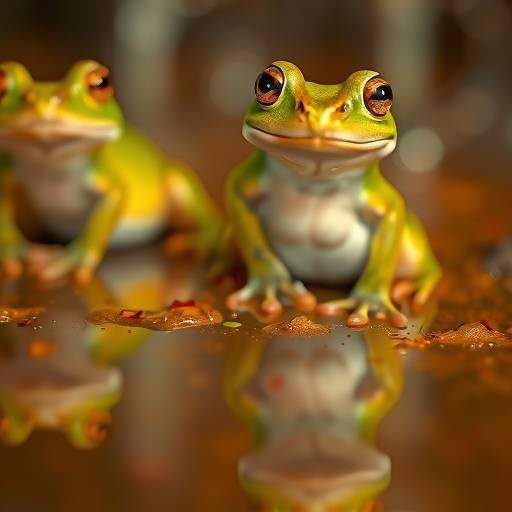} &
  \includegraphics[width=0.162\textwidth]{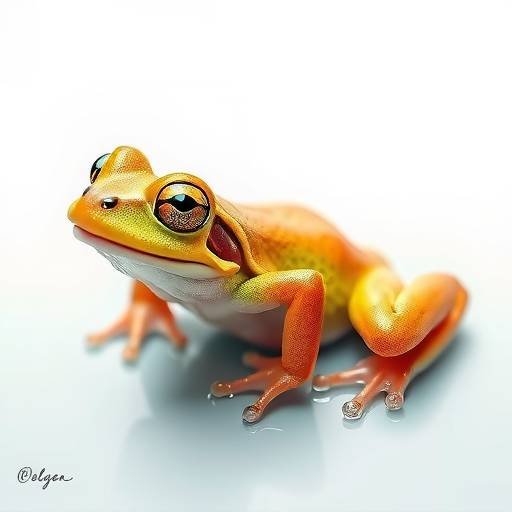} &
  \includegraphics[width=0.162\textwidth]{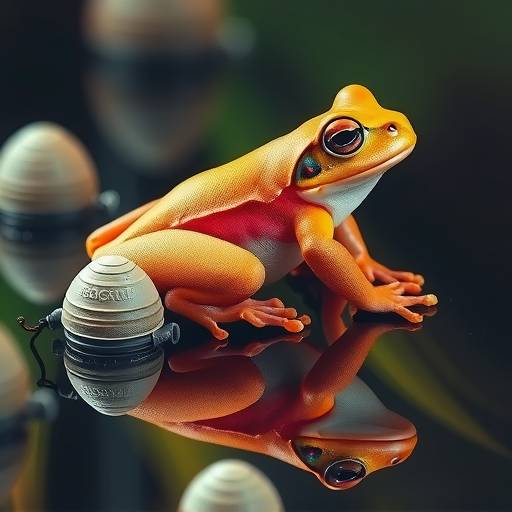} &
  \includegraphics[width=0.162\textwidth]{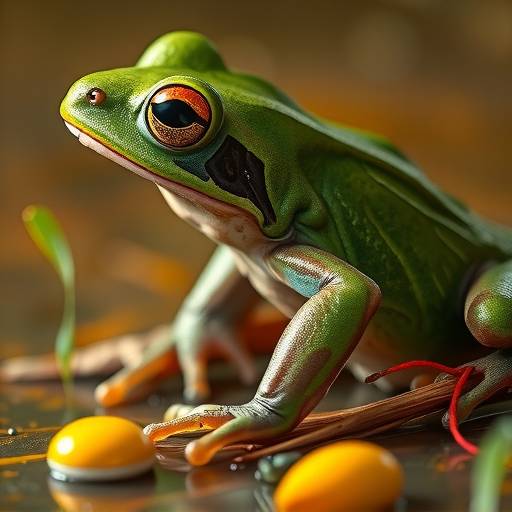} \\
& \textit{6.0279} & \textit{7.1517} & \textit{6.9119} & \textit{7.0071} & \textit{7.1417} & \textit{8.6239} \\

\midrule

\multirow{6}{0.03\textwidth}[-16pt]{\centering\rotatebox{90}{\normalsize \textbf{Text-Aligned Image Generation}}}
& \multicolumn{6}{@{}p{0.95\textwidth}@{}}{\textit{\makecell{``The yellow cone was suspended in mid-air near the orange pyramid and the green cylinder.''}}} \\[2pt]
& \includegraphics[width=0.162\textwidth]{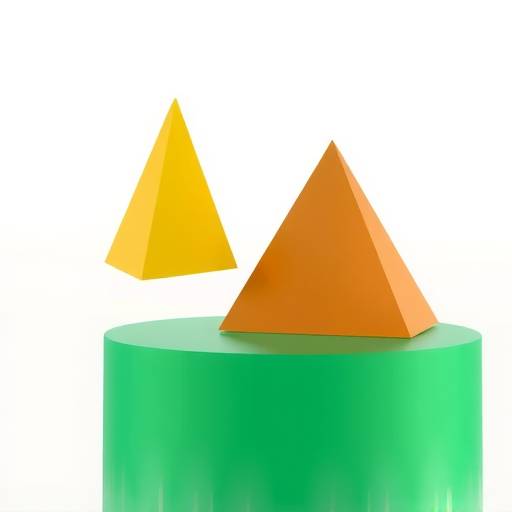} &
  \includegraphics[width=0.162\textwidth]{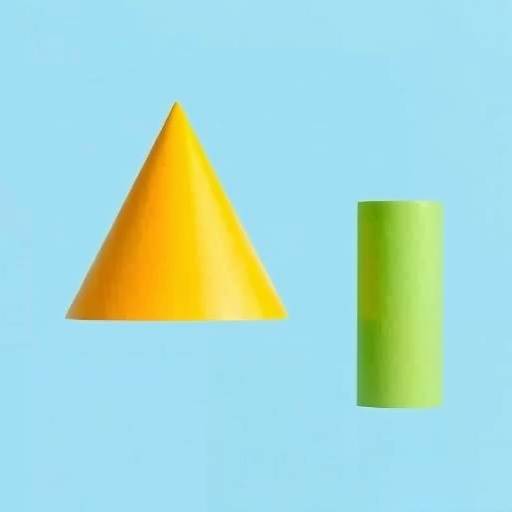} &
  \includegraphics[width=0.162\textwidth]{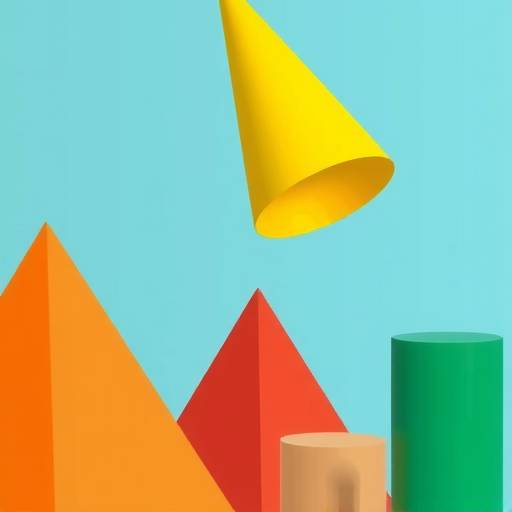} &
  \includegraphics[width=0.162\textwidth]{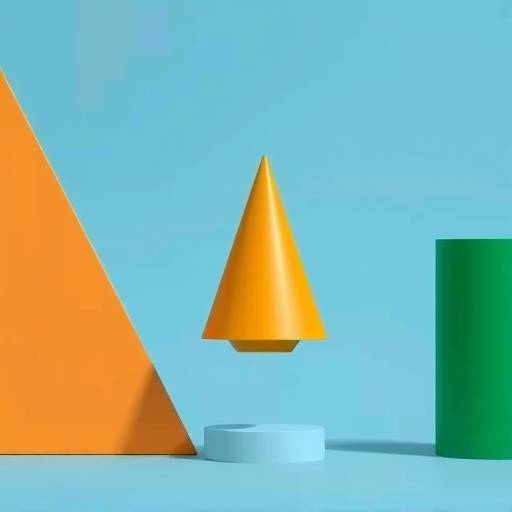} &
  \includegraphics[width=0.162\textwidth]{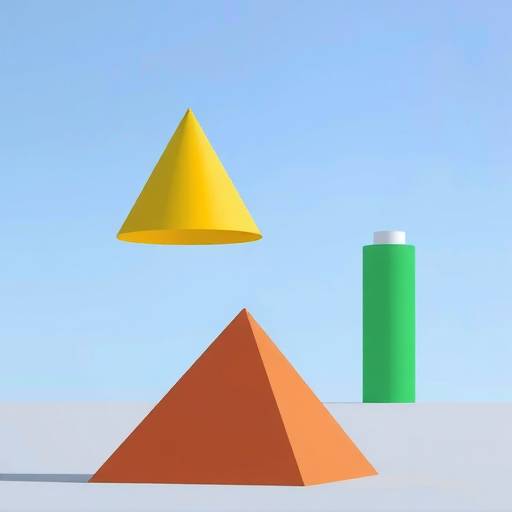} &
  \includegraphics[width=0.162\textwidth]{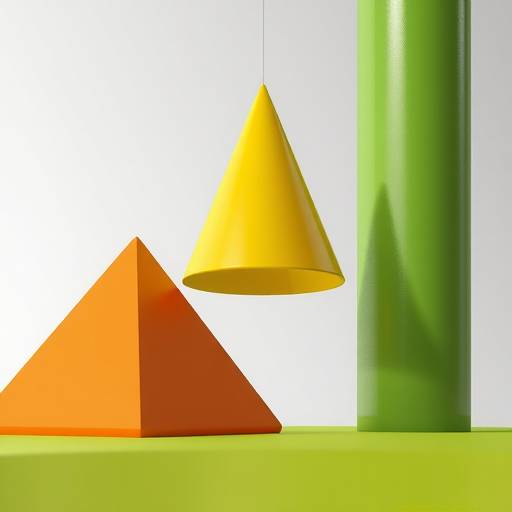} \\
& \textit{0.2190} & \textit{0.2227} & \textit{0.2235} & \textit{0.2328} & \textit{0.2224} & \textit{0.2372} \\[3pt]
& \multicolumn{6}{@{}p{0.95\textwidth}@{}}{\textit{\makecell{``The soft, billowing curtains fluttered in the gentle breeze, adding a touch of elegance to the room.''}}} \\[2pt]
& \includegraphics[width=0.162\textwidth]{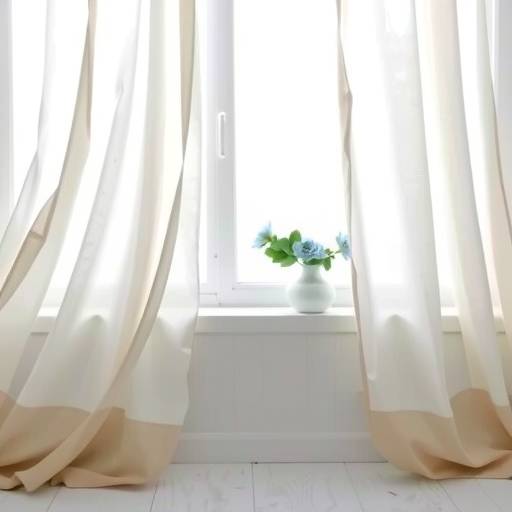} &
  \includegraphics[width=0.162\textwidth]{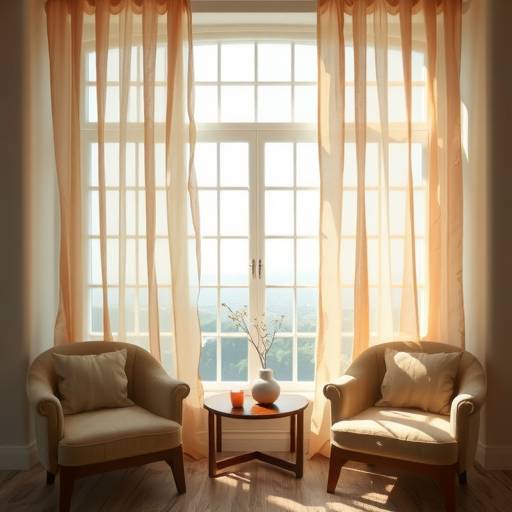} &
  \includegraphics[width=0.162\textwidth]{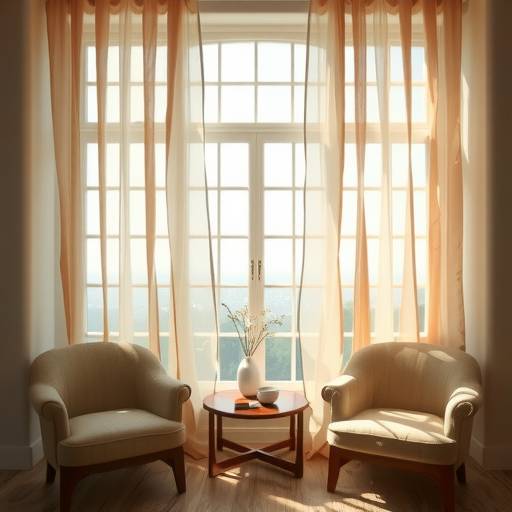} &
  \includegraphics[width=0.162\textwidth]{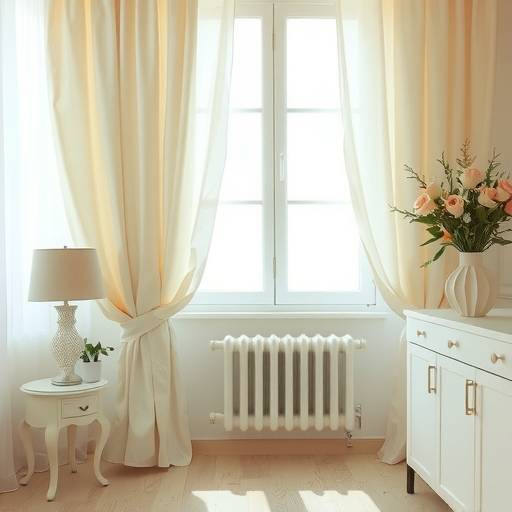} &
  \includegraphics[width=0.162\textwidth]{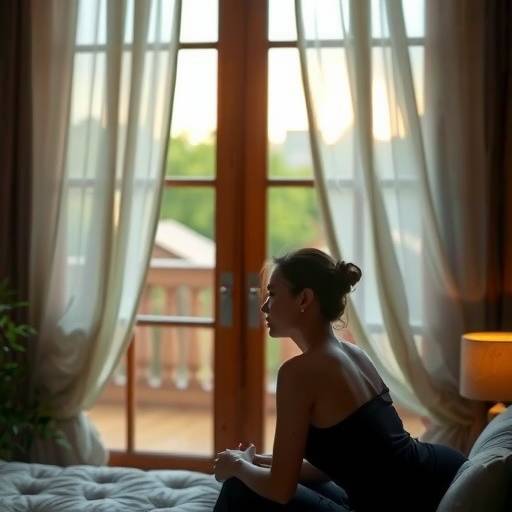} &
  \includegraphics[width=0.162\textwidth]{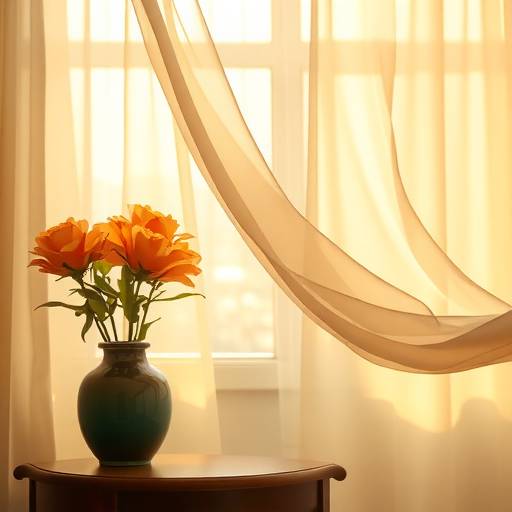} \\
& \textit{0.2070} & \textit{0.2126} & \textit{0.2122} & \textit{0.2146} & \textit{0.2055} & \textit{0.2313} \\

\bottomrule
\end{tabular}

\caption{
\footnotesize
\textbf{Qualitative comparison on image reward alignment.}
Top: aesthetic image generation (target: Aesthetic Score~\cite{Schuhmann:aesthetics}).
Bottom: text-aligned image generation (target: PickScore~\cite{Kirstain2023:pickapic}).
Scores shown in italics below each image.
For single-particle methods we augment sampling with Best-of-N to match the total NFE, denoted with \textsuperscript{\dag}.
}
\label{fig:aesthetic_qualitative}
\label{fig:text_align_qualitative}
\end{figure*}

\paragraph{Baselines.}
We compare \Ours{} against a range of inference-time reward-alignment algorithms discussed in~\cref{sec:related_work}, including both gradient-based guidance methods and search-based approaches. Specifically, we consider DPS~\cite{Chung:2023DPS} and FreeDoM~\cite{Yu:2023FreeDOM} as single-particle gradient-based methods, DAS~\cite{Kim:2025DAS} and $\Psi$-Sampler~\cite{yoon2025:psi} as multi-particle gradient-based methods, and SVDD~\cite{Li2024:SVDD}, RBF~\cite{kim2025:rbf}, and \textsc{BoN}~\cite{Stiennon:2020BoN} as search-based methods.
We additionally compare against DNO~\cite{tang2025dno}, a noise-optimization method with a distinct NFE budget structure, in \cref{app:dno_comparison}.

Note that multi-particle and search methods utilize multiple samples during sampling, whereas single-particle methods produce only a single trajectory.
For fair comparison, we fix the total number of function evaluations (NFE) across all methods.
In particular, for single-particle sampling methods such as DPS and FreeDoM, we augment sampling with Best-of-N (BoN)~\cite{Stiennon:2020BoN}, which runs multiple independent sampling processes and selects the highest-reward output, ensuring that the overall NFE is comparable across methods.
\footnote{FreeDoM~\cite{Yu:2023FreeDOM} incorporates additional MCMC sampling, which increases the computational cost.} For all quantitative results, methods augmented with BoN are marked with \textsuperscript{\dag}.
For \Ours{}, we report results both with and without BoN to isolate the effect of the proposed method.
Implementation details and hyperparameter settings are summarized in \cref{app:imple_details}.

\subsection{Aesthetic Image Generation}
\label{subsec:aesthetic}

\paragraph{Evaluation Metrics.}
In this work, we refer to the reward used for inference-time optimization as the \emph{target reward}, and to rewards not observed during optimization as \emph{held-out rewards}. In this task, the target reward is Aesthetic Score~\cite{Schuhmann:aesthetics}. As held-out rewards, we evaluate image quality using ImageReward~\cite{Xu2023:ImageReward} and HPSv2~\cite{wu2023:hpsv2}, and text–image alignment using PickScore~\cite{Kirstain2023:pickapic} and VQA Score~\cite{lin2024:vqa}.

\paragraph{Results.}
The quantitative and qualitative results are presented in~\cref{tab:aesthetic_full} and~\cref{fig:aesthetic_qualitative}, respectively. Overall, \Ours{} achieves the best target reward performance across all methods, and even outperforms all baselines with only $1/20$ of NFE.
On held-out rewards, \Ours{} also delivers the best image quality for both ImageReward and HPSv2, while remaining comparable on text-image alignment metrics.
The qualitative examples in~\cref{fig:aesthetic_qualitative} further support these trends.
Across different prompts, \Ours{} produces more visually appealing samples than the baselines, achieving the highest rewards~\cite{Schuhmann:aesthetics}.

\begin{table*}[t!]
\centering
\caption{
\footnotesize
\textbf{Quantitative comparison on preference-aligned video generation.}
The target reward is VideoReward~\cite{liu2025:videoalign}, and the held-out rewards are the metrics proposed in VBench~\cite{huang2023:vbench}.
Dark green cells indicate the best result for each metric across all runs, while light green cells denote the second best.
}
\label{tab:video_reward}
\renewcommand{\arraystretch}{1.18}
\setlength{\tabcolsep}{2.2pt}
\scriptsize

\resizebox{\textwidth}{!}{%
\begin{tabular}{l c c c c c c c c}
\toprule

\multirow{3}{*}[-2.0ex]{\textbf{Method}}
& \multirow{3}{*}[-2.0ex]{\textbf{NFE}}
& \textbf{Target Reward}
& \multicolumn{6}{c}{\textbf{Held-Out Reward}} \\

\cmidrule(lr){3-3} \cmidrule(lr){4-9}

&
&
\multirow{2}{*}[-1.0ex]{\shortstack{\textbf{VideoReward} $\uparrow$}}
& \multicolumn{2}{c}{\textbf{Motion Quality}}
& \multicolumn{2}{c}{\textbf{Visual Quality}}
& \multicolumn{2}{c}{\textbf{Text Alignment}} \\

\cmidrule(lr){4-5} \cmidrule(lr){6-7} \cmidrule(lr){8-9}

&
&
& Smooth.$\uparrow$ 
& Dynamic$\uparrow$
& Aesthetic$\uparrow$
& Imaging.$\uparrow$
& Subject$\uparrow$
& Back.$\uparrow$ \\ 

\midrule

Base~\cite{wan2025}
& 25
& -0.399 
& 0.9629 
& 0.9300 
& \cellcolor{lightgreen}0.6104 
& 0.6791 
& \cellcolor{lightgreen}0.9589 
& \cellcolor{darkgreen}0.9647 \\

DPS~\cite{Chung:2023DPS}
& 25
& \cellcolor{lightgreen}-0.130 
& 0.9621 
& \cellcolor{lightgreen}0.9350 
& 0.5867 
& 0.6738 
& 0.9398 
& 0.9473 \\

FreeDoM~\cite{Yu:2023FreeDOM}
& 25
& -0.211 
& \cellcolor{darkgreen}0.9632 
& 0.6900 
& 0.5880 
& \cellcolor{darkgreen}0.6897 
& 0.9586 
& 0.9570 \\

{\Oursbf{}} \textbf{(Ours)}
& 25
& \cellcolor{darkgreen}3.465 
& \cellcolor{lightgreen}0.9630 
& \cellcolor{darkgreen}0.9500 
& \cellcolor{darkgreen}0.6120 
& \cellcolor{lightgreen}0.6870 
& \cellcolor{darkgreen}0.9591 
& \cellcolor{lightgreen}0.9646 \\

\bottomrule
\end{tabular}%
}
\end{table*}
\begin{figure*}[t!]
\centering
\setlength{\tabcolsep}{1pt}
\renewcommand{\arraystretch}{1.05}
\scriptsize

\newcommand{\methodw}{0.11\textwidth}
\newcommand{\framew}{0.142\textwidth}
\newcommand{\gapw}{0.001\textwidth}

\begin{tabular}{@{}%
>{\centering\arraybackslash}m{\methodw}
>{\centering\arraybackslash}m{\framew}
>{\centering\arraybackslash}m{\framew}
>{\centering\arraybackslash}m{\framew}
m{\gapw}
>{\centering\arraybackslash}m{\framew}
>{\centering\arraybackslash}m{\framew}
>{\centering\arraybackslash}m{\framew}
@{}}
\toprule

& \multicolumn{3}{c}{\scriptsize\textit{``Turkey in cage.''}}
& 
& \multicolumn{3}{c}{\scriptsize\textit{``Mother whale swimming with babies.''}} \\

Base~\cite{wan2025}
& \includegraphics[width=\framew]{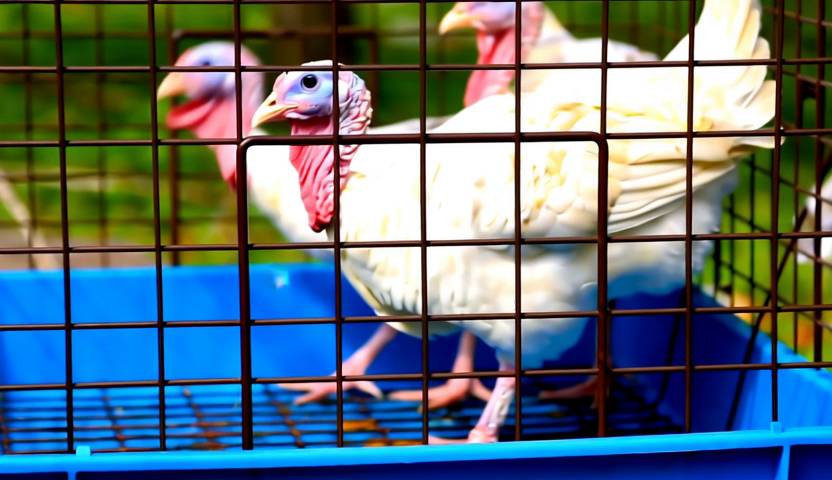}
& \includegraphics[width=\framew]{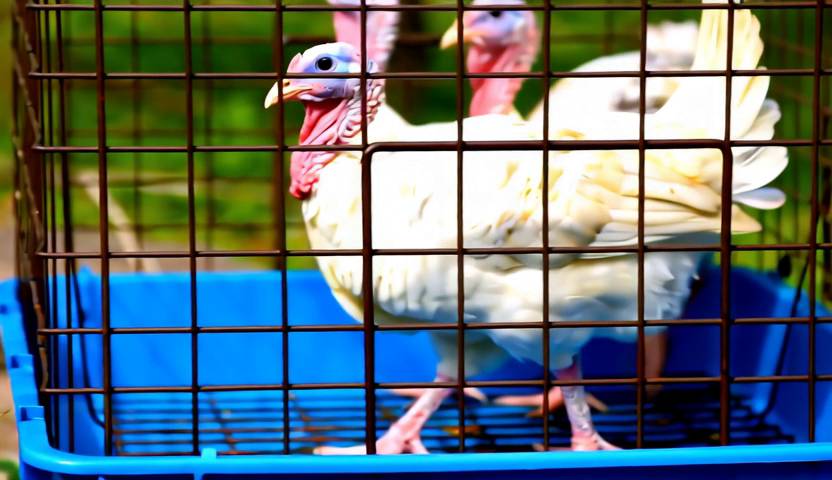}
& \includegraphics[width=\framew]{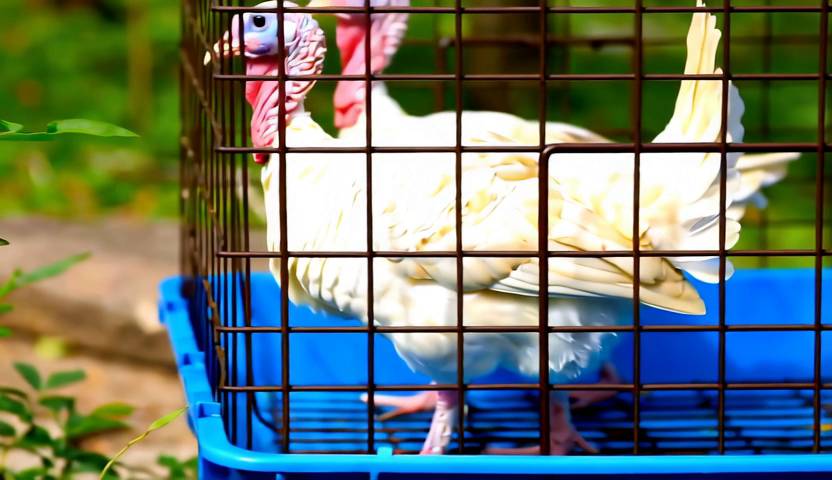}
&
& \includegraphics[width=\framew]{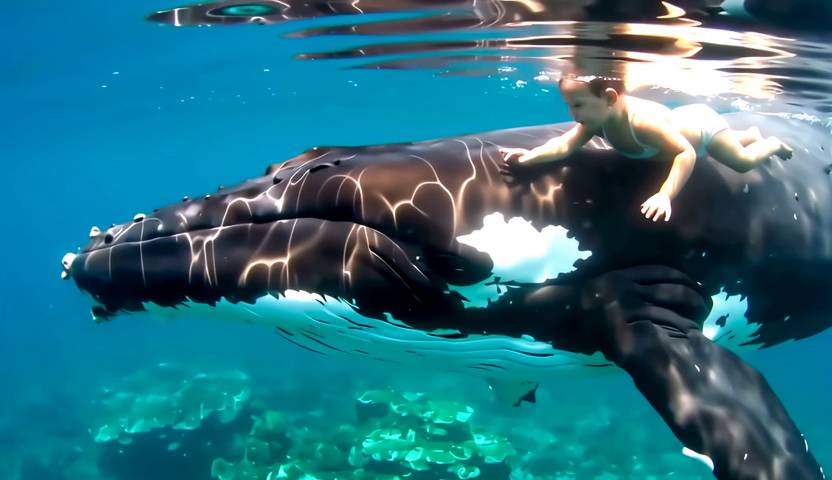}
& \includegraphics[width=\framew]{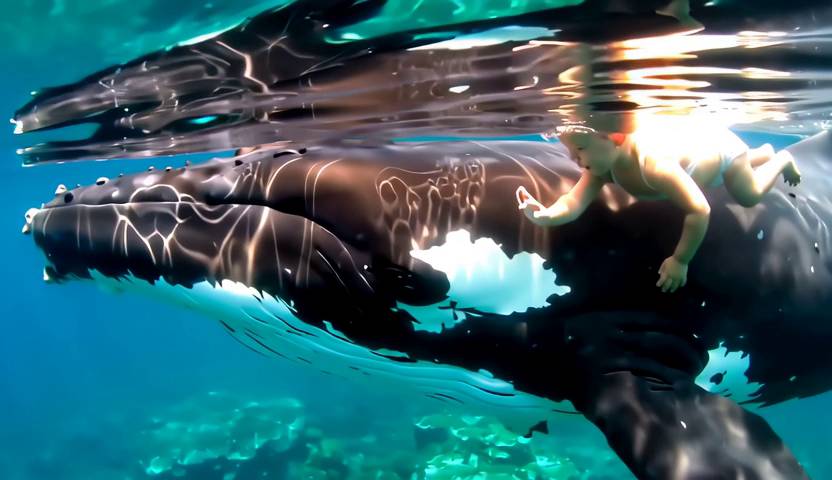}
& \includegraphics[width=\framew]{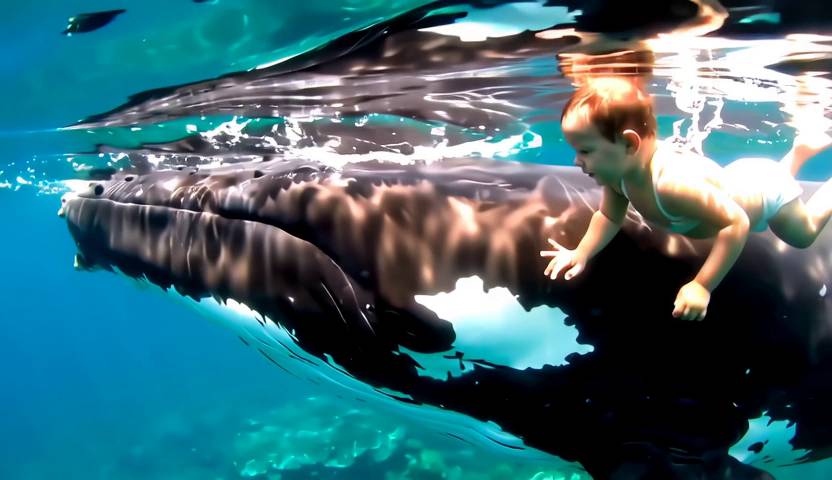} \\

DPS~\cite{Chung:2023DPS}
& \includegraphics[width=\framew]{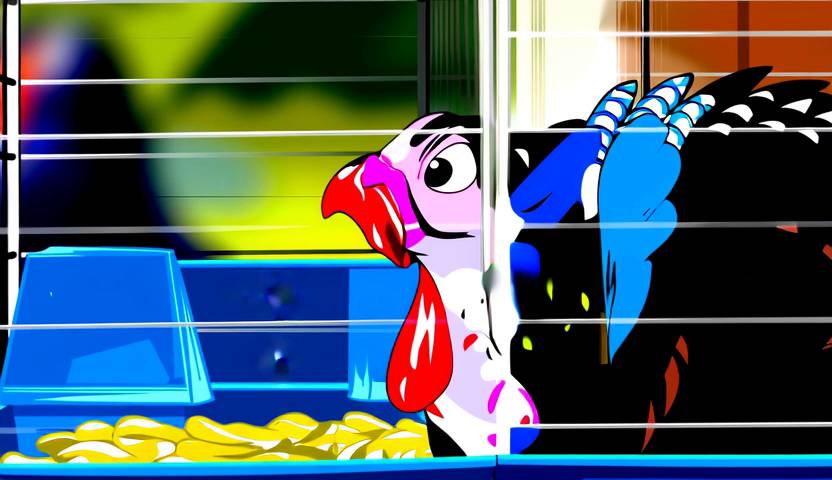}
& \includegraphics[width=\framew]{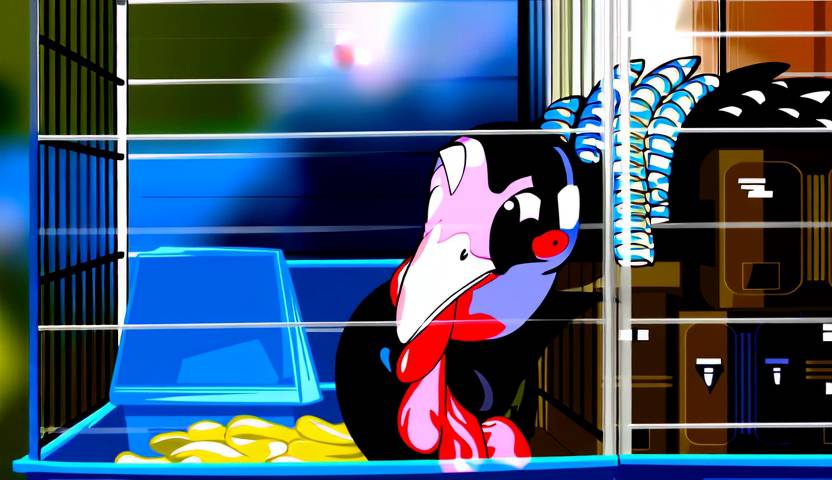}
& \includegraphics[width=\framew]{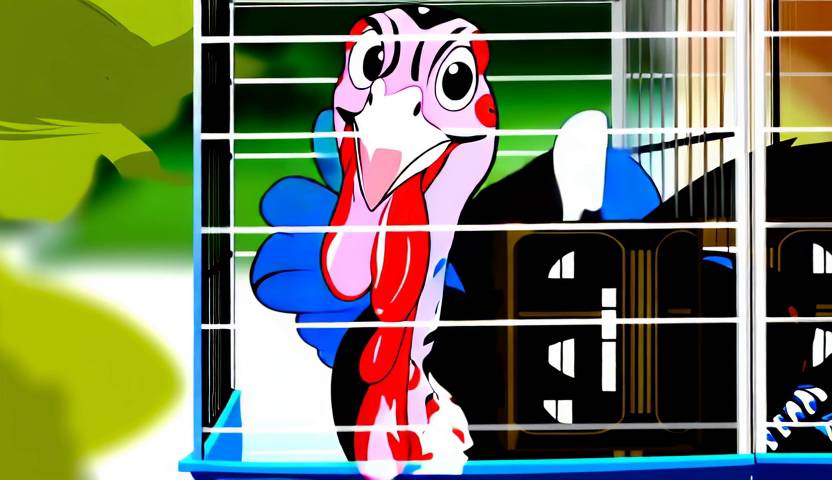}
&
& \includegraphics[width=\framew]{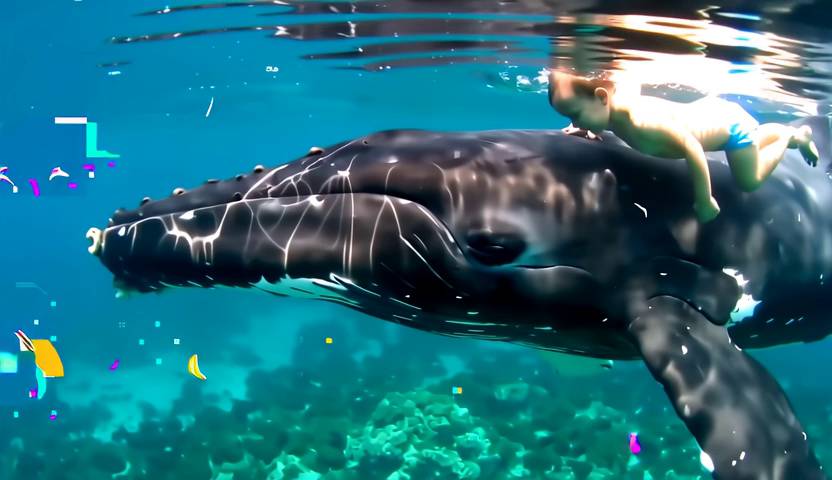}
& \includegraphics[width=\framew]{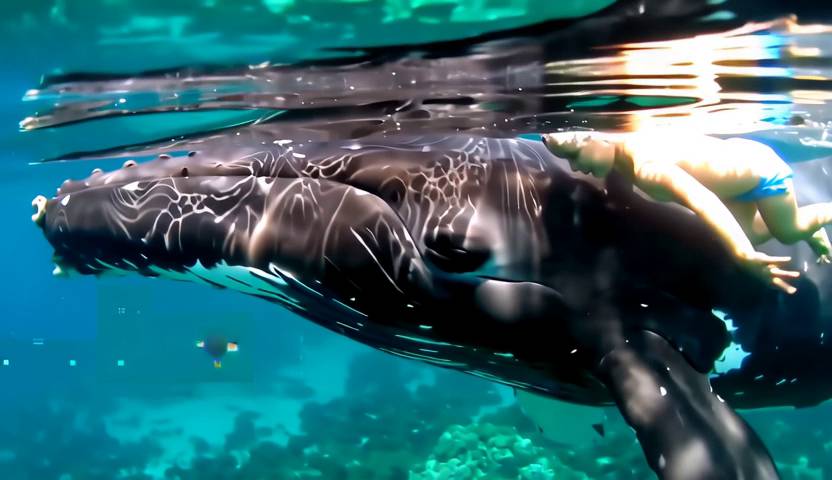}
& \includegraphics[width=\framew]{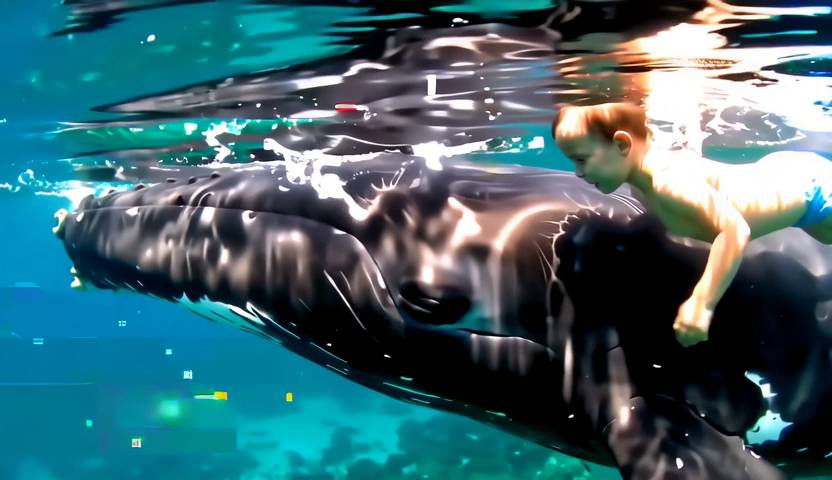} \\

\makecell{FreeDoM \\ \cite{Yu:2023FreeDOM}}
& \includegraphics[width=\framew]{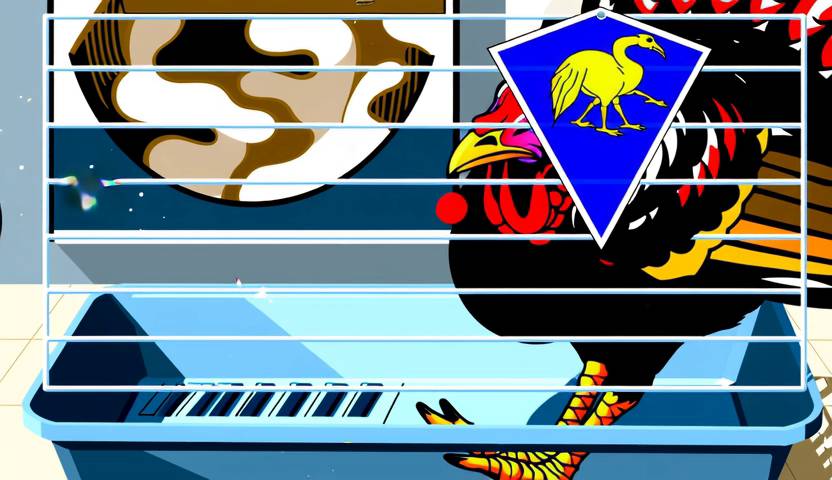}
& \includegraphics[width=\framew]{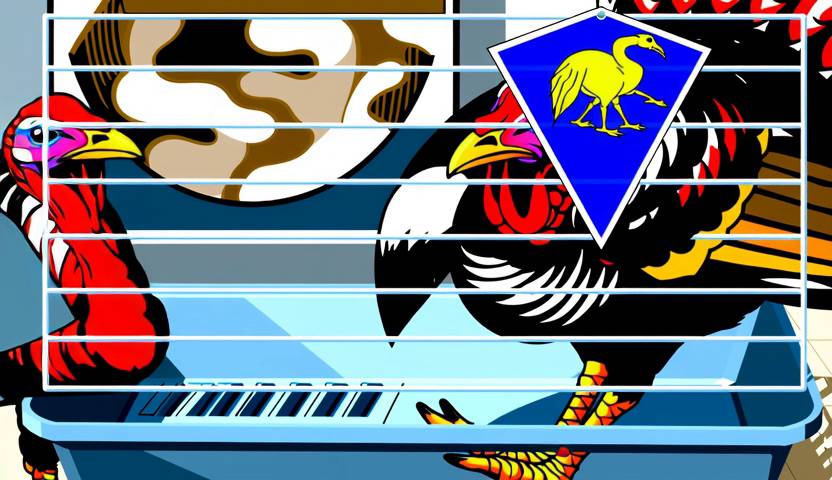}
& \includegraphics[width=\framew]{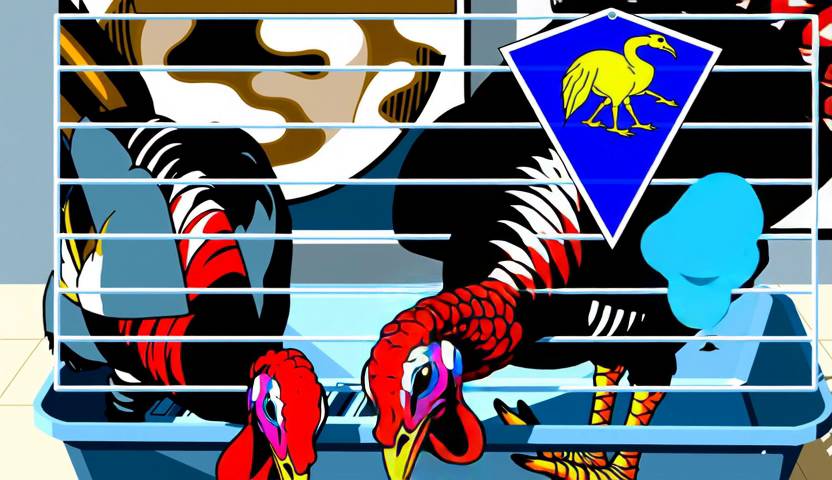}
&
& \includegraphics[width=\framew]{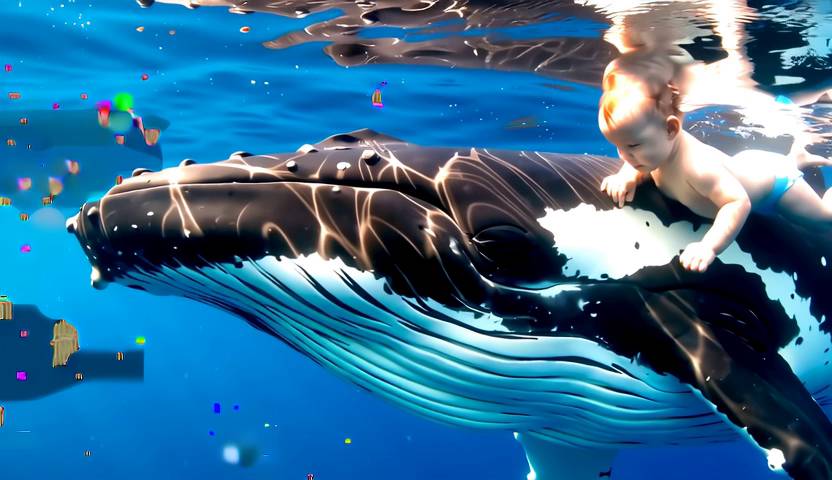}
& \includegraphics[width=\framew]{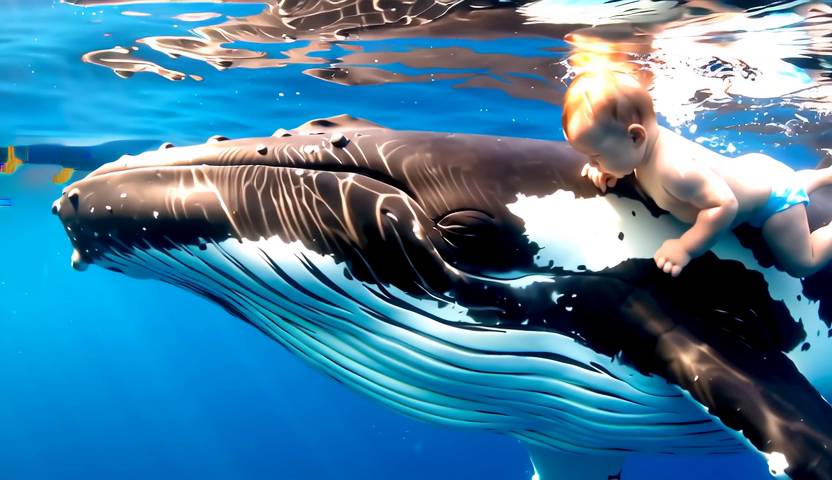}
& \includegraphics[width=\framew]{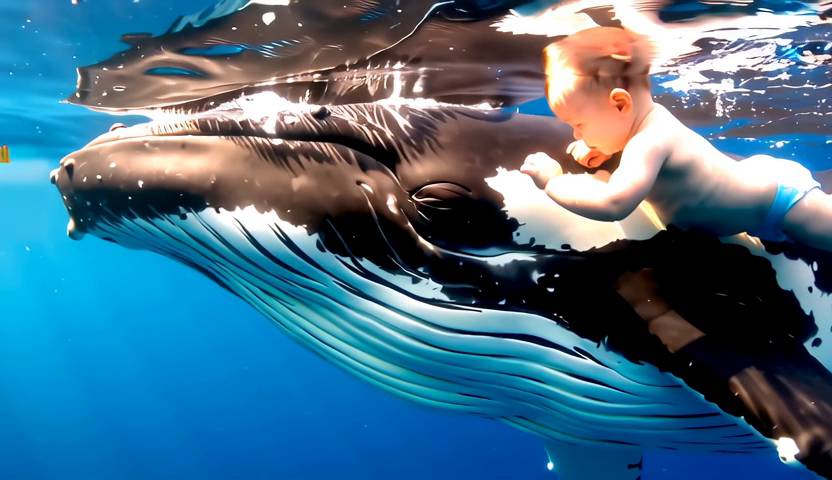} \\

{\Oursbf{}} \textbf{(Ours)}
& \includegraphics[width=\framew]{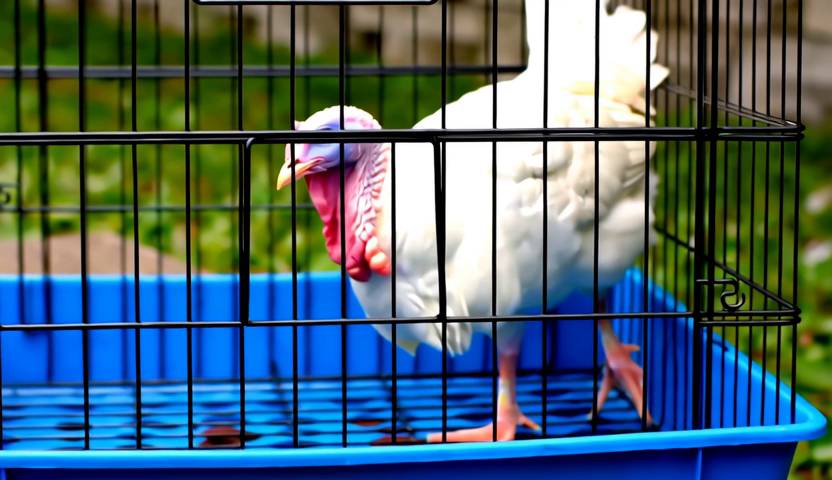}
& \includegraphics[width=\framew]{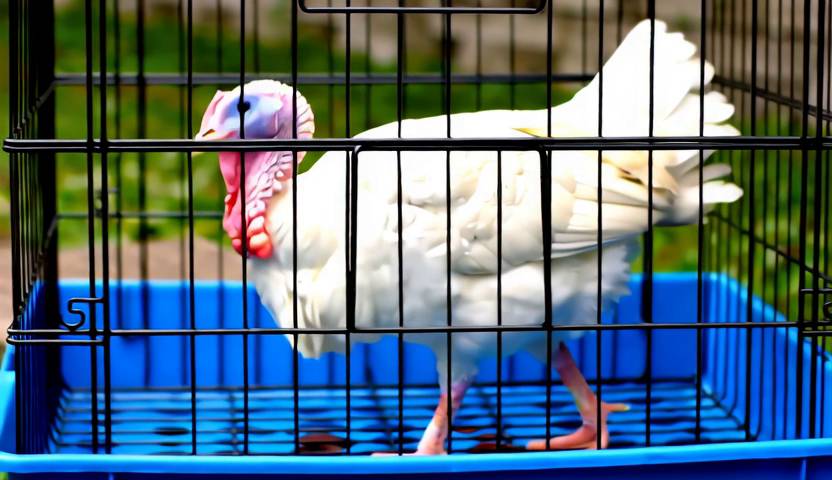}
& \includegraphics[width=\framew]{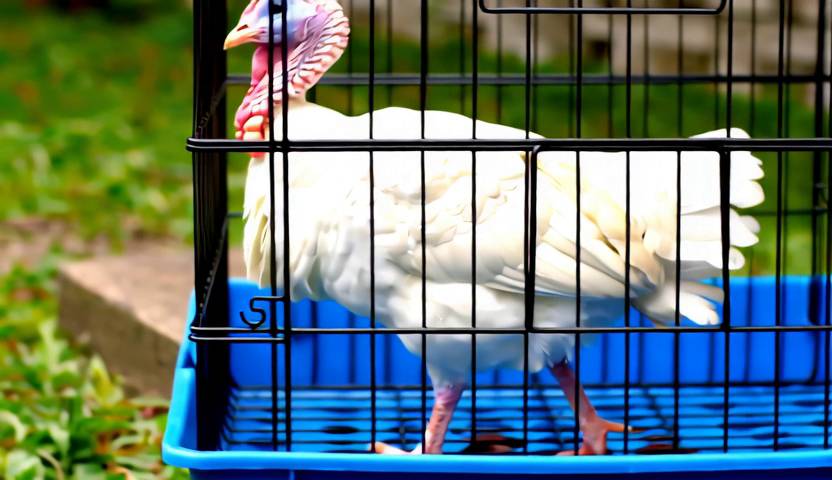}
&
& \includegraphics[width=\framew]{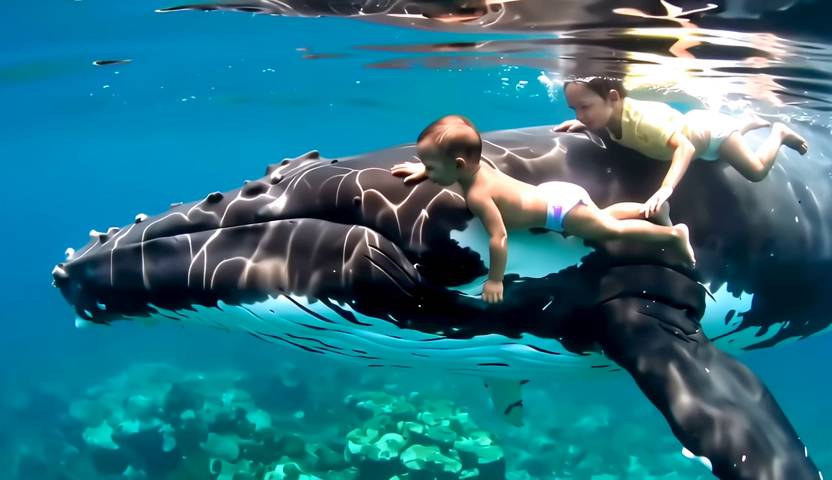}
& \includegraphics[width=\framew]{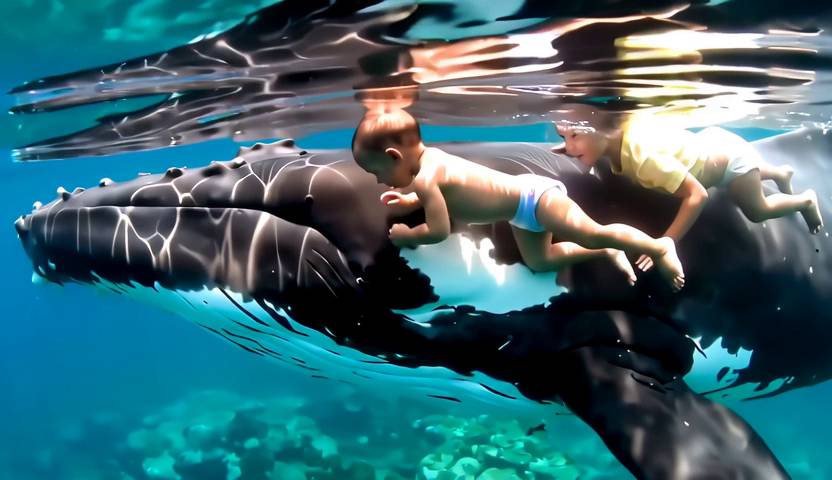}
& \includegraphics[width=\framew]{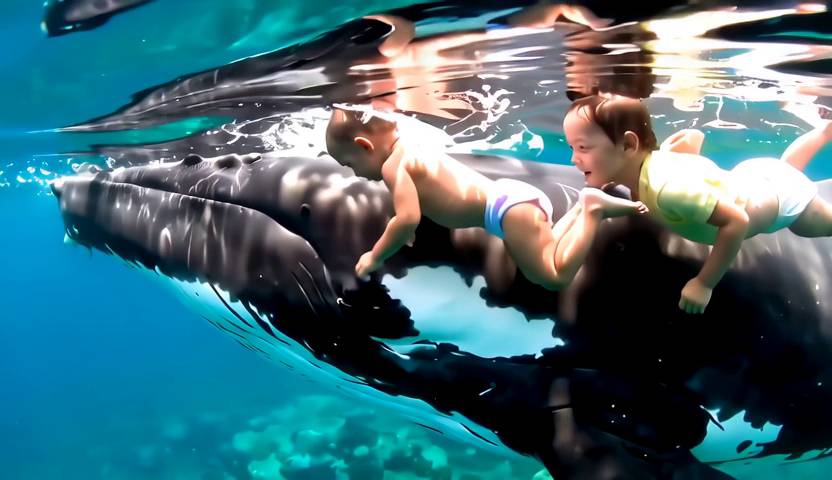} \\

\bottomrule
\end{tabular}

\caption{
\footnotesize
\textbf{Qualitative comparison on preference-aligned video generation using VideoReward~\cite{liu2025:videoalign}.} \Ours{} produces videos with better text alignment and visual quality, and additional examples are provided in \cref{fig:video_qualitative_supp}.
}
\label{fig:video_qualitative}
\end{figure*}

\subsection{Text-Aligned Image Generation}
\label{subsec:human-preference-image}

\paragraph{Evaluation Metrics.}
The target reward used to align text-image is PickScore~\cite{Kirstain2023:pickapic}.
For held-out rewards, we evaluate text–image alignment using VQA Score~\cite{lin2024:vqa}, and image quality using Aesthetic Score~\cite{Schuhmann:aesthetics}, ImageReward~\cite{Xu2023:ImageReward}, and HPSv2~\cite{wu2023:hpsv2}.

\paragraph{Results.}
The quantitative and qualitative results are presented in~\cref{tab:pick_full} and~\cref{fig:text_align_qualitative}, respectively.
As in the aesthetic image generation task, \Ours{} achieves the best target reward performance across all methods, with the same efficiency trend that the $25$-NFE setting already outperforms all baselines.
These gains also transfer to held-out rewards: \Ours{} with $500$ NFE achieves the best Aesthetic Score and HPSv2, while \Ours{} with $25$ NFE attains the best VQA Score.
Qualitatively, \Ours{} produces samples that better align with text prompts, particularly on spatial and logical relations, compared to the baselines.

\begin{table*}[h!]
\centering
\caption{
\footnotesize
\textbf{Quantitative comparison on text-aligned image generation with fine-tuned model.}
We compare DPS~\cite{Chung:2023DPS} and \Ours{} integrated with a fine-tuned model, MixGRPO~\cite{li2026:mixgrpo}.
For single-particle methods we augment sampling with Best-of-N, denoted with \textsuperscript{\dag}.
Dark green cells indicate the best result for each metric across all runs, while light green cells denote the second best.
}
\label{tab:pick_ft_tto}
\renewcommand{\arraystretch}{1.10}
\setlength{\tabcolsep}{5.0pt}
\scriptsize

\resizebox{\textwidth}{!}{%
\begin{tabular}{@{} l c c c c c c @{}}
\toprule
\multirow{2}{*}[-1.4ex]{\textbf{Method}}
& \multirow{2}{*}[-1.4ex]{\textbf{NFE}}
& \textbf{Target Reward}
& \multicolumn{4}{c}{\textbf{Held-Out Reward}} \\
\cmidrule(lr){3-3} \cmidrule(lr){4-7}
&
&
\makecell{\textbf{PickScore}} $\uparrow$
& \makecell{\textbf{Aesthetic}\\\textbf{Score}} $\uparrow$
& \textbf{HPSv2} $\uparrow$
& \makecell{\textbf{Image}\\\textbf{Reward}} $\uparrow$
& \makecell{\textbf{VQA}\\\textbf{Score}} $\uparrow$ \\
\midrule

Base~\cite{flux2024}
& 25
& 0.2054
& 5.4664
& 0.2316
& 0.1710
& 0.8011 \\

MixGRPO~\cite{li2026:mixgrpo}
& 25
& 0.2166
& 6.5245
& 0.2679
& 0.7605
& 0.8239 \\

\midrule

\hspace{1mm} $\llcorner$ DPS\textsuperscript{\dag}~\cite{Chung:2023DPS}
& 500
& \cellcolor{lightgreen}0.2235
& \cellcolor{lightgreen}6.6966
& \cellcolor{lightgreen}0.2840
& \cellcolor{lightgreen}1.0501
& \cellcolor{lightgreen}0.8376 \\

\hspace{1mm} $\llcorner$ {\Oursbf{}}\textsuperscript{\dag} \textbf{(Ours)}
& 500
& \cellcolor{darkgreen}0.2545
& \cellcolor{darkgreen}6.7281
& \cellcolor{darkgreen}0.3224
& \cellcolor{darkgreen}1.2648
& \cellcolor{darkgreen}0.8498 \\

\bottomrule
\end{tabular}%
}
\end{table*}

\subsection{Preference-Aligned Video Generation}
\label{subsec:human-preference-video}
As done in the image generation task, we use $25$ sampling steps for all methods except FreeDoM~\cite{Yu:2023FreeDOM}, for which we use 13 steps due to its additional MCMC sampling, ensuring that the total NFE remains comparable.
For video generation task, we test the baselines and our method with a single particle.

\paragraph{Evaluation Metrics.}
In this task, the target reward is VideoReward~\cite{liu2025:videoalign}, which provides three component scores for Motion Quality (MQ), Visual Quality (VQ), and Text Alignment (TA). We use the sum of these components (MQ + VQ + TA) as the target reward. For held-out evaluation, we report metrics from VBench~\cite{huang2023:vbench}, grouped into six categories: Subject Consistency and Background Consistency for text alignment, Motion Smoothness and Dynamic Degree for motion quality, and Aesthetic Quality and Imaging Quality for visual quality.

\paragraph{Results.}
\Cref{tab:video_reward} and~\cref{fig:video_qualitative} summarize the quantitative and qualitative comparisons, respectively.
\Ours{} yields the highest VideoReward score and outperforms all the baselines.
Beyond the target reward, \Ours{} achieves the best results on Dynamic Degree, Aesthetic Quality, and Subject Consistency, and remains marginally runner-up on the other held-out metrics.
The qualitative examples further support this observation, with \Ours{} producing videos that better align with the text prompts. In particular, \Ours{} generates videos that clearly capture the elements (e.g., turkey and babies) described in the prompt.

\subsection{Integration with Fine-Tuned Models}
\label{subsec:fine-tune-plug}
Fine-tuning and inference-time alignment improve reward alignment along orthogonal directions: the former adapts model parameters, while the latter guides the sampling process without modifying them. Because of this orthogonality, \Ours{} can be applied on top of a fine-tuned model, and the two strategies can be combined whenever the goal is to maximize a specific reward as much as possible. In this section, we integrate \Ours{} with MixGRPO~\cite{li2026:mixgrpo}, which fine-tunes the base FLUX model~\cite{flux2024}.

\paragraph{Evaluation Metrics.}
In this task, the target reward is PickScore~\cite{Kirstain2023:pickapic}, and we report the same held-out rewards used in the text-aligned image generation task described in \cref{subsec:human-preference-image}.

\paragraph{Results.}
As shown in~\cref{tab:pick_ft_tto}, MixGRPO~\cite{li2026:mixgrpo} improves both the target reward and all held-out rewards compared to the base model~\cite{flux2024}.
We further observe that applying inference-time reward alignment remains highly effective on top of the fine-tuned model, yielding additional improvements across all metrics.
In particular, \Ours{} consistently outperforms DPS on both the target reward and all held-out rewards, achieving the highest overall scores.
These results confirm that \Ours{} is orthogonal to fine-tuning: the two can be composed to push a specific reward beyond what either approach achieves alone.

\section{Conclusion}
\label{sec:conclusion}
In this work, we identified a fundamental trade-off in inference-time reward alignment: gradient-based guidance steers generation effectively but degrades sample quality, while search-based methods preserve quality but forgo gradient guidance. We resolved this trade-off with the Noise-Tilted Reverse Kernel (\Ours{}), which leaves the reverse mean unchanged and instead biases the noise term toward high reward through a whitening operator that makes the reward gradient safe to inject as noise without losing its guiding signal.

Across aesthetic image generation, text-aligned image generation, and preference-aligned video generation, \Ours{} consistently outperforms recent baselines in target reward alignment without losing sample quality, and on aesthetic generation surpasses the reward of the best baseline at 500 NFEs using only 25, a $20\times$ reduction in compute. Beyond this setting, \Ours{} improves reward alignment on top of fine-tuned models, transfers to a different diffusion backbone, and extends to counting-based and VLM-based rewards.

By routing reward information through the noise term rather than the mean, \Ours{} offers a simple, broadly applicable mechanism for inference-time reward alignment that leaves the pretrained sampling dynamics intact.

\section*{Acknowledgments}
We thank Takeo Igarashi for valuable discussions. I-Chao Shen acknowledges support from the UTokyo-Google AI Symbiotic Future Society Program. This work was also supported by the NRF grant (RS-2026-25486000); IITP grants (RS-2024-00399817, RS-2025-25441313, RS-2025-25443318, and RS-2026-25526850), funded by MSIT, Korea; the Industrial Technology Innovation Program grant (RS-2025-02317326), funded by MOTIE, Korea; the InnoCORE program of MSIT (AI Meta-Scientist, N10260110); the National Supercomputing Center with supercomputing resources and technical support (KSC-2025-CRE-0475); the Advanced GPU Utilization Support Program; and the DRB-KAIST SketchTheFuture Research Center.

\bibliographystyle{noisetilt_arxiv}
\bibliography{main}

@String(CVPR   = {Proceedings of the IEEE/CVF Conference on Computer Vision and Pattern Recognition (CVPR)})

@String(ICCV   = {Proceedings of the IEEE/CVF International Conference on Computer Vision (ICCV)})

@String(ECCV   = {Proceedings of the European Conference on Computer Vision (ECCV)})

@String(NeurIPS = {Advances in Neural Information Processing Systems})

@String(NIPS   = {Advances in Neural Information Processing Systems})

@String(ICML   = {International Conference on Machine Learning})

@String(ICLR   = {International Conference on Learning Representations})

@String(CVPRW  = {IEEE/CVF Conference on Computer Vision and Pattern Recognition Workshops})

@String(TOG    = {ACM Transactions on Graphics})

@inproceedings{samuel2023:norm,
  title={Norm-Guided Latent Space Exploration for Text-to-Image Generation},
  author={Samuel, Dvir and Ben-Ari, Rami and Darshan, Nir and Maron, Haggai and Chechik, Gal},
  booktitle=NeurIPS,
  volume={36},
  pages={57863--57875},
  year={2023}
}

@inproceedings{kim2025:rbf,
  title={Inference-Time Scaling for Flow Models via Stochastic Generation and Rollover Budget Forcing},
  author={Kim, Jaihoon and Yoon, Taehoon and Hwang, Jisung and Sung, Minhyuk},
  booktitle=NeurIPS,
  volume={38},
  pages={30830--30864},
  year={2025}
}

@inproceedings{yoon2025:psi,
  title={Psi-Sampler: Initial Particle Sampling for SMC-Based Inference-Time Reward Alignment in Score Models},
  author={Yoon, Taehoon and Min, Yunhong and Yeo, Kyeongmin and Sung, Minhyuk},
  booktitle=NIPS,
  volume={38},
  pages={104745--104781},
  year={2025}
}

@inproceedings{eyring2024reno,
  title={{ReNO}: Enhancing One-Step Text-to-Image Models through Reward-Based Noise Optimization},
  author={Eyring, Luca and Karthik, Shyamgopal and Roth, Karsten and Dosovitskiy, Alexey and Akata, Zeynep},
  booktitle=NeurIPS,
  volume={37},
  pages={125487--125519},
  year={2024}
}

@inproceedings{hwang2025:mpgr,
  title={Moment- and Power-Spectrum-Based {Gaussianity} Regularization for Text-to-Image Models},
  author={Hwang, Jisung and Kim, Jaihoon and Sung, Minhyuk},
  booktitle=NeurIPS,
  volume={38},
  pages={18235--18264},
  year={2025}
}

@misc{flux2024,
    author={Black Forest Labs},
    title={{FLUX}},
    year={2024},
    howpublished={\url{https://github.com/black-forest-labs/flux}},
}

@article{wan2025,
      title={Wan: Open and Advanced Large-Scale Video Generative Models}, 
      author={Wan Team, Alibaba Group},
      year={2025},
      journal={arXiv preprint arXiv:2503.20314}
}

@book{Robbins1992,
    author={Robbins, Herbert E.},
    title={An Empirical Bayes Approach to Statistics},
    publisher={Springer},
    year={1992},
}

@misc{Schuhmann:aesthetics,
    title        = {{LAION} Aesthetics},
    author       = {C. Schuhmann},
    year={2022},
    howpublished={\url{https://laion.ai/blog/laion-aesthetics}},
}

@inproceedings{Kirstain2023:pickapic,
  title={{Pick-a-Pic}: An Open Dataset of User Preferences for Text-to-Image Generation},
  author={Kirstain, Yuval and Polyak, Adam and Singer, Uriel and Matiana, Shahbuland and Penna, Joe and Levy, Omer},
  booktitle=NeurIPS,
  volume={36},
  pages={36652--36663},
  year={2023}
}

@article{qwen,
  title={{Qwen2.5-VL} Technical Report},
  author={Bai, Shuai and Chen, Keqin and Liu, Xuejing and Wang, Jialin and Ge, Wenbin and Song, Sibo and Dang, Kai and Wang, Peng and Wang, Shijie and Tang, Jun and Zhong, Humen and Zhu, Yuanzhi and Yang, Mingkun and Li, Zhaohai and Wan, Jianqiang and Wang, Pengfei and Ding, Wei and Fu, Zheren and Xu, Yiheng and Ye, Jiabo and Zhang, Xi and Xie, Tianbao and Cheng, Zesen and Zhang, Hang and Yang, Zhibo and Xu, Haiyang and Lin, Junyang},
  journal={arXiv preprint arXiv:2502.13923},
  year={2025}
}

@book{Doucet2001:SMC,
    title={Sequential Monte Carlo methods in practice},
    author={Doucet, Arnaud and De Freitas, Nando and Gordon, Neil James and others},
    year={2001},
    publisher={Springer},
}

@inproceedings{Stiennon:2020BoN,
    title = {Learning to summarize from human feedback},
    author = {Stiennon, Nisan and Ouyang, Long and Wu, Jeff and Ziegler, Daniel M. and Lowe, Ryan and Voss, Chelsea and Radford, Alec and Amodei, Dario and Christiano, Paul},
    year = {2020},
    booktitle = NIPS,
    volume = {33},
    pages = {3008--3021},
}

@article{singhal2025:fk,
  title={A general framework for inference-time scaling and steering of diffusion models},
  author={Singhal, Raghav and Horvitz, Zachary and Teehan, Ryan and Ren, Mengye and Yu, Zhou and McKeown, Kathleen and Ranganath, Rajesh},
  journal={arXiv preprint arXiv:2501.06848},
  year={2025}
}

@article{ramesh2025:search,
  title={Test-time scaling of diffusion models via noise trajectory search},
  author={Ramesh, Vignav and Mardani, Morteza},
  journal={arXiv preprint arXiv:2506.03164},
  year={2025}
}

@inproceedings{Li2024:SVDD,
    title={Derivative-Free Guidance in Continuous and Discrete Diffusion Models with Soft Value-Based Decoding},
    author={Li, Xiner and Zhao, Yulai and Wang, Chenyu and Scalia, Gabriele and Eraslan, Gokcen and Nair, Surag and Biancalani, Tommaso and Regev, Aviv and Levine, Sergey and Uehara, Masatoshi},
    booktitle=NIPS,
    volume={38},
    pages={95507--95545},
    year={2025},
}

@inproceedings{Rafailov2023:DPO,
  title={Direct preference optimization: Your language model is secretly a reward model},
  author={Rafailov, Rafael and Sharma, Archit and Mitchell, Eric and Manning, Christopher D and Ermon, Stefano and Finn, Chelsea},
  booktitle=NIPS,
  volume={36},
  pages={53728--53741},
  year={2023}
}

@inproceedings{Cardoso2024:MCGDiff,
    title={Monte Carlo guided Diffusion for Bayesian linear inverse problems}, 
    author={Cardoso, Gabriel and Idrissi, Yazid Janati El and Corff, Sylvain Le and Moulines, Eric},
    booktitle=ICLR,  
    year={2024},
}

@inproceedings{Kim:2025DAS,
    title={Test-time Alignment of Diffusion Models without Reward Over-optimization},
    author={Kim, Sunwoo and Kim, Minkyu and Park, Dongmin},
    booktitle=ICLR,
    year={2025},
}

@inproceedings{Dou2024:FPS,
  title={Diffusion posterior sampling for linear inverse problem solving: A filtering perspective},
  author={Dou, Zehao and Song, Yang},
  booktitle=ICLR,
  year={2024}
}

@article{Ma2025:SoP,
    title={Inference-Time Scaling for Diffusion Models beyond Scaling Denoising Steps}, 
    author={Ma, Nanye and Tong, Shangyuan and Jia, Haolin and Hu, Hexiang and Su, Yu-Chuan and Zhang, Mingda and Yang, Xuan and Li, Yandong and Jaakkola, Tommi and Jia, Xuhui and Xie, Saining},
    year={2025},
    journal={arXiv preprint arXiv:2501.09732},
}

@inproceedings{Xu2023:ImageReward,
    title={{ImageReward}: Learning and Evaluating Human Preferences for Text-to-Image Generation},
    author={Xu, Jiazheng and Liu, Xiao and Wu, Yuchen and Tong, Yuxuan and Li, Qinkai and Ding, Ming and Tang, Jie and Dong, Yuxiao},
    booktitle=NeurIPS,
    volume={36},
    pages={15903--15935},
    year={2023},
}

@inproceedings{Clark2024:DRaFT,
    title={Directly Fine-Tuning Diffusion Models on Differentiable Rewards}, 
    author={Clark, Kevin and Vicol, Paul and Swersky, Kevin and Fleet David J},
    booktitle=ICLR,
    year={2024},
}

@inproceedings{Black2024:DDPO,
    title={Training Diffusion Models with Reinforcement Learning},
    author={Black, Kevin and Janner, Michael and Du, Yilun and Kostrikov, Ilya and Levine, Sergey},
    booktitle=ICLR,
    year={2024},
}

@inproceedings{Fan:2023DPOK,
    title = {{DPOK}: reinforcement learning for fine-tuning text-to-image diffusion models},
    author = {Fan, Ying and Watkins, Olivia and Du, Yuqing and Liu, Hao and Ryu, Moonkyung and Boutilier, Craig and Abbeel, Pieter and Ghavamzadeh, Mohammad and Lee, Kangwook and Lee, Kimin},
    booktitle = NeurIPS,
    volume = {36},
    pages = {79858--79885},
    year = {2023},
}

@inproceedings{Wallace:2024DiffusionDPO,
    title={Diffusion Model Alignment Using Direct Preference Optimization},
    author={Wallace, Bram and Dang, Meihua and Rafailov, Rafael and Zhou, Linqi and Lou, Aaron and Purushwalkam, Senthil and Ermon, Stefano and Xiong, Caiming and Joty, Shafiq and Naik, Nikhil},
    booktitle=CVPR,
    pages={8228--8238},
    year={2024},
}

@inproceedings{Uehara:2024Finetuning,
    title={Fine-Tuning of Continuous-Time Diffusion Models as Entropy-Regularized Control},
    author={Uehara, Masatoshi and Zhao, Yulai and Black, Kevin and Hajiramezanali, Ehsan and Scalia, Gabriele and Diamant, Nathaniel Lee and Tseng, Alex M and Biancalani, Tommaso and Levine, Sergey},
    booktitle=ICLR,
    year={2025},
}

@inproceedings{Gao:2023Scaling,
    title = {Scaling laws for reward model overoptimization},
    author = {Gao, Leo and Schulman, John and Hilton, Jacob},
    booktitle = ICML,
    pages = {10835--10866},
    year = {2023},
}

@inproceedings{Ho:2020DDPM,
    title = {Denoising diffusion probabilistic models},
    author = {Ho, Jonathan and Jain, Ajay and Abbeel, Pieter},
    booktitle = NeurIPS,
    volume = {33},
    pages = {6840--6851},
    year = {2020},
}

@inproceedings{Song:2021SDE,
    title={Score-Based Generative Modeling through Stochastic Differential Equations},
    author={Song, Yang and Sohl-Dickstein, Jascha and Kingma, Diederik P and Kumar, Abhishek and Ermon, Stefano and Poole, Ben},
    booktitle=ICLR,
    year={2021},
}

@inproceedings{Rombach:2022LDM,
    title     = {High-Resolution Image Synthesis With Latent Diffusion Models},
    author    = {Rombach, Robin and Blattmann, Andreas and Lorenz, Dominik and Esser, Patrick and Ommer, Bj\"orn},
    booktitle = CVPR,
    pages     = {10674--10685},
    year      = {2022},
}

@inproceedings{Chung:2023DPS,
    title={Diffusion Posterior Sampling for General Noisy Inverse Problems},
    author={Chung, Hyungjin and Kim, Jeongsol and Mccann, Michael Thompson and Klasky, Marc Louis and Ye, Jong Chul},
    booktitle=ICLR,
    year={2023},
}

@inproceedings{Song:2023LGD,
    title={Loss-Guided Diffusion Models for Plug-and-Play Controllable Generation},
    author={Song, Jiaming and Zhang, Qinsheng and Yin, Hongxu and Mardani, Morteza and Liu, Ming-Yu and Kautz, Jan and Chen, Yongxin and Vahdat, Arash},
    booktitle=ICML,
    pages={32483--32498},
    year={2023},
}

@inproceedings{Yu:2023FreeDOM,
    title={{FreeDoM}: Training-Free Energy-Guided Conditional Diffusion Model},
    author={Yu, Jiwen and Wang, Yinhuai and Zhao, Chen and Ghanem, Bernard and Zhang, Jian},
    booktitle=ICCV,
    pages={23174--23184},
    year={2023},
}

@inproceedings{prabhudesai2024:VADER,
  title={Video Diffusion Alignment via Reward Gradients},
  author={Prabhudesai, Mihir and Mendonca, Russell and Qin, Zheyang and Fragkiadaki, Katerina and Pathak, Deepak},
  booktitle=ICLR,
  year={2025}
}

@inproceedings{daras2024:warped,
  title={Warped diffusion: Solving video inverse problems with image diffusion models},
  author={Daras, Giannis and Nie, Weili and Kreis, Karsten and Dimakis, Alex and Mardani, Morteza and Kovachki, Nikola and Vahdat, Arash},
  booktitle=NeurIPS,
  volume={37},
  pages={101116--101143},
  year={2024}
}

@inproceedings{wu2023:tds,
  title={Practical and asymptotically exact conditional sampling in diffusion models},
  author={Wu, Luhuan and Trippe, Brian and Naesseth, Christian and Blei, David and Cunningham, John P.},
  booktitle=NeurIPS,
  volume={36},
  pages={31372--31403},
  year={2023}
}

@inproceedings{song2023:pseudoinverse,
  title={Pseudoinverse-guided diffusion models for inverse problems},
  author={Song, Jiaming and Vahdat, Arash and Mardani, Morteza and Kautz, Jan},
  booktitle=ICLR,
  year={2023}
}

@inproceedings{Ye:2024TFG,
    title = {{TFG}: Unified Training-Free Guidance for Diffusion Models},
    booktitle = NeurIPS,
    author = {Ye, Haotian and Lin, Haowei and Han, Jiaqi and Xu, Minkai and Liu, Sheng and Liang, Yitao and Ma, Jianzhu and Zou, James Y and Ermon, Stefano},
    volume = {37},
    pages = {22370--22417},
    year = {2024},
}

@inproceedings{Bansal:2023UGD,
    title={Universal Guidance for Diffusion Models}, 
    author={Bansal, Arpit and Chu, Hong-Min and Schwarzschild, Avi and Sengupta, Soumyadip and Goldblum, Micah and Geiping, Jonas and Goldstein, Tom},
    booktitle=CVPRW,
    year={2023},
}

@inproceedings{Benhamu:2024DFlow,
    title={{D-Flow}: Differentiating through Flows for Controlled Generation},
    author={Ben-Hamu, Heli and Puny, Omri and Gat, Itai and Karrer, Brian and Singer, Uriel and Lipman, Yaron},
    booktitle=ICML,
    pages={3462--3483},
    year={2024},
}

@inproceedings{Korbak:2022RLDM,
    title = {On reinforcement learning and distribution matching for fine-tuning language models with no catastrophic forgetting},
    author = {Korbak, Tomasz and Elsahar, Hady and Kruszewski, Germ\'{a}n and Dymetmant, Marc},
    year = {2022},
    booktitle = NIPS,
    volume = {35},
    pages = {16203--16220},
}

@inproceedings{Uehara:2024Bridging,
    title={Bridging Model-Based Optimization and Generative Modeling via Conservative Fine-Tuning of Diffusion Models},
    author={Uehara, Masatoshi and Zhao, Yulai and Hajiramezanali, Ehsan and Scalia, Gabriele and Eraslan, G{\"o}kcen and Lal, Avantika and Levine, Sergey and Biancalani, Tommaso},
    booktitle=NeurIPS,
    volume={37},
    pages={127511--127535},
    year={2024},
}

@inproceedings{Lipman:2023CFM,
    title={Flow Matching for Generative Modeling}, 
    author={Lipman, Yaron and Chen, Ricky T. Q. and Ben-Hamu, Heli and Nickel, Maximilian and Le, Matt},
    year={2023},
    booktitle=ICLR,
}

@inproceedings{Liu:2023RF,
    title={Flow Straight and Fast: Learning to Generate and Transfer Data with Rectified Flow},
    author={Liu, Xingchao and Gong, Chengyue and Liu, Qiang},
    booktitle=ICLR,
    year={2023},
}

@article{Huang:2025T2ICompBench++,
    title={{T2I-CompBench++}: An Enhanced and Comprehensive Benchmark for Compositional Text-to-Image Generation},
    author={Huang, Kaiyi and Duan, Chengqi and Sun, Kaiyue and Xie, Enze and Li, Zhenguo and Liu, Xihui},
    journal={IEEE Transactions on Pattern Analysis and Machine Intelligence},
    volume={47},
    number={5},
    pages={3563--3579},
    doi={10.1109/TPAMI.2025.3531907},
    year={2025},
}

@inproceedings{he2024mpgd,
    title={Manifold Preserving Guided Diffusion},
    author={Yutong He and Naoki Murata and Chieh-Hsin Lai and Yuhta Takida and Toshimitsu Uesaka and Dongjun Kim and Wei-Hsiang Liao and Yuki Mitsufuji and J Zico Kolter and Ruslan Salakhutdinov and Stefano Ermon},
    booktitle=ICLR,
    year={2024},
}

@inproceedings{tang2025dno,
  title={Inference-Time Alignment of Diffusion Models with Direct Noise Optimization},
  author={Tang, Zhiwei and Peng, Jiangweizhi and Tang, Jiasheng and Hong, Mingyi and Wang, Fan and Chang, Tsung-Hui},
  booktitle=ICML,
  pages={58905--58930},
  year={2025}
}

@article{hwang2026gradient,
  title={Gradient Preconditioning for Efficient and Reliable Reward-Guided Generation},
  author={Hwang, Jisung and Sung, Minhyuk},
  journal={arXiv preprint arXiv:2602.08646},
  year={2026}
}

@book{hyvarinen2001ica,
  title     = {Independent Component Analysis},
  author    = {Hyv{\"a}rinen, Aapo and Karhunen, Juha and Oja, Erkki},
  publisher = {John Wiley \& Sons},
  year      = {2001}
}

@article{kessy2018optimal,
  title   = {Optimal Whitening and Decorrelation},
  author  = {Kessy, Agnan and Lewin, Alex and Strimmer, Korbinian},
  journal = {The American Statistician},
  volume  = {72},
  number  = {4},
  pages   = {309--314},
  year    = {2018},
  doi     = {10.1080/00031305.2016.1277159}
}

@inproceedings{liu2025:videoalign,
  title={Improving Video Generation with Human Feedback},
  author={Liu, Jie and Liu, Gongye and Liang, Jiajun and Yuan, Ziyang and Liu, Xiaokun and Zheng, Mingwu and Wu, Xiele and Wang, Qiulin and Qin, Wenyu and Xia, Menghan and others},
  booktitle=NeurIPS,
  volume={38},
  pages={82155--82192},
  year={2025}
}

@inproceedings{huang2023:vbench,
     title={{VBench}: Comprehensive Benchmark Suite for Video Generative Models},
     author={Huang, Ziqi and He, Yinan and Yu, Jiashuo and Zhang, Fan and Si, Chenyang and Jiang, Yuming and Zhang, Yuanhan and Wu, Tianxing and Jin, Qingyang and Chanpaisit, Nattapol and Wang, Yaohui and Chen, Xinyuan and Wang, Limin and Lin, Dahua and Qiao, Yu and Liu, Ziwei},
     booktitle=CVPR,
     pages={21807--21818},
     year={2024}
 }

@inproceedings{wu2023:hpsv2,
  title={Human Preference Score v2: A Solid Benchmark for Evaluating Human Preferences of Text-to-Image Synthesis},
  author={Wu, Xiaoshi and Hao, Yiming and Sun, Keqiang and Chen, Yixiong and Zhu, Feng and Zhao, Rui and Li, Hongsheng},
  booktitle=ICLR,
  year={2024}
}

@inproceedings{lin2024:vqa,
      title={Evaluating Text-to-Visual Generation with Image-to-Text Generation},
      author={Zhiqiu Lin and Deepak Pathak and Baiqi Li and Jiayao Li and Xide Xia and Graham Neubig and Pengchuan Zhang and Deva Ramanan},
      year={2024},
      booktitle=ECCV,
      pages={366--384},
}

@inproceedings{qian2025:t2icount,
               title={T2ICount: Enhancing Cross-modal Understanding for Zero-Shot Counting},
               author={Qian, Yifei and Guo, Zhongliang and Deng, Bowen and Lei, Chun Tong and Zhao, Shuai and Lau, Chun Pong and Hong, Xiaopeng and Pound, Michael P},
               year={2025},
               booktitle=CVPR,
               pages={25336--25345},
}

@inproceedings{amininaieni2025:countgd,
      title={CountGD: Multi-Modal Open-World Counting},
      author={Amini-Naieni, Niki and Han, Tengda and Zisserman, Andrew},
      year={2024},
      booktitle=NeurIPS,
      volume={37},
      pages={48810--48837},
}

@article{li2026:mixgrpo,
      title={MixGRPO: Unlocking Flow-based GRPO Efficiency with Mixed ODE-SDE}, 
      author={Junzhe Li and Yutao Cui and Tao Huang and Yinping Ma and Chun Fan and Miles Yang and Zhao Zhong and Liefeng Bo},
      year={2025},
      journal={arXiv preprint arXiv:2507.21802}
}

@inproceedings{rout2023:psld,
  title={Solving linear inverse problems provably via posterior sampling with latent diffusion models},
  author={Rout, Litu and Raoof, Negin and Daras, Giannis and Caramanis, Constantine and Dimakis, Alex and Shakkottai, Sanjay},
  booktitle=NeurIPS,
  volume={36},
  pages={49960--49990},
  year={2023}
}

@inproceedings{rout2024:STSL,
  title={Beyond first-order tweedie: Solving inverse problems using latent diffusion},
  author={Rout, Litu and Chen, Yujia and Kumar, Abhishek and Caramanis, Constantine and Shakkottai, Sanjay and Chu, Wen-Sheng},
  booktitle=CVPR,
  pages={9472--9481},
  year={2024}
}

@inproceedings{luo2025:dual,
  title={Dual-process image generation},
  author={Luo, Grace and Granskog, Jonathan and Holynski, Aleksander and Darrell, Trevor},
  booktitle=ICCV,
  pages={17972--17983},
  year={2025}
}

@inproceedings{kwon2024:video,
  title={Solving video inverse problems using image diffusion models},
  author={Kwon, Taesung and Ye, Jong Chul},
  booktitle=ICLR,
  year={2025}
}

@inproceedings{lee2023:syncdiffusion,
  title={Syncdiffusion: Coherent montage via synchronized joint diffusions},
  author={Lee, Yuseung and Kim, Kunho and Kim, Hyunjin and Sung, Minhyuk},
  booktitle=NeurIPS,
  volume={36},
  pages={50648--50660},
  year={2023}
}

@inproceedings{rozet2024:moment_matching,
  title={Learning diffusion priors from observations by expectation maximization},
  author={Rozet, Fran{\c{c}}ois and Andry, G{\'e}r{\^o}me and Lanusse, Fran{\c{c}}ois and Louppe, Gilles},
  booktitle=NeurIPS,
  volume={37},
  pages={87647--87682},
  year={2024}
}

@article{cai2025:zimage,
  title={Z-image: An efficient image generation foundation model with single-stream diffusion transformer},
  author={Cai, Huanqia and Cao, Sihan and Du, Ruoyi and Gao, Peng and Hoi, Steven and Hou, Zhaohui and Huang, Shijie and Jiang, Dengyang and Jin, Xin and Li, Liangchen and others},
  journal={arXiv preprint arXiv:2511.22699},
  year={2025}
}

@inproceedings{liu2025:flowgrpo,
  title={Flow-{GRPO}: Training Flow Matching Models via Online {RL}},
  author={Liu, Jie and Liu, Gongye and Liang, Jiajun and Li, Yangguang and Liu, Jiaheng and Wang, Xintao and Wan, Pengfei and Zhang, Di and Ouyang, Wanli},
  booktitle=NeurIPS,
  volume={38},
  pages={40783--40818},
  year={2025}
}

@inproceedings{jang2026frameguidance,
      title={Frame Guidance: Training-Free Guidance for Frame-Level Control in Video Diffusion Models}, 
      author={Sangwon Jang and Taekyung Ki and Jaehyeong Jo and Jaehong Yoon and Soo Ye Kim and Zhe Lin and Sung Ju Hwang},
      year={2026},
      booktitle=ICLR
}

@inproceedings{zhang2024:controlvideo,
      title={ControlVideo: Training-free Controllable Text-to-Video Generation}, 
      author={Yabo Zhang and Yuxiang Wei and Dongsheng Jiang and Xiaopeng Zhang and Wangmeng Zuo and Qi Tian},
      year={2024},
      booktitle=ICLR 
}

@inproceedings{lu2024:freelong,
      title={FreeLong: Training-Free Long Video Generation with SpectralBlend Temporal Attention},
      author={Yu Lu and Yuanzhi Liang and Linchao Zhu and Yi Yang},
      year={2024},
      booktitle=NeurIPS,
      volume={37},
      pages={131434--131455}
}

@article{qiu2024:freetraj,
      title={FreeTraj: Tuning-Free Trajectory Control in Video Diffusion Models}, 
      author={Haonan Qiu and Zhaoxi Chen and Zhouxia Wang and Yingqing He and Menghan Xia and Ziwei Liu},
      year={2024},
      journal={arXiv preprint arXiv:2406.16863}
}

@inproceedings{hertz2022:p2p,
      title={Prompt-to-Prompt Image Editing with Cross Attention Control},
      author={Hertz, Amir and Mokady, Ron and Tenenbaum, Jay and Aberman, Kfir and Pritch, Yael and Cohen-Or, Daniel},
      booktitle=ICLR,
      year={2023}
}

@article{ma2025:omnipainter,
  title={Training-free Stylized Text-to-Image Generation with Fast Inference},
  author={Ma, Xin and Wang, Yaohui and Chen, Xinyuan and Wong, Tien-Tsin and Chen, Cunjian},
  journal={arXiv preprint arXiv:2505.19063},
  year={2025}}

@inproceedings{hertz2024:stylealigned,
      title={Style Aligned Image Generation via Shared Attention},
      author={Amir Hertz and Andrey Voynov and Shlomi Fruchter and Daniel Cohen-Or},
      year={2024},
      booktitle=CVPR,
      pages={4775--4785},
}

@article{chefer2023:attend,
  title={Attend-and-Excite: Attention-Based Semantic Guidance for Text-to-Image Diffusion Models},
  author={Chefer, Hila and Alaluf, Yuval and Vinker, Yael and Wolf, Lior and Cohen-Or, Daniel},
  journal=TOG,
  volume={42},
  number={4},
  pages={148:1--148:12},
  doi={10.1145/3592116},
  year={2023},
}

@inproceedings{kumari2023:custom,
  title={Multi-concept customization of text-to-image diffusion},
  author={Kumari, Nupur and Zhang, Bingliang and Zhang, Richard and Shechtman, Eli and Zhu, Jun-Yan},
  booktitle=CVPR,
  pages={1931--1941},
  year={2023}
}

@inproceedings{chung2022:mcg,
  title={Improving diffusion models for inverse problems using manifold constraints},
  author={Chung, Hyungjin and Sim, Byeongsu and Ryu, Dohoon and Ye, Jong Chul},
  booktitle=NeurIPS,
  volume={35},
  pages={25683--25696},
  year={2022}
}

@article{seo2026:stressdream,
  title={StressDream: Steering Video World Models for Robust Policy Evaluation and Improvement},
  author={Seo, Junwon and Veer, Sushant and Tian, Ran and Ding, Wenhao and Sharma, Apoorva and Leung, Karen and Schmerling, Edward and Pavone, Marco and Bajcsy, Andrea},
  journal={arXiv preprint arXiv:2606.00267},
  year={2026}
}

\ntappendix

\section{Reward-Guided Reverse Kernels}
\label{sec:kernel_app}

In this section, we provide derivations and interpretations for reward-guided reverse kernels.
\Cref{fig:method_illust} and \cref{tab:kernel_summary} summarize the four kernel types considered in this section: the base reverse kernel, mean-shifted reverse kernels, search-based guidance over the base kernel, and our noise-tilted reverse kernel (\Ours).
Among them, mean-shifted guidance injects reward information through the deterministic term, whereas search-based guidance and \Ours\ act through stochastic perturbations.

Before proceeding, we elaborate on the notion of \emph{noise-compatible} used in \cref{sec:method}.
By this, we mean the regime in which the reverse update remains centered at the base transition and is perturbed by typical noise-like perturbations, so that intermediate states stay on the manifold of the learned intermediate-state distribution, rather than drifting into regions insufficiently supported by the pretrained diffusion.
This shorthand captures the key intuition; it does \emph{not} mean that the induced transition exactly matches the pretrained reverse-kernel distribution.

With this clarification, the rest of the section is organized as follows.
In \cref{subsec:prelim_app}, we recall the optimal reward-tilted reverse kernel that serves as the common theoretical target and present continuous-time formulation of the base reverse kernel. 
In \cref{subsec:mean-shift_app}, we derive mean-shifted reverse kernels as a standard approximation obtained by modifying the deterministic term.
In \cref{subsec:search_app}, we revisit search-based guidance, which instead draws multiple candidates from the base reverse kernel and selects favorable perturbations according to reward.
In \cref{subsec:ours_app}, we interpret \Ours\ from both perspectives, showing how it combines the single-draw guidance advantage of mean-shifted methods with the noise-compatible-regime advantage of search-based guidance. 
Finally, in \cref{subsec:rho_app}, we present an intuitive interpretation of the guidance strength parameter. 

\subsection{Optimal Reward-Tilted and Base Reverse Kernel}
\label{subsec:prelim_app}

In this subsection, we first formulate the optimal reward-tilted reverse kernel, which serves as the target for maximizing expected rewards under a KL divergence penalty. Subsequently, we detail the continuous-time formulation of the base reverse kernel. 

\subsubsection{Optimal Reward-Tilted Reverse Kernel}
Starting from an entropy-regularized objective over reverse transition policies, we derive the optimal reward-tilted reverse kernel, the associated value function, and the target marginal distribution.
For a more rigorous theoretical treatment and comprehensive proofs, we refer readers to previous work~\cite{Uehara:2024Bridging}.

We consider the sequence of reverse transition policies $\{q(\bm{x}_{t-1}|\bm{x}_t)\}_{t=1}^{T}$ and optimize
\begin{equation}
    \label{eq:reward_rl_obj_app}
    \max_{q}\;
    \mathbb{E}_{\bm{x}_{0:T}\sim q(\bm{x}_{0:T})}
    \left[
        r(\bm{x}_{0})
        -
        \beta
        \sum_{t=1}^{T}
        \mathcal{D}_{KL}\!\left(
            q(\bm{x}_{t-1}|\bm{x}_{t})
             \|  
            p(\bm{x}_{t-1}|\bm{x}_{t})
        \right)
    \right],
\end{equation}
which maximizes the terminal reward while regularizing step-wise deviations from the pretrained base reverse kernel $p(\bm{x}_{t-1}|\bm{x}_t)$.
The temperature $\beta>0$ controls the reward-KL trade-off.

We define the expected return functional from state $\bm{x}_t$ as
\begin{equation}
    \label{eq:return_functional_app}
    J_t(\bm{x}_t;q)
    =
    \mathbb{E}_{q(\bm{x}_{0:t-1}|\bm{x}_t)}
    \left[
        r(\bm{x}_0)
        -
        \beta
        \sum_{\tau=1}^{t}
        \mathcal{D}_{KL}\!\left(
            q(\bm{x}_{\tau-1}|\bm{x}_{\tau})
             \|  
            p(\bm{x}_{\tau-1}|\bm{x}_{\tau})
        \right)
    \right].
\end{equation}
The value function is the optimal value of this return functional:
\begin{align}
    \label{eq:bellman_value_app}
    V_t(\bm{x}_t)
    &=
    \max_{q}\; J_t(\bm{x}_t;q)
    = J_t(\bm{x}_t;p^\star) \nonumber\\
    &=
    \max_{q}
    \left\{
        \mathbb{E}_{\bm{x}_{t-1}\sim q}
        \bigl[
            V_{t-1}(\bm{x}_{t-1})
        \bigr]
        -
        \beta
        \mathcal{D}_{KL}\!\left(
            q(\bm{x}_{t-1}|\bm{x}_t)
             \| 
            p(\bm{x}_{t-1}|\bm{x}_t)
        \right)
    \right\},
\end{align}
with terminal boundary condition $V_0(\bm{x}_0)=r(\bm{x}_0)$.

\paragraph{Optimal reward-tilted reverse kernel.}
Expanding the expectation and KL divergence in \cref{eq:bellman_value_app} gives
\begin{equation}
    \label{eq:bellman_integral_app}
    V_t(\bm{x}_t)
    =
    \max_{q}
    \int
    q(\bm{x}_{t-1}|\bm{x}_t)
    \left[
        V_{t-1}(\bm{x}_{t-1})
        -
        \beta
        \log
        \frac{
            q(\bm{x}_{t-1}|\bm{x}_t)
        }{
            p(\bm{x}_{t-1}|\bm{x}_t)
        }
    \right]
    d\bm{x}_{t-1}.
\end{equation}
Writing $V_{t-1}(\bm{x}_{t-1})=\beta\log\!\bigl(\exp(V_{t-1}(\bm{x}_{t-1})/\beta)\bigr)$, we obtain
\begin{equation}
    \label{eq:bellman_single_log_app}
    V_t(\bm{x}_t)
    =
    \max_{q}
    \int
    q(\bm{x}_{t-1}|\bm{x}_t)
    \beta
    \log
    \left(
        \frac{
            p(\bm{x}_{t-1}|\bm{x}_t)
            \exp(V_{t-1}(\bm{x}_{t-1})/\beta)
        }{
            q(\bm{x}_{t-1}|\bm{x}_t)
        }
    \right)
    d\bm{x}_{t-1}.
\end{equation}
Introducing the state-dependent normalizing constant
\begin{equation}
    \label{eq:partition_function_app}
    Z_t(\bm{x}_t)
    =
    \int
    p(\bm{x}_{t-1}|\bm{x}_t)
    \exp(V_{t-1}(\bm{x}_{t-1})/\beta)
    d\bm{x}_{t-1},
\end{equation}
and multiplying and dividing the logarithm argument by $Z_t(\bm{x}_t)$ yields
\begin{align}
    \label{eq:bellman_kl_form_app}
    V_t(\bm{x}_t)
    &=
    \max_{q}
    \int
    q(\bm{x}_{t-1}|\bm{x}_t)
    \beta
    \log
    \left(
        \frac{
            Z_t(\bm{x}_t)
        }{
            q(\bm{x}_{t-1}|\bm{x}_t)
        }
        \cdot
        \frac{
            p(\bm{x}_{t-1}|\bm{x}_t)\exp(V_{t-1}(\bm{x}_{t-1})/\beta)
        }{
            Z_t(\bm{x}_t)
        }
    \right)
    d\bm{x}_{t-1} \nonumber\\
    &=
    \beta\log Z_t(\bm{x}_t)
    -
    \beta
    \min_{q}
    \mathcal{D}_{KL}
    \left(
        q(\bm{x}_{t-1}|\bm{x}_t)
         \bigg\| 
        \frac{
            p(\bm{x}_{t-1}|\bm{x}_t)\exp(V_{t-1}(\bm{x}_{t-1})/\beta)
        }{
            Z_t(\bm{x}_t)
        }
    \right).
\end{align}
Since the KL divergence is non-negative and vanishes if and only if the two distributions coincide, the optimum is uniquely attained at
\begin{equation}
    \label{eq:optimal_prop_app}
    p^\star(\bm{x}_{t-1}|\bm{x}_t)
    =
    \frac{
        p(\bm{x}_{t-1}|\bm{x}_t)\exp(V_{t-1}(\bm{x}_{t-1})/\beta)
    }{
        \int
        p(\bm{x}_{t-1}|\bm{x}_t)\exp(V_{t-1}(\bm{x}_{t-1})/\beta)
        d\bm{x}_{t-1}
    }.
\end{equation}
This is the optimal reward-tilted reverse kernel.

Evaluating the optimum further yields the relation
\begin{align}
    \label{eq:value_partition_app}
    \exp(V_t(\bm{x}_t)/\beta)
    &=
    \int
    p(\bm{x}_{t-1}|\bm{x}_t)\exp(V_{t-1}(\bm{x}_{t-1})/\beta)
    d\bm{x}_{t-1} \nonumber\\
    &=
    \mathbb{E}_{p}
    \left[
        \exp(V_{t-1}(\bm{x}_{t-1})/\beta)
         \middle| 
        \bm{x}_t
    \right]
    = \hdots =
    \mathbb{E}_{p}
    \left[
        \exp(r(\bm{x}_0)/\beta)
         \middle| 
        \bm{x}_t
    \right].
\end{align}

\paragraph{Value approximation.}
In practice, directly evaluating \cref{eq:value_partition_app} is intractable.
Following previous work~\cite{Li2024:SVDD}, we approximate the posterior expectation using the Tweedie estimate $\hat{\bm{x}}_{0|t}:=\mathbb{E}[\bm{x}_0|\bm{x}_t]$:
\begin{align}
    \label{eq:value_approx_app}
    V_t(\bm{x}_t)
    &=
    \beta\log
    \mathbb{E}_{\bm{x}_0\sim p(\cdot|\bm{x}_t)}
    \left[
        \exp(r(\bm{x}_0)/\beta)
    \right] \nonumber\\
    &\approx
    \beta\log
    \left(
        \exp(r(\hat{\bm{x}}_{0|t})/\beta)
    \right)
    =
    r(\hat{\bm{x}}_{0|t}).
\end{align}

\paragraph{Target marginal distribution.}
Sampling from the optimal step-wise proposal induces the following target marginal at timestep $t$:
\begin{align}
    \label{eq:target_marginal_app}
    p^\star(\bm{x}_t)
    &=
    \int
    p^\star(\bm{x}_{t:T})
    d\bm{x}_{t+1:T} \nonumber\\
    &=
    \frac{
        \exp(V_t(\bm{x}_t)/\beta)
    }{
        Z
    }
    \int
    p(\bm{x}_{t:T})
    d\bm{x}_{t+1:T} \nonumber\\
    &=
    \frac{1}{Z} 
    p(\bm{x}_t)\exp(V_t(\bm{x}_t)/\beta).
\end{align}
Likewise,
\begin{align}
    \label{eq:target_marginal_prev_app}
    p^\star(\bm{x}_{t-1})
    &=
    \int
    p^\star(\bm{x}_{t-1}|\bm{x}_t)
    p^\star(\bm{x}_t)
    d\bm{x}_t \nonumber\\
    &=
    \int
    \left(
        p(\bm{x}_{t-1}|\bm{x}_t)
        \frac{
            \exp(V_{t-1}(\bm{x}_{t-1})/\beta)
        }{
            \exp(V_t(\bm{x}_t)/\beta)
        }
    \right)
    \left(
        \frac{1}{Z}p(\bm{x}_t)\exp(V_t(\bm{x}_t)/\beta)
    \right)
    d\bm{x}_t \nonumber\\
    &=
    \frac{
        \exp(V_{t-1}(\bm{x}_{t-1})/\beta)
    }{
        Z
    }
    \int
    p(\bm{x}_{t-1}|\bm{x}_t)p(\bm{x}_t)
    d\bm{x}_t \nonumber\\
    &=
    \frac{1}{Z} 
    p(\bm{x}_{t-1})\exp(V_{t-1}(\bm{x}_{t-1})/\beta).
\end{align}
Recursively applying this relation down to $t=0$ recovers
\begin{equation}
    \label{eq:terminal_target_marginal_app}
    p^\star(\bm{x}_0)
    =
    \frac{1}{Z} 
    p(\bm{x}_0)\exp(r(\bm{x}_0)/\beta).
\end{equation}

\subsubsection{Base Reverse Kernel}
We connect our discrete-time formulation to the continuous-time sampling processes used in score-based generative models. 

Let $\{p_t\}_{0\le t\le T}$ be a probability path interpolating between a tractable noise prior and the data distribution:
\begin{equation}
    \label{eq:stochastic_interpolant_app}
    \bm{x}_t
    =
    \alpha_t \bm{x}_0 + \sigma_t \bm{x}_T,
\end{equation}
where $\alpha_t$ and $\sigma_t$ are smooth monotone schedules satisfying $\alpha_0=\sigma_T=1$ and $\alpha_T=\sigma_0=0$.

Flow-based models~\cite{Lipman:2023CFM,Liu:2023RF} parameterize this path using a time-dependent velocity field $\bm{u}_t:\mathbb{R}^N\rightarrow\mathbb{R}^N$, leading to the probability flow ODE:
\begin{equation}
    \label{eq:pf_ode_app}
    \mathrm{d}\bm{x}_t
    =
    \bm{u}_t(\bm{x}_t) \mathrm{d}t,
    \qquad
    \bm{x}_T\sim\mathcal{N}(\bm{0},\bm{I}).
\end{equation}
Adding stochastic exploration yields the reverse-time SDE:
\begin{equation}
    \label{eq:reverse_sde_app}
    \mathrm{d}\bm{x}_t = \bm{f}_t(\bm{x}_t) \mathrm{d}t + \sigma_t \mathrm{d}\bm{w},
    \qquad
    \bm{f}_t(\bm{x}_t) = \bm{u}_t(\bm{x}_t) - \frac{\sigma_t^2}{2}\nabla \log p_t(\bm{x}_t),
\end{equation}
where $\bm{w}$ denotes the standard Wiener process. Discretizing \cref{eq:reverse_sde_app} backward in time yields the base reverse kernel:
\begin{equation}
    \label{eq:base_reverse_kernel_app}
    \bm{x}_{t-1}
    =
    \bm{\mu}_\theta(\bm{x}_t,t)
    +
    \sigma_t \bm{\epsilon}_t,
    \qquad
    \bm{\epsilon}_t\sim\mathcal{N}(\bm{0},\bm{I}).
\end{equation}

\subsection{Approach 1: Mean-Shifted Reverse Kernel}
\label{subsec:mean-shift_app}

A standard approach to approximating the optimal reward-tilted reverse kernel is to inject the reward information directly into the deterministic component of the reverse update.
We first establish the exact relationship between the marginal velocity and the marginal score, then derive the mean-shifted reverse kernel by substituting the tilted score into the reverse SDE.

\paragraph{Relation between velocity and score.}
To systematically modify this base kernel with a reward signal, we must adjust the reverse drift $\bm{f}_t(\bm{x}_t)$ in \cref{eq:reverse_sde_app}. 
Since this drift depends on both the marginal velocity $\bm{u}_t(\bm{x}_t)$ and the marginal score $\nabla_{\bm{x}_t}\log p_t(\bm{x}_t)$, we first establish their exact relationship. By definition, the marginal velocity is given as:
\begin{equation}
    \label{eq:marginal_velocity_def}
    \bm{u}_t(\bm{x}_t) = \mathbb{E}[\dot{\alpha}_t \bm{x}_0 + \dot{\sigma}_t \bm{x}_T \mid \bm{x}_t] = \dot{\alpha}_t \mathbb{E}[\bm{x}_0 \mid \bm{x}_t] + \dot{\sigma}_t \mathbb{E}[\bm{x}_T \mid \bm{x}_t].
\end{equation}
Using the standard score identity for the Gaussian transition kernel $p_t(\bm{x}_t \mid \bm{x}_0) = \mathcal{N}(\bm{x}_t; \alpha_t \bm{x}_0, \sigma_t^2 \bm{I})$, we can express the marginal score as:
\begin{equation}
    \label{eq:marginal_score_identity}
    \nabla_{\bm{x}_t} \log p_t(\bm{x}_t) = \mathbb{E} \left[ -\frac{\bm{x}_t - \alpha_t \bm{x}_0}{\sigma_t^2} \;\middle|\; \bm{x}_t \right] = -\frac{\bm{x}_t - \alpha_t \mathbb{E}[\bm{x}_0 \mid \bm{x}_t]}{\sigma_t^2}.
\end{equation}
Rearranging this expression yields the optimal denoised estimate for $\bm{x}_0$:
\begin{equation}
    \label{eq:x0_posterior}
    \mathbb{E}[\bm{x}_0 \mid \bm{x}_t] = \frac{1}{\alpha_t}\bm{x}_t + \frac{\sigma_t^2}{\alpha_t} \nabla_{\bm{x}_t} \log p_t(\bm{x}_t).
\end{equation}
Similarly, since the forward process is defined as $\bm{x}_t = \alpha_t \bm{x}_0 + \sigma_t \bm{x}_T$, the expected noise can be written directly in terms of the score:
\begin{equation}
    \label{eq:xT_posterior}
    \mathbb{E}[\bm{x}_T \mid \bm{x}_t] = \frac{\bm{x}_t - \alpha_t \mathbb{E}[\bm{x}_0 \mid \bm{x}_t]}{\sigma_t} = -\sigma_t \nabla_{\bm{x}_t} \log p_t(\bm{x}_t).
\end{equation}
Substituting these posterior expectations back into the marginal velocity definition (\cref{eq:marginal_velocity_def}), we arrive at the exact relationship:
\begin{align}
    \bm{u}_t(\bm{x}_t) 
    &= \dot{\alpha}_t \left( \frac{1}{\alpha_t}\bm{x}_t + \frac{\sigma_t^2}{\alpha_t} \nabla_{\bm{x}_t} \log p_t(\bm{x}_t) \right) - \dot{\sigma}_t \sigma_t \nabla_{\bm{x}_t} \log p_t(\bm{x}_t) \nonumber \\
    \label{eq:marginal_velocity_score_app}
    &= \frac{\dot{\alpha}_t}{\alpha_t} \bm{x}_t - \underbrace{\left( \dot{\sigma}_t \sigma_t - \sigma_t^2 \frac{\dot{\alpha}_t}{\alpha_t} \right)}_{=:C_t} \nabla_{\bm{x}_t} \log p_t(\bm{x}_t).
\end{align} 

\paragraph{Mean-Shifted Reverse Kernel.}
With this exact relationship established, we can now inject the reward signal into the reverse generative dynamics. Starting from the optimal target marginal in \cref{eq:target_marginal_app}, we decompose its score as follows:
\begin{align}
    \label{eq:tilted_score_app}
    \nabla_{\bm{x}_t}\log p^\star(\bm{x}_t)
    &=
    \nabla_{\bm{x}_t}
    \log
    \left(
        \frac{1}{Z} 
        p_t(\bm{x}_t)\exp(V_t(\bm{x}_t)/\beta)
    \right) \nonumber\\
    &=
    \nabla_{\bm{x}_t}\log p_t(\bm{x}_t)
    +
    \frac{1}{\beta}\nabla_{\bm{x}_t}V_t(\bm{x}_t).
\end{align}
This demonstrates that the optimal tilted score is the pretrained score shifted by the temperature-scaled value gradient. 

Next, we substitute this tilted score alongside the corresponding tilted velocity, $\bm{u}^\star_t(\bm{x}_t) = \bm{u}_t(\bm{x}_t) - \frac{C_t}{\beta} \nabla_{\bm{x}_t}V_t(\bm{x}_t)$, into the reverse SDE (\cref{eq:reverse_sde_app}) to obtain the modified reverse drift:
\begin{align} 
    \label{eq:tilted_drift_app} 
    \bm{f}_t^\star(\bm{x}_t) 
    &= \bm{u}^\star_t(\bm{x}_t) - \frac{\sigma_t^2}{2} \nabla_{\bm{x}_t}\log p^\star(\bm{x}_t) \nonumber\\ 
    &= \left( \bm{u}_t(\bm{x}_t) - \frac{C_t}{\beta} \nabla_{\bm{x}_t}V_t(\bm{x}_t) \right) - \frac{\sigma_t^2}{2} \left( \nabla_{\bm{x}_t}\log p_t(\bm{x}_t) + \frac{1}{\beta}\nabla_{\bm{x}_t}V_t(\bm{x}_t) \right) \nonumber\\ 
    &= \bm{f}_t(\bm{x}_t) - \frac{C_t + \frac{\sigma_t^2}{2}}{\beta} \nabla_{\bm{x}_t}V_t(\bm{x}_t).
\end{align}
Because the exact value function gradient is analytically intractable, we employ the approximation from \cref{eq:value_approx_app}:
\begin{equation}
    \label{eq:value_grad_approx_app}
    \nabla_{\bm{x}_t}V_t(\bm{x}_t)
    \approx
    \nabla_{\bm{x}_t}r(\hat{\bm{x}}_{0|t}).
\end{equation}
We absorb the schedule-dependent coefficient $\bigl(C_t + \sigma_t^2/2\bigr)/\beta$ into a single guidance hyperparameter $\lambda_t$.
Discretizing the modified reverse SDE then gives the mean-shifted reverse kernel:
\begin{equation}
    \label{eq:mean_shift_kernel_app}
    \bm{x}_{t-1}
    =
    \bm{\mu}_\theta(\bm{x}_t, t)
    +
    \lambda_t\nabla_{\bm{x}_t} r(\hat{\bm{x}}_{0|t})
    +
    \sigma_t \bm{\epsilon}_t,
    \qquad
    \bm{\epsilon}_t \sim \mathcal{N}(\bm{0}, \bm{I}).
\end{equation}
While the derivation prescribes $\lambda_t = \bigl(C_t + \sigma_t^2/2\bigr)/\beta$, in practice $\lambda_t$ is treated as an independent hyperparameter.
This corresponds to the mean-shifted reverse kernel introduced in \cref{subsec:grad_kernel}.

\subsection{Approach 2: Search-Based Reverse Kernel}
\label{subsec:search_app}

In contrast to modifying the deterministic drift, an alternative approximation strategy draws multiple candidate perturbations from the base reverse kernel and biases the final selection toward those yielding higher rewards. 
Starting from the optimal reward-tilted reverse kernel in \cref{eq:optimal_prop_app}, we can approximate it empirically using $K$ candidate samples drawn from the base reverse kernel:
\begin{align}
    \label{eq:importance_sampling_app}
    p_\theta^\star(\bm{x}_{t-1}|\bm{x}_t)
    &\approx
    \sum_{i=1}^{K}
    \frac{
        w_{t-1}^{(i)}
    }{
        \sum_{j=1}^{K} w_{t-1}^{(j)}
    }
    \delta_{\bm{x}_{t-1}^{(i)}}(\bm{x}_{t-1}), \\
    \bm{x}_{t-1}^{(i)}
    &= 
    \bm{\mu}_\theta(\bm{x}_t,t)
    +
    \sigma_t \bm{\epsilon}_t^{(i)},
    \qquad
    \bm{\epsilon}_t^{(i)}\sim\mathcal{N}(\bm{0},\bm{I}), \nonumber\\
    w_{t-1}^{(i)}
    &=
    \exp\!\left(
        V_{t-1}(\bm{x}_{t-1}^{(i)})/\beta
    \right). \nonumber
\end{align}
Here, each candidate is generated from the same base reverse kernel, and the reward tilt appears only through the selection weights.

Using the value approximation in \cref{eq:value_approx_app}, the weights can be approximated as
\begin{equation}
    \label{eq:search_weight_approx_app}
    w_{t-1}^{(i)}
    \approx
    \exp\!\left(
        r(\hat{\bm{x}}_{0|t-1}^{(i)})/\beta
    \right),
\end{equation}
where $\hat{\bm{x}}_{0|t-1}^{(i)}$ denotes the Tweedie estimate associated with the $i$-th candidate.
This gives a search-based approximation to the optimal reward-tilted reverse kernel by favoring candidates with larger estimated terminal reward.

In practice, a common hard-selection variant replaces stochastic resampling by selecting the maximum-weight candidate.
In particular, SVDD~\cite{Li2024:SVDD} can be interpreted as the argmax form of this search procedure:
\begin{align}
    \label{eq:svdd_argmax_app}
    i^\star
    &=
    \underset{i\in\{1,\dots,K\}}{\arg\max}\,
    r(\hat{\bm{x}}_{0|t-1}^{(i)}), \\
    \bm{x}_{t-1}
    &=
    \bm{x}_{t-1}^{(i^\star)}. \nonumber
\end{align}

Conditioned on $\bm{x}_t$, search-based guidance samples multiple perturbations around the same base center $\bm{\mu}_\theta(\bm{x}_t,t)$ and selects one according to reward.
This can be interpreted as inducing a biased distribution over the perturbation variable itself, while preserving the base reverse-kernel sampling form during candidate generation.
Under a local linear reward model around the base reverse update, the induced perturbation bias is aligned with the reward-gradient direction, so search-based guidance may be viewed as an implicit form of noise tilting, as illustrated in \cref{fig:method_illust}.
This observation provides a natural bridge to our \Ours, which realizes reward guidance explicitly through a single tilted perturbation.

\subsection{Best of Both Worlds: \Ours{}}
\label{subsec:ours_app}

We now interpret \Ours\ by connecting the two approximation routes discussed above.
From the perspective of mean-shifted reverse kernels, our goal is to retain the guidance advantage of first-order reward information without modifying the deterministic term of the reverse update.
From the perspective of search-based guidance, our goal is to retain the noise-compatible-regime advantage of sampling around the base update center, while avoiding the multi-draw search cost.
These two viewpoints meet at the same principle: reward information should be injected through the perturbation variable rather than through a shift of the deterministic term.

Accordingly, \Ours\ keeps the base reverse update center $\bm{\mu}_\theta(\bm{x}_t,t)$ and replaces the standard Gaussian perturbation by a reward-tilted perturbation:
\begin{equation}
    \label{eq:ntrk_kernel_app}
    \bm{x}_{t-1}
    =
    \bm{\mu}_\theta(\bm{x}_t,t)
    +
    \sigma_t \tilde{\bm{\epsilon}}_t,
\end{equation}
where
\begin{equation}
    \label{eq:ntrk_noise_app}
    \tilde{\bm{\epsilon}}_t
    =
    \sqrt{\rho_t}\,\mathcal{W}\!\left(\nabla_{\bm{x}_t} r(\hat{\bm{x}}_{0|t})\right)
    +
    \sqrt{1-\rho_t}\,\bm{\epsilon}_t,
    \qquad
    \bm{\epsilon}_t\sim\mathcal{N}(\bm{0},\bm{I}).
\end{equation}
Here, $\rho_t\in[0,1]$ controls the strength of the reward-informed perturbation, and $\mathcal{W}(\cdot)$ denotes the whitening operator introduced in \cref{sec:method} and detailed in \cref{app:whitening_operator}.
The role of $\mathcal{W}(\cdot)$ is to map the reward gradient to a typical noise-like vector, so that the resulting update remains in the noise-compatible regime.

\paragraph{From the mean-shift perspective.}
Comparing \cref{eq:ntrk_kernel_app,eq:ntrk_noise_app} with the mean-shifted reverse kernel in \cref{eq:mean_shift_kernel_app}, both methods exploit the same local first-order reward signal $\nabla_{\bm{x}_t} r(\hat{\bm{x}}_{0|t})$.
However, the way this signal enters the reverse update is fundamentally different.
Mean-shifted guidance adds the reward gradient directly to the deterministic term, thereby changing the center of the reverse update.
In contrast, \Ours\ preserves the deterministic term $\bm{\mu}_\theta(\bm{x}_t,t)$ and injects reward information only through the perturbation.
Thus, \Ours\ may be viewed as a mean-preserving reformulation of first-order reward guidance: it retains the single-draw efficiency of mean-shifted guidance while avoiding the explicit center shift that moves the update away from the noise-compatible regime.

\paragraph{From the search-based perspective.}
Comparing \cref{eq:ntrk_kernel_app,eq:ntrk_noise_app} with the search-based update in \cref{eq:importance_sampling_app,eq:svdd_argmax_app}, both methods generate reward-guided updates around the same base center $\bm{\mu}_\theta(\bm{x}_t,t)$.
Search-based guidance does so implicitly by drawing multiple perturbations from the base reverse kernel and selecting the most favorable one according to reward.
As discussed in \cref{subsec:search_app}, under a local linear assumption, this selection induces a bias over the perturbation variable aligned with the reward-gradient direction.
From this viewpoint, \Ours\ can be interpreted as an explicit single-draw realization of that idea: instead of performing a $K$-candidate search and selecting a favorable perturbation afterward, we directly construct a reward-tilted perturbation through \cref{eq:ntrk_noise_app}.

Therefore, \Ours\ combines the key advantages of both existing viewpoints.
Unlike mean-shifted reverse kernels, it stays within the noise-compatible regime by preserving the base update center.
Unlike search-based guidance, it achieves reward alignment with a single draw per step.
In this sense, \Ours\ bridges the two approaches by realizing reward guidance explicitly through a single typical noise-like perturbation.

\subsection{Interpretation of the Guidance Strength $\rho_t$ in \Ours}
\label{subsec:rho_app}

We now give an intuitive interpretation of the guidance strength $\rho_t$ by comparing \Ours\ with search-based guidance.
Our goal is not to claim an exact equivalence, but to clarify how increasing $\rho_t$ changes the effective strength of reward guidance relative to search-based methods.
The discussion below is heuristic and relies on a local linearity assumption around the base reverse update.
Although this assumption need not hold exactly, it reflects the common working belief behind gradient-based guidance methods: the reward gradient often provides a meaningful direction not only infinitesimally, but also over a practically relevant local neighborhood.

Fix $\bm{x}_t$ and consider perturbations around the base reverse update center $\bm{\mu}_\theta(\bm{x}_t,t)$.
For a perturbation $\bm{\epsilon}$, define
\begin{equation}
    \label{eq:reverse_update_eps_app}
    \bm{x}_{t-1}(\bm{\epsilon})
    :=
    \bm{\mu}_\theta(\bm{x}_t,t) + \sigma_t \bm{\epsilon},
\end{equation}
and let $\hat{\bm{x}}_{0|t-1}(\bm{\epsilon})$ denote the Tweedie estimate obtained from $\bm{x}_{t-1}(\bm{\epsilon})$.

To connect this interpretation with the actual direction used in practice, we approximate the local reward-improving perturbation direction around $\bm{\epsilon}=\bm{0}$ by the normalized reward gradient at the current state:
\begin{equation}
    \label{eq:ut_def_app}
    \bm{u}_t
    :=
    \frac{\nabla_{\bm{x}_t} r(\hat{\bm{x}}_{0|t})}
    {\|\nabla_{\bm{x}_t} r(\hat{\bm{x}}_{0|t})\|_2},
\end{equation}
whenever $\nabla_{\bm{x}_t} r(\hat{\bm{x}}_{0|t}) \neq \bm{0}$.
Under this approximation, the reward around the base reverse update is modeled by the first-order form
\begin{equation}
    \label{eq:local_linearity_reward_app}
    r\!\left(\hat{\bm{x}}_{0|t-1}(\bm{\epsilon})\right)
    \approx
    c_t + \kappa_t \langle \bm{u}_t, \bm{\epsilon} \rangle,
\end{equation}
where
\begin{equation}
    \label{eq:ct_def_app}
    c_t := r\!\left(\hat{\bm{x}}_{0|t-1}(\bm{0})\right),
\end{equation}
and $\kappa_t > 0$ is a local sensitivity coefficient.
Thus, under this local approximation, candidate ranking depends only on the one-dimensional projected perturbation coordinate
\begin{equation}
    \label{eq:projected_coordinate_app}
    z := \langle \bm{u}_t, \bm{\epsilon} \rangle.
\end{equation}

\paragraph{Search-based perspective.}
As discussed in \cref{subsec:search_app}, search-based guidance draws $K$ candidates from the base reverse kernel and performs importance sampling according to their estimated rewards.
Under \cref{eq:local_linearity_reward_app}, the importance weight of the $i$-th candidate satisfies
\begin{equation}
    \label{eq:search_linear_weight_app}
    w_i
    \propto
    \exp\!\left(
        r\!\left(\hat{\bm{x}}_{0|t-1}(\bm{\epsilon}_t^{(i)})\right)
    \right)
    \approx
    \exp(c_t)\exp(\kappa_t z_i),
\end{equation}
where
\begin{equation}
    \label{eq:search_projected_coordinate_app}
    z_i
    :=
    \langle \bm{u}_t, \bm{\epsilon}_t^{(i)} \rangle,
    \qquad
    \bm{\epsilon}_t^{(i)} \sim \mathcal{N}(\bm{0}, \bm{I}).
\end{equation}
Since the base perturbation is isotropic and $\|\bm{u}_t\|_2=1$, each $z_i$ follows $\mathcal{N}(0,1)$.
Thus, search-based guidance increasingly favors candidates with larger projected coordinate $z_i$.
In the hard-selection form used by SVDD~\cite{Li2024:SVDD}, this reduces to selecting
\begin{equation}
    \label{eq:search_projected_max_app}
    z_K^\star = \max_{1 \le i \le K} z_i.
\end{equation}
Define
\begin{equation}
    \label{eq:mK_def_app}
    m_K
    :=
    \mathbb{E}[z_K^\star]
    =
    \mathbb{E}\!\left[\max_{1 \le i \le K} z_i\right].
\end{equation}
Then larger $K$ yields a larger average displacement along the reward-improving direction.
For large $K$, the standard extreme-value approximation gives
\begin{equation}
    \label{eq:mK_approx_app}
    m_K \approx \sqrt{2\log K}.
\end{equation}

\paragraph{NTRK perspective.}
For \Ours, the perturbation is
\begin{equation}
    \label{eq:ntrk_noise_rho_app}
    \tilde{\bm{\epsilon}}_t
    =
    \sqrt{\rho_t}\,\bm{w}_t
    +
    \sqrt{1-\rho_t}\,\bm{\epsilon}_t,
    \qquad
    \bm{\epsilon}_t \sim \mathcal{N}(\bm{0},\bm{I}),
\end{equation}
where
\begin{equation}
    \label{eq:ntrk_whitened_dir_app}
    \bm{w}_t
    :=
    \mathcal{W}\!\left(\nabla_{\bm{x}_t} r(\hat{\bm{x}}_{0|t})\right).
\end{equation}
Projecting onto the same reward-aligned unit direction $\bm{u}_t$ gives
\begin{align}
    \label{eq:ntrk_projected_coordinate_app}
    \tilde{z}_t
    :=
    \langle \bm{u}_t,\tilde{\bm{\epsilon}}_t\rangle
    &=
    \sqrt{\rho_t}\,
    \left\langle
        \bm{u}_t,
        \mathcal{W}\!\left(\nabla_{\bm{x}_t} r(\hat{\bm{x}}_{0|t})\right)
    \right\rangle
    +
    \sqrt{1-\rho_t}\,
    \langle \bm{u}_t,\bm{\epsilon}_t\rangle \nonumber\\
    &=
    \sqrt{\rho_t}\,a_t
    +
    \sqrt{1-\rho_t}\,z_t,
\end{align}
where
\begin{equation}
    \label{eq:ntrk_alignment_coeff_app}
    a_t
    :=
    \left\langle
        \bm{u}_t,
        \mathcal{W}\!\left(\nabla_{\bm{x}_t} r(\hat{\bm{x}}_{0|t})\right)
    \right\rangle,
    \qquad
    z_t
    :=
    \langle \bm{u}_t,\bm{\epsilon}_t\rangle.
\end{equation}
Since $\|\bm{u}_t\|_2=1$ and $\bm{\epsilon}_t\sim\mathcal{N}(\bm{0},\bm{I})$, we have $z_t\sim\mathcal{N}(0,1)$, and therefore
\begin{equation}
    \label{eq:ntrk_projected_stats_app}
    \mathbb{E}[\tilde{z}_t]
    =
    \sqrt{\rho_t}\,a_t.
\end{equation}

The key quantity is thus the alignment coefficient $a_t$.
It measures how strongly the whitened reward-informed perturbation projects onto the locally reward-improving direction $\bm{u}_t$.
As discussed later in \cref{app:whitening_operator}, the whitening operator $\mathcal{W}$ is designed to preserve a meaningful cosine similarity with the original reward gradient.
Consequently, even in very high dimension, $a_t$ can remain substantial.

To see this, suppose the cosine similarity between $\bm{w}_t := \mathcal{W}(\nabla_{\bm{x}_t} r(\hat{\bm{x}}_{0|t}))$ and $\bm{u}_t$ is around $0.1$.
When $N=65536$, we typically have $\|\bm{w}_t\|_2 \approx \sqrt{N} = 256$, so
\begin{equation}
    \label{eq:at_example_app}
    a_t
    =
    \langle \bm{u}_t,\bm{w}_t\rangle
    \approx
    0.1 \times 256
    =
    25.6.
\end{equation}
Thus, even a seemingly small value of $\rho_t$ can already induce a strong reward-aligned bias through the factor $\sqrt{\rho_t}\,a_t$.

Under this interpretation, a local heuristic correspondence between search-based guidance and \Ours\ is obtained by matching the average projected displacement:
\begin{equation}
    \label{eq:rho_k_match_app}
    \sqrt{\rho_t}\,a_t \approx m_K.
\end{equation}
Using the large-$K$ approximation in \cref{eq:mK_approx_app}, this gives
\begin{equation}
    \label{eq:rho_k_match_exp_app}
    K_{\mathrm{eff}}
    \approx
    \exp\!\left(\frac{\rho_t a_t^2}{2}\right).
\end{equation}
Although this correspondence is only heuristic, it clarifies the role of $\rho_t$: increasing $\rho_t$ rapidly increases the effective search strength.

For example, taking the value $a_t=25.6$ from \cref{eq:at_example_app} gives
\begin{align}
    \rho_t = 0.01
    &\quad\Longrightarrow\quad
    \sqrt{\rho_t}\,a_t = 2.56,
    \qquad
    K_{\mathrm{eff}}
    \approx
    \exp(3.28)
    \approx
    2.6\times 10^{1}, \nonumber\\
    \rho_t = 0.10
    &\quad\Longrightarrow\quad
    \sqrt{\rho_t}\,a_t = 8.10,
    \qquad
    K_{\mathrm{eff}}
    \approx
    \exp(32.77)
    \approx
    1.7\times 10^{14}. \label{eq:rho_examples_app}
\end{align}
Therefore, once the preserved directional alignment is taken into account, even moderate values of $\rho_t$ can correspond to an astronomically large effective search strength that would be infeasible for natural search-based methods.

\begin{figure*}[t]
\centering

\newlength{\ablationpairheight}
\setlength{\ablationpairheight}{6.75cm}

\begin{minipage}[t][\ablationpairheight][t]{0.60\textwidth}
    \centering
    \captionof{table}{
    \footnotesize
    \textbf{Ablation on the guidance strength $\rho_t$ for \Ours\ on aesthetic image generation.}
    We fix the sampling configuration (NFE = 25) and vary only $\rho_t$.
    Dark green and light green denote the best and second-best results, respectively.
    }
    \label{tab:rho_ablation_aesthetic}

    \vspace{0.0cm}
    \renewcommand{\arraystretch}{1.10}
    \setlength{\tabcolsep}{4.5pt}
    \scriptsize

    \resizebox{\linewidth}{!}{%
    \begin{tabular}{l c c c c c}
    \toprule
    \multirow{2}{*}{\textbf{$\rho_t$}}
    & \multicolumn{1}{c}{\makecell{\textbf{Target}\\\textbf{Reward}}}
    & \multicolumn{4}{c}{\textbf{Held-out Rewards}} \\
    \cmidrule(lr){2-2}\cmidrule(lr){3-6}
    & \makecell{\textbf{Aesthetic}\\\textbf{Score}} $\uparrow$
    & \makecell{\textbf{Pick}\\\textbf{-score}} $\uparrow$
    & \textbf{HPSv2} $\uparrow$
    & \makecell{\textbf{Image}\\\textbf{Reward}} $\uparrow$
    & \makecell{\textbf{VQA}\\\textbf{Score}} $\uparrow$ \\
    \midrule
    0.01 & \textbf{6.5127} & \cellcolor{lightgreen}0.2205 & \cellcolor{darkgreen}0.2979 & 1.2773 & 0.9724 \\
    0.03 & \textbf{6.7445} & \cellcolor{darkgreen}0.2212 & \cellcolor{lightgreen}0.2977 & \cellcolor{lightgreen}1.3135 & 0.9687 \\
    0.10 & \textbf{7.1128} & \cellcolor{lightgreen}0.2205 & 0.2959 & 1.2431 & \cellcolor{darkgreen}0.9730 \\
    0.20 & \textbf{7.3265} & 0.2204 & 0.2976 & \cellcolor{darkgreen}1.3386 & 0.9716 \\
    0.30 & \textbf{7.4510} & 0.2200 & 0.2928 & 1.2565 & \cellcolor{lightgreen}0.9728 \\
    0.50 & \cellcolor{lightgreen}\textbf{7.5892} & 0.2177 & 0.2916 & 1.1592 & 0.9674 \\
    1.00 & \cellcolor{darkgreen}\textbf{7.7571} & 0.2142 & 0.2789 & 0.9576 & 0.9492 \\
    \bottomrule
    \end{tabular}%
    }
\end{minipage}
\hfill
\begin{minipage}[t][\ablationpairheight][b]{0.37\textwidth}
    \centering
    \includegraphics[width=\linewidth]{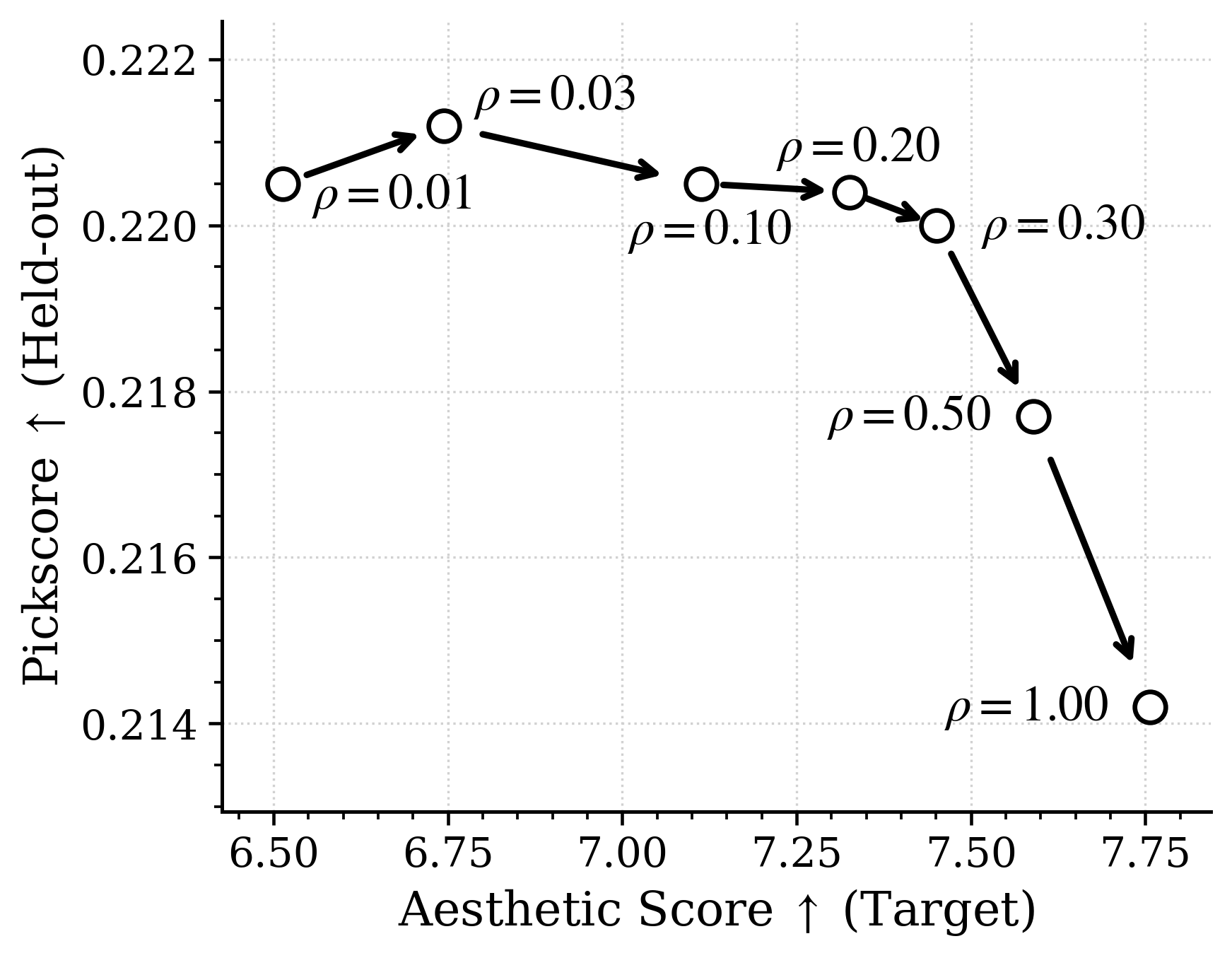}
    \vspace{-0.7cm}

    \captionof{figure}{
    \footnotesize
    \textbf{Reward trade-off trajectory as $\rho_t$ increases.}
    Moderate $\rho_t$ improves the target reward while preserving held-out quality, whereas overly large $\rho_t$ degrades the held-out reward.
    }
    \label{fig:rho_tradeoff_aesthetic}
\end{minipage}

\end{figure*}

Although the correspondence above is only heuristic, it suggests a clear practical implication: large $\rho_t$ acts like an increasingly greedy effective search over perturbations.
In multi-step reverse sampling, however, stronger greediness is not always desirable.
As $\rho_t \to 1$, the perturbation in \cref{eq:ntrk_noise_rho_app} becomes nearly deterministic, and the residual exploration term vanishes.
Thus, each reverse step relies more heavily on the reward-aligned component and less on stochastic exploration.
While this can improve the immediate target reward, repeatedly applying such nearly deterministic updates may reduce robustness to local approximation error and gradually weaken compatibility with the noise-compatible regime, which can in turn harm held-out quality.

This trade-off is consistent with the ablation results in \cref{tab:rho_ablation_aesthetic}.
Moderate values of $\rho_t$ improve the target aesthetic score while largely preserving held-out rewards, whereas overly large values of $\rho_t$ degrade held-out quality.
Thus, the most effective operating regime is neither pure random sampling ($\rho_t \approx 0$) nor fully deterministic reward chasing ($\rho_t \approx 1$), but a moderate intermediate regime that balances reward alignment and stochastic exploration.

\clearpage
\newpage

\section{Whitening Operator}
\label{app:whitening_operator}

In \cref{sec:method}, we introduced the whitening operator $\mathcal{W}$ as the key component that enables reward guidance through the noise term while preserving the pretrained reverse-kernel mean.
Its role is to transform an arbitrary input signal, such as a reward gradient, into a perturbation that remains compatible with the pretrained reverse dynamics.

More precisely, our goal is to move an arbitrary vector toward the \emph{noise-compatible regime}, namely, the regime of perturbations that exhibit the statistical characteristics of typical standard Gaussian noise and are therefore naturally accepted by the pretrained model.
As discussed in \cref{sec:method}, although $\mathcal{N}(\bm{0}, \bm{I})$ has nonzero density everywhere in $\mathbb{R}^N$, in high dimensions almost all of its probability mass concentrates on a narrow typical region.
The exact typical set is difficult to characterize in closed form, so instead of attempting to describe it exactly, we approximate it by intersecting a collection of high-confidence constraint sets induced by known statistics of the standard Gaussian distribution.

Accordingly, the whitening operator $\mathcal{W}$ is implemented as a sequence of projections onto such confidence sets.
Each component enforces one statistical property that typical white Gaussian noise satisfies with overwhelming probability.
The resulting operator is modular, computationally efficient, and practically effective: it leaves already-typical Gaussian noise almost unchanged, while substantially suppressing structured or correlated artifacts in arbitrary inputs.

\subsection{Typical Noise Accepted by a Pretrained Model}
\label{app:whitening_background}

The pretrained reverse kernel is calibrated under the assumption that its standardized perturbation follows standard Gaussian noise.
Therefore, when we inject a reward-informed vector through the noise term, the relevant question is not merely whether the vector has the correct overall scale, but whether it lies in the regime of perturbations that the pretrained model naturally accepts.
In this subsection, we clarify this regime and motivate why our whitening operator is designed through confidence-interval projections rather than simple norm matching.

Let $\bm{z}\in\mathbb{R}^N$ follow the standard multivariate Gaussian distribution
\begin{equation}
p(\bm{z})
=
(2\pi)^{-N/2}
\exp\!\left(
-\frac{\|\bm{z}\|_2^2}{2}
\right).
\label{eq:gaussian_pdf_app}
\end{equation}
Although this density is radially symmetric, its probability mass concentrates sharply in high dimensions.
In particular,
\begin{equation}
\|\bm{z}\|_2^2 \sim \chi^2_N,
\label{eq:chi_square_app}
\end{equation}
whose mean is $N$ and whose standard deviation is $\sqrt{2N}$.
As $N$ grows, the relative fluctuation of $\|\bm{z}\|_2^2$ becomes small, so most samples concentrate near a thin hyperspherical shell of radius approximately $\sqrt{N}$.
This is the familiar norm concentration phenomenon of high-dimensional Gaussian noise.

A natural first approximation is therefore to regard vectors on the hypersphere $\|\bm{z}\|_2=\sqrt{N}$ as typical.
However, this is too weak for pretrained generative models.

\begin{figure}[t]
\centering
\setlength{\tabcolsep}{1.0pt}
\newcolumntype{Y}{>{\centering\arraybackslash}X}
\renewcommand{\arraystretch}{1.0}
\small

\newcommand{\ImgBox}[1]{%
  \fbox{\includegraphics[width=\linewidth]{#1}}%
}
\setlength{\fboxsep}{0pt}      
\setlength{\fboxrule}{0.25pt}  

\newcommand{\TNoise}[1]{\ImgBox{Figures/NoiseDistinct/noises/#1}}
\newcommand{\TFlux}[1]{\ImgBox{Figures/NoiseDistinct/noises_flux/#1}}
\newcommand{\ANoise}[1]{\ImgBox{Figures/NoiseDistinct/noises/#1}}
\newcommand{\AFlux}[1]{\ImgBox{Figures/NoiseDistinct/noises_flux/#1}}

\begin{tabularx}{\linewidth}{
@{} *{2}{Y} c|c *{2}{Y}
c | c
*{2}{Y} c|c *{2}{Y} @{}}
\toprule
\multicolumn{6}{c}{\textbf{Typical Noise}} & & &
\multicolumn{6}{c}{\textbf{Atypical Noise}} \\

\multicolumn{2}{c}{\scriptsize Latent $\bm z$ ($\|\bm z\|_2=\sqrt{N}$)} & & &
\multicolumn{2}{c}{\scriptsize Sampled Image} & & &
\multicolumn{2}{c}{\scriptsize Latent $\bm z$ ($\|\bm z\|_2=\sqrt{N}$)} & & &
\multicolumn{2}{c}{\scriptsize Sampled Image} \\
\midrule

\TNoise{00.png} & \TNoise{01.png} & & & \TFlux{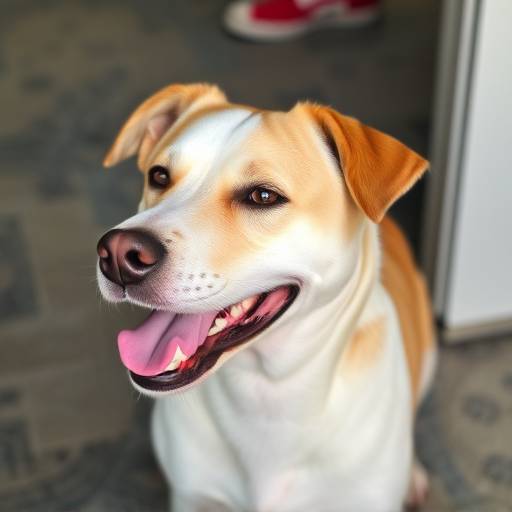} & \TFlux{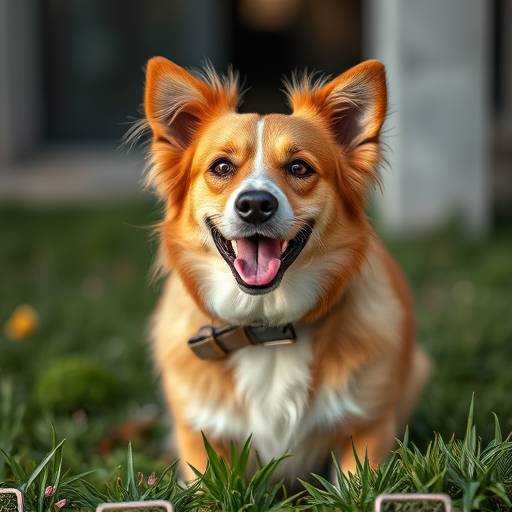} & & &
\ANoise{s00.png} & \ANoise{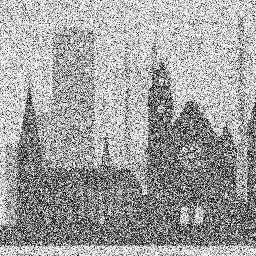} & & & \AFlux{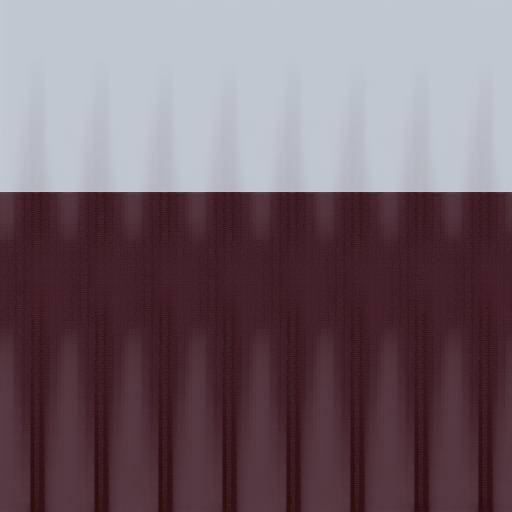} & \AFlux{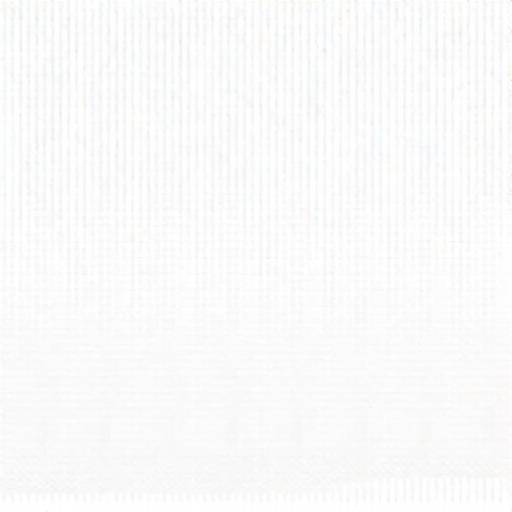} \\[-2pt]

\TNoise{02.png} & \TNoise{03.png} & & & \TFlux{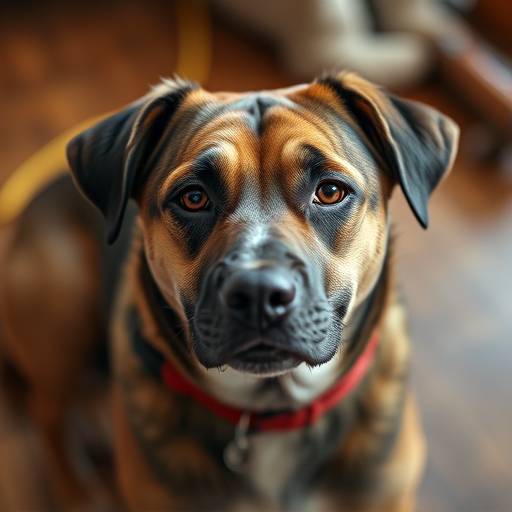} & \TFlux{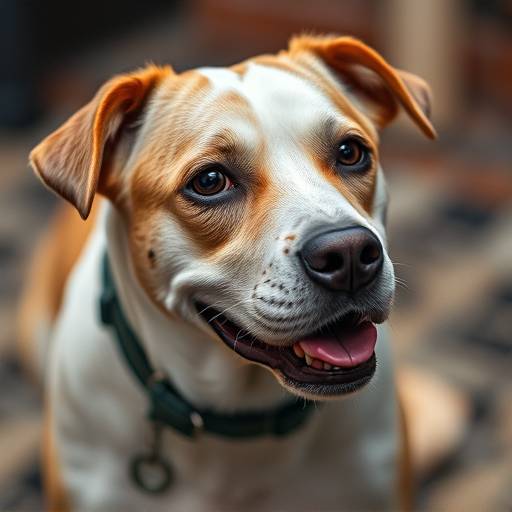} & & &
\ANoise{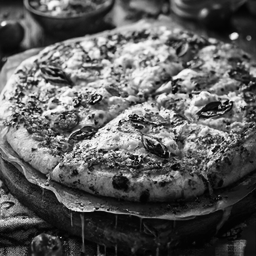} & \ANoise{s03.png} & & & \AFlux{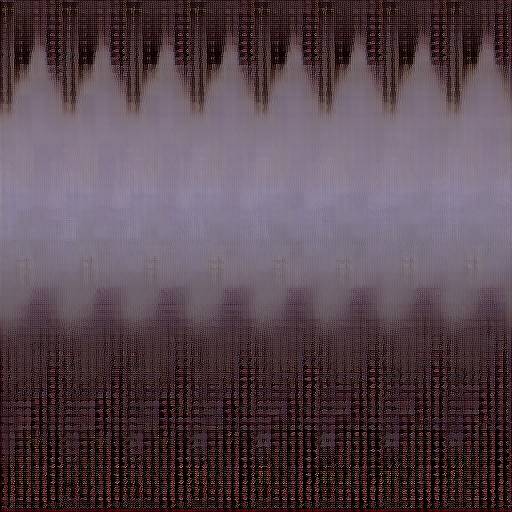} & \AFlux{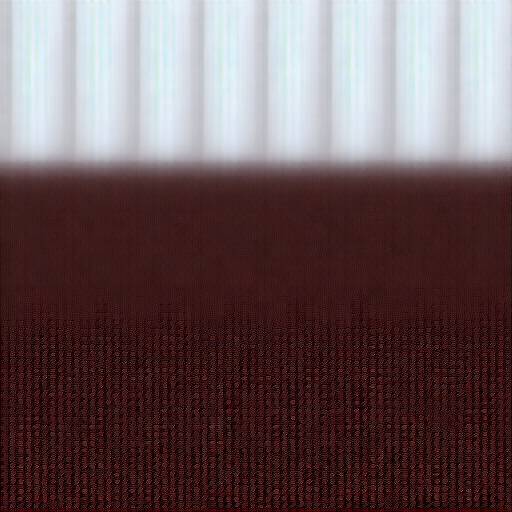} \\[-2pt]

\TNoise{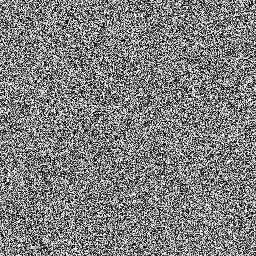} & \TNoise{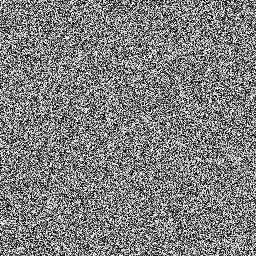} & & & \TFlux{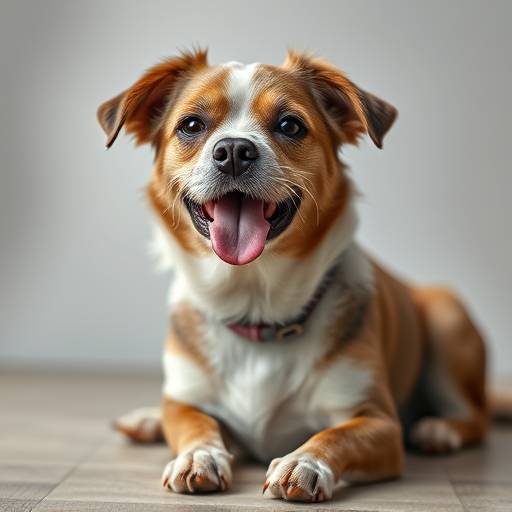} & \TFlux{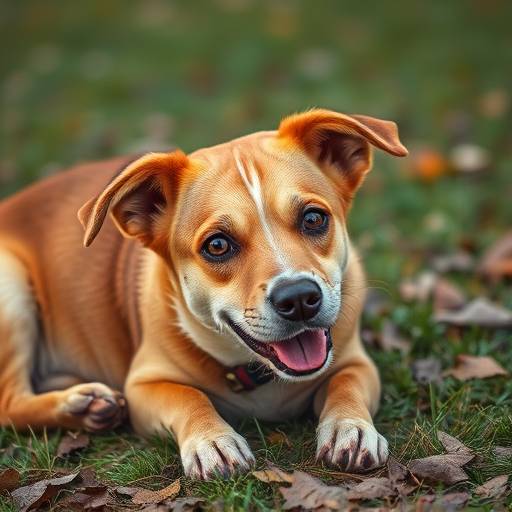} & & &
\ANoise{s04.png} & \ANoise{s05.png} & & & \AFlux{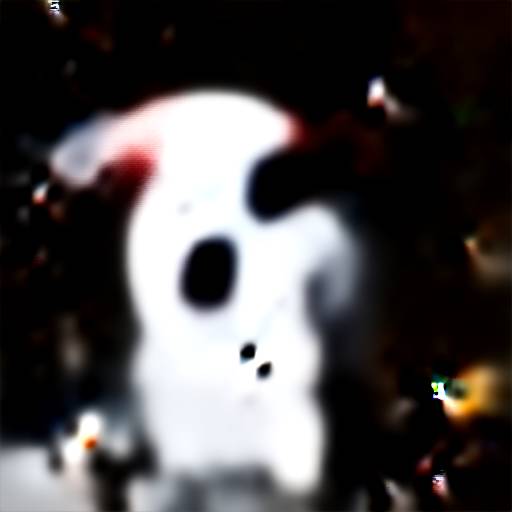} & \AFlux{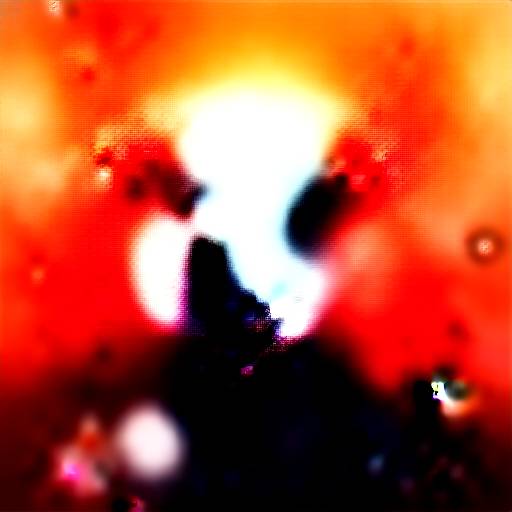} \\[-2pt]

\bottomrule
\end{tabularx}

\caption{
\footnotesize
\textbf{Identical norm does not guarantee noise compatibility.}
Latent vectors with identical norm ($\|\bm z\|_2=\sqrt{N}$) can still behave very differently under pretrained reverse dynamics.
Vectors whose statistics resemble typical standard Gaussian noise produce valid samples, whereas structured or spatially correlated vectors, despite lying on the same hypersphere, lead to severe failure cases.
}
\label{fig:noise_init_examples_app}
\end{figure}

As illustrated in \cref{fig:noise_init_examples_app}, latent vectors can have the same norm while behaving very differently in a pretrained model.
When the vector exhibits the fine-scale irregularity characteristic of standard Gaussian noise, the pretrained model produces valid samples.
In contrast, when the vector contains structured spatial correlation or other atypical regularities, the same pretrained model can produce severe failure cases, even though the vector lies on the same hypersphere.
Thus, norm concentration alone does not characterize the perturbation regime that a pretrained model reliably accepts.

This observation suggests that the practically relevant object is not the hypersphere itself, but the subset of vectors whose statistics resemble those of \emph{typical standard Gaussian noise}.
More importantly, these are precisely the perturbations that the pretrained reverse kernel is calibrated to process: during pretraining, the model repeatedly encounters standard Gaussian samples drawn from the reference noise distribution, not arbitrary vectors on the hypersphere.
Therefore, if we wish to inject a reward-informed signal through the noise term while preserving the pretrained reverse dynamics, the injected vector should mimic the statistical characteristics of those typical sampled perturbations as closely as possible.
In this sense, our goal is to move an arbitrary input toward the \emph{noise-compatible regime} of standard Gaussian noise.

Intuitively, such vectors should not only have the correct global scale, but also avoid unnatural local structure, excessive correlation, or anomalous concentration of values and energies.
The exact typical set of high-dimensional Gaussian noise is difficult to describe in a closed form that is mathematically explicit.
Accordingly, rather than attempting to characterize it exactly, we construct a tractable surrogate based on a collection of high-confidence constraints derived from standard Gaussian reference statistics.

Our whitening operator is built from this perspective.
Each component enforces one property that standard Gaussian noise satisfies with overwhelming probability, while modifying the input as little as possible in Euclidean norm.
Concretely, throughout this appendix we use $99.99\%$ confidence intervals (CI) under the reference distribution and realize whitening as a sequence of projections onto the corresponding confidence sets.
This gives a practical approximation of the noise-compatible regime: already-typical Gaussian noise is changed only minimally, while structured or correlated inputs are pushed toward the statistics of typical standard Gaussian noise.

The remaining subsections make this construction explicit.
We first introduce the basic two-level order-statistic (2OS) projection in \cref{app:whitening_2os}, then extend it to tile-wise mean and energy statistics in \cref{app:whitening_tile_mean_energy}, and finally apply the same principle in multiple orthogonal transform domains in \cref{app:whitening_orthogonal}.

\subsection{Confidence-Interval Projection via Two-Level Order Statistics}
\label{app:whitening_2os}

We now introduce the two-level order-statistic projection (2OS), which is a key building block of our whitening operator.
Its purpose is to move an arbitrary input toward \emph{typical standard Gaussian noise} by projecting onto high-confidence sets derived from the reference distribution.
To motivate why 2OS is needed, we first consider two simpler confidence-interval (CI) constructions.
These preliminary constructions are included only to explain the design choice.
They are not components of the full whitening operator.
Rather, they show why we ultimately use two-level order statistics instead of relying only on global value bounds or one-level sorted statistics.

\cref{fig:ci_hierarchy_app} summarizes this build-up.
Panel (a) shows the coarsest possibility: a global confidence interval for a single Gaussian variable.
Panel (b) refines this by using rank-dependent confidence intervals for one-level order statistics.
Panel (c) then shows the two-level construction used in our method, which additionally conditions on the inner rank within each local chunk.
The key distinction is that only this final construction constrains how values are distributed \emph{inside} each chunk, which is what allows it to suppress localized structure.

\begin{figure}[t]
\centering
\setlength{\tabcolsep}{2pt}
\renewcommand{\arraystretch}{0.5}

\begin{tabular}{ccc}
\small \textbf{(a)} \makecell{Global CI}
&
\small \textbf{(b)} \makecell{1OS CI}
&
\small \textbf{(c)} \makecell{2OS CI}
\\
\includegraphics[width=0.33\linewidth]{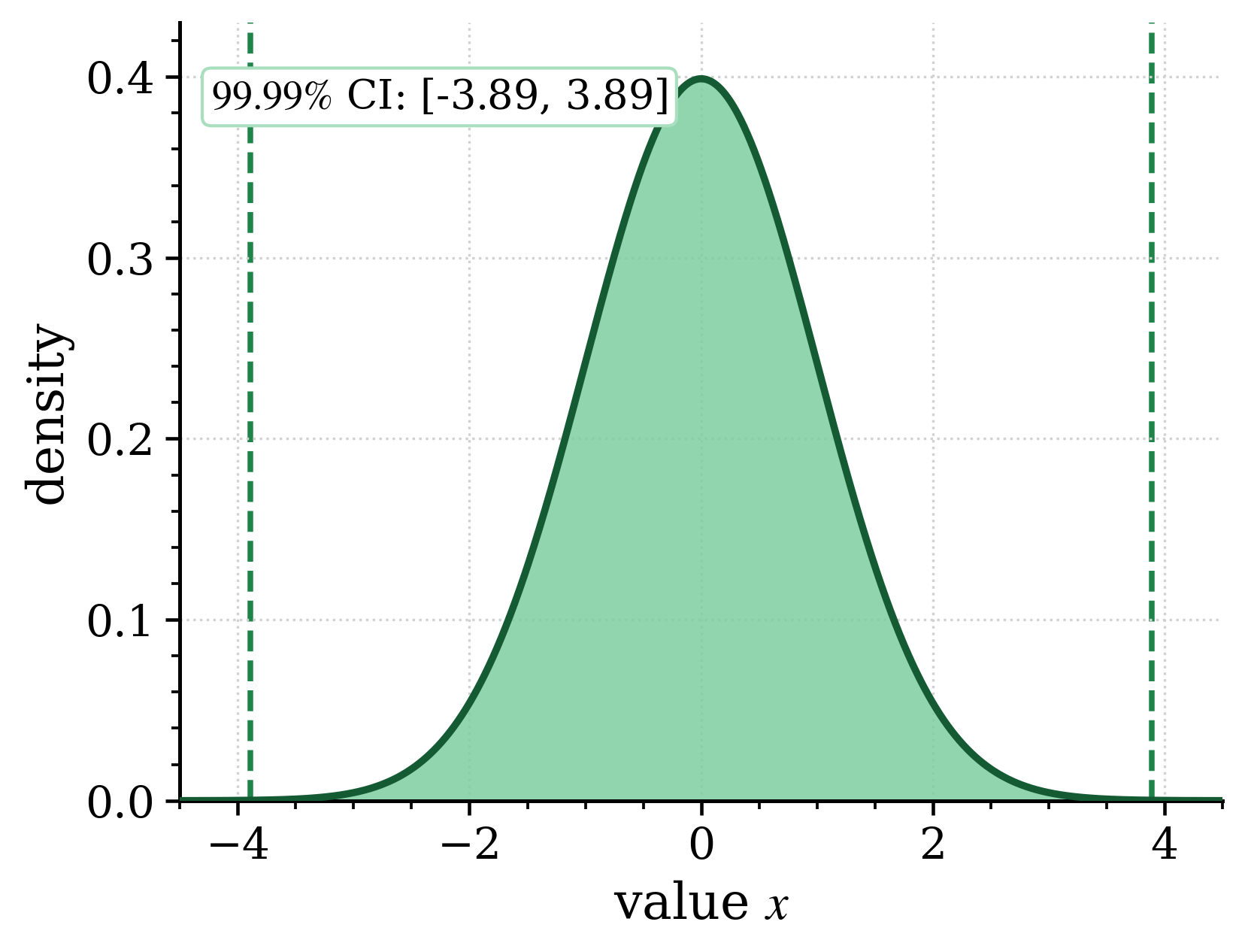}
&
\includegraphics[width=0.33\linewidth]{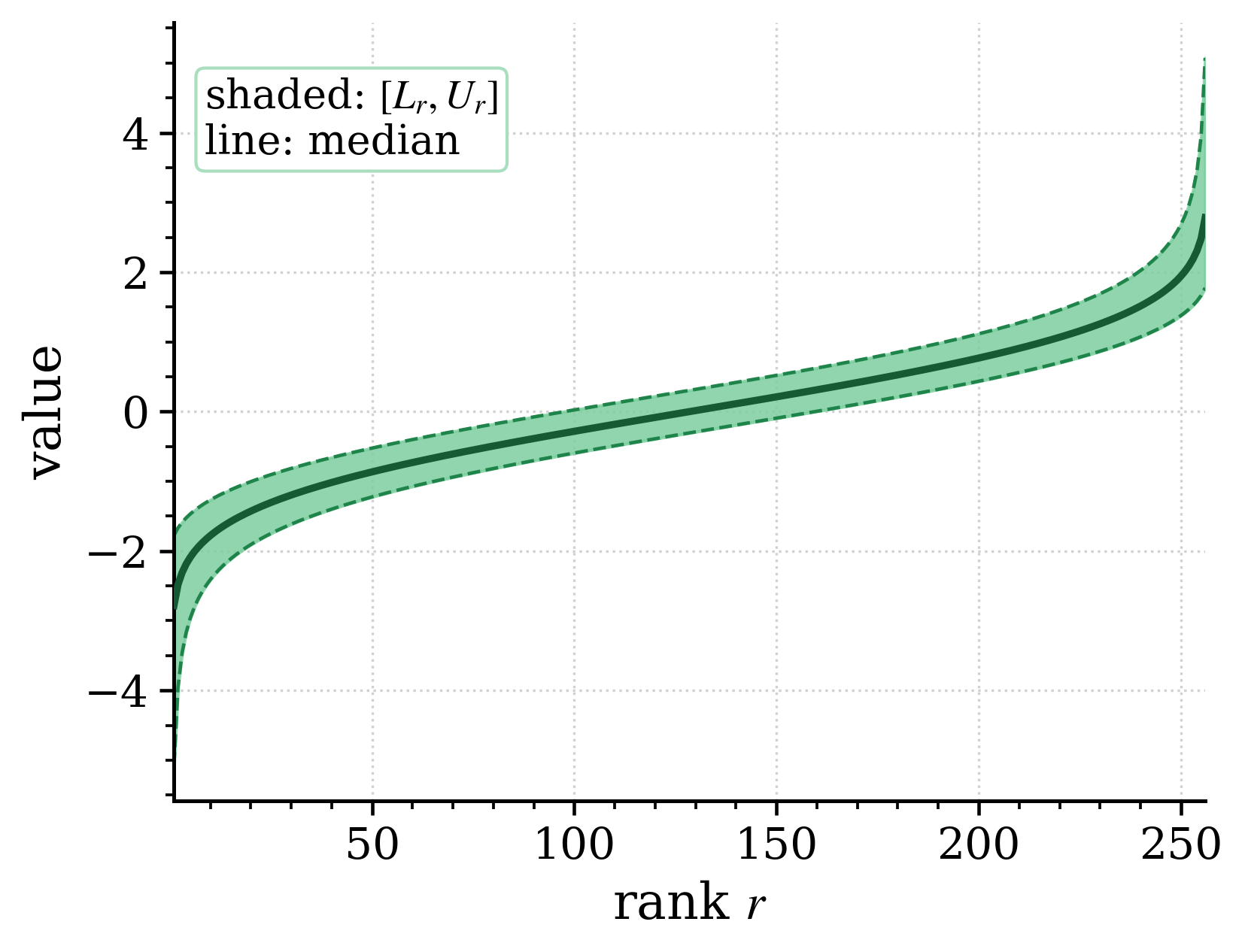}
&
\includegraphics[width=0.33\linewidth]{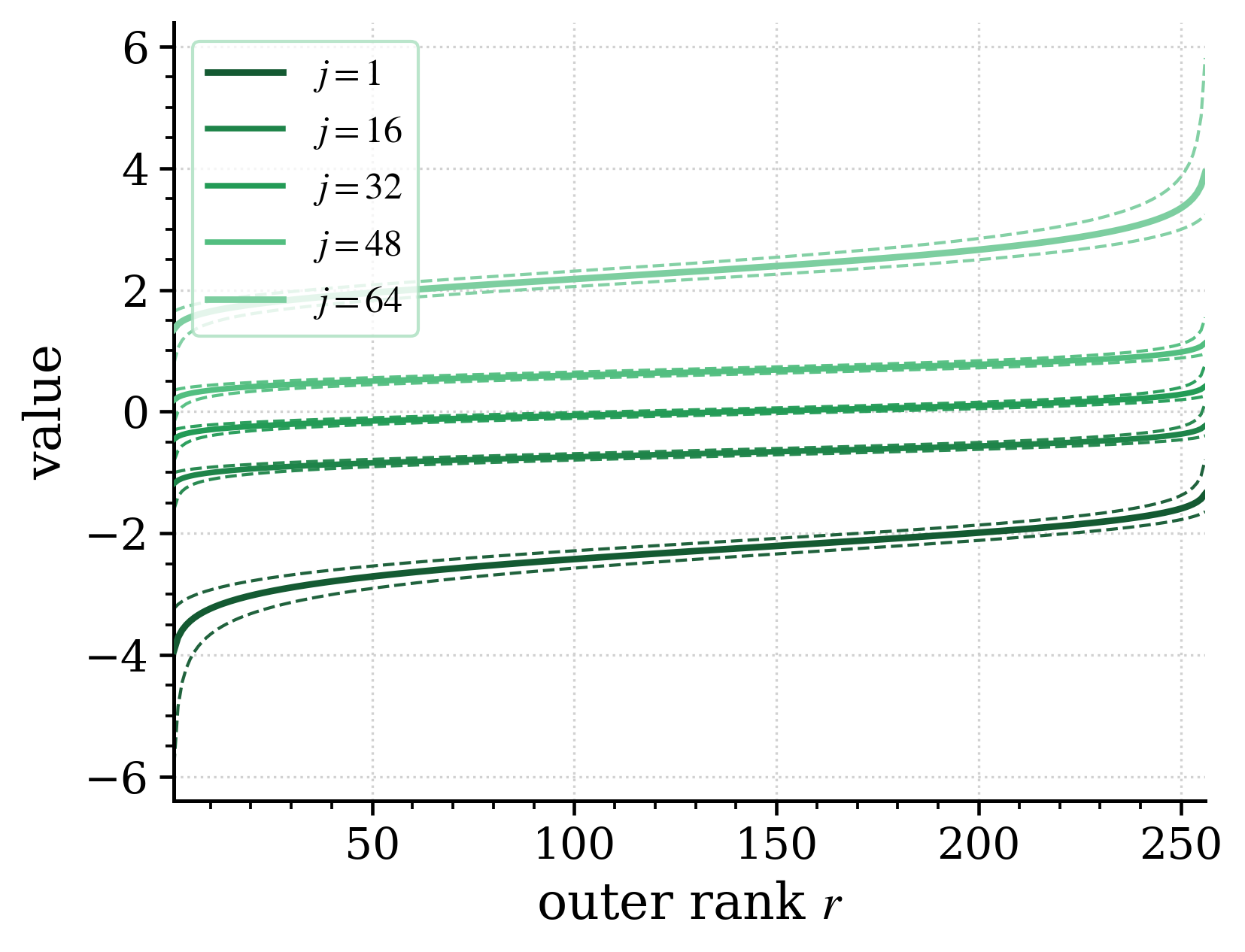}
\end{tabular}
\vspace{-0.3cm}
\caption{
\footnotesize
\textbf{Illustration of the confidence-interval hierarchy.} Confidence-intervals (CI) for the standard normal distribution with confidence level $99.99\%$ are shown.
\textbf{(a)} The coarsest constraint is the global confidence interval for a single variable.
\textbf{(b)} A tighter characterization is obtained by using rank-dependent confidence intervals for one-level order statistics (1OS).
\textbf{(c)} Our two-level order-statistic (2OS) construction further refines the admissible range by conditioning on both the outer rank $r$ and the inner rank $j$ within each chunk.
}
\label{fig:ci_hierarchy_app}
\end{figure}

Throughout this subsection, let $H$ denote a continuous reference CDF and let $H^{-1}$ denote its inverse CDF.
In our Gaussian instantiation, we use $H=\Phi$, the standard normal CDF, and fix the confidence level to $99.99\%$.

\paragraph{A first construction: global confidence intervals (Global CI).}
The coarsest confidence-based construction is to require each coordinate to lie in a global high-confidence interval under the reference distribution; see \cref{fig:ci_hierarchy_app}(a).
Given confidence level $1-\alpha$, define
\begin{equation}
\underline{a}
=
H^{-1}\!\left(\frac{\alpha}{2}\right),
\qquad
\overline{a}
=
H^{-1}\!\left(1-\frac{\alpha}{2}\right),
\label{eq:global_ci_bounds_app}
\end{equation}
and let
\begin{equation}
\mathcal{S}_{\mathrm{box}}
=
\Bigl\{
\bm{y}\in\mathbb{R}^{N}
:\;
\underline{a}\le y_i\le \overline{a}
\;\; \forall i
\Bigr\}.
\label{eq:box_set_app}
\end{equation}

\begin{proposition}
\label{prop:box_projection_app}
The Euclidean projection onto $\mathcal{S}_{\mathrm{box}}$ is given by elementwise clipping:
\begin{equation}
\Pi_{\mathcal{S}_{\mathrm{box}}}(\bm{x})
=
\bigl(
\mathrm{clip}(x_i;\underline{a},\overline{a})
\bigr)_{i=1}^{N}.
\label{eq:box_projection_app}
\end{equation}
\end{proposition}

\begin{proof}
Because
\begin{equation}
\|\bm{y}-\bm{x}\|_2^2
=
\sum_{i=1}^{N}(y_i-x_i)^2,
\end{equation}
and the constraints in \cref{eq:box_set_app} decouple across coordinates, the minimization separates into $N$ independent scalar projections onto the interval $[\underline{a},\overline{a}]$, which is exactly elementwise clipping.
\end{proof}

This construction preserves natural variability, since it imposes an interval rather than a deterministic equality.
However, it is still far too weak for our purpose.
It constrains only marginal value range, so many structured vectors can satisfy \cref{eq:box_set_app} while remaining far from the noise-compatible regime of typical standard Gaussian noise.

\paragraph{A second construction: one-level order statistics (1OS).}
A more informative construction is to constrain the sorted values themselves, as visualized in \cref{fig:ci_hierarchy_app}(b).
Let $\{x_i\}_{i=1}^{M}$ be i.i.d.\ samples from the reference distribution with CDF $H$, and let $x_{(r)}$ denote the $r$-th smallest sample.
By the probability integral transform, $u_i := H(x_i)$ are i.i.d.\ $\mathrm{Unif}(0,1)$, and therefore
\begin{equation}
u_{(r)}
\sim
\mathrm{Beta}\bigl(r,\,M-r+1\bigr).
\label{eq:orderstat_beta_app}
\end{equation}
This yields rank-dependent confidence bounds in the quantile domain:
\begin{equation}
p^{\mathrm{lo}}_{r}(M)
=
\mathrm{Beta}^{-1}\!\Bigl(
\frac{\alpha}{2};\,
r,\,
M-r+1
\Bigr),
\quad
p^{\mathrm{hi}}_{r}(M)
=
\mathrm{Beta}^{-1}\!\Bigl(
1-\frac{\alpha}{2};\,
r,\,
M-r+1
\Bigr).
\label{eq:beta_bounds_outer_app}
\end{equation}
Mapping them back through $H^{-1}$ gives the value-domain thresholds
\begin{equation}
L_r
=
H^{-1}\!\bigl(p^{\mathrm{lo}}_{r}(M)\bigr),
\qquad
U_r
=
H^{-1}\!\bigl(p^{\mathrm{hi}}_{r}(M)\bigr).
\label{eq:gauss_bounds_outer_app}
\end{equation}

Let $\bm{x}_{\uparrow}=(x_{(1)},\dots,x_{(M)})$ denote the sorted version of $\bm{x}$, and let $P_{\bm{x}}$ be the permutation matrix satisfying $\bm{x}_{\uparrow}=P_{\bm{x}}\bm{x}$.
Define
\begin{equation}
\mathcal{S}_{\mathrm{os}}
=
\Bigl\{
\bm{y}\in\mathbb{R}^{M}
:\;
L_r\le y_{(r)}\le U_r,
\ \ r=1,\dots,M
\Bigr\}.
\label{eq:os_set_app}
\end{equation}

\begin{proposition}
\label{prop:order_projection_app}
A Euclidean projection onto $\mathcal{S}_{\mathrm{os}}$ is obtained by sorting, clipping in sorted order, and undoing the sort:
\begin{equation}
\Pi_{\mathcal{S}_{\mathrm{os}}}(\bm{x})
=
P_{\bm{x}}^{\top}\,
\mathrm{clip}\bigl(
\bm{x}_{\uparrow};\,
\bm{L},\,
\bm{U}
\bigr),
\label{eq:os_projection_app}
\end{equation}
where $\bm{L}=(L_1,\dots,L_M)$ and $\bm{U}=(U_1,\dots,U_M)$.
\end{proposition}

\begin{proof}
Any $\bm{y}\in\mathcal{S}_{\mathrm{os}}$ can be written as $\bm{y}=P^\top \bm{y}_{\uparrow}$ for some permutation $P$ and some nondecreasing vector $\bm{y}_{\uparrow}$ satisfying $L_r\le (\bm{y}_{\uparrow})_r\le U_r$.
Since the Euclidean norm is permutation-invariant,
\begin{equation}
\|\bm{y}-\bm{x}\|_2^2
=
\|\bm{y}_{\uparrow}-P\bm{x}\|_2^2.
\end{equation}
For fixed $\bm{y}_{\uparrow}$, this is minimized when $P\bm{x}$ is sorted in the same order as $\bm{y}_{\uparrow}$, that is, when $P=P_{\bm{x}}$, by the rearrangement inequality.
The remaining problem is therefore the Euclidean projection of the sorted vector $\bm{x}_{\uparrow}$ onto the box constraints $L_r\le y_r\le U_r$, which is achieved by elementwise clipping.
\end{proof}

This one-level construction is much tighter than the global interval, since it constrains each rank separately rather than only the overall value range.
For example, when $H=\Phi$, $\alpha=10^{-4}$, and $M=65536$, the global $99.99\%$ confidence interval is $[-3.8906,\,3.8906]$, whose width is $7.7812$.
In contrast, the one-level interval width $U_r-L_r$ is only $0.0381$ near the median rank and $0.0414$ even at $r=16384$, making it roughly $200\times$ tighter than the global interval.
Even at the most extreme rank $r=1$, the width is $2.4282$, which is still more than $3\times$ tighter.

Nevertheless, one-level order statistics still do not solve the main problem.
Because they depend only on the globally sorted values, they are permutation-invariant with respect to the original coordinates.
They constrain the global histogram, but they do not constrain how small and large values are distributed across local chunks.
As a result, substantial spatial structure can still survive.

\begin{figure}[t]
\centering
\setlength{\tabcolsep}{1pt}
\renewcommand{\arraystretch}{0.5}
\small

\newcommand{\cellimg}[1]{\includegraphics[width=\linewidth]{Figures/whitening_process_app/#1}}

\newcolumntype{C}{>{\centering\arraybackslash}X}

\begin{tabularx}{\linewidth}{@{} c *{1}{C} | *{3}{C} @{}}
\toprule
& Input & \makecell{Global} & \makecell{1OS} & \makecell{2OS} \\

\midrule
\raisebox{3.0ex}{\rotatebox{90}{Latent Visualization}} &
\cellimg{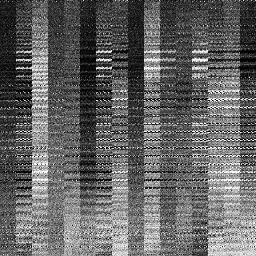} &
\cellimg{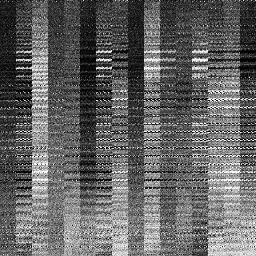} &
\cellimg{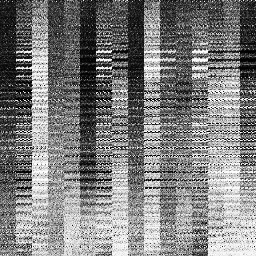} &
\cellimg{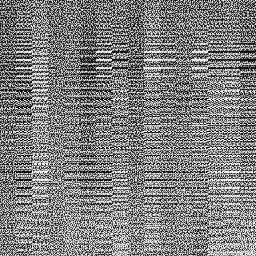} \\
\raisebox{5.5ex}{\rotatebox{90}{Sampled Image}} &
\cellimg{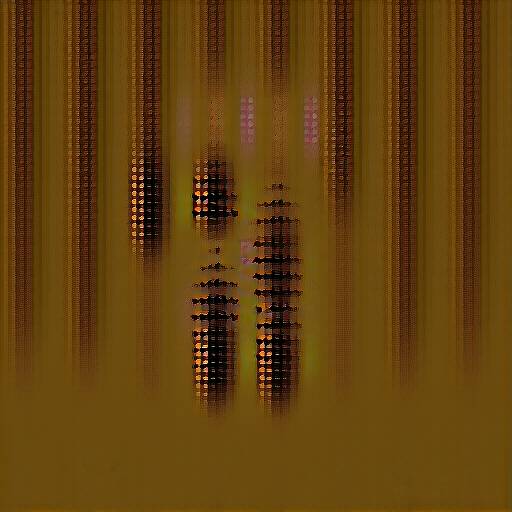} &
\cellimg{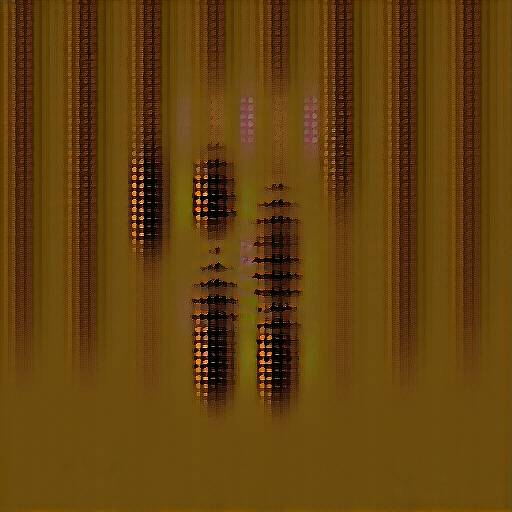}  &
\cellimg{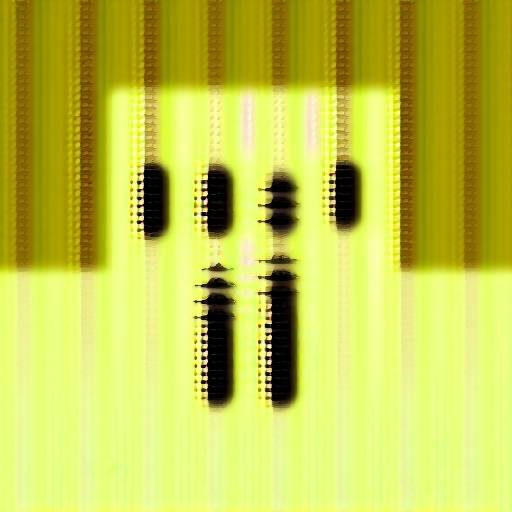}  &
\cellimg{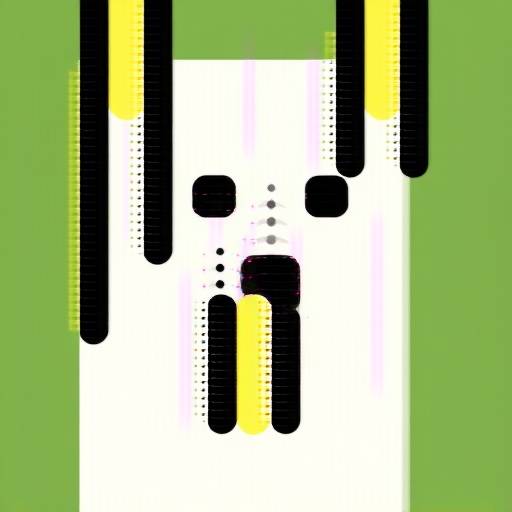} \\[-2px]
\bottomrule
\end{tabularx}

\caption{
\footnotesize
\textbf{Effect of confidence-interval (CI) projection via two-level order statistics (2OS).}
Global CI clipping only truncates extreme values and leaves the dominant structured pattern largely unchanged.
1OS CI further matches the global sorted histogram, but still preserves substantial local structure.
In contrast, 2OS CI enforces rank consistency both within and across local chunks, forcing each chunk to contain a balanced spread of small-to-large values.
This suppresses localized artifacts much more effectively and makes the latent substantially more like a typical noise.
The bottom row shows the corresponding effect on the sampled image.
}
\label{fig:whitening_two_level_app}
\end{figure}

This limitation is illustrated in \cref{fig:whitening_two_level_app}.
Both the global CI and one-level CI projections reduce extreme values, but visible stripe-like structure remains in the latent visualization.
This is precisely what motivates the two-level construction below: we need a confidence-interval projection that controls not only \emph{which} values appear globally, but also \emph{how} those values are distributed locally.

\paragraph{Two-level order statistics (2OS).}
We now introduce the construction actually used in our whitening operator.
To make the confidence-interval constraint sensitive to local organization, we partition the vector into local chunks and enforce order-statistic consistency both \emph{within} each chunk and \emph{across} chunks; see \cref{fig:ci_hierarchy_app}(c).

Specifically, reshape the vector into a matrix
\begin{equation}
\bm{X}\in\mathbb{R}^{M\times D},
\end{equation}
where each row corresponds to a chunk of size $D$.
We first sort within each row, and then, for each within-row rank $j$, sort across the $M$ rows.
This yields the doubly sorted representation
\begin{equation}
\bm{Z}
=
\mathrm{sort}_0\bigl(\mathrm{sort}_1(\bm{X})\bigr),
\label{eq:twolevel_sort_def_app}
\end{equation}
where $r$ indexes the outer rank across chunks and $j$ indexes the inner rank within each chunk.
Equivalently, $Z_{r,j}$ is the $r$-th smallest value among the $j$-th order statistics collected from all chunks.

Under the i.i.d.\ reference model, the induced quantiles are obtained by composing Beta laws.
After fixing the outer rank $r$, the confidence interval for the corresponding quantile is given by \cref{eq:beta_bounds_outer_app}.
Conditioning further on the inner rank $j$ within a chunk yields
\begin{align}
p^{\mathrm{lo}}_{r,j}(M,D)
&=
\mathrm{Beta}^{-1}\!\Bigl(
p^{\mathrm{lo}}_{r}(M);\,
j,\,
D-j+1
\Bigr),
\label{eq:beta_bounds_twolevel_lo_app}
\\
p^{\mathrm{hi}}_{r,j}(M,D)
&=
\mathrm{Beta}^{-1}\!\Bigl(
p^{\mathrm{hi}}_{r}(M);\,
j,\,
D-j+1
\Bigr),
\label{eq:beta_bounds_twolevel_hi_app}
\end{align}
and therefore
\begin{equation}
L_{r,j}
=
H^{-1}\!\bigl(p^{\mathrm{lo}}_{r,j}(M,D)\bigr),
\qquad
U_{r,j}
=
H^{-1}\!\bigl(p^{\mathrm{hi}}_{r,j}(M,D)\bigr).
\label{eq:gauss_bounds_twolevel_app}
\end{equation}

Now define the feasible set
\begin{equation}
\mathcal{S}_{\mathrm{2os}}
=
\Bigl\{
\bm{Y}\in\mathbb{R}^{M\times D}
:\;
L_{r,j}
\le
\bigl(\mathrm{sort}_0(\mathrm{sort}_1(\bm{Y}))\bigr)_{r,j}
\le
U_{r,j},
\ \forall r,j
\Bigr\}.
\label{eq:twolevel_set_app}
\end{equation}

\begin{proposition}
\label{prop:two_level_projection_app}
Let $\bm{Z}$ be defined by \cref{eq:twolevel_sort_def_app}.
Then clipping each entry of $\bm{Z}$ into its corresponding interval $[L_{r,j},U_{r,j}]$ and undoing the two sorting permutations yields a Frobenius-norm projection onto $\mathcal{S}_{\mathrm{2os}}$:
\begin{equation}
\Pi_{\mathcal{S}_{\mathrm{2os}}}(\bm{X})
\in
\arg\min_{\bm{Y}\in\mathcal{S}_{\mathrm{2os}}}
\|\bm{Y}-\bm{X}\|_F^2.
\label{eq:twolevel_projection_app}
\end{equation}
\end{proposition}

\begin{proof}
The row-wise and column-wise sorting operations are compositions of permutation matrices and therefore preserve the Frobenius norm.
Accordingly, for the purpose of Euclidean projection, we may pass to the canonical representative obtained by doubly sorting $\bm{X}$.

Now fix a candidate doubly sorted matrix $\widetilde{\bm{Z}}$ satisfying
\begin{equation}
L_{r,j}\le \widetilde{Z}_{r,j}\le U_{r,j}
\qquad
\forall r,j.
\end{equation}
Among all matrices whose doubly sorted representation equals $\widetilde{\bm{Z}}$, the one closest to $\bm{X}$ is obtained by undoing the sorting permutations induced by $\bm{X}$ itself.
This follows by repeated application of the rearrangement inequality: first within each row, and then across rows for each column.
Consequently, the projection problem reduces to minimizing
\begin{equation}
\|\widetilde{\bm{Z}}-\bm{Z}\|_F^2
\end{equation}
subject to the box constraints above.
Since these constraints decouple entrywise, the minimizer is obtained by clipping each $Z_{r,j}$ independently to $[L_{r,j},U_{r,j}]$.
Undoing the sorting permutations then gives a Frobenius-norm minimizer in the original coordinates.
\end{proof}

The key difference from the one-level construction is that clipping is applied to every rank pair $(r,j)$ of the doubly sorted statistic.
As a result, extreme values cannot concentrate in only a few chunks.
Instead, each chunk is forced to contain a balanced spread of small-to-large values consistent with typical standard Gaussian noise.
In this sense, 2OS constrains not only the global histogram, but also the local distributional composition of each chunk.
This is the core mechanism by which it suppresses localized structure and pushes an arbitrary vector toward typical noise.

\subsection{Confidence-Interval Projection on Tile-wise Mean and Energy}
\label{app:whitening_tile_mean_energy}

The 2OS projection in \cref{app:whitening_2os} acts directly on the entries of a partitioned representation.
This already suppresses localized structure by forcing each chunk to contain a balanced spread of small-to-large values.
However, entrywise 2OS alone does not fully control \emph{partition-wise statistics}.
For example, a chunk may still contain values from all bands while having an unusually low or high mean, or it may contain values that are too concentrated around that mean, leading to an atypical centered energy.
To further move the input toward typical standard Gaussian noise, we therefore apply the same confidence-interval principle to tile-wise statistics themselves.

\begin{figure}[t]
\centering
\setlength{\tabcolsep}{1pt}
\renewcommand{\arraystretch}{0.5}
\small

\newcommand{\cellimg}[1]{\includegraphics[width=\linewidth]{Figures/whitening_process_app/#1}}

\newcolumntype{C}{>{\centering\arraybackslash}X}

\begin{tabularx}{\linewidth}{@{} c C | C C C @{}}
\toprule
& Input
& \makecell{2OS}
& \makecell{$+\,2\times2$ Tile\\Mean/Energy}
& \makecell{$+\,8\times8$ Tile\\Mean/Energy} \\
\midrule
\raisebox{3.0ex}{\rotatebox{90}{Latent Visualization}} &
\cellimg{original.png} &
\cellimg{clamp2.png} &
\cellimg{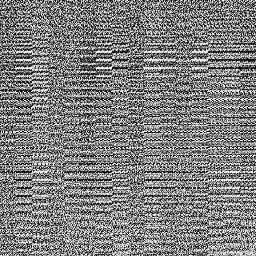} &
\cellimg{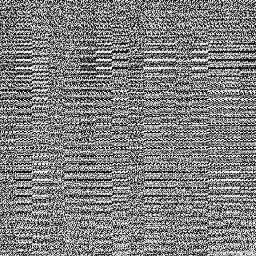} \\
\raisebox{5.5ex}{\rotatebox{90}{Sampled Image}} &
\cellimg{original_img.jpg} &
\cellimg{clamp2_img.jpg}  &
\cellimg{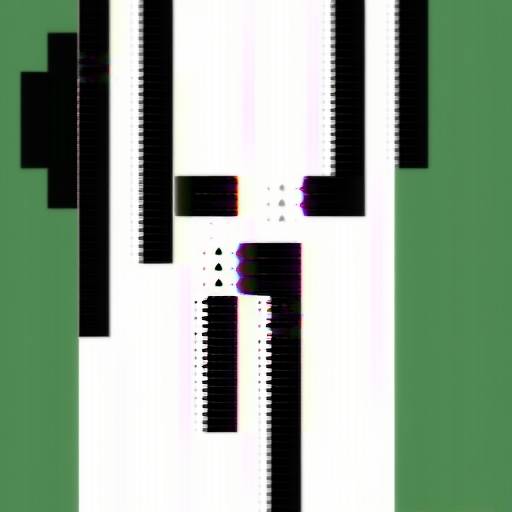}  &
\cellimg{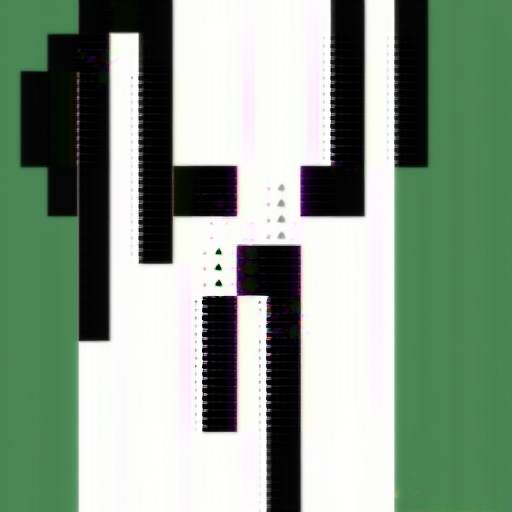} \\[-2px]
\bottomrule
\end{tabularx}

\caption{
\footnotesize
\textbf{Effect of confidence-interval (CI) projection on tile-wise statistics.}
Starting from entrywise 2OS, we additionally constrain tile-wise mean and centered energy at $2\times2$ and $8\times8$ tile scales.
These tile-wise constraints further suppress residual low-frequency structure.
}
\label{fig:whitening_tilewise_app}
\end{figure}

\cref{fig:whitening_tilewise_app} illustrates this effect.
Starting from the output of 2OS CI on entries, we additionally constrain tile-wise statistics at multiple tile scales.
This progressively removes residual low-frequency bias.
This is exactly the role of tile-wise statistics in the whitening operator: they complement entrywise 2OS by directly controlling whether each partition has Gaussian-like aggregate behavior.

Let $\bm{x}\in\mathbb{R}^{N}$ be partitioned into $P$ disjoint tiles of equal size $F$ (so $N=PF$), denoted by
\begin{equation}
\bm{x}^{(p)}\in\mathbb{R}^{F},
\qquad
p=1,\dots,P.
\end{equation}
For each tile, define the tile-wise mean and centered energy
\begin{equation}
\bar{x}^{(p)}
=
\frac{1}{F}\bm{1}^{\top}\bm{x}^{(p)},
\qquad
v^{(p)}
=
\left\|
\bm{x}^{(p)}-\bar{x}^{(p)}\bm{1}
\right\|_2^2,
\label{eq:block_stats_app}
\end{equation}
where $\bm{1}\in\mathbb{R}^{F}$ is the all-ones vector.

Under the standard Gaussian reference model, these statistics have known distributions.
Indeed, if $\bm{x}^{(p)}\sim\mathcal{N}(\bm{0},\bm{I}_F)$, then
\begin{equation}
s^{(p)}
:=
\sqrt{F}\,\bar{x}^{(p)}
\sim
\mathcal{N}(0,1),
\qquad
v^{(p)}
\sim
\chi^2_{F-1}.
\label{eq:block_stat_dists_app}
\end{equation}
Thus the normalized tile-wise mean is standard normal, while the centered energy follows a chi-square law with $F-1$ degrees of freedom.

We now collect these quantities over all tiles.
Define the tile-wise statistic vectors
\begin{equation}
\bm{s}\in\mathbb{R}^{P},
\qquad
\bm{v}\in\mathbb{R}^{P},
\label{eq:block_stat_fields_app}
\end{equation}
whose $p$-th entries are the normalized mean $s^{(p)}$ and centered energy $v^{(p)}$ of the corresponding tile.
We then apply the same 2OS projection introduced in \cref{app:whitening_2os} to these statistic vectors, using the appropriate reference distributions.
Let $H_s$ and $H_v$ denote the CDFs of the standard normal and $\chi^2_{F-1}$ distributions, respectively.
We define
\begin{equation}
\bm{n}
=
\Pi^{\mathrm{(2os)}}_{H_s}(\bm{s}),
\qquad
\bm{m}
=
\Pi^{\mathrm{(2os)}}_{H_v}(\bm{v}),
\label{eq:tile_stat_projection_app}
\end{equation}
where $\Pi^{\mathrm{(2os)}}_{H}$ denotes the two-level order-statistic (2OS) projection of \cref{eq:twolevel_projection_app}, with confidence bounds computed from the reference CDF $H$.

The projected statistic fields $\bm{n}$ and $\bm{m}$ specify target tile-wise mean and energy values.
For each tile, we then update the tile to match those targets with minimal Euclidean modification.
Specifically, let
\begin{equation}
\mu_p^\star
=
\frac{n^{(p)}}{\sqrt{F}},
\qquad
v_p^\star
=
m^{(p)},
\label{eq:block_targets_app}
\end{equation}
and define
\begin{equation}
\bm{x}^{(p)}_{\mathrm{new}}
=
\mu_p^\star \bm{1}
+
\left(
\bm{x}^{(p)}-\bar{x}^{(p)}\bm{1}
\right)
\sqrt{\frac{v_p^\star}{v^{(p)}}}.
\label{eq:block_mean_energy_update_app}
\end{equation}
Thus each tile is shifted to the target mean and rescaled to the target centered energy, while preserving its centered direction.

\begin{proposition}
\label{prop:block_mean_energy_projection_app}
Fix a tile $\bm{x}^{(p)}\in\mathbb{R}^{F}$ and target values $(\mu_p^\star,v_p^\star)$ with $v_p^\star\ge 0$.
Define the feasible set
\begin{equation}
\mathcal{S}(\mu_p^\star,v_p^\star)
=
\Bigl\{
\bm{y}\in\mathbb{R}^{F}
:\;
\frac{1}{F}\bm{1}^{\top}\bm{y}=\mu_p^\star,
\ \ 
\bigl\|
\bm{y}-\mu_p^\star\bm{1}
\bigr\|_2^2=v_p^\star
\Bigr\}.
\label{eq:block_feasible_exact_app}
\end{equation}
Let $\bar{x}^{(p)}=\frac{1}{F}\bm{1}^{\top}\bm{x}^{(p)}$ and
$v^{(p)}=\|\bm{x}^{(p)}-\bar{x}^{(p)}\bm{1}\|_2^2$.
If $v^{(p)}>0$, then the update in \cref{eq:block_mean_energy_update_app} is the Euclidean projection of $\bm{x}^{(p)}$ onto $\mathcal{S}(\mu_p^\star,v_p^\star)$:
\begin{equation}
\bm{x}^{(p)}_{\mathrm{new}}
=
\Pi_{\mathcal{S}(\mu_p^\star,v_p^\star)}
\bigl(\bm{x}^{(p)}\bigr)
\in
\underset{\bm{y}\in\mathcal{S}(\mu_p^\star,v_p^\star)}{\arg\min}
\|\bm{y}-\bm{x}^{(p)}\|_2^2.
\label{eq:block_proj_exact_app}
\end{equation}
When $v^{(p)}=0$, the minimizer is not unique.
\end{proposition}

\begin{proof}
Write
\begin{equation}
\bm{x}^{(p)}
=
\bar{x}^{(p)}\bm{1}+\bm{c},
\qquad
\bm{1}^{\top}\bm{c}=0,
\qquad
\|\bm{c}\|_2^2=v^{(p)}.
\end{equation}
Any feasible $\bm{y}\in\mathcal{S}(\mu_p^\star,v_p^\star)$ can be written as
\begin{equation}
\bm{y}
=
\mu_p^\star\bm{1}+\bm{d},
\qquad
\bm{1}^{\top}\bm{d}=0,
\qquad
\|\bm{d}\|_2^2=v_p^\star.
\end{equation}
By orthogonality,
\begin{equation}
\|\bm{y}-\bm{x}^{(p)}\|_2^2
=
\|(\mu_p^\star-\bar{x}^{(p)})\bm{1}\|_2^2
+
\|\bm{d}-\bm{c}\|_2^2
=
F(\mu_p^\star-\bar{x}^{(p)})^2
+
\|\bm{d}-\bm{c}\|_2^2.
\label{eq:block_obj_decompose_app}
\end{equation}
The first term is fixed once $\mu_p^\star$ is fixed.
For the second term,
\begin{equation}
\|\bm{d}-\bm{c}\|_2^2
=
\|\bm{d}\|_2^2+\|\bm{c}\|_2^2-2\langle \bm{d},\bm{c}\rangle,
\end{equation}
so minimizing it is equivalent to maximizing $\langle \bm{d},\bm{c}\rangle$ subject to $\|\bm{d}\|_2^2=v_p^\star$.
This is achieved when $\bm{d}$ is colinear with $\bm{c}$.
Hence, when $v^{(p)}>0$, the unique minimizer is
\begin{equation}
\bm{d}
=
\bm{c}\sqrt{\frac{v_p^\star}{v^{(p)}}},
\end{equation}
which yields \cref{eq:block_mean_energy_update_app}.
If $v^{(p)}=0$, then $\bm{c}=\bm{0}$, and any feasible $\bm{d}$ with $\bm{1}^{\top}\bm{d}=0$ and $\|\bm{d}\|_2^2=v_p^\star$ attains the same minimum.
\end{proof}

The role of this construction is complementary to entrywise 2OS.
The projection in \cref{app:whitening_2os} ensures that each tile contains a balanced spread of values across different rank bands.
The present construction checks whether the \emph{aggregate behavior} of each tile is statistically typical.
By constraining tile-wise mean and centered energy through their known reference distributions, we prevent a subset of tiles from having unusually biased averages or atypical local energy, even if their entrywise order statistics already look plausible.
This is why tile-wise statistics are a natural next component of the whitening operator after 2OS.

\subsection{Confidence-Interval Projection in Orthogonal Transform Domains}
\label{app:whitening_orthogonal}

The constructions in \cref{app:whitening_2os,app:whitening_tile_mean_energy} constrain entrywise order statistics and tile-wise statistics in the original coordinates.
These components already suppress substantial local structure, but some artifacts are still more naturally exposed after an orthogonal change of basis.
In particular, residual global patterns may remain diffuse in the original indexing while becoming much more explicit in a transformed domain.
For this reason, our whitening operator also applies the same confidence-interval projections in additional \emph{orthogonal transform domains}.

\begin{figure}[t]
\centering
\setlength{\tabcolsep}{1pt}
\renewcommand{\arraystretch}{0.5}
\small

\newcommand{\cellimg}[1]{\includegraphics[width=\linewidth]{Figures/whitening_process_app/#1}}

\newcolumntype{C}{>{\centering\arraybackslash}X}

\begin{tabularx}{\linewidth}{@{} c *{1}{C} | *{3}{C} @{}}
\toprule
& Input & \makecell{2OS + Tile-wise} & \makecell{+ Fourier} & \makecell{+ Hadamard} \\

\midrule
\raisebox{3.0ex}{\rotatebox{90}{Latent Visualization}} &
\cellimg{original.png} &
\cellimg{mean_var.png} &
\cellimg{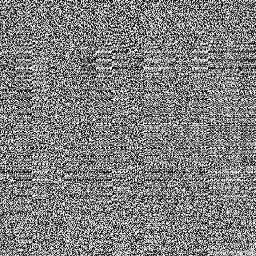} &
\cellimg{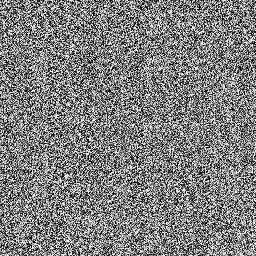} \\
\raisebox{5.5ex}{\rotatebox{90}{Sampled Image}} &
\cellimg{original_img.jpg} &
\cellimg{mean_var_img.jpg}  &
\cellimg{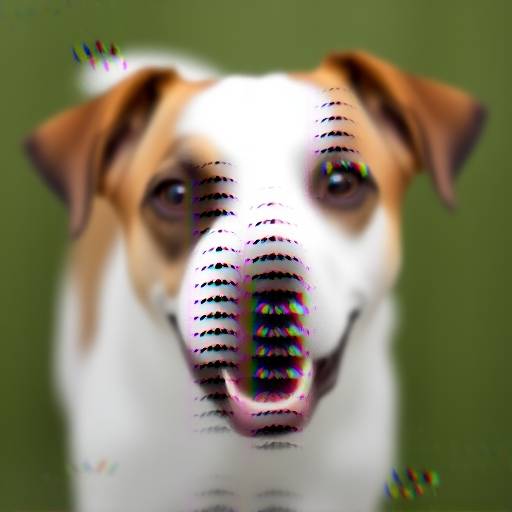}  &
\cellimg{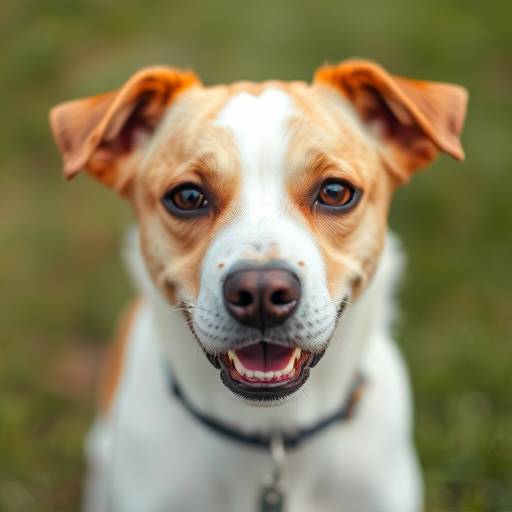} \\[-2px]
\bottomrule
\end{tabularx}

\caption{
\footnotesize
\textbf{Effect of confidence-interval (CI) projection in orthogonal transform domains.}
Starting from the output of 2OS and tile-wise constraints in the original coordinates, we additionally apply the same confidence-interval projections in a compact Fourier domain and then in a Hadamard-style mixing domain.
These transformed views expose complementary global structure, progressively suppressing residual artifacts and improving the sampled image.
}
\label{fig:whitening_domain_app}
\end{figure}

This effect is illustrated in \cref{fig:whitening_domain_app}.
After applying 2OS together with tile-wise constraints in the original coordinates, the latent is already substantially whitened, but the sampled image still contains severe global artifacts.
Applying the same confidence-interval projections in a compact Fourier domain further suppresses frequency-localized structure and yields a recognizable image, although visible artifacts remain.
Applying them again in a Hadamard-style mixing domain removes much of the remaining coherent structure and produces a valid sample.
Thus, orthogonal transform domains provide complementary global views that are not fully captured by coordinate-domain or tile-wise constraints alone.

The validity of this strategy rests on a simple geometric fact: Euclidean projection is preserved under orthogonal changes of coordinates.

\begin{proposition}
\label{prop:isometry_projection_app}
Let $T:\mathbb{R}^N\to\mathbb{R}^N$ be orthogonal, i.e.,
\begin{equation}
T^\top T=\bm{I}.
\end{equation}
For any nonempty closed set $\mathcal{S}\subset\mathbb{R}^N$, define
\begin{equation}
T(\mathcal{S})
=
\{T\bm{y}:\bm{y}\in\mathcal{S}\}.
\end{equation}
Then, for any $\bm{x}\in\mathbb{R}^N$,
\begin{equation}
T\,\Pi_{\mathcal{S}}(\bm{x})
\in
\Pi_{T(\mathcal{S})}(T\bm{x}),
\label{eq:isometry_proj_commute_app}
\end{equation}
equivalently,
\begin{equation}
\Pi_{\mathcal{S}}(\bm{x})
\in
T^\top \Pi_{T(\mathcal{S})}(T\bm{x}).
\label{eq:isometry_proj_commute_inv_app}
\end{equation}
Hence, projecting in the transformed domain and mapping back is exactly the Euclidean projection onto the corresponding pullback set in the original domain.
\end{proposition}

\begin{proof}
Since $T$ is orthogonal,
\begin{equation}
\|T\bm{u}-T\bm{v}\|_2
=
\|\bm{u}-\bm{v}\|_2
\qquad
\forall \bm{u},\bm{v}\in\mathbb{R}^N.
\end{equation}
Let $\bm{y}^\star\in\Pi_{\mathcal{S}}(\bm{x})$.
Then for every $\bm{y}\in\mathcal{S}$,
\begin{equation}
\|T\bm{x}-T\bm{y}^\star\|_2
=
\|\bm{x}-\bm{y}^\star\|_2
\le
\|\bm{x}-\bm{y}\|_2
=
\|T\bm{x}-T\bm{y}\|_2.
\end{equation}
Hence $T\bm{y}^\star$ is a minimizer of the distance from $T\bm{x}$ over $T(\mathcal{S})$, proving \cref{eq:isometry_proj_commute_app}.
Applying $T^\top$ yields \cref{eq:isometry_proj_commute_inv_app}.
\end{proof}

A second fact is that the standard Gaussian reference distribution is itself invariant under orthogonal transforms.

\begin{proposition}
\label{prop:gaussian_invariant_app}
If $T^\top T=\bm{I}$ and $\bm{\epsilon}\sim\mathcal{N}(\bm{0},\bm{I}_N)$, then
\begin{equation}
T\bm{\epsilon}\sim\mathcal{N}(\bm{0},\bm{I}_N).
\label{eq:gaussian_invariant_app}
\end{equation}
Consequently, confidence intervals derived under i.i.d.\ standard Gaussian assumptions remain valid after applying $T$.
\end{proposition}

\begin{proof}
Since $T$ is linear, $T\bm{\epsilon}$ is Gaussian with mean $\bm{0}$ and covariance
\begin{equation}
T\bm{I}_N T^\top
=
\bm{I}_N.
\end{equation}
Therefore $T\bm{\epsilon}\sim\mathcal{N}(\bm{0},\bm{I}_N)$.
\end{proof}

\paragraph{Compact orthogonal Fourier domain.}
Our first transformed view is a compact real representation of the orthonormal discrete Fourier transform.
For an even-length real vector $\bm{x}\in\mathbb{R}^N$, let
\begin{equation}
\hat{\bm{x}}
=
\mathcal{F}_{\mathrm{r}}(\bm{x})
\in
\mathbb{C}^{N/2+1},
\end{equation}
where $\mathcal{F}_{\mathrm{r}}$ denotes the orthonormal real Fourier transform.
We then define the compact real transform
\begin{equation}
\mathcal{T}_{\mathrm{F}}(\bm{x})
=
\Bigl[
\Re(\hat{x}_0),\,
\sqrt{2}\Re(\hat{\bm{x}}_{\mathrm{int}})^{\top},\,
\Re(\hat{x}_{N/2}),\,
\sqrt{2}\Im(\hat{\bm{x}}_{\mathrm{int}})^{\top}
\Bigr]^{\top}
\in
\mathbb{R}^{N},
\label{eq:compact_fourier_transform_app}
\end{equation}
where $
\hat{\bm{x}}_{\mathrm{int}}
=
(\hat{x}_1,\dots,\hat{x}_{N/2-1})
$
collects the interior Fourier coefficients.
That is, we keep the two real boundary frequencies explicitly and pack the independent interior Fourier coefficients by concatenating their real and imaginary parts.
The factor $\sqrt{2}$ on the interior frequencies accounts for the Hermitian symmetry of the Fourier transform of a real signal and preserves the Euclidean norm.
Consequently, $\mathcal{T}_{\mathrm{F}}$ is an orthogonal linear bijection from $\mathbb{R}^N$ to $\mathbb{R}^N$.
Its inverse simply reconstructs the Hermitian-symmetric Fourier coefficients and then applies the inverse real FFT.

The purpose of this domain is to expose structure that is more naturally described in frequency space.
Periodic artifacts, low-frequency bias, and anomalous spectral concentration may remain difficult to suppress when viewed only through local statistics in the original coordinates.
In the compact Fourier domain, the same CI projections from \cref{app:whitening_2os,app:whitening_tile_mean_energy} directly constrain those frequency-domain coordinates while remaining statistically consistent with the same standard Gaussian distribution.

\paragraph{Hadamard-style mixing domain.}
Our second transformed view is a sparse orthogonal mixing operator based on repeated pairwise sum/difference butterflies.
Let a tile of size $H\times W$ be flattened into a vector $\bm{z}\in\mathbb{R}^{D}$ with $D=HW$.
A single butterfly step maps each adjacent pair according to
\begin{equation}
(z_{2k-1},z_{2k})
\mapsto
\left(
\frac{z_{2k-1}+z_{2k}}{\sqrt{2}},
\frac{z_{2k-1}-z_{2k}}{\sqrt{2}}
\right),
\qquad
k=1,\dots,\frac{D}{2}.
\label{eq:hadamard_butterfly_app}
\end{equation}
Applied independently across all tiles, this defines an orthogonal transform.
Repeated application produces a Hadamard-style mixing effect that progressively redistributes local information across coordinates.

The role of this domain is different from the Fourier domain.
Whereas the compact Fourier transform makes frequency-localized anomalies explicit, the Hadamard-style mixing domain exposes coherent global structure through orthogonal linear combinations of coordinates.
In practice, this complementary view is important: after Fourier-domain filtering, some residual artifacts can still remain spatially coherent, and Hadamard-style mixing makes them easier to detect and suppress by the same CI projections.

In both transformed domains, the projection principle itself does not change:
we transform the signal, apply the CI projections developed in \cref{app:whitening_2os,app:whitening_tile_mean_energy}, and map the result back.
Because the transforms are orthogonal and preserve the standard Gaussian law, these projections continue to target the same noise-compatible regime of typical standard Gaussian noise.

\subsection{Component-Level Ablation of the Whitening Operator}
\label{app:whitening_component_ablation}

\cref{tab:whitening_component_ablation} examines the contribution of each component of $\mathcal{W}$ to reward alignment performance.
Starting from the 2OS projection alone, we progressively add tile-wise statistics and then multi-domain projections, holding the sampling configuration fixed throughout (NFE\,=\,25).

\begin{figure*}[h]
\centering

\newlength{\whiteningcomppairheight}
\setlength{\whiteningcomppairheight}{5.5cm}

\begin{minipage}[t][\whiteningcomppairheight][t]{0.63\textwidth}
    \centering
    \captionof{table}{
    \footnotesize
    \textbf{Component-level ablation of the whitening operator $\mathcal{W}$ on aesthetic image generation.}
    We fix the sampling configuration (NFE = 25) and progressively add components to the whitening operator.
    Dark green denotes the best target reward; light green denotes the second best.
    }
    \label{tab:whitening_component_ablation}

    \vspace{0.2em}
    \renewcommand{\arraystretch}{1.10}
    \setlength{\tabcolsep}{4.5pt}
    \scriptsize

    \resizebox{\linewidth}{!}{%
    \begin{tabular}{l c c c c}
    \toprule
    \multirow{2}{*}{\textbf{Whitening Operator}}
    & \multicolumn{1}{c}{\makecell{\textbf{Target}\\\textbf{Reward}}}
    & \multicolumn{3}{c}{\textbf{Held-out Rewards}} \\
    \cmidrule(lr){2-2}\cmidrule(lr){3-5}
    & \makecell{\textbf{Aesthetic}\\\textbf{Score}} $\uparrow$
    & \makecell{\textbf{Pick}\\\textbf{-score}} $\uparrow$
    & \textbf{HPSv2} $\uparrow$
    & \makecell{\textbf{Image}\\\textbf{Reward}} $\uparrow$ \\
    \midrule
    Base & \textbf{6.0282} & 0.2144 & 0.2759 & 1.0538 \\
    \midrule
    2OS proj. & \textbf{7.3439} & 0.2179 & 0.2891 & 1.1430 \\
    $+$ Tile-wise stat. & \cellcolor{darkgreen}\textbf{7.4607} & 0.2181 & 0.2909 & 1.0585 \\
    $+$ Multi. domain (Ours) & \cellcolor{lightgreen}\textbf{7.4510} & 0.2200 & 0.2928 & 1.2565 \\
    \bottomrule
    \end{tabular}%
    }
\end{minipage}
\hfill
\begin{minipage}[t][\whiteningcomppairheight][b]{0.34\textwidth}
    \centering
    \includegraphics[width=\linewidth]{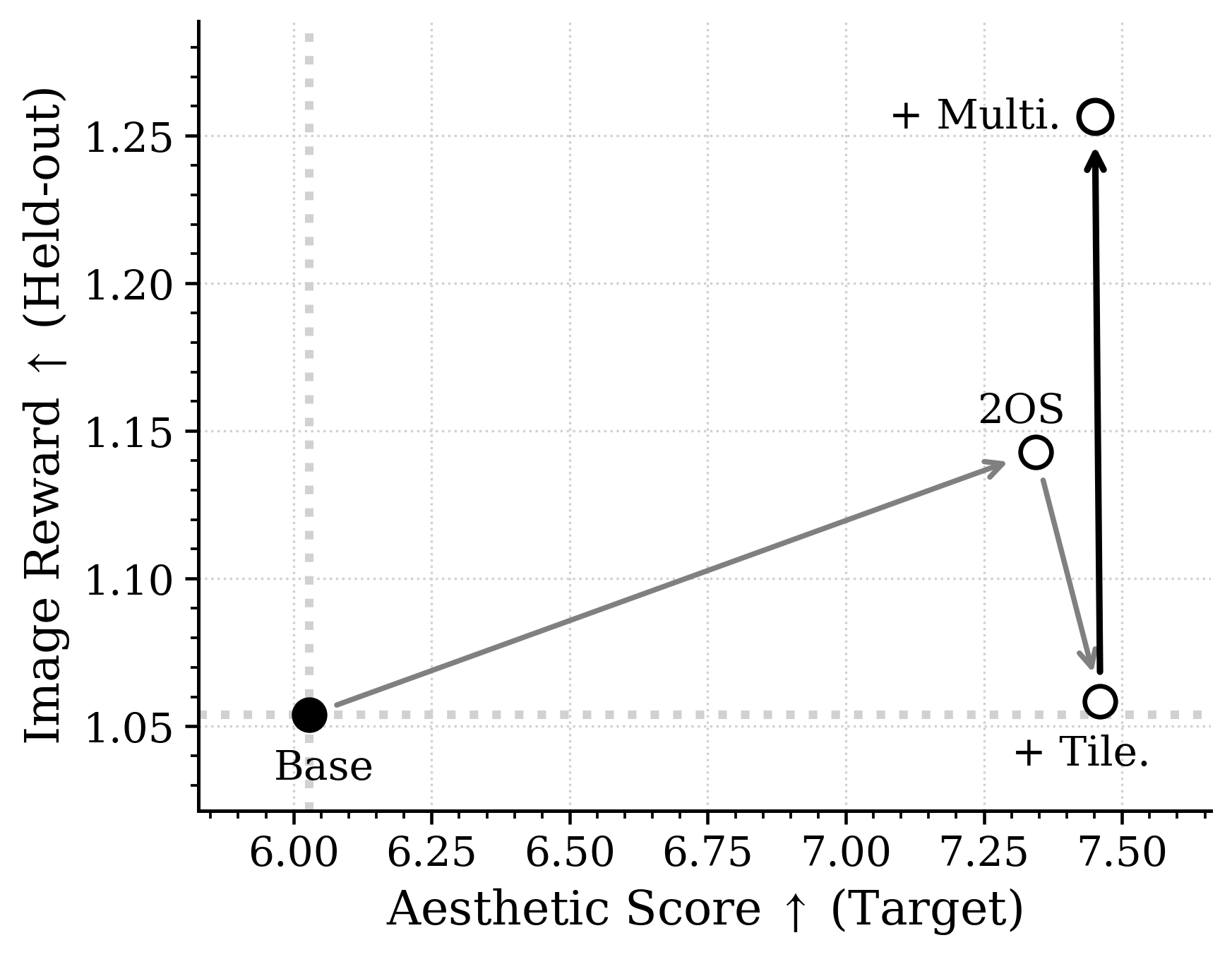}
    \vspace{-0.7cm}

    \captionof{figure}{
    \footnotesize
    \textbf{Trade-off under component ablation.}
    Full components achieve the best trade-off.
    }
    \label{fig:whitening_component_ablation}
\end{minipage}

\end{figure*}

The 2OS projection alone already provides a substantial gain in the target reward over the unguided base.
Adding tile-wise statistics yields the highest target reward among the three configurations, but reduces ImageReward, reflecting a trade-off between target optimization and distributional balance across rewards.
Incorporating multi-domain projections restores the held-out metrics to their highest values while preserving most of the target-reward gain, achieving the best overall trade-off.
These results suggest that each component addresses a complementary aspect of noise compatibility: 2OS controls local value statistics, tile-wise statistics regulate aggregate per-tile behavior, and multi-domain projections target residual structure that is more naturally exposed in alternative orthogonal bases.
\cref{fig:whitening_component_ablation} further visualizes the target-vs-held-out trade-off as components are added.

\subsection{Implementation and Optimization of the Whitening Operator}
\label{app:whitening_efficiency}

The whitening operator admits many possible configurations, since it is determined by the combination of the 2OS chunk partitions in \cref{app:whitening_2os}, the tile partitions for mean and centered energy in \cref{app:whitening_tile_mean_energy}, and the orthogonal transform domains in \cref{app:whitening_orthogonal}.
In practice, however, the computational bottleneck is clear: the dominant cost comes from the repeated sort--clip--unsort pattern inside the 2OS projection.
By contrast, the tile-wise mean and centered-energy updates are comparatively cheap, since they consist only of reductions, rescaling, and broadcasting.
Accordingly, our implementation is designed to make the 2OS projections as GPU-friendly as possible.

\paragraph{Optimization.}
Our first optimization is to keep the 2OS chunk sizes small.
In practice, we use chunk sizes such as $4$ and $64$, so that the inner sorting dimension remains modest.
When the chunk size is $4$, the sorting can be handled essentially at thread-local scale.
When the chunk size is $64$, it still fits comfortably within a single GPU block and can be processed efficiently in shared memory.
This matters because each 2OS projection repeatedly sorts within chunks and then across chunks at a fixed rank index.
By keeping the chunk size small, we keep these repeated sorting operations efficient even when the whitening operator is applied many times.

Our second optimization is to precompute all confidence bounds once at initialization.
For each configuration, we cache the Gaussian and chi-square thresholds for all relevant tuples of outer-rank count, chunk size, and tile size.
As a result, the online cost of each whitening call is dominated by tensor reshaping, sorting, clipping, FFTs, and orthogonal mixing, rather than by repeated evaluation of Beta, Gaussian, or chi-square quantiles.
In addition, because the same sort--clip--unsort pattern is reused across many nested calls, temporary tensors and indexing patterns can also be reused efficiently on GPU.

\paragraph{Configurations.}
For FLUX latents with shape $[1024,64]$, we use 2OS chunk sizes $(2,2)$ and $(8,8)$, tile-wise statistic sizes $(1,1)$, $(2,2)$, and $(8,8)$, and two Hadamard-style mixing domains: the full latent resolution and a $4\times4$ tiled mixing, repeated $32$ and $8$ times, respectively.
For the larger Wan2.1 latent with shape $[16,13,60,104]$, reshaped as $(16\!\cdot\!13\!\cdot\!60,\,104)$, we use chunk sizes $(2,2)$ and $(16,4)$, tile-wise statistic sizes $(2,2)$ and $(13,13)$, and the same two mixing domains, repeated $42$ and $8$ times.
In all cases, we use confidence level $\alpha=10^{-4}$.

\begin{table}[h]
\centering
\caption{
\footnotesize
\textbf{Wall-clock overhead of the whitening operator.}
}
\label{tab:whitening_runtime_app}
\setlength{\tabcolsep}{4.0pt}
\renewcommand{\arraystretch}{1.08}
\normalsize
\begin{tabular}{lcccc}
\toprule
\textbf{Model} & \makecell{\textbf{Per-step}\\\textbf{Model Time}} & \makecell{\textbf{Whitening}\\\textbf{Time}} & \makecell{\textbf{Whitening}\\\textbf{Slowdown}} & \makecell{\textbf{DPS}~\cite{Chung:2023DPS} $\rightarrow$ \Ours} \\
\midrule
FLUX~\cite{flux2024}    & 0.873\,s & 0.068\,s & +7.8\% & 87\,s $\rightarrow$ 94\,s \\
Wan2.1~\cite{wan2025} & 26.1\,s  & 0.5\,s   & +1.9\% & 650\,s $\rightarrow$ 663\,s \\
\bottomrule
\end{tabular}
\vspace{-0.8em}
\end{table}

\paragraph{Wall-clock time.}
These configurations were selected so that whitening accounts for less than $10\%$ of the overall reward-alignment cost.
For FLUX, one reward-guided model step takes about $0.87$\,s, while the whitening operator takes about $0.068$\,s, i.e., about $7.8\%$ of the per-step wall-clock time.
At the level of full runs, DPS~\cite{Chung:2023DPS} takes about $87$\,s, whereas adding whitening increases this only to $94$\,s.
For Wan2.1, one model step takes about $26$\,s, while whitening takes about $0.5$\,s, which is roughly $2\%$ of the runtime.
Correspondingly, DPS~\cite{Chung:2023DPS} takes about $650$\,s, and the whitened version takes about $663$\,s.
Thus, although the whitening operator combines several CI projections across multiple domains, its overhead remains modest relative to the underlying generative model.

An important practical advantage is that this cost can be adjusted continuously through the configuration.
Reducing the number of chunk sizes, tile sizes, transform domains, or mixing repetitions yields a cheaper operator, while adding them yields a stronger but slower one.
This makes the whitening operator easy to adapt to models with different latent sizes and runtime budgets.

\subsection{Comparison with Simpler Operators}
\label{app:whitening_comparison}

To understand why the quality of the whitening operator matters, we compare \Ours\ with three simpler alternatives.
The key question is not merely whether an operator removes visible structure, but whether it moves an arbitrary input toward the noise-compatible regime of typical standard Gaussian noise while leaving already-typical noise nearly unchanged.

We consider four whitening operators.
The first is the identity map,
\begin{equation}
\mathcal{W}_{\mathrm{None}}(\bm{z})=\bm{z},
\end{equation}
which corresponds to using the raw reward-informed direction without whitening.
The second is norm projection,
\begin{equation}
\mathcal{W}_{\mathrm{Norm}}(\bm{z})
=
\sqrt{N}\,\frac{\bm{z}}{\|\bm{z}\|_2},
\label{eq:norm_projection_app}
\end{equation}
which places the input on the high-probability hypersphere of radius $\sqrt{N}$, but does not constrain its finer statistics.
The third is $\mathcal{W}_{\mathrm{WGNC}}$, based on White Gaussian Noise Constraints (WGNC)~\cite{hwang2026gradient}, which projects onto a spectral feasible set defined by hard blockwise constraints in a compact Fourier domain.
This suppresses structured frequency artifacts much more strongly than norm projection, but its hard equalities remove the natural variability that genuine Gaussian noise should retain.
The last is our proposed operator, $\mathcal{W}_{\mathrm{Ours}}$.

\begin{figure*}[h]
\centering

\newlength{\whiteningpairheight}
\setlength{\whiteningpairheight}{5.7cm} 

\begin{minipage}[t][\whiteningpairheight][t]{0.65\textwidth}
    \centering
    \captionof{table}{
    \footnotesize
    \textbf{Ablation on the whitening operator for \Ours\ on aesthetic image generation.}
    We fix the sampling configuration (NFE = 25) and replace only the whitening operator.
    Dark green cells indicate the best result in the target reward, while light green cells denote the second best.
    }
    \label{tab:whitening_ablation_aesthetic}

    \vspace{0.2em}
    \renewcommand{\arraystretch}{1.10}
    \setlength{\tabcolsep}{4.5pt}
    \scriptsize

    \resizebox{\linewidth}{!}{%
    \begin{tabular}{l c c c c c}
    \toprule
    \multirow{2}{*}{\textbf{Setting}}
    & \multicolumn{1}{c}{\makecell{\textbf{Target}\\\textbf{Reward}}}
    & \multicolumn{4}{c}{\textbf{Held-out Rewards}} \\
    \cmidrule(lr){2-2}\cmidrule(lr){3-6}
    & \makecell{\textbf{Aesthetic}\\\textbf{Score}} $\uparrow$
    & \makecell{\textbf{Pick}\\\textbf{-score}} $\uparrow$
    & \textbf{HPSv2} $\uparrow$
    & \makecell{\textbf{Image}\\\textbf{Reward}} $\uparrow$
    & \makecell{\textbf{VQA}\\\textbf{Score}} $\uparrow$ \\
    \midrule

    Base~\cite{flux2024} & \textbf{6.0282} & 0.2144 & 0.2759 & 1.0538 & 0.9644 \\
    \Ours\ w/ $\mathcal{W}_{\text{None}}$ & \textbf{6.2392} & 0.2207 & 0.2981 & 1.2914 & 0.9690 \\
    \Ours\ w/ $\mathcal{W}_{\text{Norm}}$ & \textbf{7.1036} & 0.2147 & 0.2794 & 0.8908 & 0.9557 \\
    \Ours\ w/ $\mathcal{W}_{\text{WGNC}}$ & \cellcolor{lightgreen}\textbf{7.2096} & 0.2174 & 0.2901 & 1.1391 & 0.9672 \\
    \Ours\ w/ $\mathcal{W}_{\text{Ours}}$ & \cellcolor{darkgreen}\textbf{7.4510} & 0.2200 & 0.2928 & 1.2565 & 0.9728 \\

    \bottomrule
    \end{tabular}%
    }
\end{minipage}
\hfill
\begin{minipage}[t][\whiteningpairheight][b]{0.33\textwidth}
    \centering
    \includegraphics[width=\linewidth]{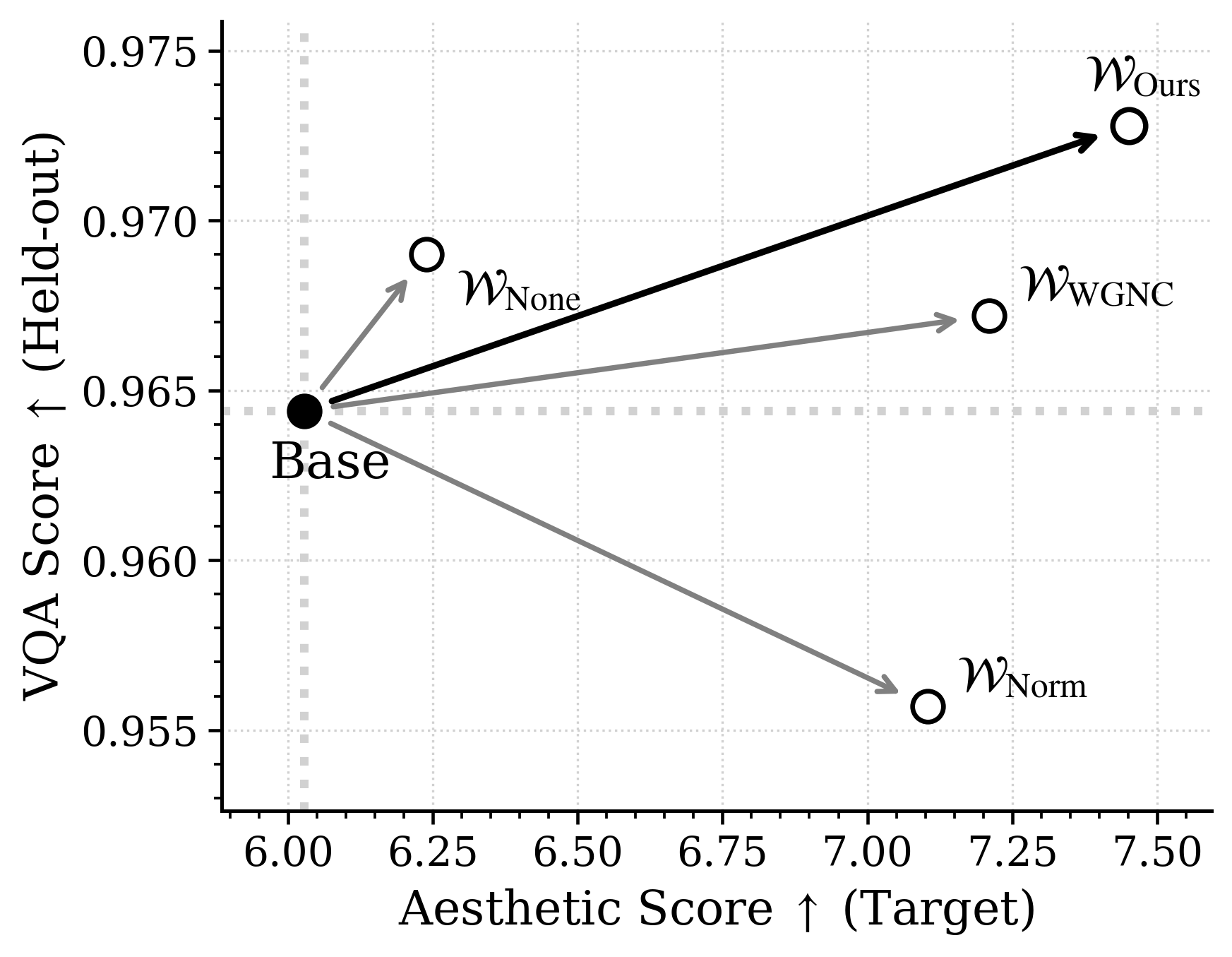}
    \vspace{-0.7cm}

    \captionof{figure}{
    \footnotesize
    \textbf{Trade-off under different whitening operators.}
    Only $\mathcal{W}_{\text{Ours}}$ improves the target reward without degrading the held-out quality.
    }
    \label{fig:whitening_ablation_tradeoff}
\end{minipage}

\end{figure*}

\cref{tab:whitening_ablation_aesthetic} shows an ablation in which we fix the \Ours\ sampling setting and replace only the whitening operator.
Using no whitening, i.e.\ $\mathcal{W}_{\mathrm{None}}$, improves the target reward only modestly.
Norm projection, $\mathcal{W}_{\mathrm{Norm}}$, increases the target reward further, but incurs the largest degradation in held-out quality.
$\mathcal{W}_{\mathrm{WGNC}}$ provides a better trade-off than norm projection and recovers held-out quality above the base level, yet still lags behind $\mathcal{W}_{\mathrm{Ours}}$ on both target reward and held-out metrics.
In contrast, $\mathcal{W}_{\mathrm{Ours}}$ achieves the highest target reward while also preserving, and in this case slightly improving, the held-out metrics.
Therefore, the benefit of \Ours\ does not come merely from making the guidance stronger; it comes from injecting reward information through a perturbation that remains much more faithful to the statistics of typical standard Gaussian noise.

\cref{fig:whitening_comparison_app} compares the four operators on two representative inputs.
In the top example of \cref{fig:whitening_comparison_app}, the input is strongly structured and leads to a severe failure case.
Applying $\mathcal{W}_{\mathrm{None}}$ or $\mathcal{W}_{\mathrm{Norm}}$ leaves the dominant structure largely intact, so the sampled image remains invalid.
$\mathcal{W}_{\mathrm{WGNC}}$ removes much more of the visible structure, but still produces a noticeably distorted sample.
By contrast, $\mathcal{W}_{\mathrm{Ours}}$ produces a much more typical noise and yields a realistic image.
In the bottom example of \cref{fig:whitening_comparison_app}, the input is already a typical Gaussian noise.
Here the desired behavior is the opposite: a good whitening operator should change the input as little as possible.
Indeed, $\mathcal{W}_{\mathrm{Ours}}$ leaves the latent almost unchanged, whereas $\mathcal{W}_{\mathrm{WGNC}}$ still perturbs it visibly.
\Cref{fig:whitening_comparison_app} illustrates the main difference between the two approaches: \Ours\ removes atypical structure when necessary, but preserves genuine noise when it is already compatible with the pretrained model.

Overall, the quantitative results above and \cref{fig:whitening_comparison_app} together show that a good whitening operator is crucial for reward alignment.
Operators that are too weak fail to remove harmful structure, whereas operators that are too rigid distort even already-typical noise.
The proposed whitening operator is effective precisely because it balances these two requirements: it strongly suppresses atypical structure when needed, while minimally altering inputs that already lie near the noise-compatible regime.

\begin{figure}[t]
\centering
\setlength{\tabcolsep}{1pt}
\renewcommand{\arraystretch}{0.5}
\small

\newcommand{\cellimg}[1]{\includegraphics[width=\linewidth]{Figures/whitening_process_app/#1}}
\newcommand{\cellimgp}[1]{\includegraphics[width=\linewidth]{Figures/pure_pass/#1}}

\newcolumntype{C}{>{\centering\arraybackslash}X}

\begin{tabularx}{\linewidth}{@{} c *{1}{C} | *{3}{C} @{}}
\toprule
& $\bm{z}$ ($=\mathcal{W}_{\text{None}}(\bm{z})$) & \makecell{$\mathcal{W}_{\text{Norm}}(\bm{z})$} & \makecell{$\mathcal{W}_{\text{WGNC}}(\bm{z})$} & \makecell{$\mathcal{W}_{\text{Ours}}(\bm{z})$} \\

\midrule
\raisebox{3.0ex}{\rotatebox{90}{Latent Visualization}} &
\cellimg{original.png} &
\cellimg{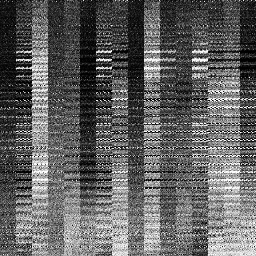} &
\cellimg{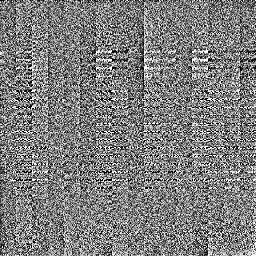} &
\cellimg{ours.png} \\
\raisebox{5.5ex}{\rotatebox{90}{Sampled Image}} &
\cellimg{original_img.jpg} &
\cellimg{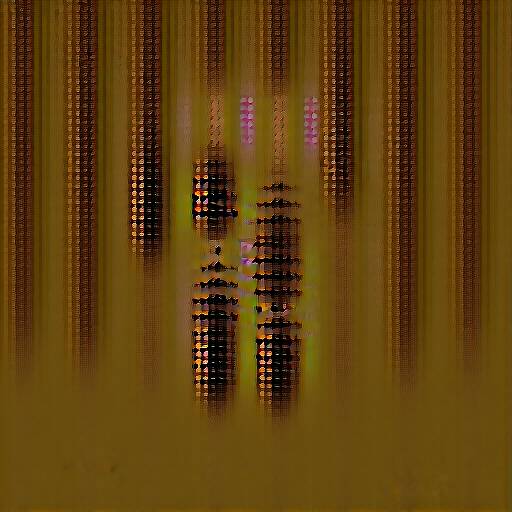}  &
\cellimg{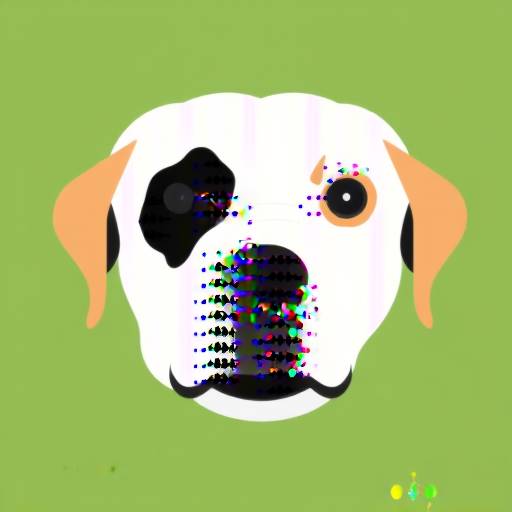}  &
\cellimg{ours_img.jpg} \\[3px]
& CosSim($\bm{z}, \mathcal{W}(\bm{z})$)
& 1.000000
& 0.642806
& 0.480325 \\[3px]
& $\|\bm{z} - \mathcal{W}(\bm{z})\|_2$
& 37.63
& 203.35
& 244.39 \\
\bottomrule
\toprule
& $\bm{z}$ ($=\mathcal{W}_{\text{None}}(\bm{z})$) & \makecell{$\mathcal{W}_{\text{Norm}}(\bm{z})$} & \makecell{$\mathcal{W}_{\text{WGNC}}(\bm{z})$} & \makecell{$\mathcal{W}_{\text{Ours}}(\bm{z})$} \\
\midrule
\raisebox{3.0ex}{\rotatebox{90}{Latent Visualization}} &
\cellimgp{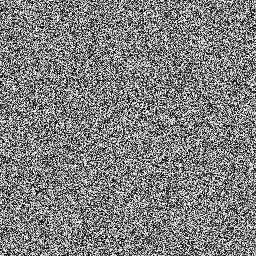} &
\cellimgp{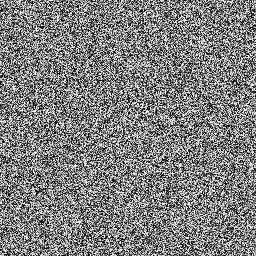} &
\cellimgp{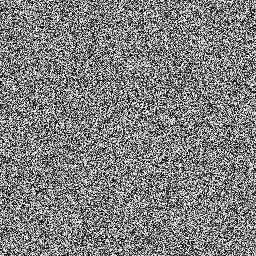} &
\cellimgp{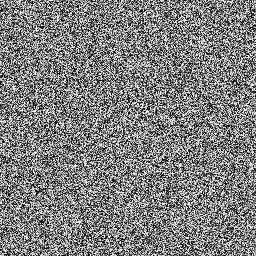} \\
\raisebox{5.5ex}{\rotatebox{90}{Sampled Image}} &
\cellimgp{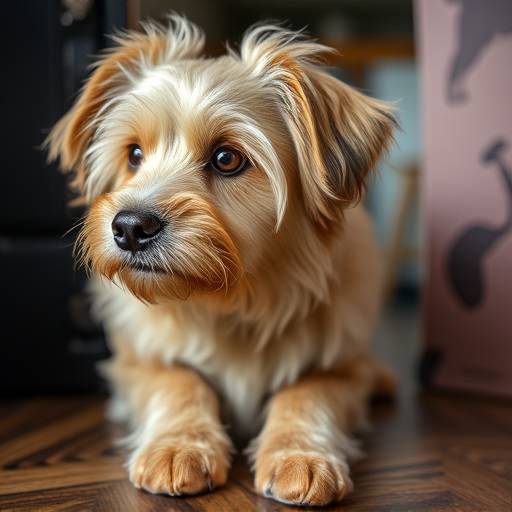} &
\cellimgp{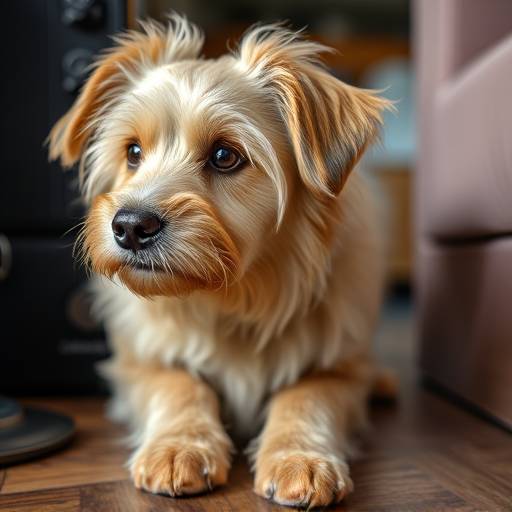}  &
\cellimgp{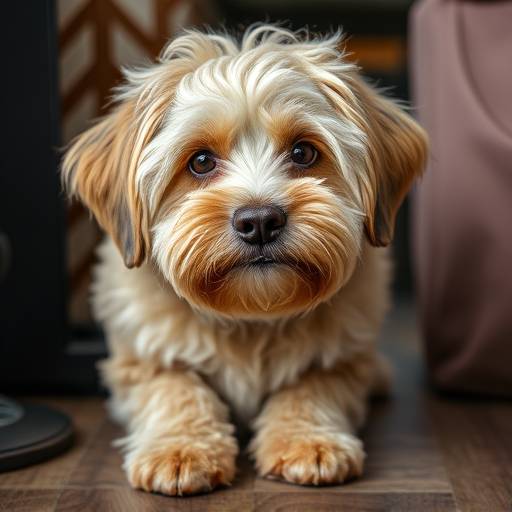}  &
\cellimgp{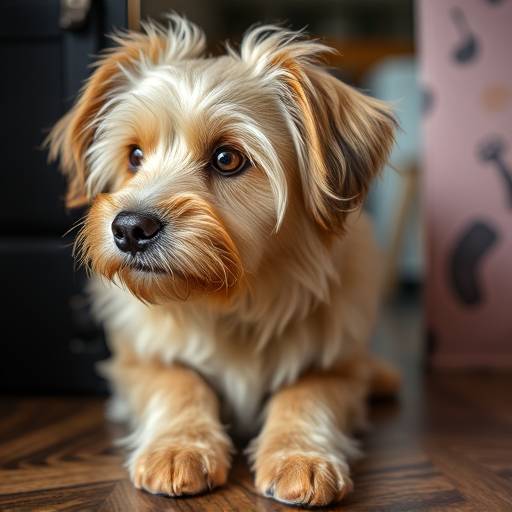} \\[3px]
& CosSim($\bm{z}, \mathcal{W}(\bm{z})$)
& 1.000000
& 0.989371
& 0.999996 \\[3px]
& $\|\bm{z} - \mathcal{W}(\bm{z})\|_2$
& 1.98
& 37.24
& 0.77 \\
\bottomrule
\end{tabularx}

\caption{
\footnotesize
\textbf{Comparison with simpler whitening operators.}
Top: when the input latent $\bm{z}$ is strongly atypical, $\mathcal{W}_{\mathrm{None}}$ and $\mathcal{W}_{\mathrm{Norm}}$ fail to remove the dominant structure, while $\mathcal{W}_{\mathrm{WGNC}}$ suppresses it more strongly but still yields a distorted sample.
In contrast, $\mathcal{W}_{\mathrm{Ours}}$ produces a much more typical noise and a realistic image.
Bottom: when $\bm{z}$ is already close to typical Gaussian noise, a good whitening operator should leave it nearly unchanged.
Here, $\mathcal{W}_{\mathrm{Ours}}$ preserves the input most faithfully, whereas $\mathcal{W}_{\mathrm{WGNC}}$ still perturbs it visibly.
The cosine similarity and $\ell_2$ distance quantify this trade-off: \Ours\ changes atypical inputs when needed, but minimally alters already-typical noise.
}
\label{fig:whitening_comparison_app}
\end{figure}

\clearpage
\newpage

\clearpage
\newpage
\section{Additional Experimental Results}
\label{app:additional_experiments}

In this section, we present additional experimental results and implementation details that complement the experiments in \cref{sec:exp}. We first summarize the implementation details across applications in~\cref{app:imple_details}. We then compare with noise optimization methods in~\cref{app:dno_comparison}. We further provide results on additional applications, including quantity-aware generation and VLM-based reward alignment, in~\cref{app:counting,app:vlm_reward_alignment}. We further report results with a different base model, Z-Image~\cite{cai2025:zimage}, for aesthetic and text-aligned image generation in~\cref{app:zimage_results}. Finally, we present additional results for the main experiments in~\cref{app:additional_main}.

\subsection{Implementation Details}
\label{app:imple_details}

\begin{wraptable}{r}{0.46\textwidth}
\vspace{-1.0\baselineskip}
\centering
\caption{
\footnotesize
\textbf{Hyperparameter configurations for each application.}
}
\label{tab:hyperparameters}
\renewcommand{\arraystretch}{1.05}
\setlength{\tabcolsep}{3pt}
\footnotesize
\begin{tabular}{@{} l ccccc @{}}
\toprule
& \makecell[c]{\textbf{Aesthetic}\\\textbf{Image}}
& \makecell[c]{\textbf{Text-}\\\textbf{Aligned}}
& \makecell[c]{\textbf{Quantity-}\\\textbf{Aware}}
& \textbf{VLM}
& \textbf{Video} \\
\midrule
$\sigma_t$ & $0.2t$ & $0.2t$ & $0.2$ & $0.2$ & $0.2$ \\
$\rho_t$   & $0.3$  & $0.3$  & $0.5$ & $0.5$ & $0.5$ \\
\bottomrule
\end{tabular}
\vspace{-0.8\baselineskip}
\end{wraptable}
We summarize the implementation details across applications in this section. For image generation tasks, we use FLUX~\cite{flux2024} and Z-Image~\cite{cai2025:zimage} as the base models for aesthetic image generation and text-aligned image generation, and FLUX~\cite{flux2024} for quantity-aware generation and VLM-based reward alignment. For preference-aligned video generation, we use Wan2.1~\cite{wan2025} as the base flow model. Across all experiments, we fix the number of sampling steps to 25. The diffusion coefficient $\sigma_t$ and the whitened guidance strength $\rho_t$ used for each application are summarized in~\cref{tab:hyperparameters}.

\subsection{Comparison with Noise Optimization}
\label{app:dno_comparison}

\begin{wraptable}{r}{0.56\textwidth}
\vspace{-1.0\baselineskip}
\centering
\caption{
\footnotesize
\textbf{Comparison with DNO~\cite{tang2025dno} under the same NFE budget.} Aesthetic image generation at 500 NFE.
}
\label{tab:dno_comparison}
\vspace{-0.2\baselineskip}
\renewcommand{\arraystretch}{1.10}
\setlength{\tabcolsep}{5.0pt}
\footnotesize
\begin{tabular}{@{} l c ccc @{}}
\toprule
\multirow{2}{*}[-1.4ex]{\textbf{Method}} & \multirow{2}{*}[-1.4ex]{\textbf{NFE}} & \textbf{Target Reward} & \multicolumn{2}{c}{\textbf{Held-Out Reward}} \\
\cmidrule(lr){3-3}\cmidrule(lr){4-5}
& & \makecell{\textbf{Aesthetic}\\\textbf{Score}} $\uparrow$ & \makecell{\textbf{Pick-}\\\textbf{Score}} $\uparrow$ & \textbf{HPSv2} $\uparrow$ \\
\midrule
Base                                & 25  & 6.028 & 0.2144 & 0.2759 \\[0.5em]
\makecell[l]{DNO~\cite{tang2025dno}\\[-0.2em]{\scriptsize(20 updates)}} & 500 & 6.772 & 0.2143 & 0.2717 \\
\makecell[l]{\textbf{\Ours{}}\\[-0.2em]{\scriptsize(20 particles)}} & 500 & 7.966 & 0.2197 & 0.2932 \\
\bottomrule
\end{tabular}
\vspace{-0.8\baselineskip}
\end{wraptable}

DNO~\cite{tang2025dno} belongs to a distinct class of methods: rather than modifying the per-step reverse kernel, it treats all injected noise vectors as inference-time optimization variables and refines them through iterative Adam updates with a regularization term.
This design has a fundamental budgetary consequence: each Adam update requires a complete forward pass through all $T$ denoising steps, so generating a single sample with $K$ updates costs $T \times K$ NFEs.
At the standard 25-NFE budget ($K{=}0$), DNO produces output identical to base sampling since no optimization has yet taken place.

We compare DNO and \Ours{} under a matched budget of 500 NFE on the aesthetic image generation task.
For DNO this corresponds to $K{=}20$ Adam updates over the full noise trajectory; for \Ours{} it corresponds to $N{=}20$ particles with Best-of-N.
As shown in~\cref{tab:dno_comparison}, \Ours{} achieves substantially higher aesthetic reward while better preserving held-out metrics, demonstrating that per-step kernel modification is a more budget-efficient strategy than trajectory-level noise optimization.
\FloatBarrier

\subsection{Additional Application: Quantity-Aware Generation}
\label{app:counting}

\paragraph{Experiment Setup.} We evaluate quantity-aware generation using 40 prompts from prior work, $\Psi$-Sampler~\cite{yoon2025:psi}, covering 20 object categories with randomly assigned target counts of up to 90. 
For all other experimental details, we follow the same setup described in \cref{subsec:exp_setup}.

\paragraph{Evaluation Metrics.} The target reward is negative smooth L1 loss computed with T2I-Count~\cite{qian2025:t2icount}, which takes a generated image and the corresponding
text prompt as input and returns a density map. For the held-out reward, we use alternative counting model, CountGD~\cite{amininaieni2025:countgd}, and report MAE and counting accuracy (\%), where a prediction is considered correct if the predicted count exactly matches the target count. As held-out rewards, we evaluate image quality using ImageReward~\cite{Xu2023:ImageReward} and HPSv2~\cite{wu2023:hpsv2}, and text–image alignment using VQA Score~\cite{lin2024:vqa}. For all methods, methods augmented with BoN~\cite{Stiennon:2020BoN} and $\Psi$-Sampler~\cite{yoon2025:psi} are marked with \textsuperscript{\dag} and \textsuperscript{$\ddagger$}, respectively. 

\paragraph{Results.}
The quantitative and qualitative results are presented in~\cref{tab:quantity_aware} and~\cref{fig:counting_qualitative}, respectively. Overall, \Ours{} achieves the best target reward performance by a large margin over all baselines. On held-out rewards, \Ours{} also attains the best counting performance in terms of both MAE and accuracy, while preserving image quality relative to the base model. Qualitatively, \Ours{} consistently generates the desired number of objects across diverse categories and target counts, whereas the baselines often generate either fewer or more objects than the target count.

\begin{table*}[t!]
\centering
\caption{
\footnotesize 
\textbf{Quantitative comparison on quantity-aware generation.} 
The target reward is T2I-Count~\cite{qian2025:t2icount}.
For single-particle methods, we augment sampling with Best-of-N~\cite{Stiennon:2020BoN} and $\Psi$-Sampler~\cite{yoon2025:psi} to match the total NFE, denoted with \textsuperscript{\dag} and \textsuperscript{$\ddagger$}, respectively. 
Dark green cells indicate the best result for each metric across all runs, while light green cells denote the second best. 
}
\label{tab:quantity_aware}
\renewcommand{\arraystretch}{1.10}
\setlength{\tabcolsep}{5.0pt}
\footnotesize
\resizebox{\textwidth}{!}{%
\begin{tabular}{@{} l c cccccc @{}}
\toprule
\multirow{2}{*}[-1.4ex]{\textbf{Method}} & \multirow{2}{*}[-1.4ex]{\textbf{NFE}} & \textbf{Target Reward} & \multicolumn{5}{c}{\textbf{Held-Out Reward}} \\
\cmidrule(lr){3-3} \cmidrule(lr){4-8}
& & \makecell{\textbf{T2I-Count}} $\downarrow$
& \textbf{MAE} $\downarrow$
& \makecell{\textbf{Acc}\\\textbf{(\%)}} $\uparrow$
& \makecell{\textbf{Image}\\\textbf{Reward}} $\uparrow$
& \textbf{HPSv2} $\uparrow$
& \makecell{\textbf{VQA}\\\textbf{Score}} $\uparrow$ \\
\midrule
Base~\cite{flux2024} & 25 & 16.365 & 16.700 & 5.0 & 0.579 & 0.262 & 0.928 \\
BoN~\cite{Stiennon:2020BoN} & 1000 & 2.668 & 4.925 & 22.5 & 0.642 & 0.269 & 0.904 \\
DPS\textsuperscript{\dag}~\cite{Chung:2023DPS} & 1000 & 2.565 & 4.900 & 15.0 & 0.716 & 0.268 & 0.925 \\
FreeDoM\textsuperscript{\dag}~\cite{Yu:2023FreeDOM} & 1045 & 5.210 & 7.925 & 20.0 & 0.303 & 0.259 & 0.931 \\
SVDD~\cite{Li2024:SVDD} & 1000 & \cellcolor{lightgreen}1.530 & \cellcolor{lightgreen}3.100 & 17.5 & 0.729 & \cellcolor{lightgreen}0.271 & \cellcolor{darkgreen}0.953 \\
RBF~\cite{kim2025:rbf} & 1000 & 1.796 & 4.025 & \cellcolor{lightgreen}30.0 & 0.646 & \cellcolor{lightgreen}0.271 & 0.936 \\
DAS~\cite{Kim:2025DAS} & 1000 & 2.061 & 3.875 & 27.5 & \cellcolor{lightgreen}0.789 & \cellcolor{lightgreen}0.271 & 0.950 \\
$\Psi$-Sampler~\cite{yoon2025:psi} & 1000 & 1.426 & 3.215 & \cellcolor{lightgreen}30.0 & \cellcolor{darkgreen}0.854 & \cellcolor{darkgreen}0.273 & \cellcolor{lightgreen}0.951 \\
\midrule
\textbf{\Ours\textsuperscript{$\ddagger$} (Ours)} & 1000 & \cellcolor{darkgreen}0.045 & \cellcolor{darkgreen}2.325 & \cellcolor{darkgreen}37.5 & 0.681 & 0.266 & 0.915 \\
\bottomrule
\end{tabular}%
}
\end{table*}
\begin{figure*}[t!]
\centering
\setlength{\tabcolsep}{0pt}
\renewcommand{\arraystretch}{1.03}
\scriptsize

\newcommand{\colw}{0.159\textwidth}
\newcommand{\imgw}{0.159\textwidth}

\begin{tabular}{@{}%
>{\centering\arraybackslash}m{\colw}
>{\centering\arraybackslash}m{\colw}
>{\centering\arraybackslash}m{\colw}
>{\centering\arraybackslash}m{\colw}
>{\centering\arraybackslash}m{\colw}
>{\centering\arraybackslash}m{\colw}
@{}}
\toprule

Base~\cite{flux2024} &
BoN~\cite{Stiennon:2020BoN} &
DPS\textsuperscript{\dag}~\cite{Chung:2023DPS} &
SVDD~\cite{Li2024:SVDD} &
$\Psi$-Sampler~\cite{yoon2025:psi} &
{\Oursbf{}}\textsuperscript{$\ddagger$} (\textbf{Ours}) \\

\midrule

\multicolumn{6}{@{}c@{}}{\textit{``\textbf{37} strawberries''}} \\[1pt]
\includegraphics[width=\imgw]{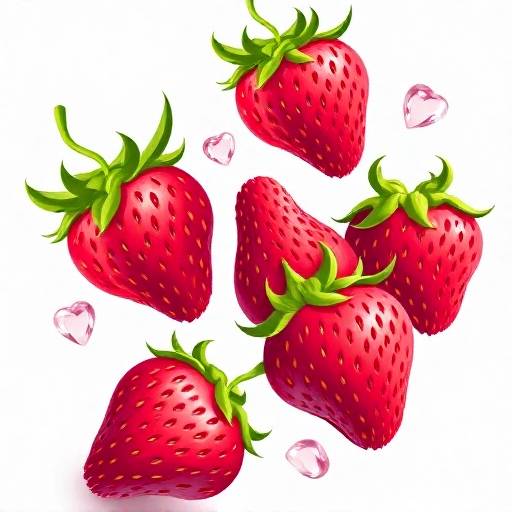} &
\includegraphics[width=\imgw]{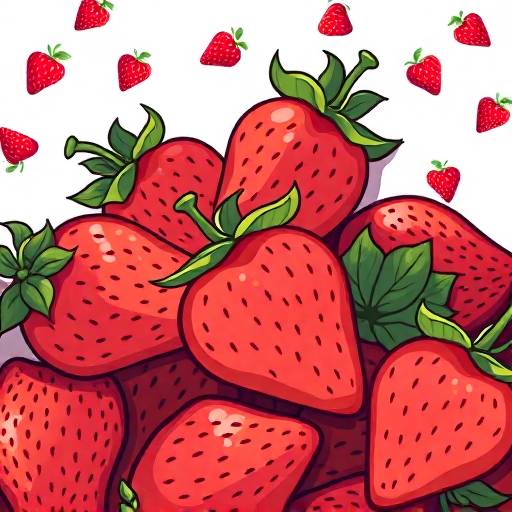} &
\includegraphics[width=\imgw]{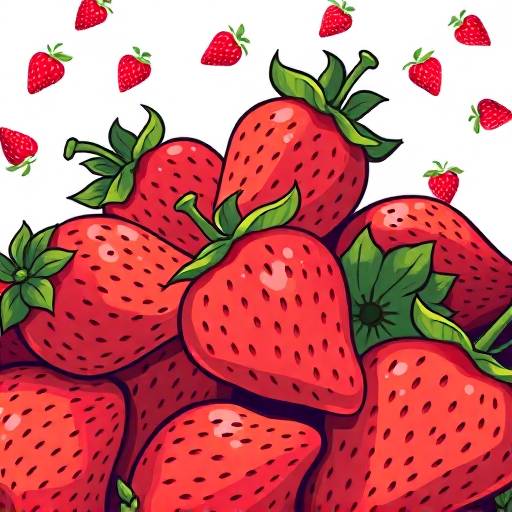} &
\includegraphics[width=\imgw]{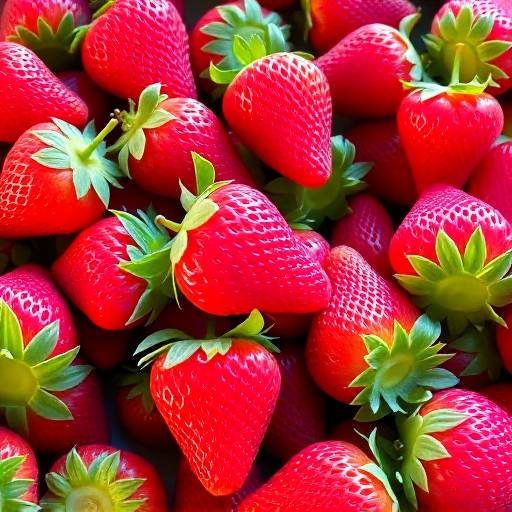} &
\includegraphics[width=\imgw]{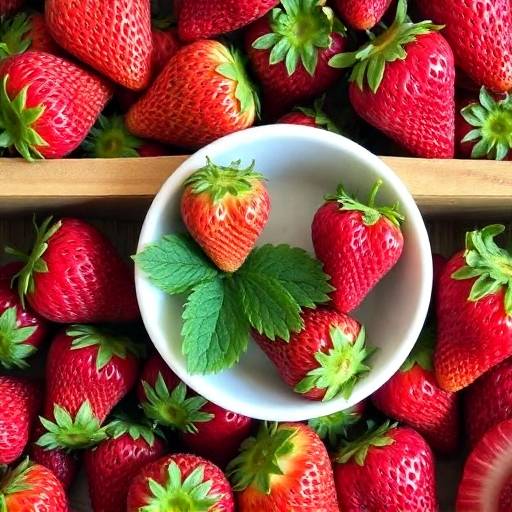} &
\includegraphics[width=\imgw]{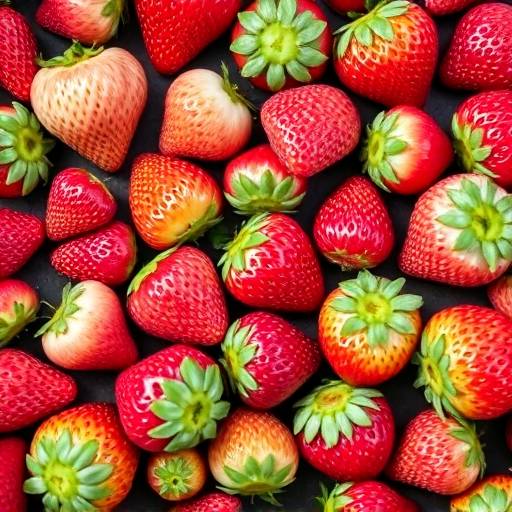} \\
\includegraphics[width=\imgw]{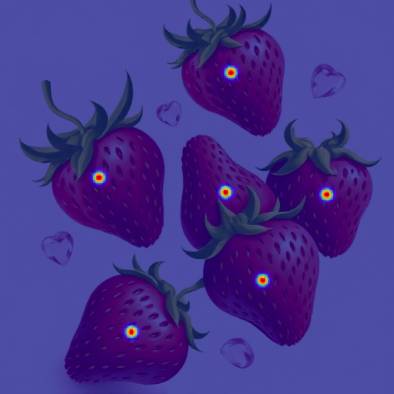} &
\includegraphics[width=\imgw]{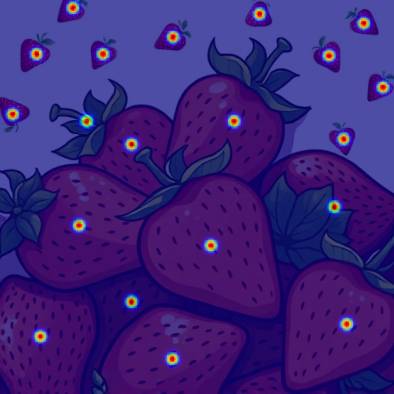} &
\includegraphics[width=\imgw]{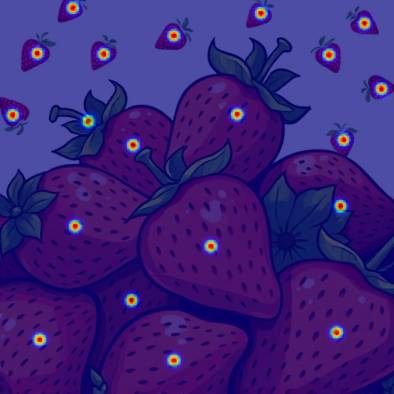} &
\includegraphics[width=\imgw]{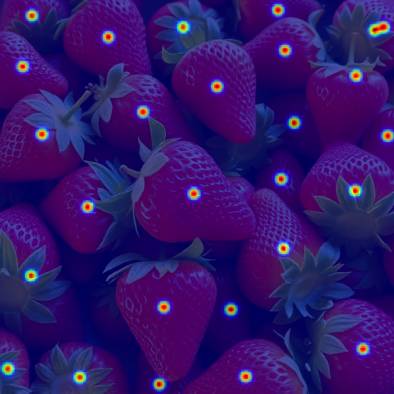} &
\includegraphics[width=\imgw]{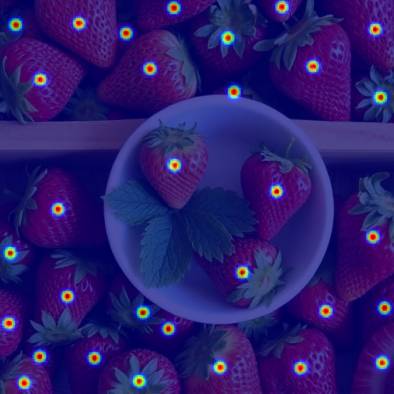} &
\includegraphics[width=\imgw]{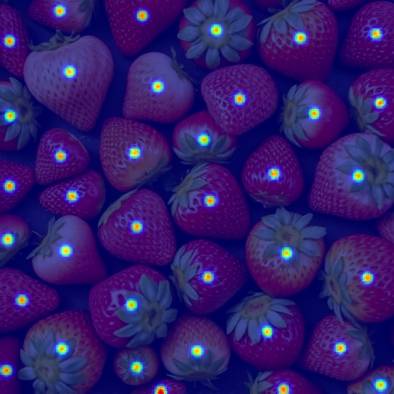} \\
6 ($\Delta 31$) &
20 ($\Delta 17$) &
20 ($\Delta 17$) &
27 ($\Delta 10$) &
31 ($\Delta 6$) &
37 ($\Delta 0$) \\[4pt]

\multicolumn{6}{@{}c@{}}{\textit{``\textbf{22} grapes''}} \\[1pt]
\includegraphics[width=\imgw]{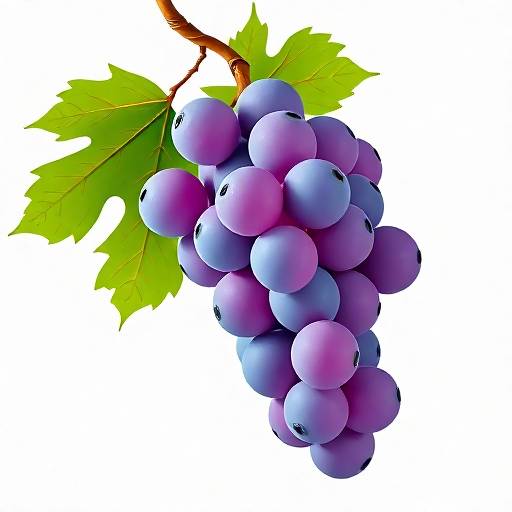} &
\includegraphics[width=\imgw]{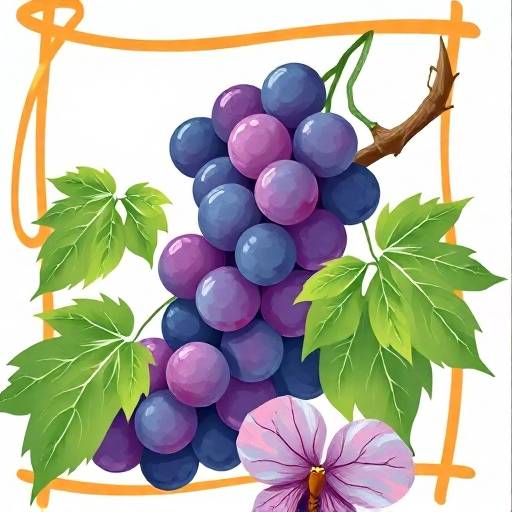} &
\includegraphics[width=\imgw]{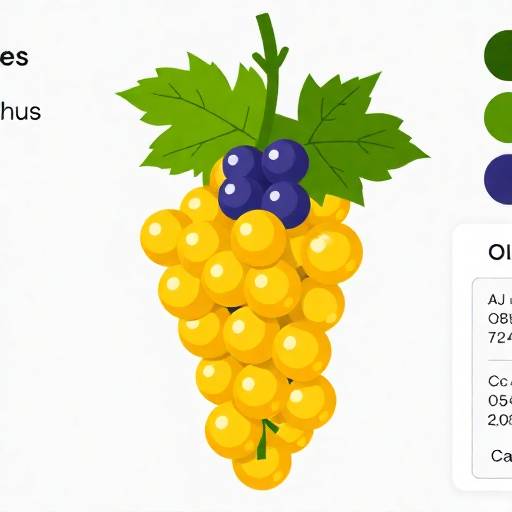} &
\includegraphics[width=\imgw]{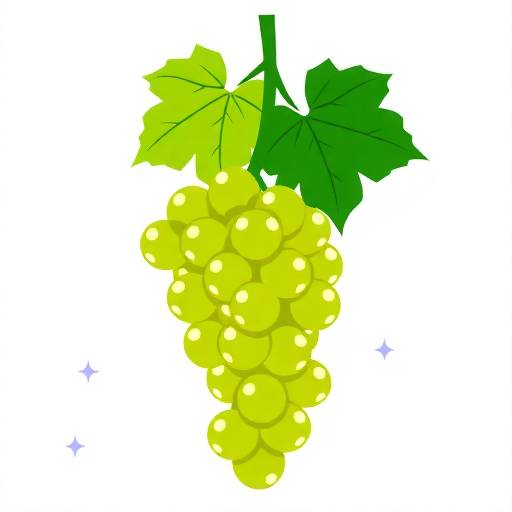} &
\includegraphics[width=\imgw]{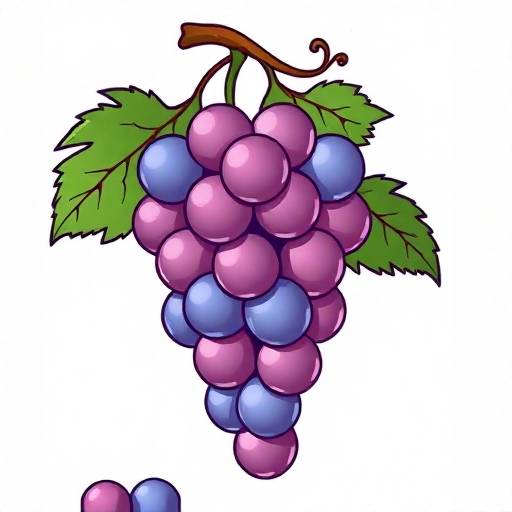} &
\includegraphics[width=\imgw]{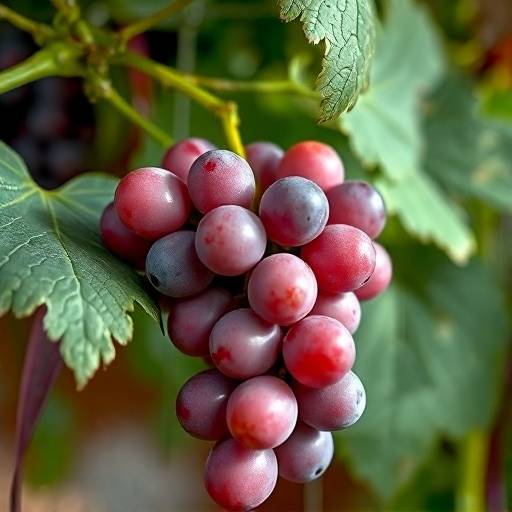} \\
\includegraphics[width=\imgw]{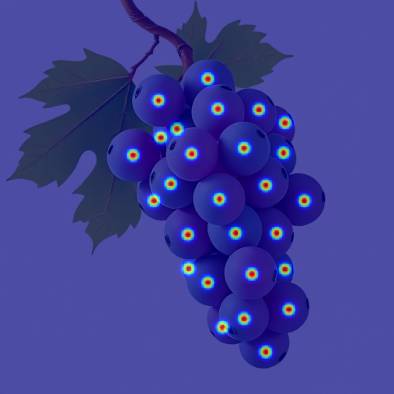} &
\includegraphics[width=\imgw]{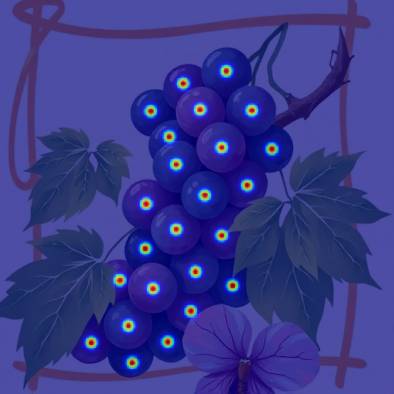} &
\includegraphics[width=\imgw]{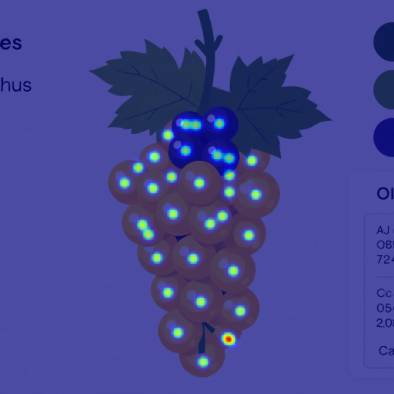} &
\includegraphics[width=\imgw]{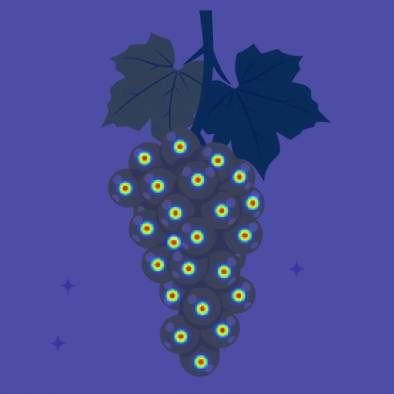} &
\includegraphics[width=\imgw]{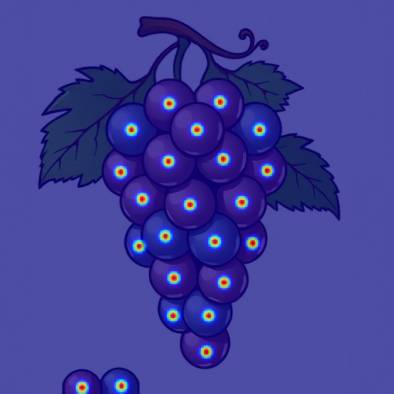} &
\includegraphics[width=\imgw]{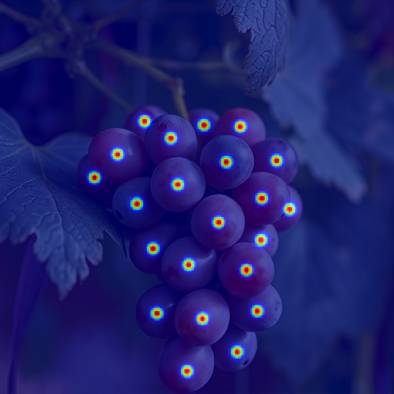} \\
27 ($\Delta 5$) &
24 ($\Delta 2$) &
33 ($\Delta 11$) &
23 ($\Delta 1$) &
23 ($\Delta 1$) &
22 ($\Delta 0$) \\[4pt]

\multicolumn{6}{@{}c@{}}{\textit{``\textbf{11} lychees''}} \\[1pt]
\includegraphics[width=\imgw]{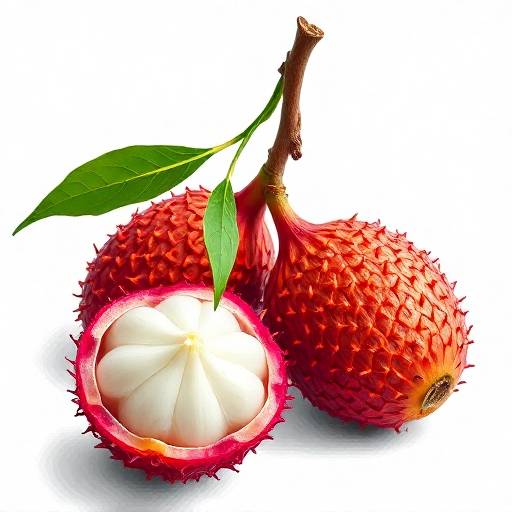} &
\includegraphics[width=\imgw]{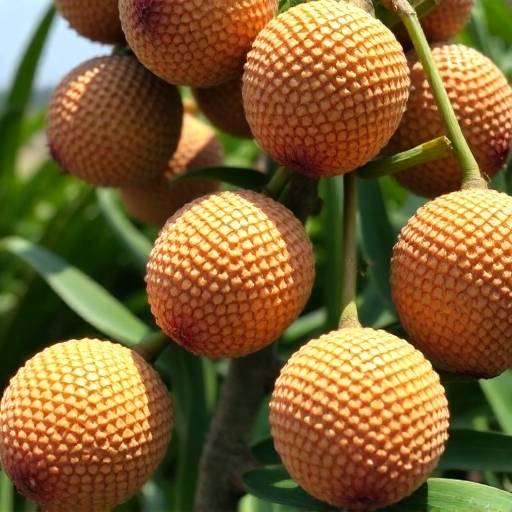} &
\includegraphics[width=\imgw]{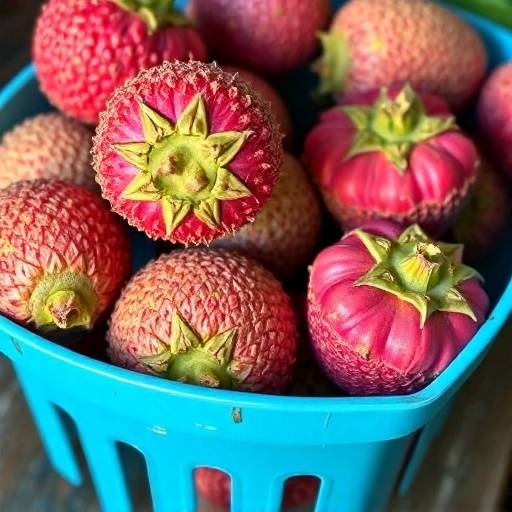} &
\includegraphics[width=\imgw]{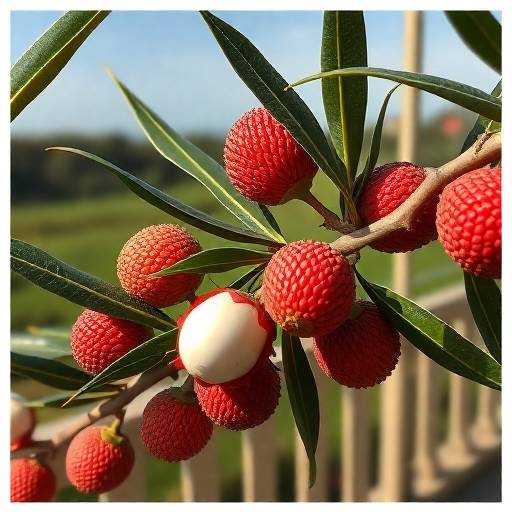} &
\includegraphics[width=\imgw]{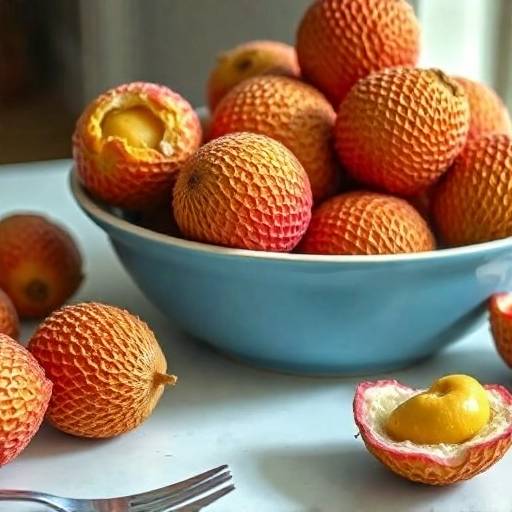} &
\includegraphics[width=\imgw]{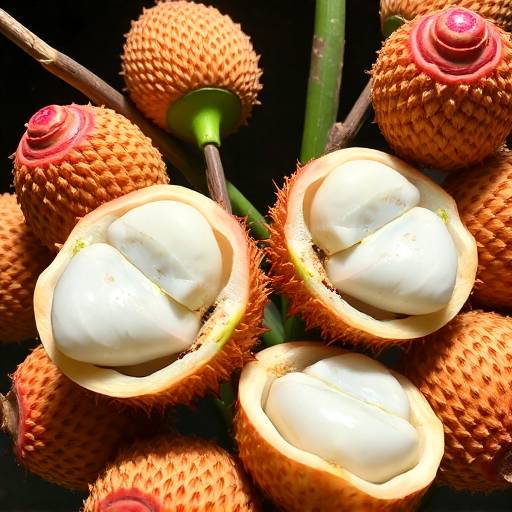} \\
\includegraphics[width=\imgw]{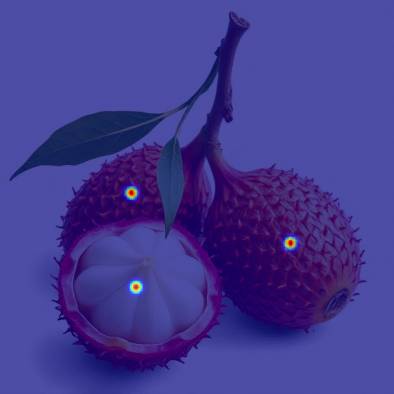} &
\includegraphics[width=\imgw]{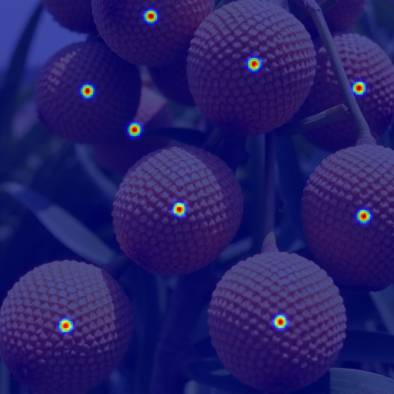} &
\includegraphics[width=\imgw]{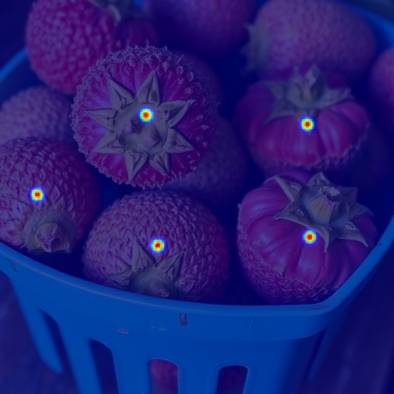} &
\includegraphics[width=\imgw]{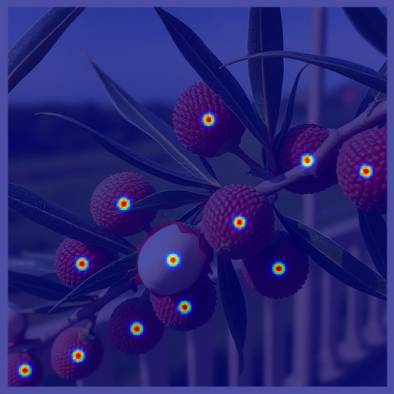} &
\includegraphics[width=\imgw]{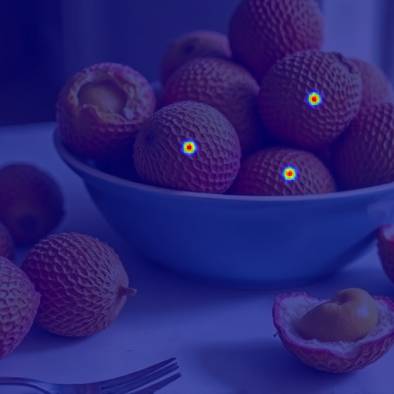} &
\includegraphics[width=\imgw]{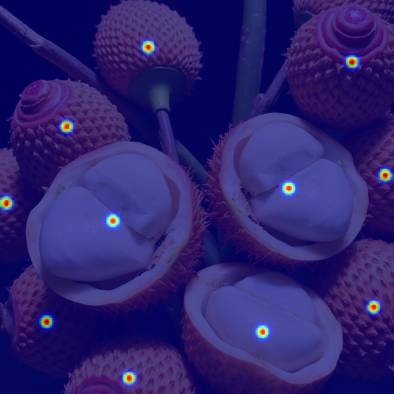} \\
3 ($\Delta 8$) &
9 ($\Delta 2$) &
5 ($\Delta 6$) &
12 ($\Delta 1$) &
3 ($\Delta 8$) &
11 ($\Delta 0$) \\

\bottomrule
\end{tabular}

\caption{
\footnotesize
\textbf{Qualitative comparison on quantity-aware generation using T2I-Count~\cite{qian2025:t2icount}.}
For each example, the first row shows generated images and the second row shows the corresponding heatmaps.
For single-particle methods, we augment sampling with Best-of-N~\cite{Stiennon:2020BoN} and $\Psi$-Sampler~\cite{yoon2025:psi} to match the total NFE, denoted with \textsuperscript{\dag} and \textsuperscript{$\ddagger$}, respectively.
}
\label{fig:counting_qualitative}
\end{figure*}

\subsection{Additional Application: VLM-Based Reward Alignment}
\label{app:vlm_reward_alignment}
In this section, we present qualitative results of applying \Ours{} to a VLM-based reward application: relative depth and horizon position conditioned image generation. 

\paragraph{Experiment Setup.} 
We consider the relative depth and horizon position applications introduced in the Dual-Process framework~\cite{luo2025:dual}. 
For relative depth, the task is to enforce a specific depth ordering between two locations, specified by overlaid red dots labeled as Point A and Point B. 
For horizon position, the objective is to align the scene horizon with a red line overlaid on the image. 

To enforce these conditions, the generated image containing these visual overlays is fed into a VLM, Qwen2.5-VL~\cite{qwen}. 
The VLM is then queried with a task-specific instruction prompt (e.g., ``Is Point B much closer to the camera than Point A?'' or ``Is the overlaid red line aligned with the horizon of the scene?''). 
The reward is computed as the probability of the VLM outputting the desired affirmative answer. 

\paragraph{Results.}
We present qualitative results in \cref{tab:dual_process}. 
As shown in these examples, \Ours{} demonstrates the capability to align generated images with instruction prompts. 
Specifically, we observe that the generated scenes can successfully place the horizon along the specified red line and adjust object positioning to satisfy the target depth ordering. 

\begin{table}[t]
\centering
\setlength{\tabcolsep}{3pt} 
\renewcommand{\arraystretch}{1.0} 
\begin{tabular}{cccc}
\multicolumn{4}{c}{\textbf{Relative Depth}} \\
\midrule 
\multicolumn{2}{c}{Base~\cite{flux2024}} & \multicolumn{2}{c}{{\Oursbf{}} (\textbf{Ours})} \\ 
\includegraphics[width=0.24\linewidth]{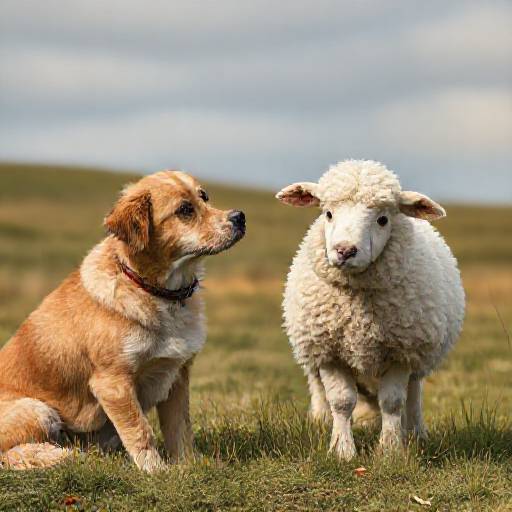} &
\includegraphics[width=0.24\linewidth]{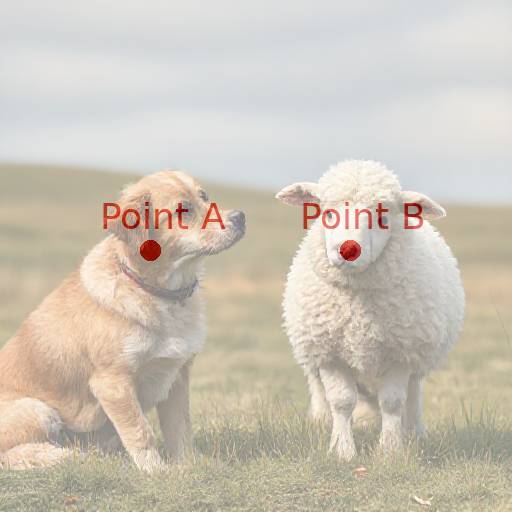} &
\includegraphics[width=0.24\linewidth]{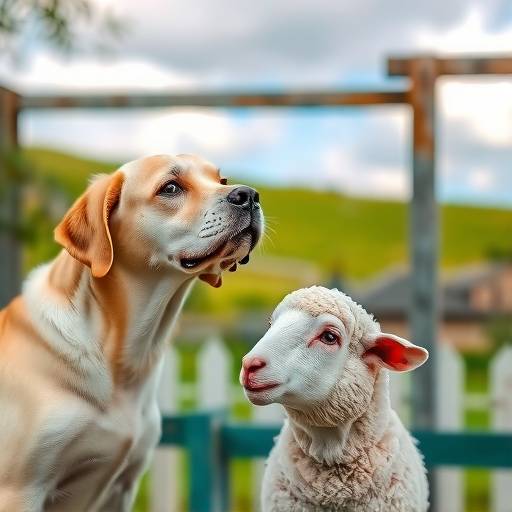} &
\includegraphics[width=0.24\linewidth]{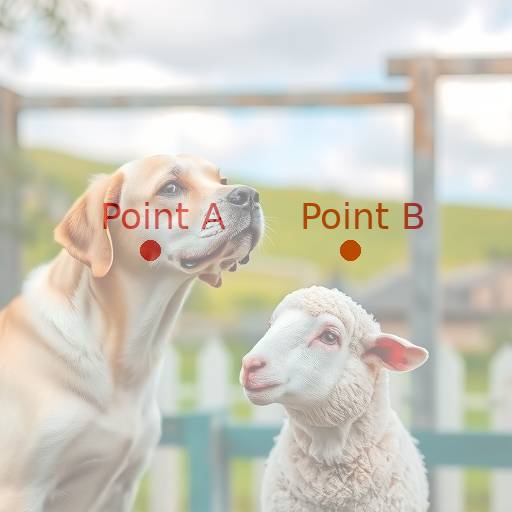} \\
\multicolumn{4}{c}{ \textit{``Is \colorbox{blue!20}{Point A} much closer to the camera than \colorbox{red!20}{Point B}?''}} \\ [1ex]

\includegraphics[width=0.24\linewidth]{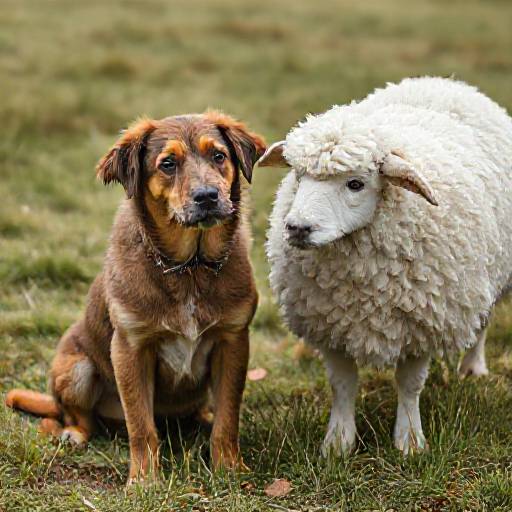} &
\includegraphics[width=0.24\linewidth]{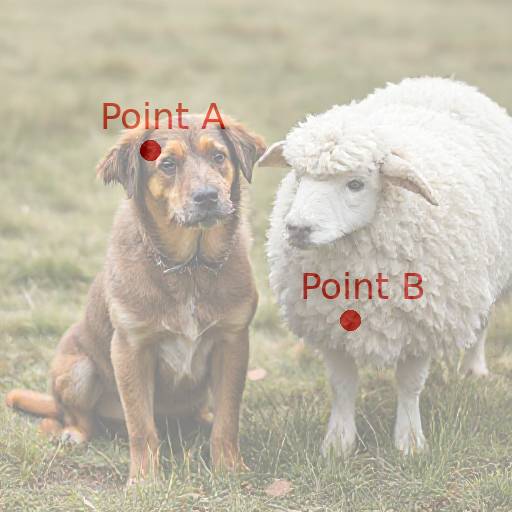} &
\includegraphics[width=0.24\linewidth]{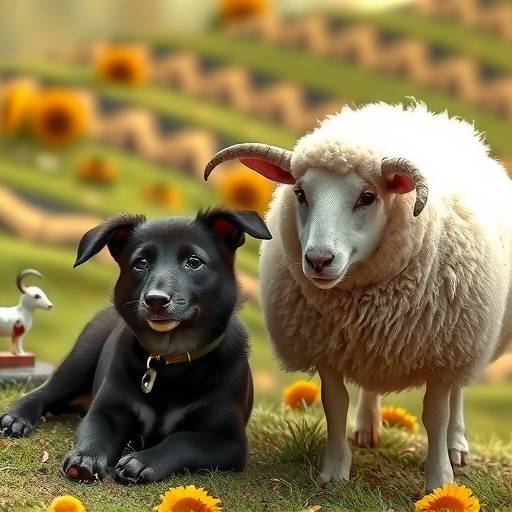} &
\includegraphics[width=0.24\linewidth]{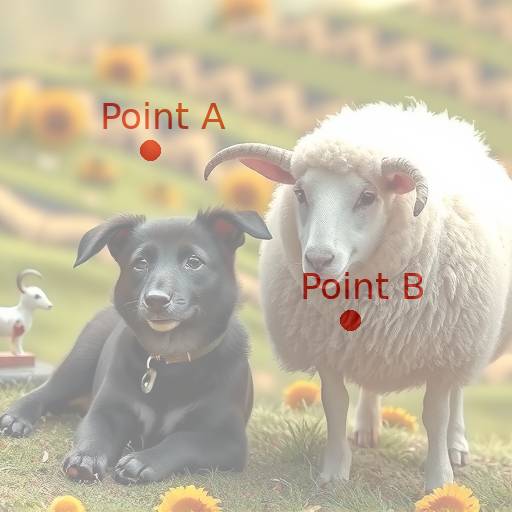} \\
\multicolumn{4}{c}{ \textit{``Is \colorbox{red!20}{Point B} much closer to the camera than \colorbox{blue!20}{Point A}?''}} \\ [1ex]

\multicolumn{4}{c}{\textbf{Horizon Position}} \\
\midrule 
\multicolumn{2}{c}{Base~\cite{flux2024}} & \multicolumn{2}{c}{{\Oursbf{}} (\textbf{Ours})} \\ 
\includegraphics[width=0.24\linewidth]{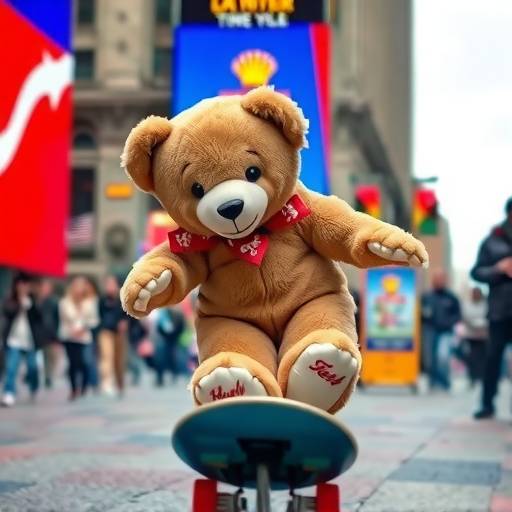} &
\includegraphics[width=0.24\linewidth]{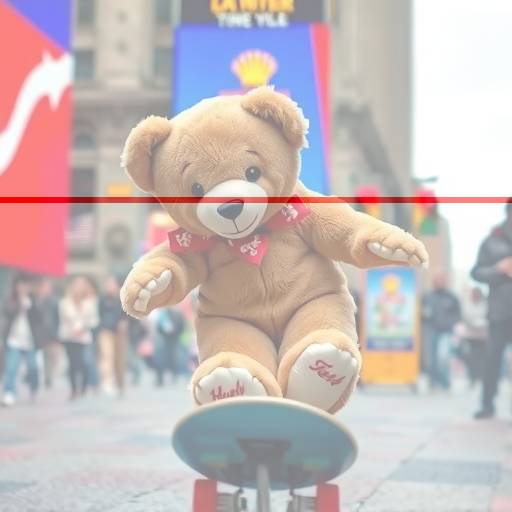} &
\includegraphics[width=0.24\linewidth]{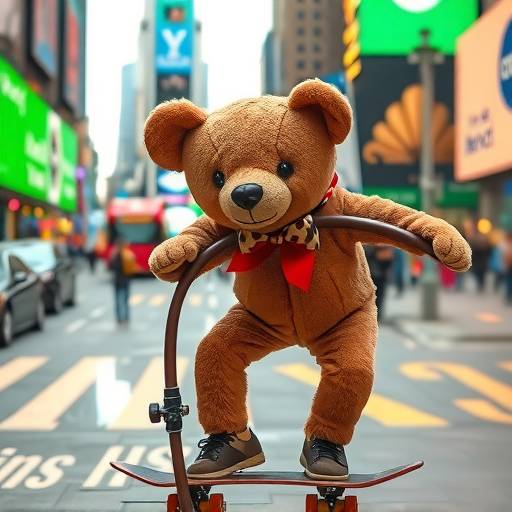} &
\includegraphics[width=0.24\linewidth]{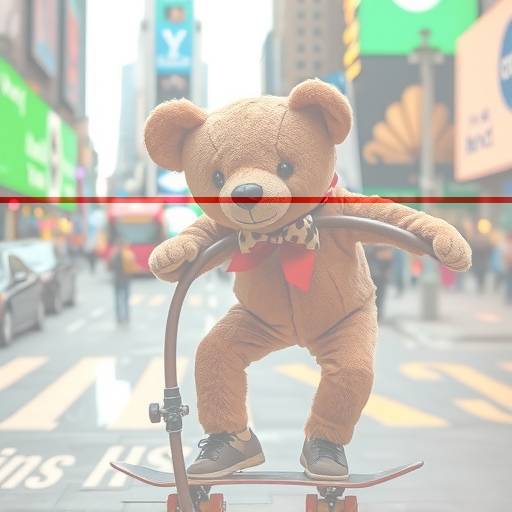} \\
\multicolumn{4}{c}{ \textit{``Does the overlaid red line represent the horizon of the scene?''}} \\ [1ex]
\includegraphics[width=0.24\linewidth]{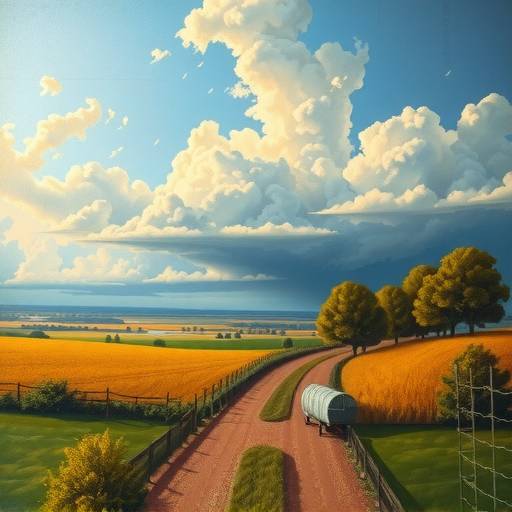} &
\includegraphics[width=0.24\linewidth]{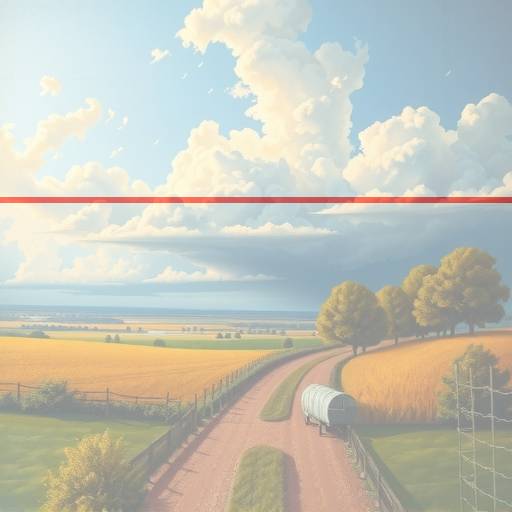} &
\includegraphics[width=0.24\linewidth]{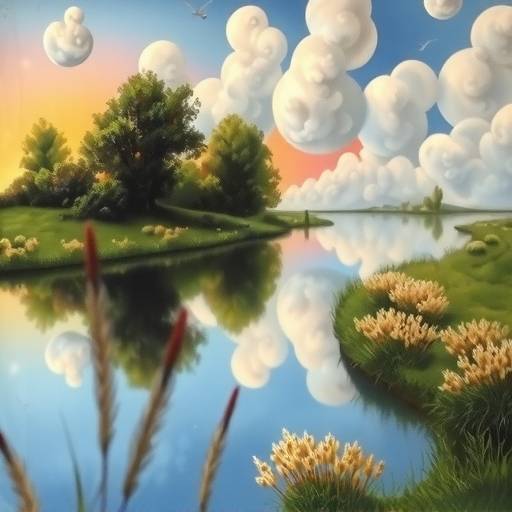} &
\includegraphics[width=0.24\linewidth]{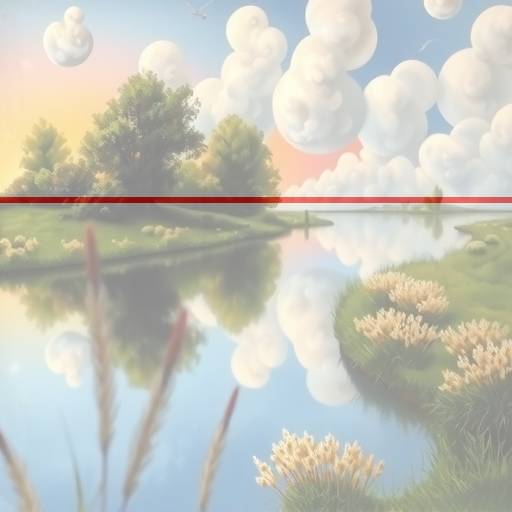} \\
\multicolumn{4}{c}{ \textit{``Does the overlaid red line represent the horizon of the scene?''}} \\ [1ex]
\end{tabular}
\captionof{figure}{
    \footnotesize
    \textbf{Applications of Dual-Process~\cite{luo2025:dual}}. Horizon position aims to align the horizontal line with the image, and relative depth to place objects with the given relative depth information. 
    The prompts describe the instructions provided to the VLM~\cite{qwen}, and for generation, we use the prompts provided in the Dual-Process~\cite{luo2025:dual}. 
}
\label{tab:dual_process}
\end{table}
\FloatBarrier

\subsection{Additional Results with Different Diffusion Model}
\label{app:zimage_results}

In this section, we present additional diffusion reward alignment results with a different base model, Z-Image~\cite{cai2025:zimage}. 

\paragraph{Experiment Setup.}
We follow the identical experimental setup described in \cref{subsec:exp_setup}, with only the base model replaced by Z-Image~\cite{cai2025:zimage}. Specifically, for aesthetic image generation, we use $45$ animal prompts from DDPO~\cite{Black2024:DDPO}. For text-aligned image generation, we use $100$ prompts from the complex category of T2I-CompBench++~\cite{Huang:2025T2ICompBench++}. We fix the sampling steps to $25$ for the base model. To ensure a fair comparison, we fix the total number of function evaluations (NFE) across all baseline methods, augmenting single-particle methods with Best-of-N (BoN)~\cite{Stiennon:2020BoN} or $\Psi$-Sampler~\cite{yoon2025:psi} to match the computational budget.

\paragraph{Aesthetic Image Generation.}
In this task, the target reward is Aesthetic Score~\cite{Schuhmann:aesthetics}, and the held-out rewards are ImageReward~\cite{Xu2023:ImageReward}, HPSv2~\cite{wu2023:hpsv2}, PickScore~\cite{Kirstain2023:pickapic}, and VQA Score~\cite{lin2024:vqa}. The quantitative results are presented in \cref{tab:aesthetic_full_zimage}. Consistent with our findings using FLUX~\cite{flux2024}, \Ours{} achieves the best target reward performance across all methods. Remarkably, \Ours{} with a single particle and only $25$ NFE achieves a higher target reward (7.3500) than the best-performing multi-particle baselines (e.g., DAS~\cite{Kim:2025DAS} and SVDD~\cite{Li2024:SVDD}) evaluated at $500$ NFE. When scaled to $500$ NFE, \Ours{} further improves the target reward to 8.7860 and attains the highest scores across most held-out rewards.

\paragraph{Text-Aligned Image Generation.}
For this task, the target reward is PickScore~\cite{Kirstain2023:pickapic}, while held-out rewards include VQA Score~\cite{lin2024:vqa}, Aesthetic Score~\cite{Schuhmann:aesthetics}, ImageReward~\cite{Xu2023:ImageReward}, and HPSv2~\cite{wu2023:hpsv2}. The quantitative results are summarized in \cref{tab:pick_full_zimage}. Similar to the aesthetic image generation, \Ours{} consistently outperforms other baselines in target reward performance. With just $25$ NFE, \Ours{} surpasses the performance of all baselines evaluated at $500$ NFE. In addition to superior target reward alignment, \Ours{} generalizes effectively to held-out rewards, attaining the highest ImageReward and HPSv2 scores while remaining comparable in Aesthetic Score and VQA Score. These results confirm that \Ours{} can be seamlessly applied to different diffusion models without compromising fundamental image quality.

\begin{table*}[t!]
\centering
\caption{
\footnotesize
\textbf{Full quantitative comparison on aesthetic image generation with Z-Image~\cite{cai2025:zimage}.}
The target reward is Aesthetic Score~\cite{Schuhmann:aesthetics}.
For single-particle methods, we augment sampling with Best-of-$N$~\cite{Stiennon:2020BoN} and $\Psi$-Sampler~\cite{yoon2025:psi} to match the total NFE, denoted with \textsuperscript{\dag} and \textsuperscript{$\ddagger$}, respectively.
Dark green cells indicate the best result for each metric across all runs, while light green cells denote the second best.
}
\label{tab:aesthetic_full_zimage}
\renewcommand{\arraystretch}{1.10}
\setlength{\tabcolsep}{4.0pt}

\scriptsize
\resizebox{\textwidth}{!}{%
\begin{tabular}{@{} l c c c >{\bfseries}c c c c c @{}}
\toprule
\multirow{2}{*}{\textbf{Method}}
& \multicolumn{3}{c}{\textbf{Settings}}
& \multicolumn{1}{c}{\makecell{\textbf{Target}\\\textbf{Reward}}}
& \multicolumn{4}{c}{\makecell{\textbf{Held-out}\\\textbf{Rewards}}} \\
\cmidrule(lr){2-4}\cmidrule(lr){5-5}\cmidrule(lr){6-9}
& \makecell{\textbf{\#}\\\textbf{Particles}}
& \makecell{\textbf{\#}\\\textbf{Steps}}
& \textbf{NFE}
& \makecell{\textbf{Aesthetic}\\\textbf{Score}} $\uparrow$
& \makecell{\textbf{Pick}\\\textbf{-score}} $\uparrow$
& \textbf{HPSv2} $\uparrow$
& \makecell{\textbf{Image}\\\textbf{Reward}} $\uparrow$
& \makecell{\textbf{VQA}\\\textbf{Score}} $\uparrow$ \\
\midrule

\multicolumn{9}{@{}l}{Base~\cite{cai2025:zimage}} \\
\quad & 1 & 25 & 25  & 5.7627 & 0.2206 & 0.2949 & 1.2856 & 0.9708 \\

\midrule

\multicolumn{9}{@{}l}{BoN~\cite{Stiennon:2020BoN}} \\
\quad & 4 & 25 & 100 & 5.9375 & 0.2208 & 0.2968 & 1.3292 & 0.9707 \\
\quad & 20 & 25 & 500 & 6.1749 & 0.2209 & 0.2962 & 1.3474 & 0.9709 \\

\midrule

\multicolumn{9}{@{}l}{DPS\textsuperscript{\dag}~\cite{Chung:2023DPS}} \\
\quad & 4 & 25 & 100 & 5.9327 & 0.2207 & 0.2963 & 1.3186 & \cellcolor{lightgreen}0.9736 \\
\quad & 20 & 25 & 500 & 6.1587 & 0.2208 & 0.2966 & \cellcolor{lightgreen}1.3784 & 0.9657 \\

\midrule

\multicolumn{9}{@{}l}{FreeDoM\textsuperscript{\dag}~\cite{Yu:2023FreeDOM}} \\
\quad & 3 & 25 & 123 & 5.9502 & 0.2209 & 0.2944 & 1.3189 & 0.9718 \\
\quad & 13 & 25 & 533 & 6.0753 & 0.2203 & 0.2942 & 1.3119 & 0.9678 \\

\midrule

\multicolumn{9}{@{}l}{SVDD~\cite{Li2024:SVDD}} \\
\quad & 4 & 25 & 100 & 6.0228 & 0.2208 & 0.2979 & 1.3307 & 0.9665 \\
\quad & 20 & 25 & 500 & 6.2394 & 0.2209 & 0.2945 & 1.2996 & 0.9698 \\

\midrule

\multicolumn{9}{@{}l}{RBF~\cite{kim2025:rbf}} \\
\quad & 4 & 25 & 100 & 6.0037 & 0.2208 & 0.2961 & 1.3203 & 0.9715 \\
\quad & 20 & 25 & 500 & 6.1600 & 0.2210 & 0.2941 & 1.3486 & 0.9692 \\

\midrule

\multicolumn{9}{@{}l}{DAS~\cite{Kim:2025DAS}} \\
\quad & 4 & 25 & 100 & 5.9745 & 0.2202 & 0.2942 & 1.3035 & 0.9678 \\
\quad & 20 & 25 & 500 & 6.2449 & \cellcolor{lightgreen}0.2214 & 0.2970 & 1.3174 & 0.9673 \\

\midrule

\multicolumn{9}{@{}l}{$\Psi$-Sampler~\cite{yoon2025:psi}} \\
\quad & 10 & 25 + 25 & 500 & 6.1078 & 0.2206 & 0.2942 & 1.2843 & 0.9689 \\

\midrule

\multicolumn{9}{@{}l}{\textbf{\Ours\textsuperscript{$\ddagger$} (Ours)}} \\
\quad & 10 & 25 + 25 & 500 & 7.3139 & 0.2205 & 0.2975 & 1.3296 & 0.9669 \\

\midrule

\multicolumn{9}{@{}l}{\textbf{\Ours\textsuperscript{\dag} (Ours)}} \\
\quad & 1 & 25  & 25  & 7.3500 & 0.2212 & \cellcolor{lightgreen}0.3003 & 1.2781 & 0.9666 \\
\quad & 2 & 50  & 100 & \cellcolor{lightgreen}8.0537 & 0.2204 & 0.2975 & 1.3034 & 0.9706 \\
\quad & 4 & 25  & 100 & 7.5989 & 0.2210 & 0.2996 & 1.3185 & 0.9684 \\
\quad & 20 & 25 & 500 & 7.8527 & \cellcolor{darkgreen}0.2218 & \cellcolor{darkgreen}0.3009 & 1.3701 & 0.9686 \\
\quad & 5 & 100 & 500 & \cellcolor{darkgreen}8.7860 & 0.2208 & 0.2978 & \cellcolor{darkgreen}1.4165 & \cellcolor{darkgreen}0.9738 \\

\bottomrule
\end{tabular}%
}
\end{table*}
\begin{table*}[t!]
\centering
\caption{
\footnotesize 
\textbf{Full quantitative comparison on text-aligned generation with Z-Image~\cite{cai2025:zimage}.} 
The target reward is PickScore~\cite{Kirstain2023:pickapic}. 
For single-particle methods, we augment sampling with Best-of-N~\cite{Stiennon:2020BoN} and $\Psi$-Sampler~\cite{yoon2025:psi} to match the total NFE, denoted with \textsuperscript{\dag} and \textsuperscript{$\ddagger$}, respectively. 
Dark green cells indicate the best result for each metric across all runs, while light green cells denote the second best. 
}
\label{tab:pick_full_zimage}
\renewcommand{\arraystretch}{1.10}
\setlength{\tabcolsep}{4.0pt}
\scriptsize
\resizebox{\textwidth}{!}{%
\begin{tabular}{@{} l c c c >{\bfseries}c c c c c @{}}
\toprule
\multirow{2}{*}{\textbf{Method}}
& \multicolumn{3}{c}{\textbf{Settings}}
& \multicolumn{1}{c}{\makecell{\textbf{Target}\\\textbf{Reward}}}
& \multicolumn{4}{c}{\makecell{\textbf{Held-out}\\\textbf{Rewards}}} \\
\cmidrule(lr){2-4}\cmidrule(lr){5-5}\cmidrule(lr){6-9}
& \makecell{\textbf{\#}\\\textbf{Particles}}
& \makecell{\textbf{\#}\\\textbf{Steps}}
& \textbf{NFE}
& \makecell{\textbf{Pick}\\\textbf{-Score}} $\uparrow$
& \makecell{\textbf{Aesthetic}\\\textbf{Score}} $\uparrow$
& \textbf{HPSv2} $\uparrow$
& \makecell{\textbf{Image}\\\textbf{Reward}} $\uparrow$
& \makecell{\textbf{VQA}\\\textbf{Score}} $\uparrow$ \\
\midrule

\multicolumn{9}{@{}l}{Base~\cite{cai2025:zimage}} \\
\quad & 1  & 25 & 25  & 0.2116 & 5.7061 & 0.2561 & 0.6567 & 0.8393 \\

\midrule
\multicolumn{9}{@{}l}{BoN~\cite{Stiennon:2020BoN}} \\
\quad & 4  & 25 & 100 & 0.2145 & 5.7627 & 0.2627 & 0.7788 & 0.8396 \\
\quad & 20 & 25 & 500 & 0.2162 & 5.7791 & 0.2654 & 0.7522 & 0.8299 \\
\midrule

\multicolumn{9}{@{}l}{DPS\textsuperscript{\dag}~\cite{Chung:2023DPS}} \\
\quad & 4  & 25 & 100 & 0.2145 & 5.7501 & 0.2624 & 0.7747 & 0.8392 \\
\quad & 20 & 25 & 500 & 0.2162 & 5.7422 & 0.2633 & 0.7380 & 0.8385 \\
\midrule

\multicolumn{9}{@{}l}{FreeDoM\textsuperscript{\dag}~\cite{Yu:2023FreeDOM}} \\
\quad & 3  & 25 & 123 & 0.2133 & 5.7087 & 0.2589 & 0.7054 & 0.8317 \\
\quad & 13 & 25 & 533 & 0.2155 & 5.7486 & 0.2625 & 0.6872 & 0.8435 \\
\midrule

\multicolumn{9}{@{}l}{SVDD~\cite{Li2024:SVDD}} \\
\quad & 4  & 25 & 100 & 0.2151 & 5.7813 & 0.2621 & 0.7275 & 0.8374 \\
\quad & 20 & 25 & 500 & 0.2181 & 5.7910 & 0.2667 & 0.7893 & 0.8389 \\
\midrule

\multicolumn{9}{@{}l}{RBF~\cite{kim2025:rbf}} \\
\quad & 4  & 25 & 100 & 0.2144 & 5.7461 & 0.2594 & 0.6940 & 0.8399 \\
\quad & 20 & 25 & 500 & 0.2165 & 5.7778 & 0.2628 & 0.7610 & 0.8445 \\
\midrule

\multicolumn{9}{@{}l}{DAS~\cite{Kim:2025DAS}} \\
\quad & 4  & 25 & 100 & 0.2144 & 5.7360 & 0.2605 & 0.7012 & 0.8426 \\
\quad & 20 & 25 & 500 & 0.2163 & 5.7657 & 0.2656 & 0.8306 & 0.8445 \\
\midrule

\multicolumn{9}{@{}l}{$\Psi$-Sampler~\cite{yoon2025:psi}} \\
\quad & 10 & 25 + 25 & 500 & 0.2159 & 5.8097 & 0.2639 & 0.7419 & \cellcolor{darkgreen}0.8451 \\
\midrule

\multicolumn{9}{@{}l}{\textbf{\Ours\textsuperscript{$\ddagger$} (Ours)}} \\
\quad & 10 & 25 + 25 & 500 & 0.2158 & 5.7939 & 0.2656 & 0.7352 & 0.8336 \\
\midrule

\multicolumn{9}{@{}l}{\textbf{\Ours\textsuperscript{\dag} (Ours)}} \\
\quad & 1  & 25  & 25  & 0.2334 & \cellcolor{darkgreen}5.8487 & 0.2837 & 0.8128 & 0.8423 \\
\quad & 2  & 50  & 100 & \cellcolor{lightgreen}0.2436 & 5.8095 & \cellcolor{lightgreen}0.2907 & \cellcolor{lightgreen}0.8525 & \cellcolor{lightgreen}0.8446 \\
\quad & 4  & 25  & 100 & 0.2361 & \cellcolor{lightgreen}5.8108 & 0.2877 & 0.8224 & 0.8364 \\
\quad & 20 & 25  & 500 & 0.2393 & 5.8092 & 0.2886 & 0.8515 & 0.8405 \\
\quad & 5  & 100 & 500 & \cellcolor{darkgreen}0.2613 & 5.7398 & \cellcolor{darkgreen}0.2984 & \cellcolor{darkgreen}0.8654 & 0.8444 \\
\bottomrule
\end{tabular}%
}
\end{table*}
\FloatBarrier

\subsection{Additional Results of the Main Paper Experiments}
\label{app:additional_main}

In this section, we present additional results for the main experiments, including expanded quantitative comparisons across diverse sampling configurations and additional qualitative examples for all three tasks in \cref{sec:exp}.

\paragraph{Expanded Quantitative Results.}
We provide the expanded quantitative results of aesthetic image generation and text-aligned image generation with FLUX~\cite{flux2024} in~\cref{tab:aesthetic_full_supp,tab:pick_full_supp}, covering diverse settings with different numbers of particles, sampling steps, and total NFEs. Across both diffusion models and reward alignment tasks, \Ours{} consistently outperforms the baselines in target reward alignment while achieving competitive or superior held-out reward performance, thereby yielding the most favorable trade-offs. These results further demonstrate that the advantages of \Ours{} remain robust across different base models and sampling configurations.

\begin{table*}[t!]
\centering
\caption{
\footnotesize 
\textbf{Full quantitative comparison on aesthetic image generation with FLUX~\cite{flux2024}.} 
The target reward is Aesthetic Score~\cite{Schuhmann:aesthetics}. 
For single-particle methods, we augment sampling with Best-of-N~\cite{Stiennon:2020BoN} and $\Psi$-Sampler~\cite{yoon2025:psi} to match the total NFE, denoted with \textsuperscript{\dag} and \textsuperscript{$\ddagger$}, respectively. 
Dark green cells indicate the best result for each metric across all runs, while light green cells denote the second best. 
}
\label{tab:aesthetic_full_supp}
\renewcommand{\arraystretch}{1.10}
\setlength{\tabcolsep}{4.0pt}

\scriptsize
\resizebox{\textwidth}{!}{%
\begin{tabular}{@{} l c c c >{\bfseries}c c c c c @{}}
\toprule
\multirow{2}{*}{\textbf{Method}}
& \multicolumn{3}{c}{\textbf{Settings}}
& \multicolumn{1}{c}{\makecell{\textbf{Target}\\\textbf{Reward}}}
& \multicolumn{4}{c}{\makecell{\textbf{Held-out}\\\textbf{Rewards}}} \\
\cmidrule(lr){2-4}\cmidrule(lr){5-5}\cmidrule(lr){6-9}
& \makecell{\textbf{\#}\\\textbf{Particles}}
& \makecell{\textbf{\#}\\\textbf{Steps}}
& \textbf{NFE}
& \makecell{\textbf{Aesthetic}\\\textbf{Score}} $\uparrow$
& \makecell{\textbf{Pick}\\\textbf{-score}} $\uparrow$
& \textbf{HPSv2} $\uparrow$
& \makecell{\textbf{Image}\\\textbf{Reward}} $\uparrow$
& \makecell{\textbf{VQA}\\\textbf{Score}} $\uparrow$ \\
\midrule

\multicolumn{9}{@{}l}{Base~\cite{flux2024}} \\
\quad & 1 & 25 & 25 & 6.0282 & 0.2144 & 0.2759 & 1.0538 & 0.9644 \\

\midrule

\multicolumn{9}{@{}l}{BoN~\cite{Stiennon:2020BoN}} \\
\quad & 4 & 25 & 100 & 6.4694 & 0.2183 & 0.2832 & 0.9583 & 0.9635 \\
\quad & 20 & 25 & 500 & 6.7310 & \cellcolor{lightgreen}0.2197 & 0.2890 & 1.1419 & 0.9597 \\

\midrule

\multicolumn{9}{@{}l}{DPS\textsuperscript{\dag}~\cite{Chung:2023DPS}} \\
\quad & 4 & 25 & 100 & 6.4773 & 0.2182 & 0.2805 & 0.9604 & 0.9634 \\
\quad & 20 & 25 & 500 & 6.7647 & 0.2191 & 0.2861 & 1.0639 & 0.9624 \\

\midrule

\multicolumn{9}{@{}l}{FreeDoM\textsuperscript{\dag}~\cite{Yu:2023FreeDOM}} \\
\quad & 3 & 25 & 123 & 6.5196 & 0.2179 & 0.2797 & 0.8771 & 0.9699 \\
\quad & 13 & 25 & 533 & 6.8406 & 0.2185 & 0.2853 & 0.9941 & 0.9635 \\

\midrule

\multicolumn{9}{@{}l}{SVDD~\cite{Li2024:SVDD}} \\
\quad & 4 & 25 & 100 & 6.5348 & 0.2176 & 0.2835 & 1.1208 & 0.9652 \\
\quad & 20 & 25 & 500 & 7.1363 & 0.2177 & 0.2814 & 1.0256 & 0.9510 \\

\midrule

\multicolumn{9}{@{}l}{RBF~\cite{kim2025:rbf}} \\
\quad & 4 & 25 & 100 & 6.5497 & 0.2173 & 0.2818 & 1.1237 & 0.9695 \\
\quad & 20 & 25 & 500 & 6.9900 & 0.2183 & 0.2826 & 1.0761 & 0.9689 \\

\midrule

\multicolumn{9}{@{}l}{DAS~\cite{Kim:2025DAS}} \\
\quad & 4 & 25 & 100 & 6.5759 & 0.2177 & 0.2821 & 1.0096 & 0.9658 \\
\quad & 20 & 25 & 500 & 6.9384 & 0.2183 & 0.2860 & 1.0568 & 0.9706 \\

\midrule

\multicolumn{9}{@{}l}{$\Psi$-Sampler~\cite{yoon2025:psi}} \\
\quad & 10 & 25 + 25 & 500 & 7.0116 & 0.2188 & 0.2847 & 1.1235 & \cellcolor{darkgreen}0.9737 \\

\midrule

\multicolumn{9}{@{}l}{\textbf{\Ours\textsuperscript{$\ddagger$} (Ours)}} \\
\quad & 10 & 25 + 25 & 500 & 7.4180 & 0.2170 & 0.2812 & 1.0209 & 0.9656 \\

\midrule

\multicolumn{9}{@{}l}{\textbf{\Ours\textsuperscript{\dag} (Ours)}} \\
\quad & 1 & 25 & 25 & 7.4510 & \cellcolor{darkgreen}0.2200 & \cellcolor{lightgreen}0.2928 & \cellcolor{lightgreen}1.2565 & \cellcolor{lightgreen}0.9728 \\
\quad & 1 & 100 & 100 & 8.0538 & 0.2173 & 0.2854 & 1.1546 & 0.9682 \\
\quad & 2 & 50 & 100 & 7.8454 & 0.2160 & 0.2831 & 1.0744 & 0.9520 \\
\quad & 4 & 25 & 100 & 7.7247 & 0.2176 & 0.2835 & 0.9874 & 0.9648 \\
\quad & 5 & 100 & 500 & \cellcolor{darkgreen}8.5394 & 0.2183 & 0.2868 & \cellcolor{darkgreen}1.2613 & 0.9566 \\
\quad & 10 & 50 & 500 & \cellcolor{lightgreen}8.2081 & 0.2181 & 0.2913 & 1.1445 & 0.9643 \\
\quad & 20 & 25 & 500 & 7.9656 & \cellcolor{lightgreen}0.2197 & \cellcolor{darkgreen}0.2932 & 1.1669 & 0.9609 \\

\bottomrule
\end{tabular}%
}
\end{table*}

\begin{table*}[t!]
\centering
\caption{
\footnotesize 
\textbf{Full quantitative comparison on text-aligned generation with FLUX~\cite{flux2024}.} 
The target reward is PickScore~\cite{Kirstain2023:pickapic}. 
For single-particle methods, we augment sampling with Best-of-N~\cite{Stiennon:2020BoN} and $\Psi$-Sampler~\cite{yoon2025:psi} to match the total NFE, denoted with \textsuperscript{\dag} and \textsuperscript{$\ddagger$}, respectively. 
Dark green cells indicate the best result for each metric across all runs, while light green cells denote the second best. 
}
\label{tab:pick_full_supp}
\renewcommand{\arraystretch}{1.10}
\setlength{\tabcolsep}{4.0pt}
\scriptsize
\resizebox{\textwidth}{!}{%
\begin{tabular}{@{} l c c c >{\bfseries}c c c c c @{}}
\toprule
\multirow{2}{*}{\textbf{Method}}
& \multicolumn{3}{c}{\textbf{Settings}}
& \multicolumn{1}{c}{\makecell{\textbf{Target}\\\textbf{Reward}}}
& \multicolumn{4}{c}{\makecell{\textbf{Held-out}\\\textbf{Rewards}}} \\
\cmidrule(lr){2-4}\cmidrule(lr){5-5}\cmidrule(lr){6-9}
& \makecell{\textbf{\#}\\\textbf{Particles}}
& \makecell{\textbf{\#}\\\textbf{Steps}}
& \textbf{NFE}
& \makecell{\textbf{Pick}\\\textbf{-Score}} $\uparrow$
& \makecell{\textbf{Aesthetic}\\\textbf{Score}} $\uparrow$
& \textbf{HPSv2} $\uparrow$
& \makecell{\textbf{Image}\\\textbf{Reward}} $\uparrow$
& \makecell{\textbf{VQA}\\\textbf{Score}} $\uparrow$ \\
\midrule

\multicolumn{9}{@{}l}{Base~\cite{flux2024}} \\
\quad & 1 & 25 & 25  & 0.2054 & 5.4664 & 0.2316 & 0.1710 & 0.8011 \\
\midrule

\multicolumn{9}{@{}l}{BoN~\cite{Stiennon:2020BoN}} \\
\quad & 4  & 25 & 100 & 0.2103 & 5.6987 & 0.2463 & 0.4272 & 0.7767 \\
\quad & 20 & 25 & 500 & 0.2146 & 5.8582 & 0.2619 & 0.6883 & 0.8021 \\
\midrule

\multicolumn{9}{@{}l}{DPS\textsuperscript{\dag}~\cite{Chung:2023DPS}} \\
\quad & 4  & 25 & 100 & 0.2105 & 5.6926 & 0.2470 & 0.5169 & 0.7926 \\
\quad & 20 & 25 & 500 & 0.2147 & 5.8073 & 0.2622 & 0.6310 & 0.8028 \\
\midrule

\multicolumn{9}{@{}l}{FreeDoM\textsuperscript{\dag}~\cite{Yu:2023FreeDOM}} \\
\quad & 3  & 25 & 123 & 0.2099 & 5.8158 & 0.2458 & 0.3275 & 0.8041 \\
\quad & 13 & 25 & 533 & 0.2133 & 5.8492 & 0.2572 & 0.5354 & 0.7990 \\
\midrule

\multicolumn{9}{@{}l}{SVDD~\cite{Li2024:SVDD}} \\
\quad & 4  & 25 & 100 & 0.2143 & 5.6745 & 0.2587 & 0.5536 & 0.8102 \\
\quad & 20 & 25 & 500 & 0.2204 & \cellcolor{lightgreen}5.8743 & 0.2699 & \cellcolor{darkgreen}0.7592 & 0.8201 \\
\midrule

\multicolumn{9}{@{}l}{RBF~\cite{kim2025:rbf}} \\
\quad & 4  & 25 & 100 & 0.2150 & 5.6813 & 0.2603 & 0.5993 & 0.8031 \\
\quad & 20 & 25 & 500 & 0.2202 & 5.8618 & 0.2682 & \cellcolor{lightgreen}0.7583 & 0.8149 \\
\midrule

\multicolumn{9}{@{}l}{DAS~\cite{Kim:2025DAS}} \\
\quad & 4  & 25 & 100 & 0.2104 & 5.6601 & 0.2457 & 0.4134 & 0.8064 \\
\quad & 20 & 25 & 500 & 0.2139 & 5.8385 & 0.2568 & 0.5226 & 0.7990 \\
\midrule

\multicolumn{9}{@{}l}{$\Psi$-Sampler~\cite{yoon2025:psi}} \\
\quad & 10 & 25 + 25 & 500 & 0.2120 & 5.7329 & 0.2551 & 0.4590 & 0.8145 \\
\midrule

\multicolumn{9}{@{}l}{\textbf{\Ours\textsuperscript{$\ddagger$} (Ours)}} \\
\quad & 10 & 25 + 25 & 500 & 0.2133 & 5.8495 & 0.2594 & 0.6226 & 0.8198 \\
\midrule

\multicolumn{9}{@{}l}{\textbf{\Ours\textsuperscript{\dag} (Ours)}} \\
\quad & 1  & 25  & 25  & 0.2224 & 5.7720 & 0.2601 & 0.5257 & \cellcolor{darkgreen}0.8210 \\
\quad & 1  & 100 & 100 & 0.2363 & 5.8397 & 0.2743 & 0.5971 & 0.8177 \\
\quad & 2  & 50  & 100 & 0.2317 & 5.8097 & 0.2720 & 0.5918 & 0.8100 \\
\quad & 4  & 25  & 100 & 0.2282 & 5.7478 & 0.2700 & 0.6399 & 0.7881 \\
\quad & 5  & 100 & 500 & \cellcolor{darkgreen}0.2439 & 5.8016 & \cellcolor{darkgreen}0.2830 & 0.6771 & \cellcolor{lightgreen}0.8205 \\
\quad & 10 & 50  & 500 & \cellcolor{lightgreen}0.2385 & 5.7252 & 0.2796 & 0.6739 & 0.8081 \\
\quad & 20 & 25  & 500 & 0.2327 & \cellcolor{darkgreen}5.9020 & \cellcolor{lightgreen}0.2817 & 0.7370 & 0.8017 \\
\bottomrule
\end{tabular}%
}
\end{table*}

\clearpage

\paragraph{Additional Qualitative Results.}
We provide additional qualitative results for aesthetic image generation, text-aligned image generation and preference-aligned video generation in~\cref{fig:aesthetic_qualitative_supp},~\cref{fig:text_align_qualitative_supp} and~\cref{fig:video_qualitative_supp}, respectively. Across all applications, \Ours{} generates images and videos that better reflect the target reward, while maintaining strong visual quality and fidelity to the input text compared to the baselines.

\begin{figure*}[h!]
\centering
\setlength{\tabcolsep}{0pt}
\renewcommand{\arraystretch}{1.05}
\scriptsize

\newcommand{\colw}{0.164\textwidth}  
\newcommand{\imgw}{0.165\textwidth}

\begin{tabular}{@{}%
>{\centering\arraybackslash}m{\colw}
>{\centering\arraybackslash}m{\colw}
>{\centering\arraybackslash}m{\colw}
>{\centering\arraybackslash}m{\colw}
>{\centering\arraybackslash}m{\colw}
>{\centering\arraybackslash}m{\colw}
@{}}
\toprule

Base~\cite{flux2024} &
BoN~\cite{Stiennon:2020BoN} &
DPS\textsuperscript{\dag}~\cite{Chung:2023DPS} &
SVDD~\cite{Li2024:SVDD} &
$\Psi$-Sampler~\cite{yoon2025:psi} &
{\Oursbf{}}\textsuperscript{\dag} (\textbf{Ours}) \\

\midrule

\multicolumn{6}{@{}c@{}}{"dolphin"} \\[2pt]
\includegraphics[width=\imgw]{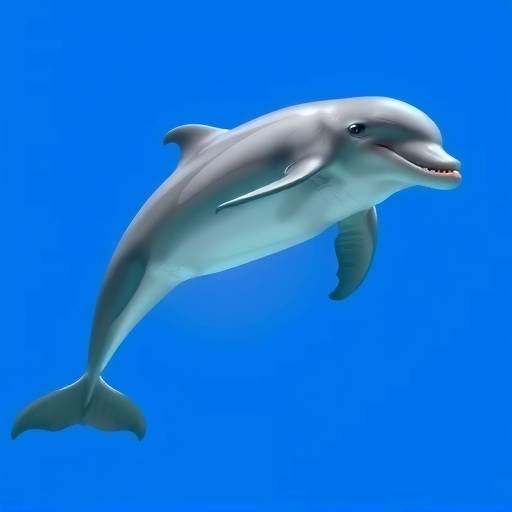} &
\includegraphics[width=\imgw]{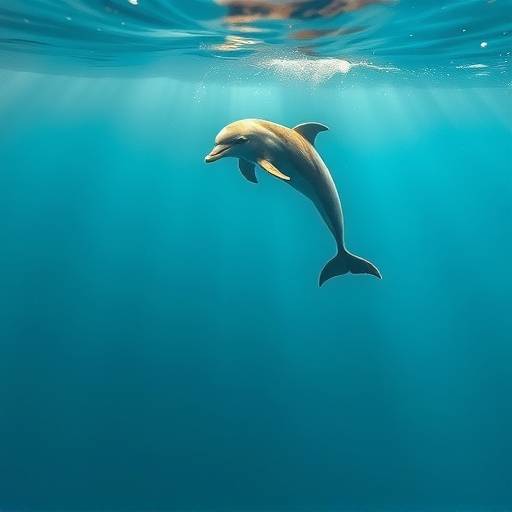} &
\includegraphics[width=\imgw]{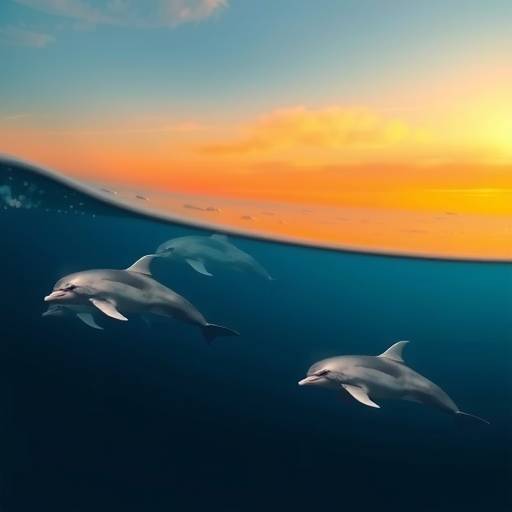} &
\includegraphics[width=\imgw]{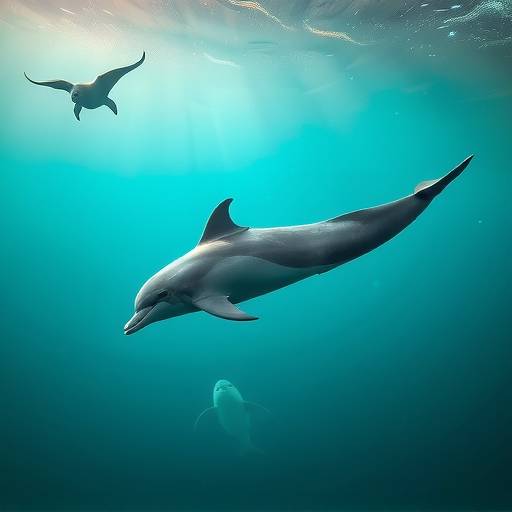} &
\includegraphics[width=\imgw]{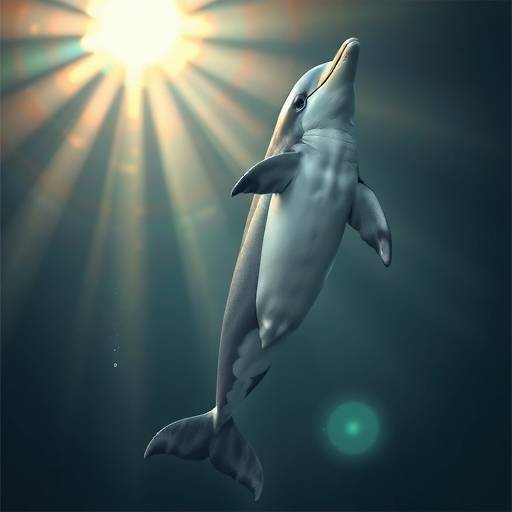} &
\includegraphics[width=\imgw]{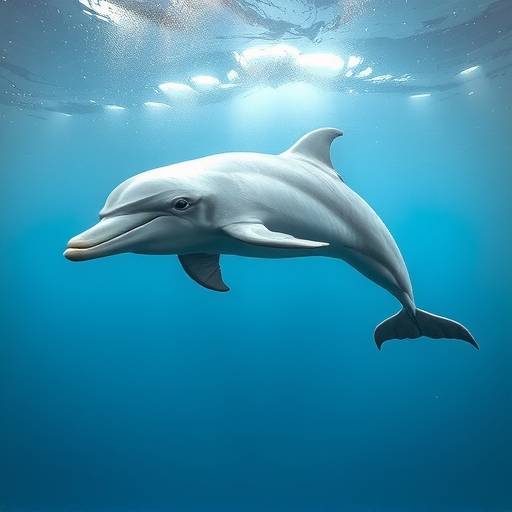} \\
\textit{5.8065} & \textit{6.7978} & \textit{6.7873} & \textit{7.5315} & \textit{6.9614} & \textit{7.9347} \\[3pt]

\multicolumn{6}{@{}c@{}}{"hedgehog"} \\[2pt]
\includegraphics[width=\imgw]{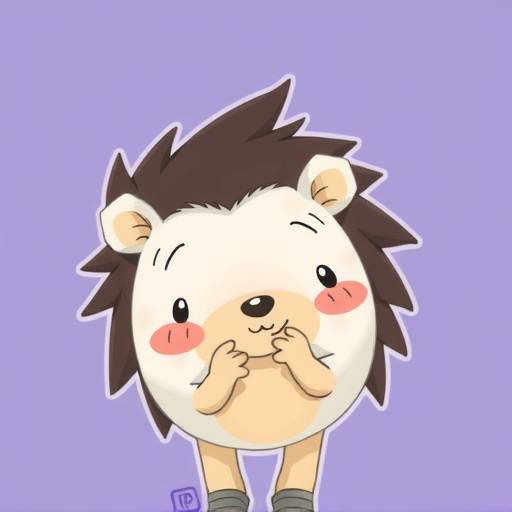} &
\includegraphics[width=\imgw]{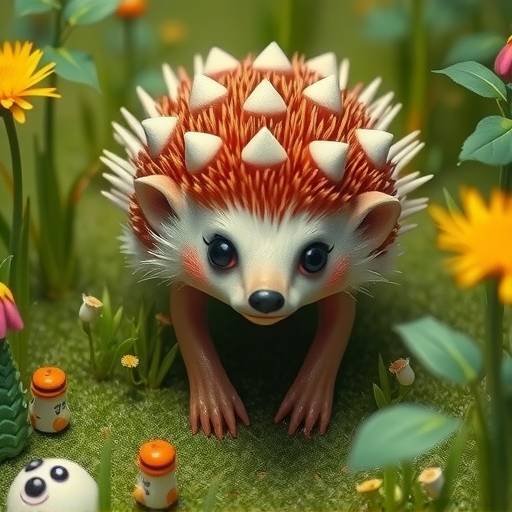} &
\includegraphics[width=\imgw]{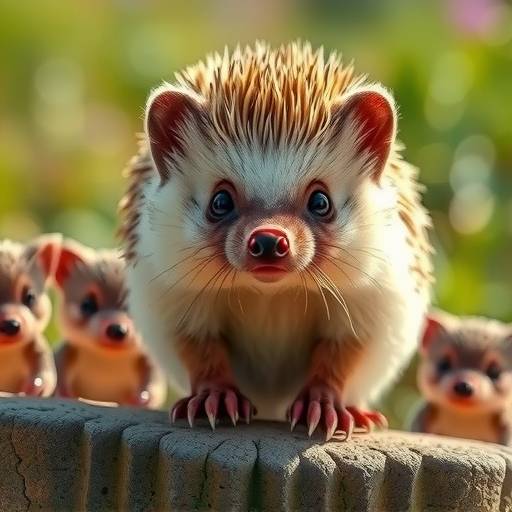} &
\includegraphics[width=\imgw]{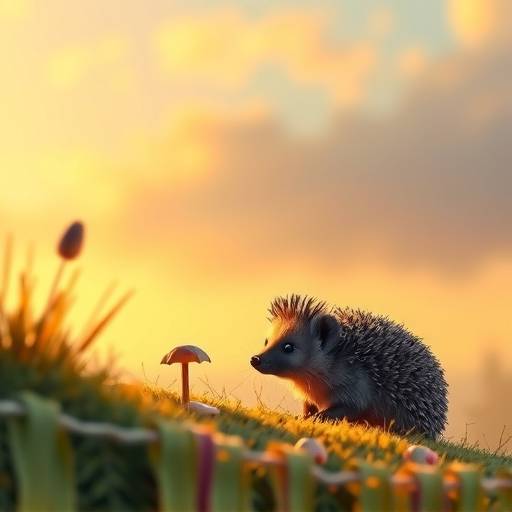} &
\includegraphics[width=\imgw]{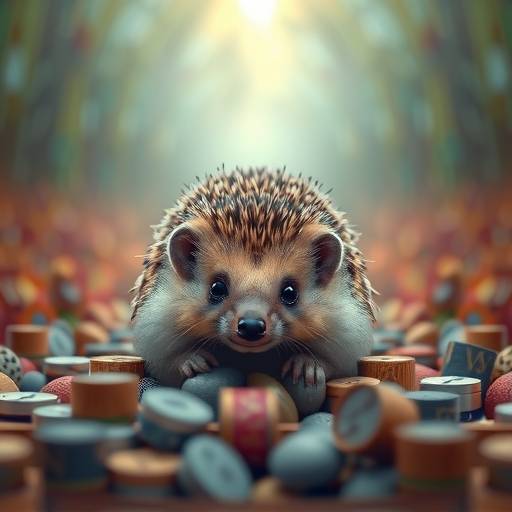} &
\includegraphics[width=\imgw]{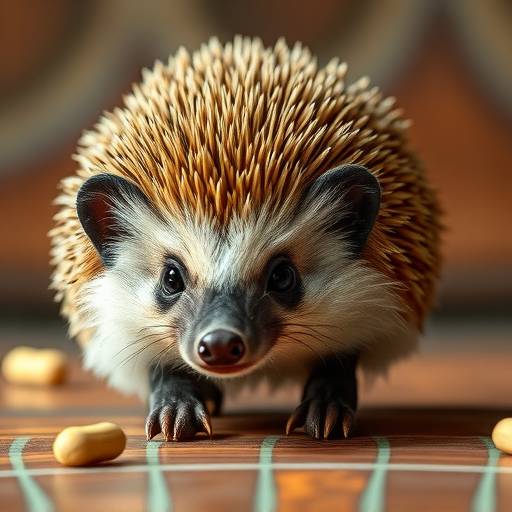} \\
\textit{5.8580} & \textit{6.5064} & \textit{6.7468} & \textit{7.3300} & \textit{7.0763} & \textit{7.9429} \\

\bottomrule
\end{tabular}

\caption{
\footnotesize 
\textbf{Additional qualitative comparison on aesthetic image generation using Aesthetic Score~\cite{Schuhmann:aesthetics}.} 
For single-particle methods, we augment sampling with Best-of-N to match the total NFE, denoted with \textsuperscript{\dag}. 
}
\label{fig:aesthetic_qualitative_supp}
\end{figure*}

\begin{figure*}[h!]
\centering
\setlength{\tabcolsep}{0pt}
\renewcommand{\arraystretch}{1.05}
\scriptsize

\newcommand{\colw}{0.164\textwidth}   
\newcommand{\imgw}{0.165\textwidth}

\begin{tabular}{@{}%
>{\centering\arraybackslash}m{\colw}
>{\centering\arraybackslash}m{\colw}
>{\centering\arraybackslash}m{\colw}
>{\centering\arraybackslash}m{\colw}
>{\centering\arraybackslash}m{\colw}
>{\centering\arraybackslash}m{\colw}
@{}}
\toprule

Base~\cite{flux2024} &
BoN~\cite{Stiennon:2020BoN} &
DPS\textsuperscript{\dag}~\cite{Chung:2023DPS} &
SVDD~\cite{Li2024:SVDD} &
$\Psi$-Sampler~\cite{yoon2025:psi} &
{\Oursbf{}}\textsuperscript{\dag} \textbf{(Ours)} \\

\midrule

\multicolumn{6}{@{}p{0.95\textwidth}@{}}{%
\centering\tiny``The intricate, colorful patterns of the mandala radiated outward from the center, forming a dazzling and hypnotic display.''} \\[3pt]
\includegraphics[width=\imgw]{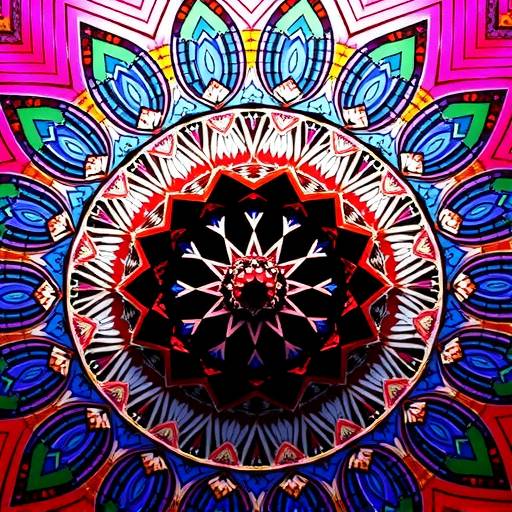} &
\includegraphics[width=\imgw]{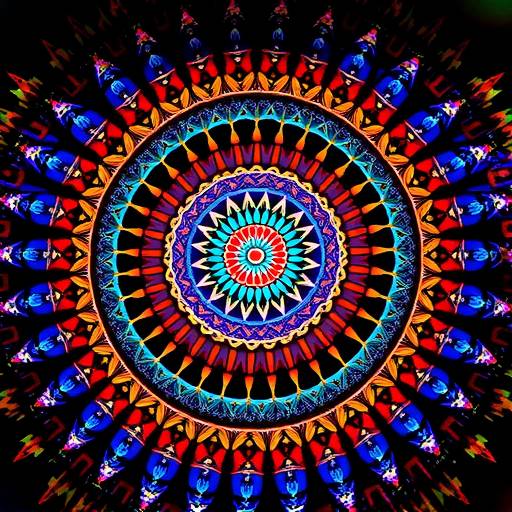} &
\includegraphics[width=\imgw]{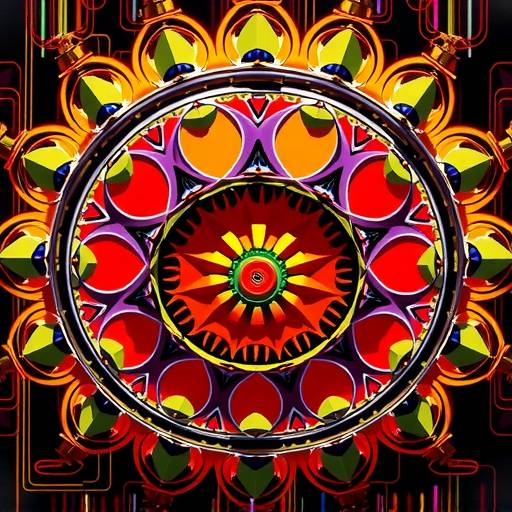} &
\includegraphics[width=\imgw]{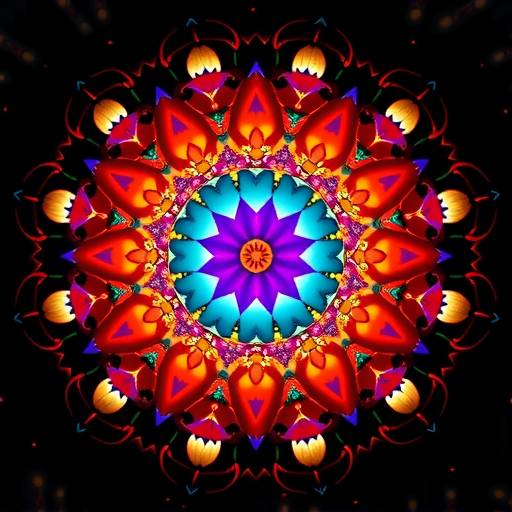} &
\includegraphics[width=\imgw]{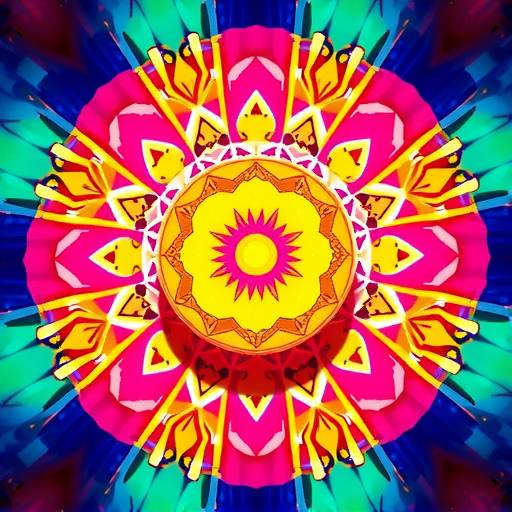} &
\includegraphics[width=\imgw]{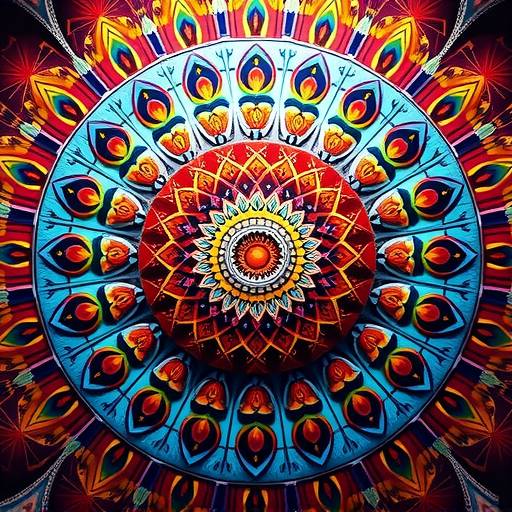} \\
\textit{0.2126} & \textit{0.2175} & \textit{0.2140} & \textit{0.2199} & \textit{0.2116} & \textit{0.2299} \\[2pt]

\multicolumn{6}{@{}p{0.95\textwidth}@{}}{%
\centering\tiny``The tall, slender birch trees swayed gently in the cool autumn breeze, their golden leaves rustling softly.''} \\[3pt]
\includegraphics[width=\imgw]{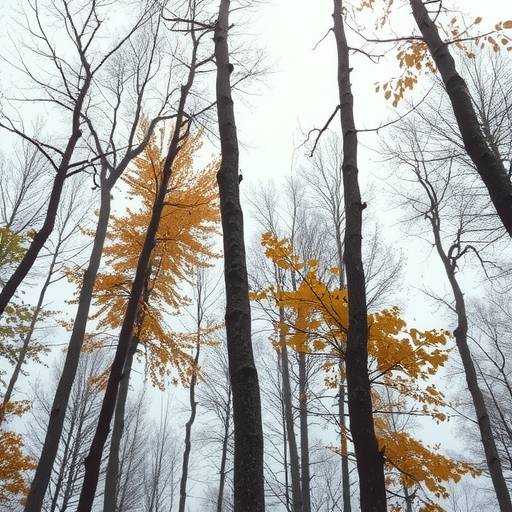} &
\includegraphics[width=\imgw]{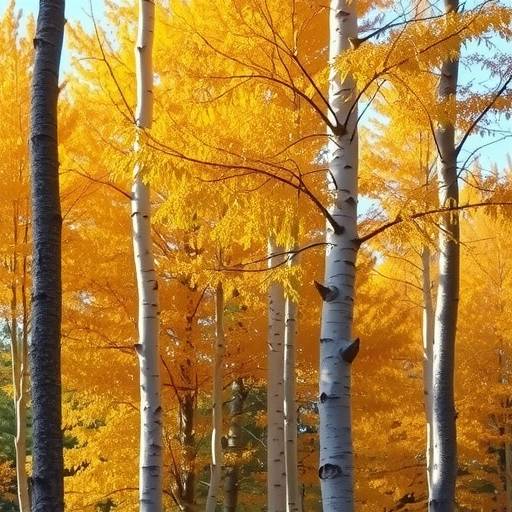} &
\includegraphics[width=\imgw]{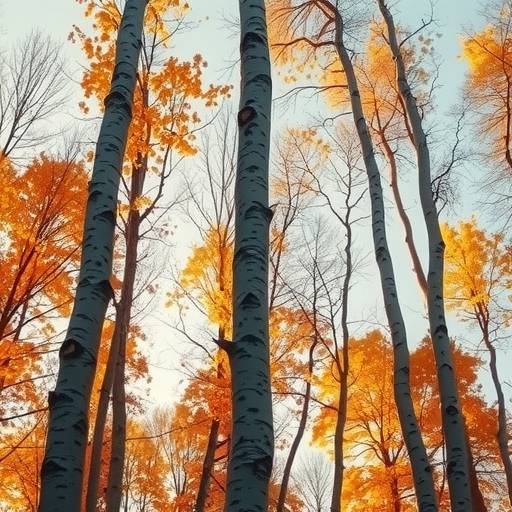} &
\includegraphics[width=\imgw]{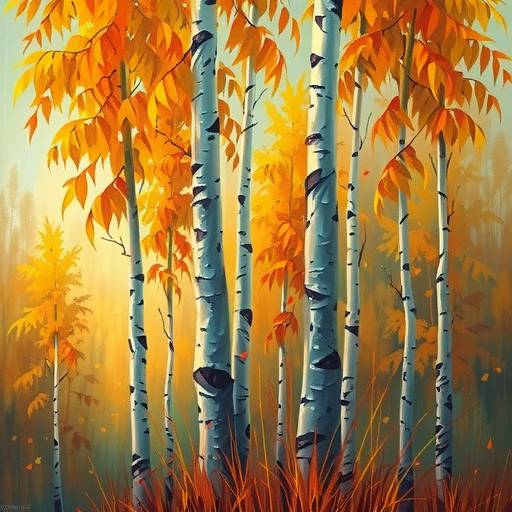} &
\includegraphics[width=\imgw]{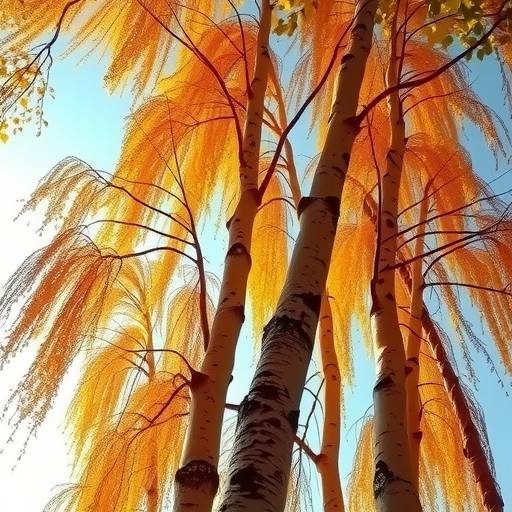} &
\includegraphics[width=\imgw]{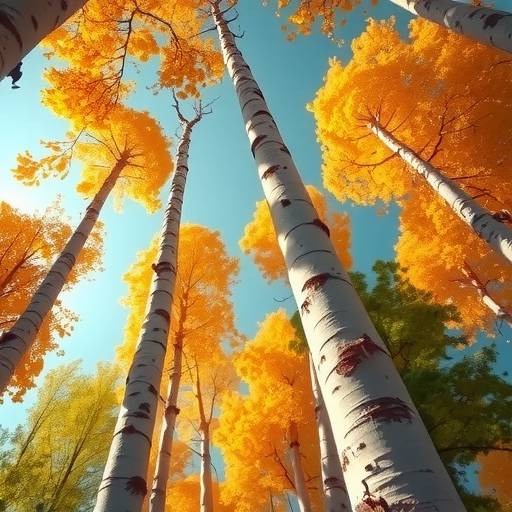} \\
\textit{0.2004} & \textit{0.2106} & \textit{0.2111} & \textit{0.2231} & \textit{0.2104} & \textit{0.2403} \\

\bottomrule
\end{tabular}

\caption{
\footnotesize
\textbf{Additional qualitative comparison on text-aligned image generation using PickScore~\cite{Kirstain2023:pickapic}.}
For single-particle methods we augment sampling with Best-of-N to match the total NFE, denoted with \textsuperscript{\dag}.
}
\label{fig:text_align_qualitative_supp}
\end{figure*}

\begin{figure*}[t!]
\centering
\setlength{\tabcolsep}{1pt}
\renewcommand{\arraystretch}{1.05}
\scriptsize

\newcommand{\methodw}{0.11\textwidth}
\newcommand{\framew}{0.142\textwidth}
\newcommand{\gapw}{0.001\textwidth}

\begin{tabular}{@{}%
>{\centering\arraybackslash}m{\methodw}
>{\centering\arraybackslash}m{\framew}
>{\centering\arraybackslash}m{\framew}
>{\centering\arraybackslash}m{\framew}
m{\gapw}
>{\centering\arraybackslash}m{\framew}
>{\centering\arraybackslash}m{\framew}
>{\centering\arraybackslash}m{\framew}
@{}}

& \multicolumn{3}{c}{\begin{tabular}[c]{@{}c@{}}\scriptsize\textit{``A sheep eating yellow flowers} \\ \scriptsize\textit{from behind a wire fence.''}\end{tabular}}
& 
& \multicolumn{3}{c}{\scriptsize\textit{``A harbour seal swimming near the shore.''}} \\

Base~\cite{wan2025}
& \includegraphics[width=\framew]{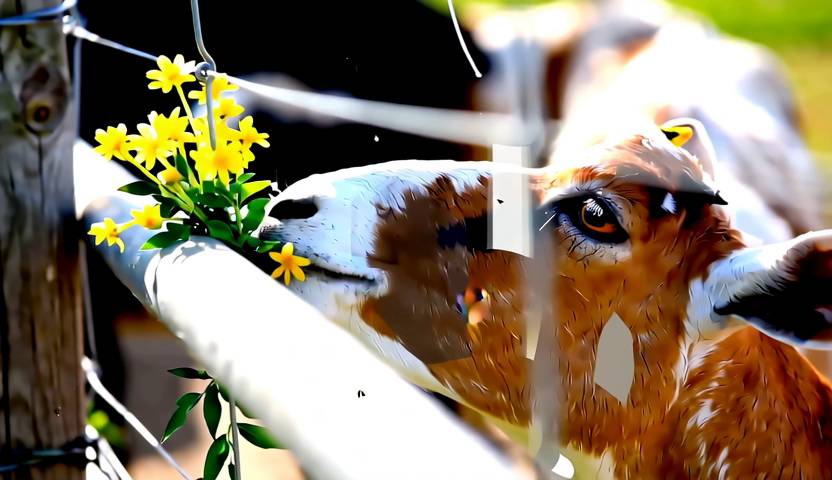}
& \includegraphics[width=\framew]{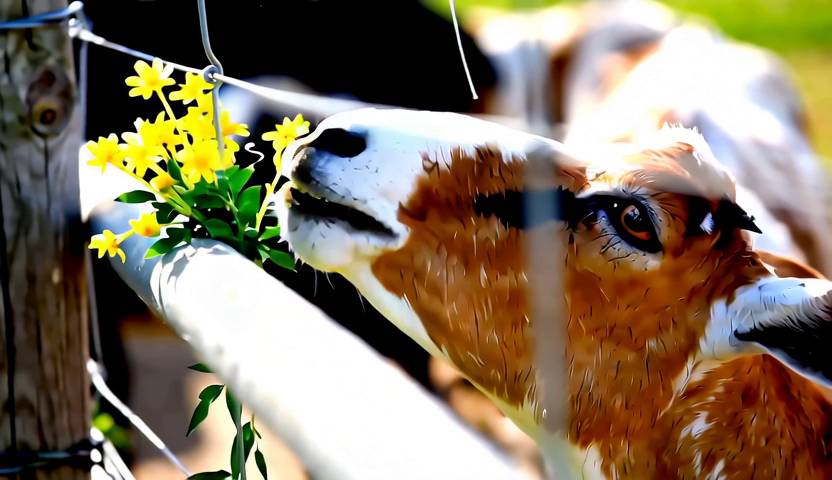}
& \includegraphics[width=\framew]{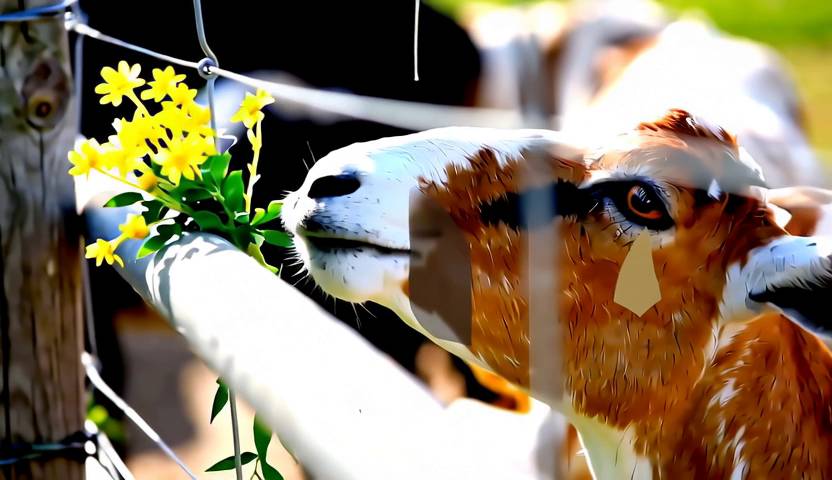}
&
& \includegraphics[width=\framew]{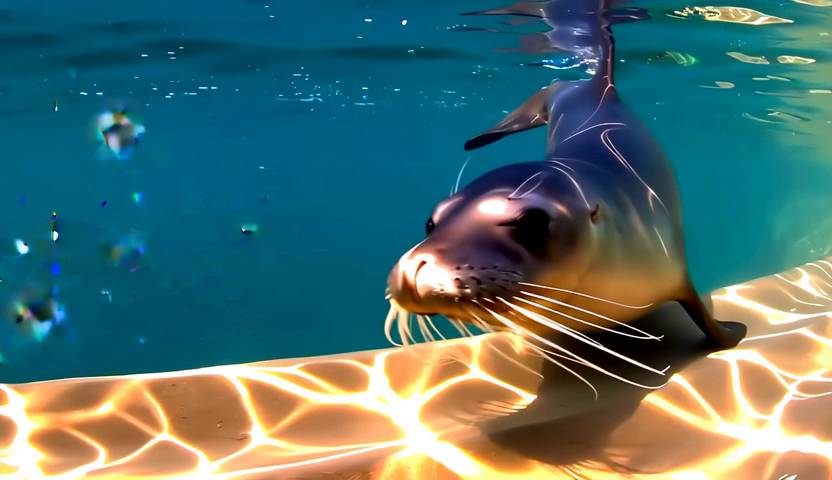}
& \includegraphics[width=\framew]{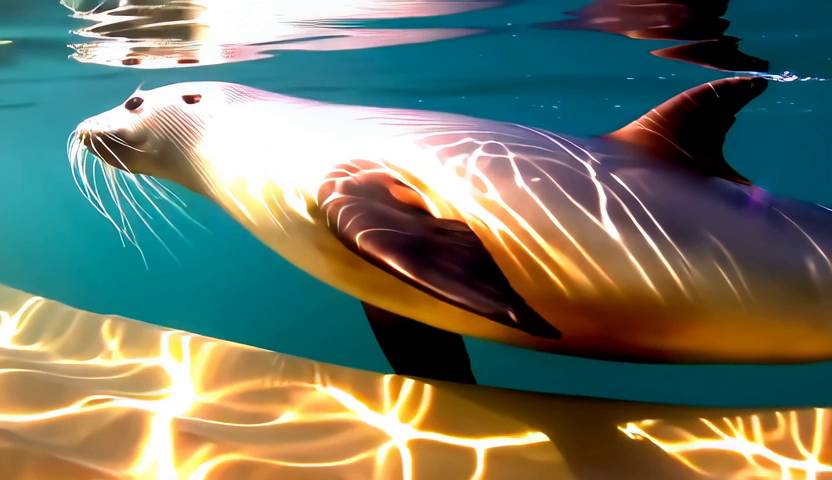}
& \includegraphics[width=\framew]{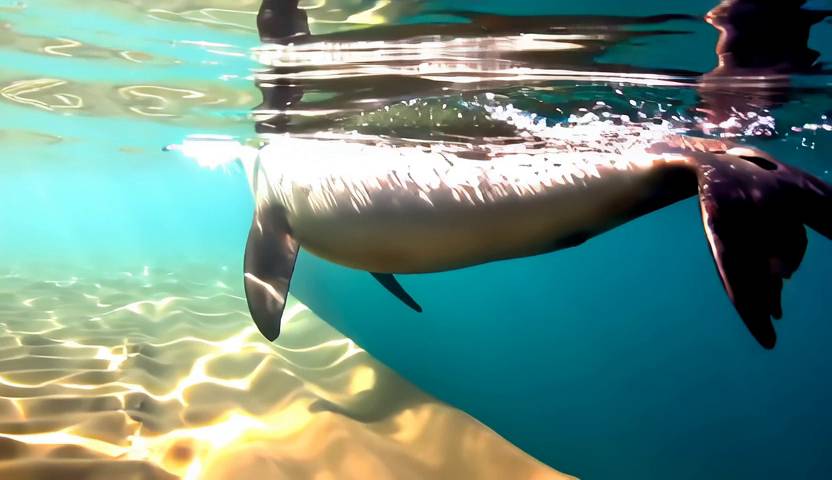} \\

DPS~\cite{Chung:2023DPS}
& \includegraphics[width=\framew]{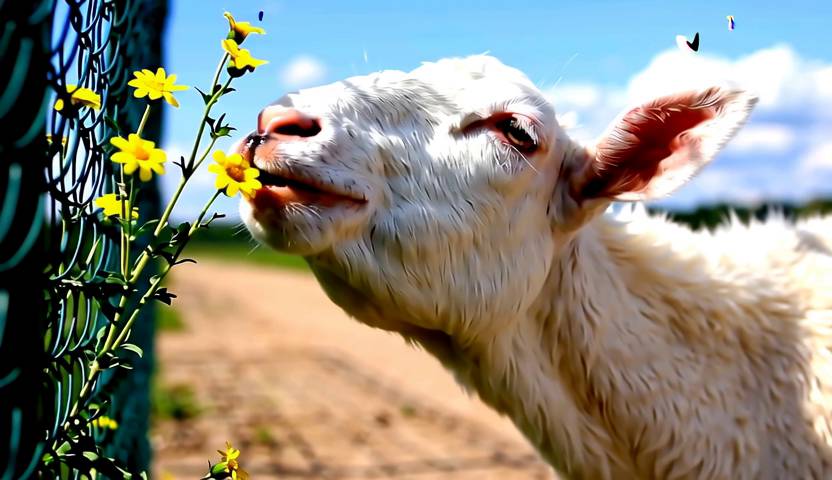}
& \includegraphics[width=\framew]{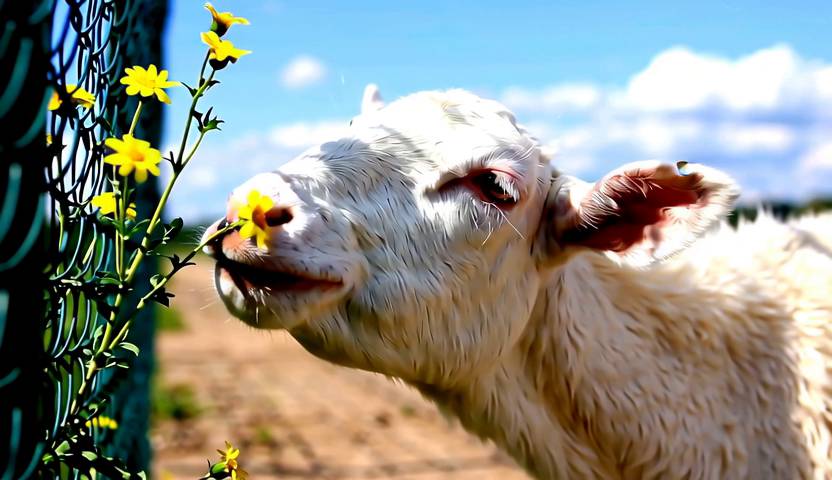}
& \includegraphics[width=\framew]{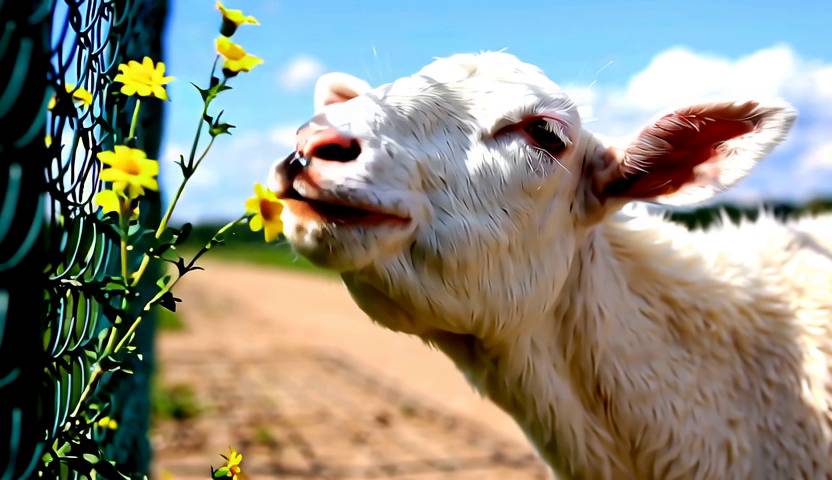}
&
& \includegraphics[width=\framew]{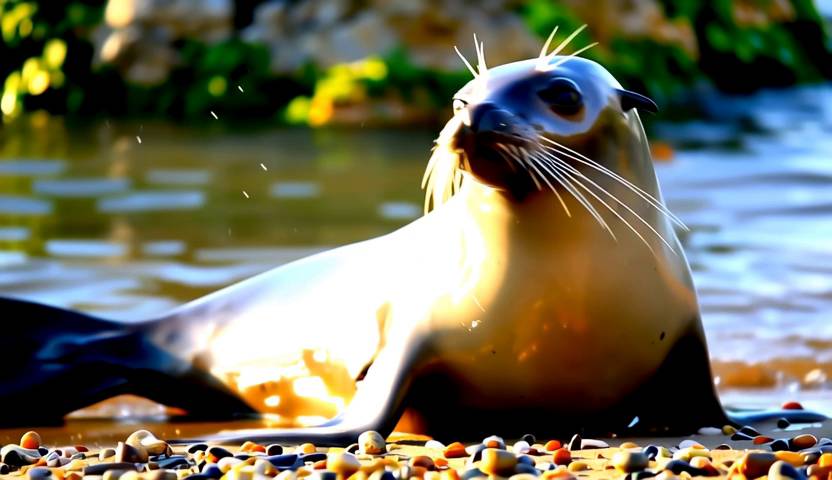}
& \includegraphics[width=\framew]{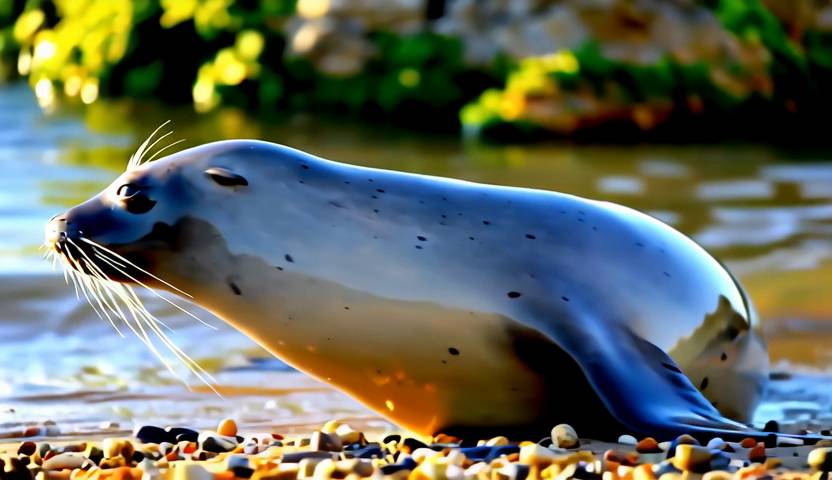}
& \includegraphics[width=\framew]{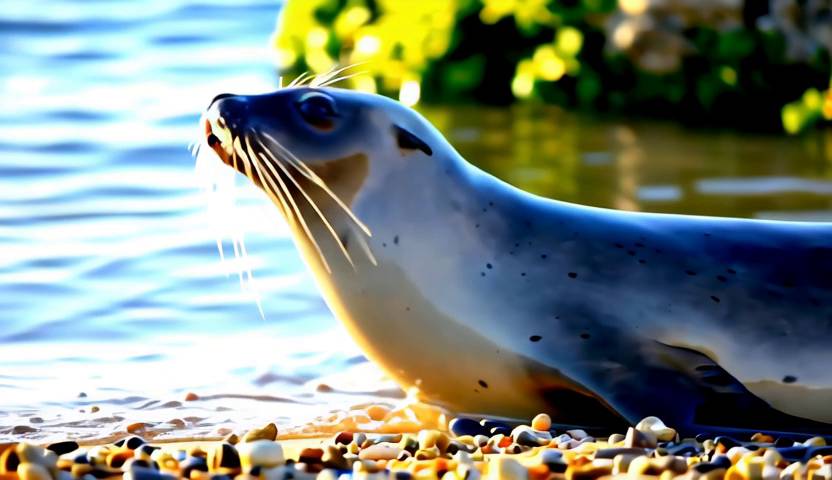} \\

\makecell{FreeDoM \\ \cite{Yu:2023FreeDOM}}
& \includegraphics[width=\framew]{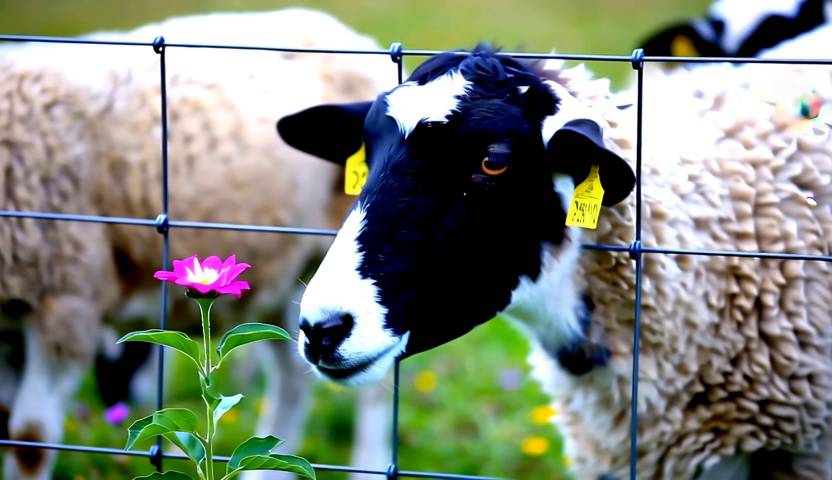}
& \includegraphics[width=\framew]{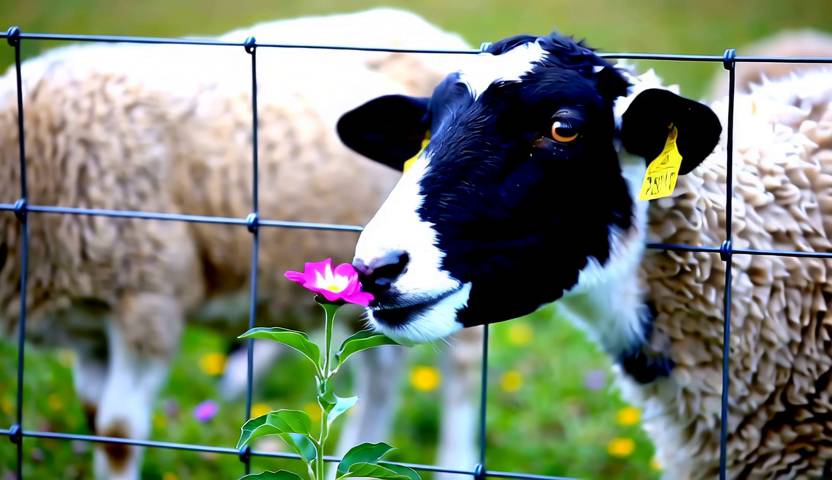}
& \includegraphics[width=\framew]{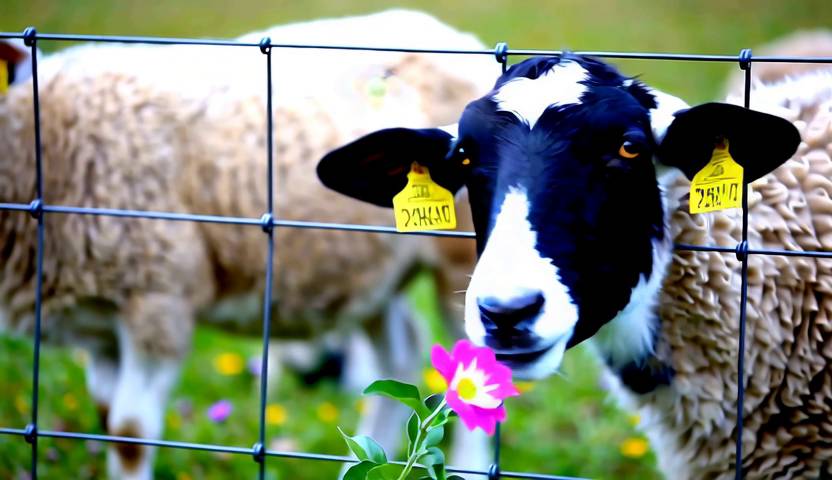}
&
& \includegraphics[width=\framew]{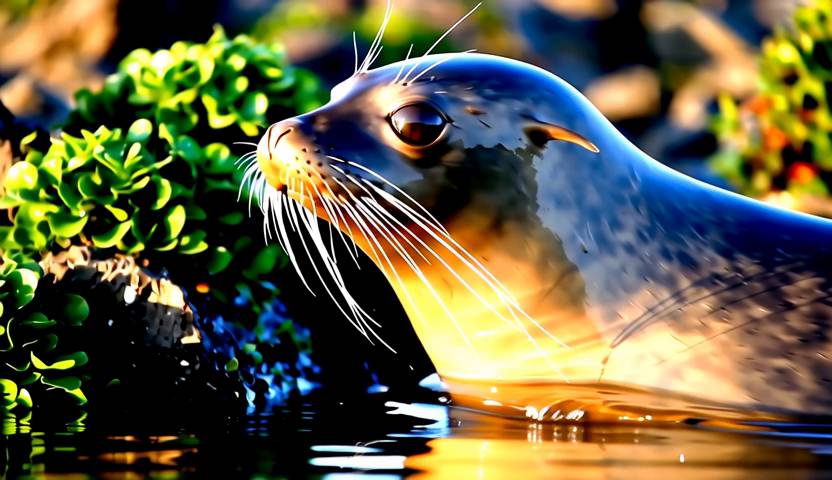}
& \includegraphics[width=\framew]{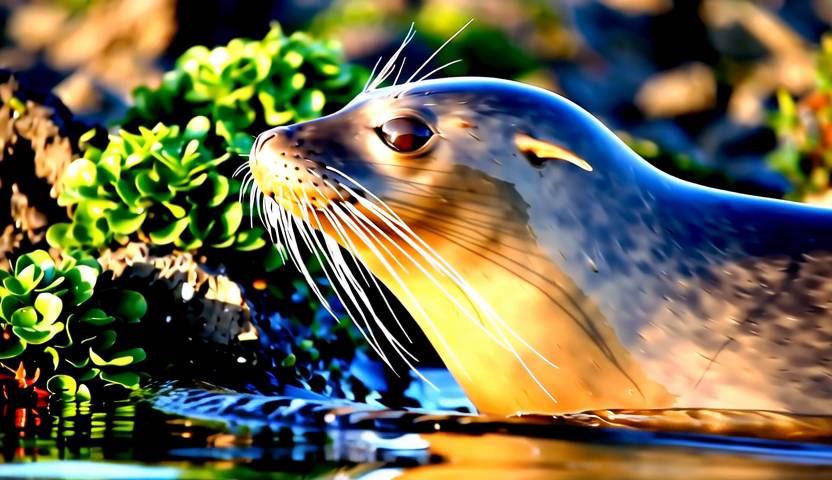}
& \includegraphics[width=\framew]{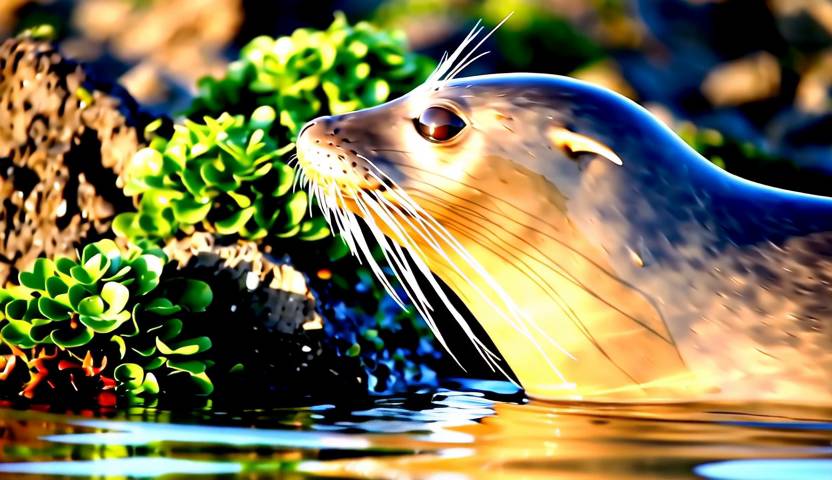} \\

{\Oursbf{}} \textbf{(Ours)}
& \includegraphics[width=\framew]{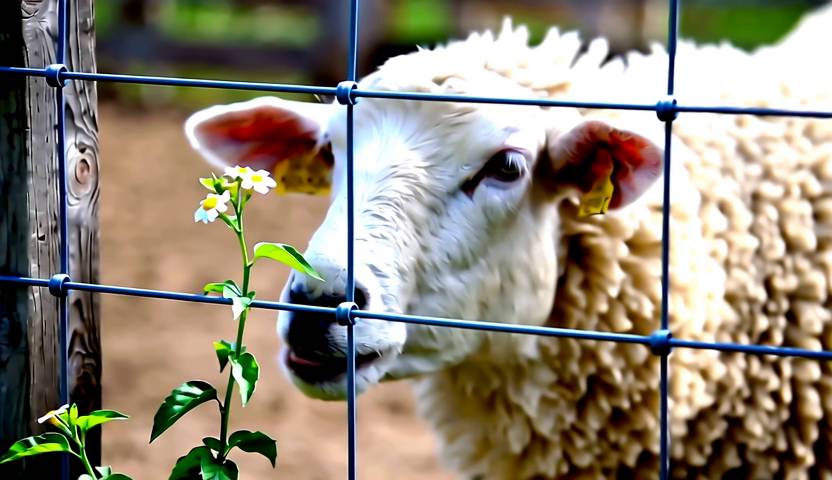}
& \includegraphics[width=\framew]{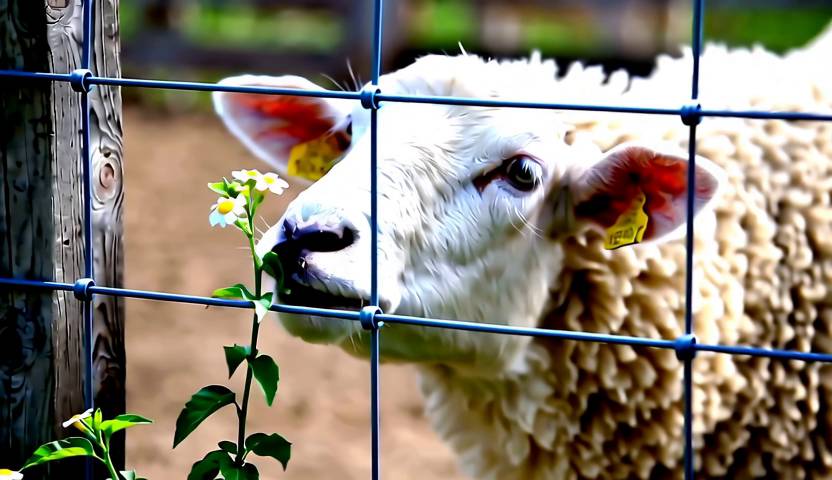}
& \includegraphics[width=\framew]{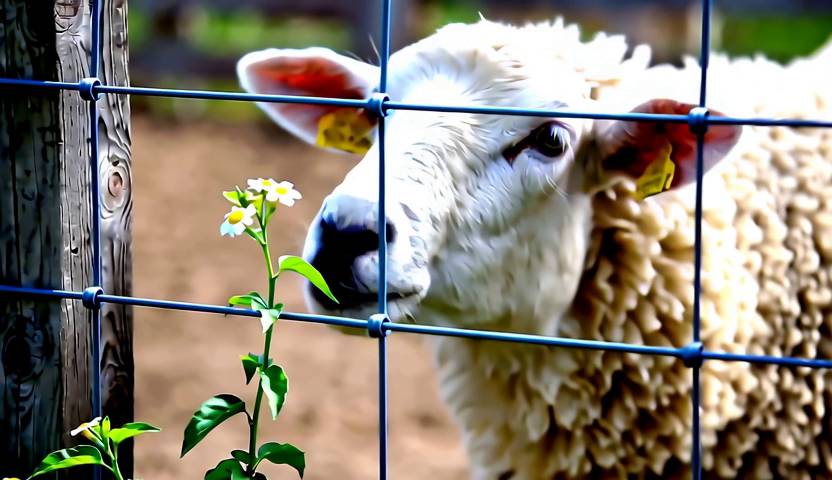}
&
& \includegraphics[width=\framew]{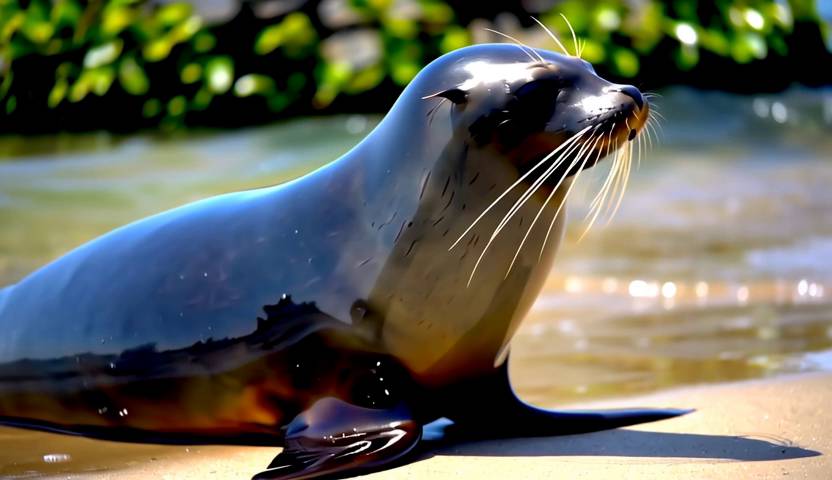}
& \includegraphics[width=\framew]{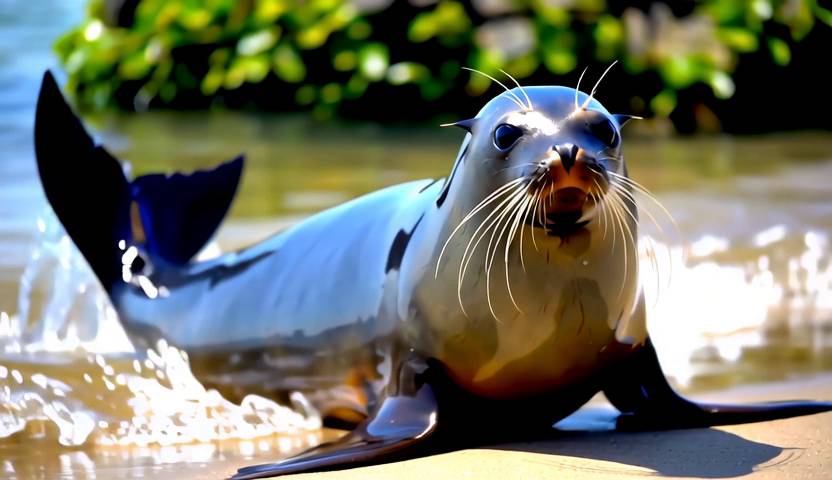}
& \includegraphics[width=\framew]{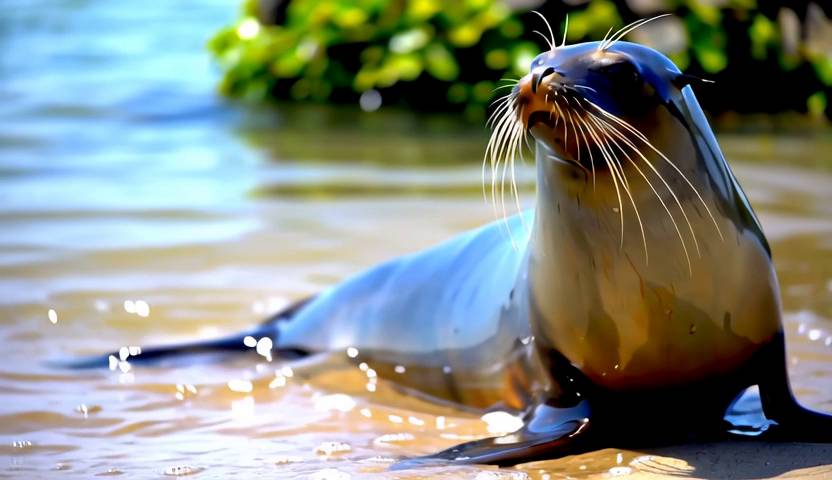} \\

\end{tabular}

\caption{
\footnotesize 
\textbf{Additional qualitative comparison on preference-aligned video generation using VideoReward~\cite{liu2025:videoalign}.} \Ours{} produces videos with better text alignment, visual quality and motion quality. 
}
\label{fig:video_qualitative_supp}
\end{figure*}

\end{document}